%% file: nips2020.tex
\documentclass{article}

\usepackage[preprint,nonatbib]{neurips_2020}

\usepackage[utf8]{inputenc} 
\usepackage[T1]{fontenc}    
\usepackage{url}            
\usepackage{booktabs}       
\usepackage{amsfonts}       
\usepackage{nicefrac}      
\usepackage{microtype}      

\usepackage{relsize}
\usepackage{amsmath,amscd}
\usepackage{mathtools}
\usepackage{amsthm}
\usepackage{amssymb}
\usepackage{tikz-cd}
\usepackage{color}
\usepackage{bm}
\usepackage{subfig}
\usepackage{caption}
\usepackage{algorithm}
\usepackage{algorithmic}
\usepackage[normalem]{ulem}
\usepackage{graphicx}

\usepackage{xr}
\makeatletter
\newcommand*{\addFileDependency}[1]{
  \typeout{(#1)}
  \@addtofilelist{#1}
  \IfFileExists{#1}{}{\typeout{No file #1.}}
}
\makeatother

\newif\ifdebugdoc\debugdoctrue

\ifdebugdoc

\newcommand{\outline}[1]{\colorbox{yellow}{Outline:}\textcolor{red}{#1}}
\newcommand{\todo}[1]{\large{\colorbox{red}{TODO:} #1. }}

\newcommand{\del}[1]{\textcolor{blue}{\sout{#1}}}
\newcommand{\save}[1]{}
\else
\newcommand{\outline}[1]{}
\newcommand{\todo}[1]{}
\newcommand{\chunyi}[1]{}

\newcommand{\del}[1]{}
\newcommand{\save}[1]{}

\fi

\newcommand{\diam}{\mathrm{diam}}

\newcommand{\mean}{\mathrm{mean}}
\newcommand{\tr}{\mathrm{tr}}

\newcommand{\bxe}{B_\eps^{d_X}(x)}

\newcommand{\R}{\mathbb{R}}
\newcommand{\eps}{\varepsilon}

\newcommand{\norm}[1]{\left\lVert#1\right\rVert}
\newcommand{\pf}{\mathcal{P}_f}

\newcommand{\p}{^\prime}

\newcommand{\me}{m^{\mathsmaller{(\eps)}}}

\newcommand{\lc}{\left(}
\newcommand{\rc}{\right)}

\newcommand{\dcov}{d_{\mathrm{cov}}}
\newcommand{\dgt}{d^{\mathsmaller{(\eps,\lambda)}}_{{\alpha,d_X}}}

\newcommand{\dw}{d_{\mathrm{W},2}}

\newtheorem{theorem}{Theorem}[section]
\newtheorem{proposition}[theorem]{Proposition}

\newtheorem{lemma}[theorem]{Lemma}
\newtheorem{remark}[theorem]{Remark}

\newtheorem{definition}{Definition}

\title{The Gaussian Transform}

\author{%
  Kun Jin \\
  Department of Computer Science and Engineering\\
  The Ohio State University\\
  \texttt{jin.810@osu.edu} \\
   \And
   Facundo M\'emoli \\
   Department of Mathematics \\
   Department of Computer Science and Engineering \\
   The Ohio State University\\
   \texttt{memoli@math.osu.edu} \\ 
   \AND
   Zhengchao Wan \\
   Department of Mathematics \\
   The Ohio State University\\
   \texttt{wan.252@osu.edu} \\
}

\begin{document}

\maketitle

\begin{abstract}
  We introduce the Gaussian transform (GT), an optimal transport inspired iterative method for denoising and enhancing latent structures in datasets. Under the hood, GT generates a new distance function (GT distance) on a given dataset  by computing the $\ell^2$-Wasserstein distance between certain Gaussian density estimates obtained by localizing the dataset to individual points. Our contribution is twofold: (1) theoretically, we establish firstly that GT is stable under perturbations and secondly that in the continuous case, each point possesses an asymptotically ellipsoidal neighborhood with respect to the GT distance; (2) computationally, we accelerate GT both by identifying a strategy for reducing the number of matrix square root computations inherent to the $\ell^2$-Wasserstein distance between Gaussian measures, and by avoiding redundant computations of GT distances between points via enhanced neighborhood mechanisms. We also observe that GT is both a generalization and a strengthening of the mean shift (MS) method, and it is also a computationally efficient specialization of the recently proposed Wasserstein Transform (WT) method. We perform extensive experimentation comparing their performance  in different scenarios. 
\end{abstract}

\section{Introduction}
\label{sec:intro}

Optimal Transport (OT) studies how to find an optimal strategy for transporting a source probability measure to a target probability measure \cite{villani2008optimal}. Recently, OT has been widely applied in Machine Learning~\cite{courty2016optimal,peyre2019computational}, Deep Neural Network~\cite{arjovsky2017wasserstein,gulrajani2017improved} and Natural Language Processing 
(NLP)~\cite{alvarez-melis-jaakkola-2018-gromov}, etc.

In \cite{pmlr-v97-memoli19a}, the authors introduced the \emph{Wasserstein transform} (WT) as a method for enhancing and denoising datasets. The WT alters the distance function on a dataset by computing the dissimilarity between neighborhoods of data points via methods from OT. WT can be regarded as a generalization and strengthening of mean shift (MS) \cite{fukunaga1975estimation,cheng1995mean}. Inspired by the construction of WT, in this paper we propose the \emph{Gaussian transform} (GT), a computationally efficient specialization of WT.

GT takes as input a point cloud and iteratively alters its metric based on how different are the local covariance matrices around points. {This is done via the computation of the $\ell^2$-Wasserstein distance between certain Gaussian distributions associated to these neighborhoods.} See Figure~(\ref{fig:ill-cov}) for an illustration. Due to the fact that there exists a closed form solution for the $\ell^2$-Wasserstein distance between Gaussian distributions
, the computation of GT is substantially more efficient than that of WT (which requires solving an OT problem for \emph{each pair of points} thus having complexity which scales cubically with the size of neighborhoods). 
We show ways of accelerating the computation of GT which include a {novel observation} stemming from computational linear algebra techniques (see Theorem \ref{thm:redux}) and some refinements about neighborhood mechanisms (see Section \ref{sec:nbh-mech}). We also prove that GT is stable with respect to perturbations (cf. Theorem \ref{thm:GT-stable}). 

One important feature of GT is its \emph{sensitivity to anisotropy}, an often desired feature in methods for image denoising and image segmentation 
\cite{perona1990scale,wang2004image}; see Figure \ref{fig:GT}. GT contains an intrinsic parameter $\lambda$ providing flexibility in tuning the degree of sensitivity to anisotropy in data features.
We apply GT to clustering and denoising tasks, and verify that in these, by tuning the parameter $\lambda$, GT has either comparable or superior performance over WT. We further apply GT to image segmentation, a task in which GT outperforms MS \cite{demirovic2019implementation}, and also in NLP tasks to boost word embeddings performance. Our experiments indicate that GT is effective in enhancing and denoising datasets. 
\begin{figure}[htb]
    \centering
    \includegraphics[width=0.9\linewidth]{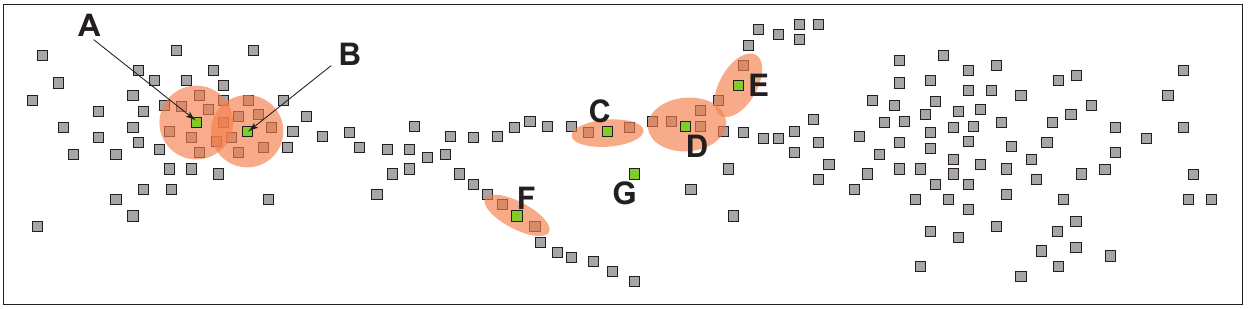}
    \caption{Illustration of GT. 
    We illustrate the idea of GT based on a point cloud in $\R^2$. Different points in the point cloud have different neighborhood structures, which are captured by their local covariance matrices. We  generate a Gaussian distribution \emph{on each data point} based on local covariance matrices. In the figure, we represent these Gaussians by ellipses (iso-contours of Gaussian distributions at a given height). Note that these localized Gaussian distributions reflect the neighborhood structures: the Gaussian is round when the neighborhood is relatively isotropic (\textbf{A} and \textbf{B}); the Gaussian is {flat} when the neighborhood is oblong (\textbf{C, D, E} and \textbf{F}); and the Gaussian is degenerate for an isolated point (cf. \textbf{G}). In a nutshell, the Gaussian Transform produces a new distance function on the dataset by computing the $\ell^2$-Wasserstein distance between these localized Gaussian distributions.
    }
    \label{fig:ill-cov}
\end{figure}
\input{background}
\input{GT}
\input{computational}
\input{experiments}
\input{discussion}

\bibliography{GT}
\bibliographystyle{alpha}

\appendix

\input{supp_algorithm}

\input{supp_proof}

\input{supp_experiments}

\end{document}

%% file: background.tex
\section{Background}\label{sec:bg}
\textbf{Optimal transport concepts.}
For $p\in[1,\infty]$, the $\ell^p$-\emph{Wassertein distance} $d_{\mathrm{W},p}$ \cite{villani2008optimal} measures the dissimilarity between two probability measures $\alpha,\beta$ on a compact metric space $X$. For $p<\infty$, it is defined as : 
$d_{\mathrm{W},p}(\alpha,\beta) \coloneqq  \left(\inf_{\pi\in \Pi(\alpha,\beta)} \iint_{X\times X} d_X^p(x,x')\,d\pi(x\times x')\right)^\frac{1}{p}$,
where $\Pi(\alpha,\beta)$ is the set of all \emph{couplings} $\pi$ (also named \emph{transport plans}) between $\alpha$ and $\beta$, i.e., $\pi$ is a probability measure on $X\times X$ with marginals $\alpha$ and $\beta$, respectively. See \cite{villani2008optimal} for a formula of $d_{\mathrm{W},\infty}$.
Solving the optimization problem for computing the Wasserstein distance is usually time consuming \cite{cuturi2013sinkhorn}. However, in the case of Gaussian measures, the distance enjoys a \emph{closed form formula} which allows for efficient computation. Given two Gaussian distributions $\gamma_1=\mathcal{N}(x_1,\Sigma_1)$ and $\gamma_2=\mathcal{N}(x_2,\Sigma_2)$ on $\R^m$, we have for $p=2$ that 
$
    d_{\mathrm{W},2}^2(\gamma_1,\gamma_2)=\norm{x_1-x_2}^2+ \dcov^2(\Sigma_1,\Sigma_2), 
$
where 
$
    \dcov(\Sigma_1,\Sigma_2)\coloneqq\sqrt{\tr\left(\Sigma_1+\Sigma_2-2\left(\Sigma_1^\frac{1}{2}\Sigma_2\Sigma_1^\frac{1}{2}\right)^\frac{1}{2}\right)}
$ \cite{givens1984class}.
Note that ${\dcov}$ is also known as the Bures distance \cite{bures1969extension} between positive semi-definite matrices. See the supplementary material for further remarks on the Bures distance. 

\textbf{Mean shift.}
Mean shift \cite{fukunaga1975estimation,cheng1995mean} is a mode seeking algorithm used in data analysis which operates by iteratively shifting each data point towards the mean of its neighborhood. To capture  neighborhood information, the MS algorithm  requires a kernel function $K:\R_+\rightarrow \R_+$, such as the Gaussian kernel $e^{-t^2/2}$, or the truncation kernel (which equals $1$ if $t\in[0,1]$ and is zero otherwise). Given a point cloud $X = \{x_i\}_{i=1}^n\subset\R^m$ and a scale parameter $\eps>0$, one shifts each point $x_i$ towards the weighted mean of its neighborhood (as defined by the choice of $K$ and $\eps)$.

\textbf{Wasserstein transform.}
We recall the definition of the Wasserstein transform from \cite{pmlr-v97-memoli19a}. Similarly to the case of MS a kernel function is used to capture neighborhood information, WT uses the \emph{localization operator} to reflect neighborhood information. A localization operator $L$ is a map which for any metric space $(X,d_X)$ and a probability measure $\alpha$ on $X$, assigns to every point $x\in X$ a probability measure $m_\alpha^L(x)$, which is referred to as the \emph{localized measure} at $x$ (based on $\alpha$).

\begin{definition}[The Wasserstein transform]\label{def:wt}
Let $(X,d_X)$ be a metric space together with a probability measure $\alpha$. Given a localization operator $L$ and $p\geq 1$, the \emph{Wasserstein transform} generates the distance function $d_\alpha^{L}$ on $X$ defined by
$d_\alpha^{L}(x,x')\coloneqq d_{\mathrm{W},p}\lc m_\alpha^{L}(x),m_\alpha^{L}(x')\rc ,\forall x,x\p\in X.$ 
\end{definition}
Definition \ref{def:wt} is slightly different from the one in \cite{pmlr-v97-memoli19a} which only consider the case when $p=1$. 
In this paper, in order to compare WT  with GT, we mainly focus on the case when $p=2$. 

%% file: GT.tex
\section{The Gaussian transform}
\label{sec:gt}
From now on, unless otherwise specified, \emph{we always assume $X$ to be a compact subset of $\R^m$ together with a metric $d_X$}. In practice, $d_X$ usually coincides with the underlying Euclidean distance between points. We allow general (non-Euclidean) $d_X$ for the convenience of later introducing an iterative algorithm for GT. 

\textbf{Theoretical background.}
Denote by $\mathcal{P}_f(X)$ the set of all probability measures on $X$ with full support and let $\alpha\in\mathcal{P}_f(X)$. Given a parameter $\eps>0$, we denote by $B_\eps^{d_X}(x)\coloneqq\{x'\in X:\, d_X(x,x')\leq\eps\}$ the closed ball with respect to $d_X$ centered at $x\in X$ with radius $\eps$. We assign to $x$ a probability measure $\me_{{\alpha,d_X}}(x)\coloneqq\frac{\alpha|_{B_{\eps}^{d_X}(x)}}{\alpha\left(B_{\eps}^{d_X}(x)\right)}$, which is the renormalized restriction of $\alpha$ to $B_\eps^{d_X}(x)$. Denote by $\mu^{\mathsmaller{(\eps)}}_{{\alpha,d_X}}(x)$ the mean of $\me_{{\alpha,d_X}}(x)$, i.e., 
$\mu^{\mathsmaller{(\eps)}}_{{\alpha,d_X}}(x)\coloneqq \int_{\R^m}y\,\me_{{\alpha,d_X}}(x)(dy).$
Then, we denote by $\Sigma_{{\alpha,d_X}}^{\mathsmaller{(\eps)}}(x)$ the covariance matrix of $\me_{{\alpha,d_X}}(x)$, i.e., the matrix defined below:
$
    \Sigma_{{\alpha,d_X}}^{\mathsmaller{(\eps)}}(x) \coloneqq\int_{\mathbb{R}^m}\lc{y}-{\mu^{\mathsmaller{(\eps)}}_{{\alpha,d_X}}(x)}\rc\otimes \lc {y}- {\mu^{\mathsmaller{(\eps)}}_{{\alpha,d_X}}(x)}\rc\,\me_{{\alpha,d_X}}(x)(d {y}),
$
where $x\otimes y$ is the bilinear form on $\R^m$ such that $x\otimes y(u,v)=\langle x,u\rangle\cdot\langle y,v\rangle$ for any $u,v\in\R^m$. See also the supplementary material for formulas corresponding to the discrete case.

\begin{definition}[GT distance]\label{def:gtd}
Given parameters $\lambda\geq 0$ and $\varepsilon>0$, we define the GT distance $d^{\mathsmaller{(\eps,\lambda)}}_{{\alpha,d_X}}(x,x')$ between $x,x'\in X$ by the following quantity
\begin{equation}
\label{eq:gtd}
    d^{\mathsmaller{(\eps,\lambda)}}_{{\alpha,d_X}}(x,x'):=\left(\norm{x-x'}^2+\lambda\cdot \dcov^2(\Sigma_{{\alpha,d_X}}^{\mathsmaller{(\eps)}}(x),\Sigma_{{\alpha,d_X}}^{\mathsmaller{(\eps)}}(x'))\right)^\frac{1}{2}.
\end{equation}
\end{definition}

\begin{definition}[The Gaussian transform]\label{def:gt}
The \emph{Gaussian transform} (GT) is the distance altering process which takes $(X,d_X,\alpha)$ into $(X,d^{\mathsmaller{(\eps,\lambda)}}_{{\alpha,d_X}})$. 
\end{definition}

In the case when $d_X$ {agrees with the underlying Euclidean distance}, GT is stable with respect to perturbations on the probability measure $\alpha$ under certain conditions. Let $c,\Lambda>0$ be constants. Denote by $\mathcal{P}_f^{c,\Lambda}(X)$ the set of all $\alpha\in\mathcal{P}_f(X)$ such that $\frac{\alpha\left(B^{d_X}_{r_1}(x)\right)}{\alpha\left(B^{d_X}_{r_2}(x)\right)}\leq\lc\frac{r_1}{r_2}\rc^\Lambda$ for any $x\in X$ and $r_1\geq r_2>0$, and $\alpha(S)\leq c\cdot\mathcal{L}_m(S)$ for all measurable set $S\subset\R^m$, where $\mathcal{L}_m$ stands for the $m$ dimensional Lebesgue measure. Let $D\geq 0$ be given such that $\diam(X)\leq D$. We have the stability theorem below whose proof   (and a remark on $\mathcal{P}_f^{c,\Lambda}(X)$) we relegate  to the supplementary material. 

\begin{theorem}[Stability of GT]\label{thm:GT-stable}
Assume that $d_X$ agrees with the underlying Euclidean distance and $\alpha,\beta\in\mathcal{P}_f^{c,\Lambda}(X)$. Then, there exists a positive constant $A=A(\eps,m,D)$ such that 
$\norm{d^{\mathsmaller{(\eps,\lambda)}}_{{\alpha,d_X}}-d^{\mathsmaller{(\eps,\lambda)}}_{{\beta,d_X}}}_\infty\leq 2\sqrt{m\,\lambda\,\Psi_{\Lambda,D,\eps}^{c,A}\lc d_{\mathrm{W},\infty}(\alpha,\beta)\rc}$,
where $\Psi_{\Lambda,D,\eps}^{c,A}:[0,\infty)\rightarrow[0,\infty)$ is an increasing function such that $\Psi_{\Lambda,D,\eps}^{c,A}(0)=0$. 
See the supplementary material for an explicit formula.
\end{theorem}

\textbf{Theoretical comparison with WT and MS.}
The idea of GT originates from the Wasserstein transform. We formulate GT as an instance of WT (cf. Definition \ref{def:wt}) as follows. For given $\lambda,\eps\geq 0$, define a localization operator $L_{\mathrm{GT}}^{\mathsmaller{(\eps,\lambda)}}$ that assigns to each point $x\in X$ a Gaussian distribution $\gamma_{\alpha,d_X}^{\mathsmaller{(\eps,\lambda)}}( {x})\coloneqq\mathcal{N}\left(x,\lambda\cdot\Sigma_{\alpha,d_X}^{\mathsmaller{(\eps)}}(x)\right)$. \emph{Then, GT is a version of WT arising from applying $\dw$ to the localization operator $L_{\mathrm{GT}}^{\mathsmaller{(\eps,\lambda)}}$.} 
In \cite{pmlr-v97-memoli19a}, the authors focused on a particular type of localization operator, \emph{local truncation}, which assigns to each point the localized probability $\me_{\alpha,d_X}(x)$. Though in Definition \ref{def:wt} WT is a general scheme, from now on, WT only refers to the \emph{local truncation based WT}. Whenever necessary we use WT1 and WT2 to denote WT with respect to $d_{\mathrm{W},1}$ and $d_{\mathrm{W},2}$, respectively. We view $\gamma_{\alpha,d_X}^{\mathsmaller{(\eps,1)}}( \cdot)$ as a (Gaussian) approximation of $\me_{\alpha,d_X}(\cdot)$ and thus GT as an approximation of WT2 when $\lambda=1$. Next, we compare GT with both WT2 and MS in a special case. 

\textit{Two lines.} In applications such as
crack detection in Civil Engineering~\cite{yamaguchi2010fast,koch2015review}, one may often encounter data points concentrated on line segments or curves. For such a data set $X$ as shown in Figure~(\ref{fig:non-intersect}), suppose the closed balls $B_\eps^{d_X}(x)$ and $B_\eps^{d_X}(x')$ are of the form of two non-intersecting (approximate) line segments for $x,x'\in X$, then it turns out that distances generated by GT and WT2 between $x$ and $x'$ are very similar. In fact, in the following idealized case of two perfect line segments, we show that the distances are the same. Moreover, we show that both distances are larger than the distance generated by MS. 
\begin{figure}[htb]
    \centering
    \subfloat[Non-intersecting line segments]{
        \includegraphics[width=0.3\linewidth]{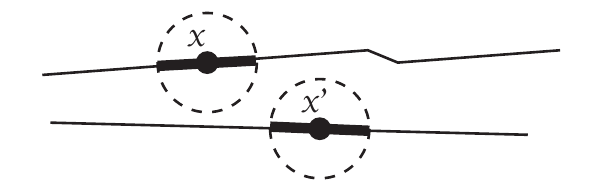}
        \label{fig:non-intersect}
    }
    \subfloat[Original Data]{
		\includegraphics[width=0.25\linewidth]{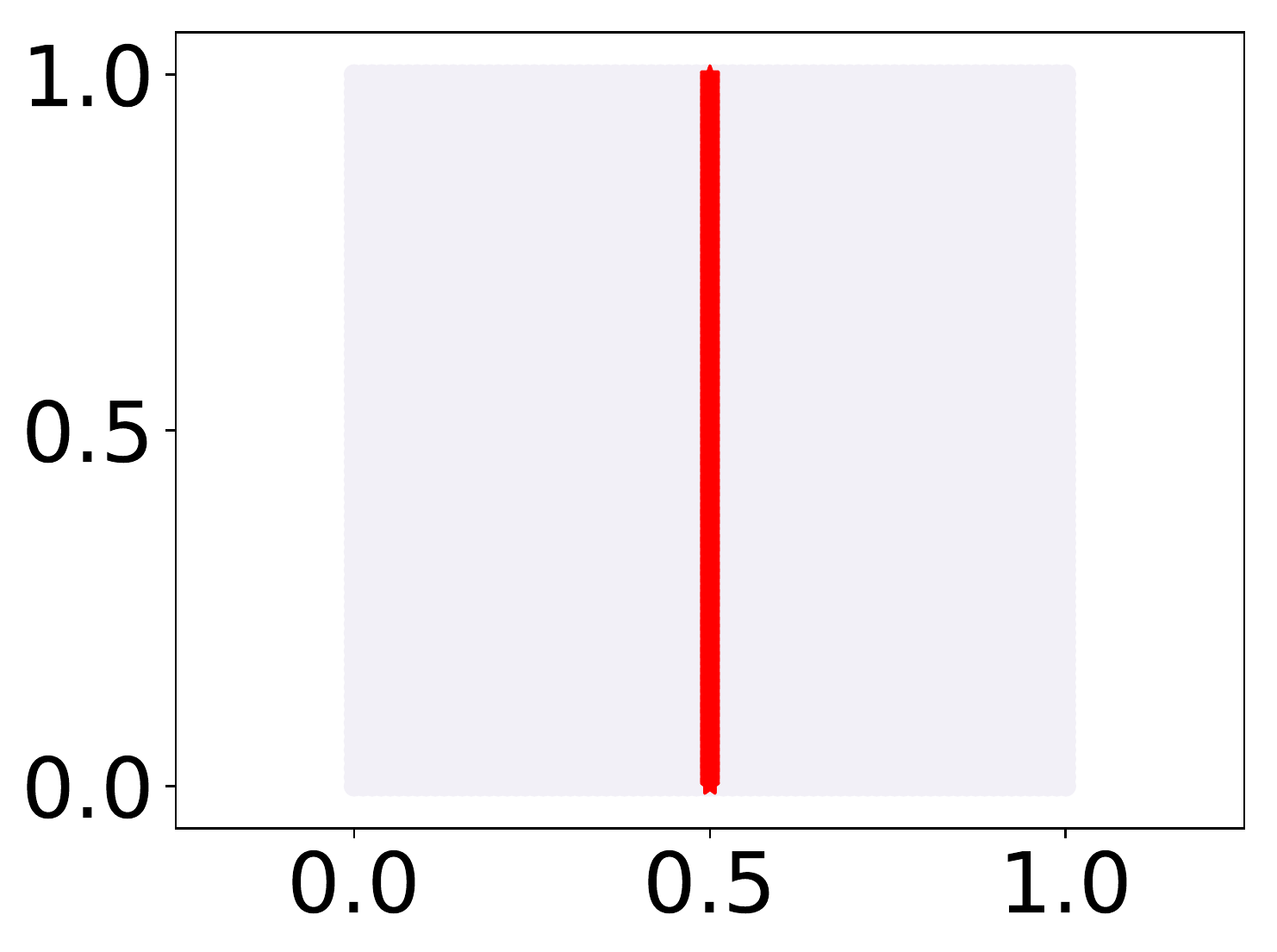}
		\label{fig:aniso-1}
	}
	\subfloat[Neighborhood $B_\eps^{\lambda}(x_0)$]{
		\includegraphics[width=0.25\linewidth]{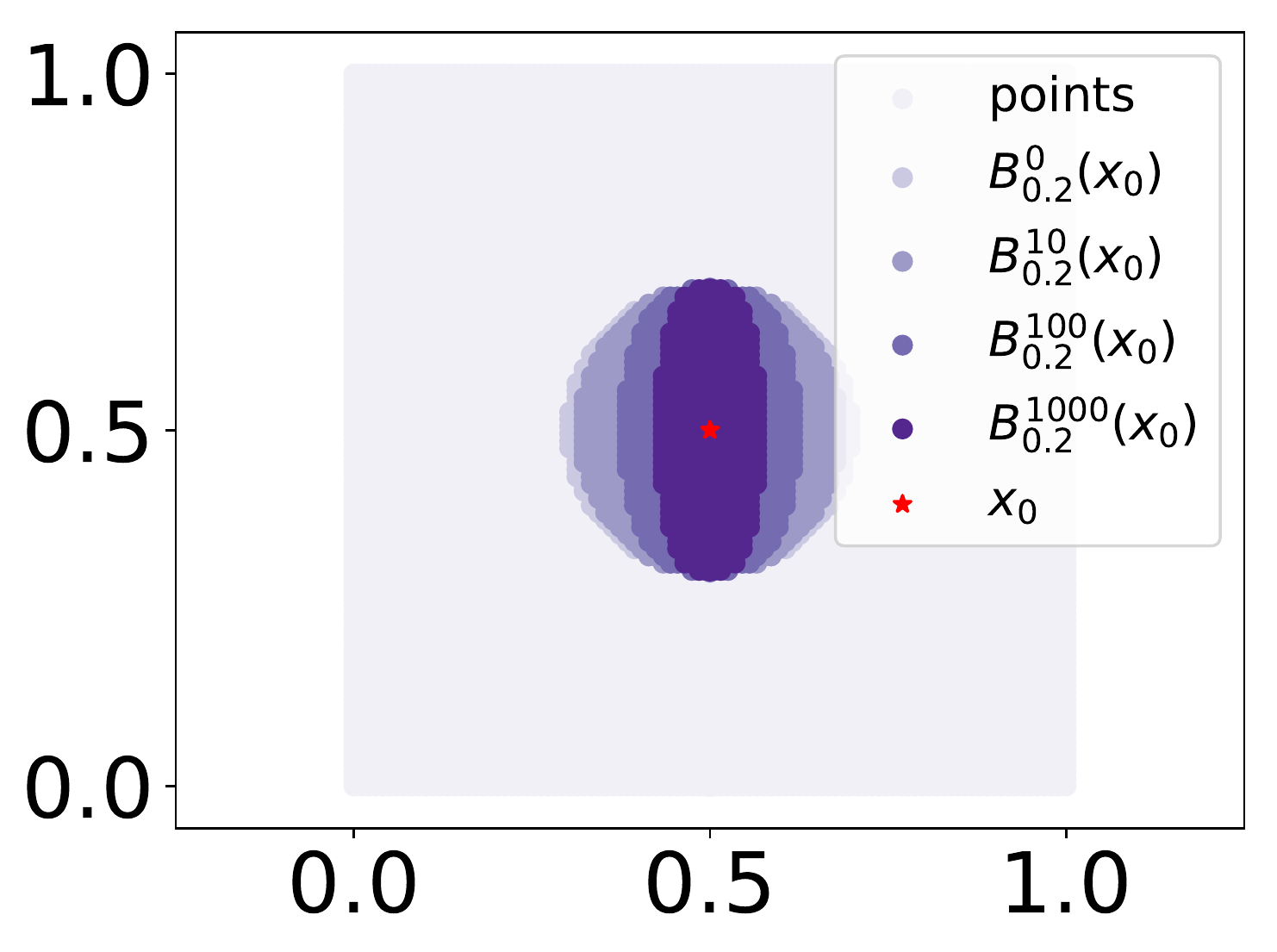}
		\label{fig:aniso-2}
	}
    \caption{(a) Illustration of line segments. (b) The original data with $100 \times 100$ grid points inside $[0,1]\times[0,1]$ and 1001 uniformly distributed points on the line segment from $(0.5,0)$ to $(0.5,1)$. (c) The $\eps$-neighborhood of the point $x_0$=(0.5, 0.5) with respect to $\dgt$ under various choice of $\lambda$, where $\eps=0.2$, $\alpha$ is the normalized empirical measure, and we abbreviate $B_\eps^{\lambda,\alpha}(x_0)$ to $B_\eps^{\lambda}(x_0)$. }
\end{figure}
Consider two non-intersecting line segments $l_1$ and $l_2$ in $\R^2$ with length $s_1$ and $s_2$, respectively. Assume the angle between them is $\theta\in[0,\frac{\pi}{2}]$. Let $\alpha_i$ be the {normalized length measure} on $l_i$ for $i=1,2$. Suppose the mean of $\alpha_i$ is $x_i$ (which happens to be the center of $l_i$) and the covariance of $\alpha_i$ is $\Sigma_i$ for $i=1,2$. Then we show the following result; see the supplementary material for the proof.
\begin{proposition}\label{prop:line}
For $i=1,2$ let $\gamma_i$ be the Gaussian distribution $\mathcal{N}(x_i,\Sigma_i)$. Then we have that
$\norm{x_1-x_2}\leq\dw(\gamma_1,\gamma_2)=\dw(\alpha_1,\alpha_2).$
\end{proposition}

\begin{remark}
Based on the proposition above, we know that although covariance constructions might lose some information of the local neighborhoods, GT has the \emph{same performance} as WT2 in the case of non-intersecting line segments (see also Figure~(\ref{fig:tlines}) for experimental validation). This, together with the fact that the computation of GT is much more efficient than that of WT2, suggests that GT is a sound alternative to WT2. 
\end{remark}

\textbf{The anisotropy of GT-distance neighborhoods.}
Now we assume that $X=\R^m$ with $d_X$ being the underlying Euclidean distance and that $\alpha\in\mathcal{P}_f(X)$ has a smooth non-vanishing density function $f$ with respect to the Lebesgue measure $\mathcal{L}_m$.  We denote by $B_\eps^{\lambda,\alpha}(x_0)\coloneqq B_\eps^{d^{\mathsmaller{(\eps,\lambda)}}_{{\alpha,d_X}}}(x_0)$ the ball with respect to $d^{\mathsmaller{(\eps,\lambda)}}_{{\alpha,d_X}}$ centered at $x_0\in X$ with radius $\eps$, which we will refer to as the \textbf{GT-distance neighborhood} of $x_0$. We study the asymptotic shape of $B^{\lambda,\alpha}_\eps(x_0)$ when $\lambda$ tends to 0 with $\eps$ at a precise rate:

\begin{theorem}\label{thm:ellipsoid}
Let $\lambda=\eps^{-6}$, then $B_\eps^{\lambda,\alpha}(x_0)$ becomes an ellipsoid when $\eps$ approaches $0$. More precisely, the closure of $\limsup_{\eps\rightarrow 0}\frac{1}{\eps}B_\eps^{\lambda,\alpha}(x_0)$ is an ellipsoid in $\R^m$. 
\end{theorem}

In the theorem above, $\frac{1}{\eps}B_\eps^{\lambda,\alpha}(x_0)\coloneqq\big\{x_0+\frac{1}{\eps}(x-x_0):\,x\in B_\eps^{\lambda,\alpha}(x_0)\big\}$. Though the norms of the covariance matrices are of order $O(\eps^2)$, $\dcov^2$ between covariance matrices of points $\eps$-close to each other is of order $O(\eps^8)$. This explains the choice of $\lambda=\eps^{-6}$ in Theorem \ref{thm:ellipsoid}: with this choice the Euclidean term and the covariance term in Equation (\ref{eq:gtd}) are of the same order $O(\eps^2)$. See the supplementary material for a proof and Figure~(\ref{fig:aniso-1}) and (\ref{fig:aniso-2}) for an illustration of the theorem. 
Theorem {\ref{thm:ellipsoid}} demonstrates that neighborhoods with respect to the GT distance (Equation (\ref{eq:gtd})) are anisotropic. This indicates that GT is sensitive to boundaries/edges in datasets and thus suggests potential applications to edge detection and preservation tasks in image processing. See our image segmentation experiment in Section \ref{sec:img-seg}. Anisotropy sensitive ideas, such as anisotropic diffusion \cite{perona1990scale} or anisotropic mean shift \cite{wang2004image}, are prevalent in the literature and have been applied to image denoising and image segmentation. See also \cite{martinez2013multiscale,martinez2020shape}  for applications to shape analysis.

\textbf{Algorithm for iterative GT.} Note that, after applying GT to $(X,d_X)$ once, we obtain a new metric $d^{\mathsmaller{(\eps,\lambda)}}_{{\alpha,d_X}}$ that is sensitive to directions/edges in $X$ and generates an anisotropic neighborhood for each point $x\in X$ as discussed above. We view this step as an initialization and then incorporate a \textit{point updating process} to iterate GT. In words, after obtaining $d\coloneqq \dgt$ on $X$, we generate a probability measure $m_{\alpha,d}^{\mathsmaller{(\eps)}}(x)$ for $x\in X$ by restricting $\alpha$ to the GT ball $B_\eps^{d}(x)$ (the ball with respect to $d$ centered at $x$ with radius $\eps$) and shift each data point $x$ towards the mean of $m_{\alpha,d}^{\mathsmaller{(\eps)}}(x)$. Denote by $X'$ the set of points after shifting. Now, we have obtained a new point cloud $X'$. Still denote by $\alpha$ the pushforward of $\alpha$ itself under the shifting map. Then, we apply GT to $(X',d,\alpha)$ to obtain a new metric $d'\coloneqq d^{\mathsmaller{(\eps,\lambda)}}_{{\alpha,d}}$ and iterate the process for $(X',d',\alpha)$ as described above. See the algorithm structure of GT in Figure~(\ref{fig:alggt}) and the supplementary material for a pseudocode of the iterative GT algorithm. 

\begin{figure}[htb]
    \centering
    \subfloat[MS]{
    \label{fig:algms}
		\includegraphics[width=0.32\linewidth]{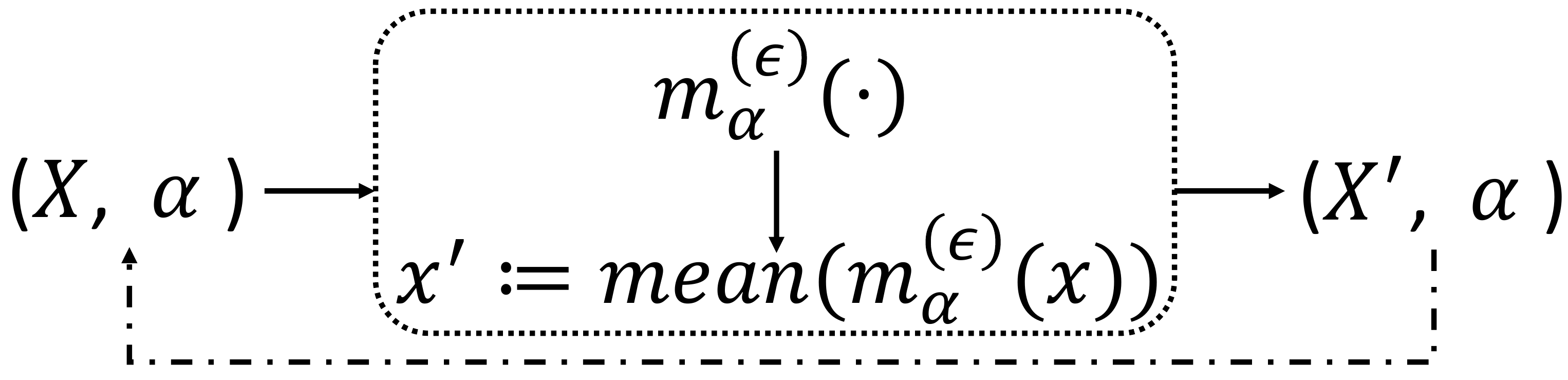}
	}
	\subfloat[WT]{
	\label{fig:algwt}
		\includegraphics[width=0.3\linewidth]{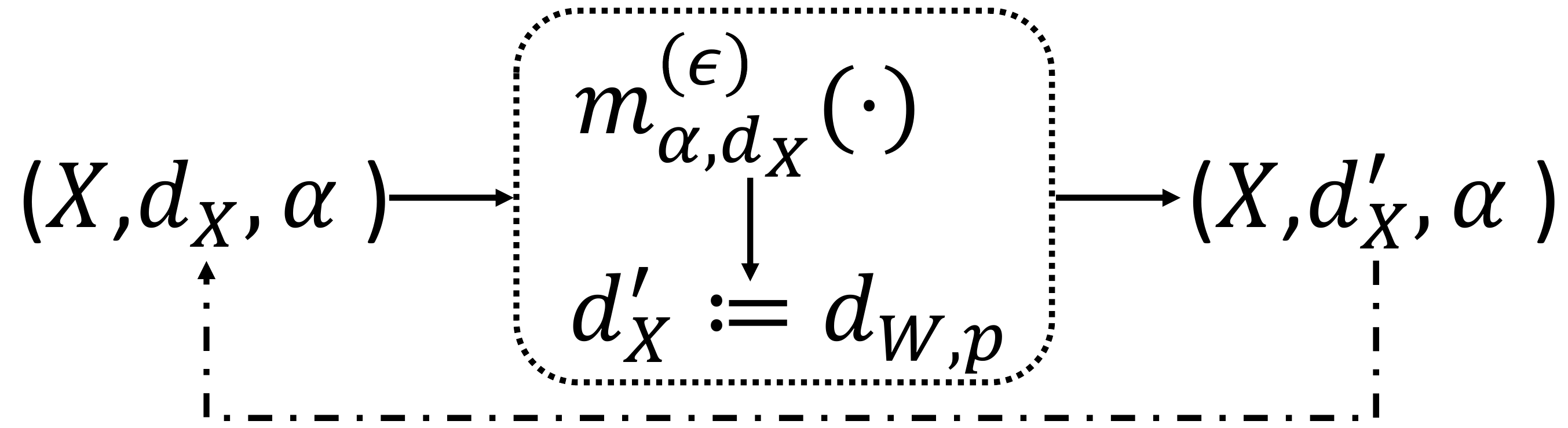}
	}
	\subfloat[GT]{
	\label{fig:alggt}
		\includegraphics[width=0.37\linewidth]{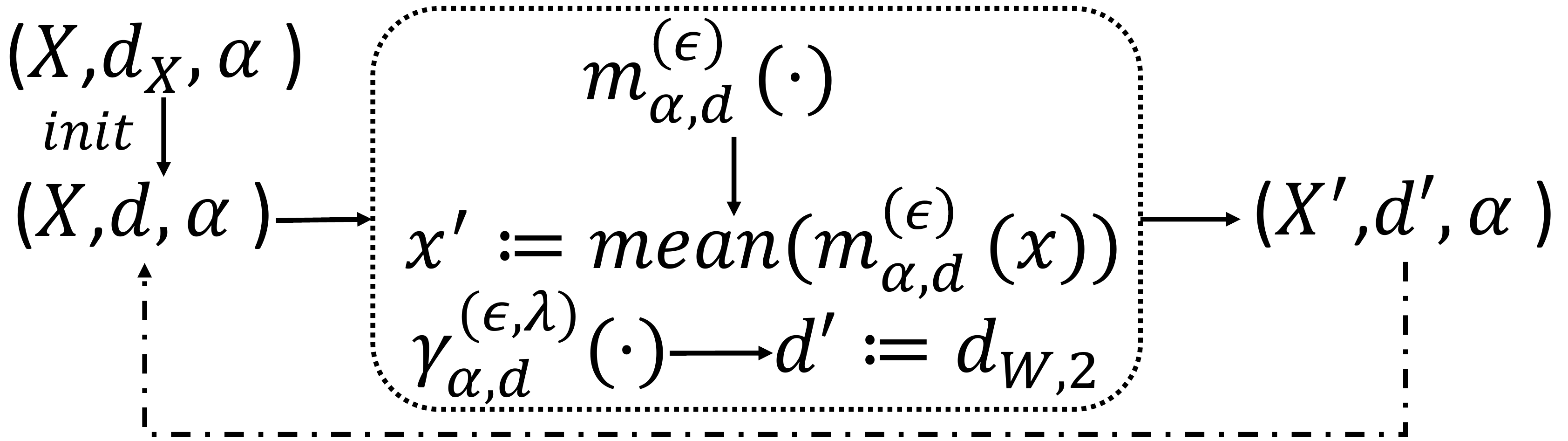}
	}
    \caption{Algorithmic structures.
    (a) MS algorithm structure; (b) WT algorithm structure; (c) GT algorithm structure. $X'$ is the set of updated points $x'$. $d'$ is the updated distance. $\me_\alpha(x)$ is the abbreviation of $\me_{\alpha,\norm{\cdot}}$, where $\norm{\cdot}$ denotes the Euclidean distance. }
    \label{fig:algcom}
\end{figure}

\textbf{Algorithmic similarities with MS and WT.} The iterative GT algorithm draws inspiration from the iterative MS and WT algorithms. Please see Figure \ref{fig:algcom} for an illustration. Note that the MS algorithm is a \emph{point updating} process whereas the WT algorithm is a \emph{distance updating process}. The GT algorithm is a hybrid between the MS and the WT algorithms in that \emph{it is composed of both a point updating and a distance updating process}. Thus, the GT algorithm inherently provides us with two features, a point cloud and a distance matrix, which can be leveraged in different applications and thus provides an advantage over WT. For example, the point updating process allows GT to adapt from MS~\cite{demirovic2019implementation} to the task of image segmentation in Section \ref{sec:exp} whereas WT is not applicable.

\begin{remark}
Note that when $\lambda=0$, the GT algorithm boils down to the MS algorithm.
\end{remark}

%% file: computational.tex
\section{Computational optimizations \& complexity}
\textbf{Computation of $\dcov$: a new formula.}
The main challenge in implementing GT is the computational cost associated with $\dcov$: that is, the computation of $\tr\lc\Sigma_1+\Sigma_2-2\lc\Sigma_1^\frac{1}{2}\Sigma_2\Sigma_1^\frac{1}{2}\rc^\frac{1}{2}\rc$. The most time consuming part is taking the square root of a matrix. We identified the following methods to accelerate this computation. In the given formula for $\dcov$, one has to carry out square root computations twice: once for $\Sigma_1^\frac{1}{2}$ and another one for $\left(\Sigma_1^\frac{1}{2}\Sigma_2\Sigma_1^\frac{1}{2}\right)^\frac{1}{2}$. Since all we care about is the trace, it turns out that, for each pair of $\Sigma_1$ and $\Sigma_2$ we only need to compute the eigenvalues of the matrix $\Sigma_1\Sigma_2$ by the theorem below (whose proof is given in the supplementary material). For a square matrix $A$, we denote by $\mathrm{spec}(A)$ the multiset of eigenvalues of $A$ counted with multiplicities. 

\begin{theorem} \label{thm:redux}
Given two square positive semi-definite matrices $A$ and $B$, we have that
$\mathrm{tr}\left(\left(A^\frac{1}{2}BA^\frac{1}{2}\right)^\frac{1}{2}\right)=\sum_{\lambda\in \mathrm{spec}(AB)}\lambda^\frac{1}{2}.$
\end{theorem}

\textit{Computation of eigenvalues}. 
Computing the eigenvalues of a (square) matrix becomes expensive when the size of the matrix is large. However, it is relatively cheap to just compute the first few largest eigenvalues (for example, via the so called  ``power method" ~\cite{quarteroni2010numerical}). In our experiments, $\tr\lc\left(\Sigma_1^\frac{1}{2}\Sigma_2\Sigma_1^\frac{1}{2}\right)^\frac{1}{2}\rc$ is always approximated as follows: fix an integer $i_0\leq \dim(\Sigma_1)$, compute the first $i_0$ largest eigenvalues of $\Sigma_1\Sigma_2$, take the square root of these eigenvalues and compute their sum.

\textbf{Neighborhood mechanism: acceleration of point updating process}. 
\label{sec:nbh-mech}
If we only require the GT point updating process in some tasks, such as image segmentation and classification, the following proposition allows us to restrict computations of GT distance to only pairs of points within small Euclidean distance instead of computing all GT pairwise distances. 
See the supplementary material for a detailed description and experimental verification. 
\begin{proposition}\label{pro:nbh-mech}
    In each iteration, given a data point $x$, its GT-distance neighborhood is contained in the corresponding Euclidean neighborhood. 
\end{proposition}

\textbf{Worst-case complexity comparison.} We now carry out a comparison of the worst case computational complexity associated to GT and WT. Denote the point cloud size by $n$, its dimension by $m$ and the maximum neighborhood size of each point by $N$. The complexity of GT (distance updating process) for one iteration is $O(n^2(Nm^2 + m^3))$ and that of WT is $O(n^2(N^3\log N))$. Note that GT is significantly faster than WT in the regime when $m \ll N$. We also emphasize that as we described earlier in this section, in practice one is able to use spectral methods for approximating the square root matrix calculations intrinsic to GT. See the supplementary materials for the complexity of GT using neighborhood mechanism and derivations of all reported time complexities. 

%% file: experiments.tex
\section{Implementation, examples and applications}\label{sec:exp}
We now apply GT to various datasets.  In all of our experiments, $\alpha$ is the normalized empirical measure, and the radius adopted in each dataset, $\eps$, varies across different experiment but, in a given experiment, it remains fixed (by default) throughout iterations. In figures, we use $\tau$ to represent the number of iterations. 
We compare the performance of GT with MS and WT if applicable. We only present results of WT2 in the paper and see the supplementary material for results of WT1.

\textbf{Clustering of a T-junction dataset.} 
We compare the clustering results based on GT with those of MS and WT2 on the T-junction dataset {shown in Figure~(\ref{fig:tline-original})}, which is composed of a vertical line with 200 uniformly arranged points spanning from $(0,1)$ to $(0,200)$ and a horizontal line with 201 uniformly arranged points spanning from $(-100,0)$ to $(100,0)$. 
We set $\eps$=10. 
Figure~(\ref{fig:tline-mds}) illustrates the updated point cloud using MS and 2D/3D MDS plots using GT with $\lambda$=1, GT with $\lambda$=5 and WT2 after 2 iterations. Figure~(\ref{fig:tline-dend}) shows the corresponding dendrograms, based on which we split the data into 4 clusters. 
Note that 3D MDS plots of GT and WT2 have comparable structures, which indicates similarity between their distance matrices. This agrees with our analysis of the two line dataset in Section \ref{sec:gt} and validates our claim that GT is an approximation of WT2. Although it is clear that both the dendrogram and the MDS plot of GT-$\lambda$-1 are degraded compared with those of WT2, GT allows us to fine-tune $\lambda$: when $\lambda$=5, the performance of GT is visually comparable to that of WT2. 

\begin{figure}[htb!] 
    \centering
    \subfloat[Original Data]{
    \label{fig:tline-original}
        \fbox{
        \begin{minipage}[t][2.35cm][t]{0.11\textwidth}
		    \centering
		    $\tau=0$ \\ 
            \includegraphics[width=1\textwidth]{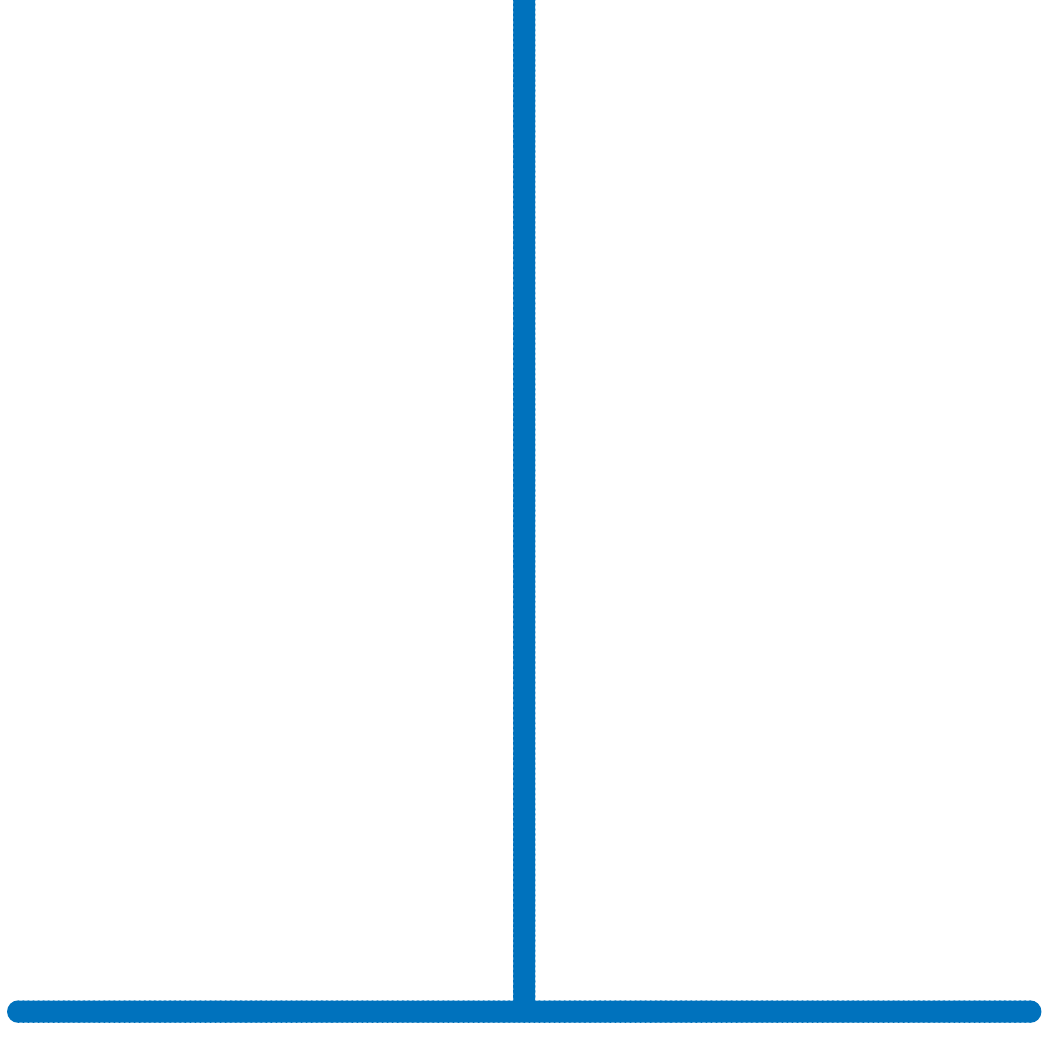}
        \end{minipage}
        }
    }
    \subfloat[2D and 3D MDS at $\tau=2$]{
    \label{fig:tline-mds}
        \fbox{
    	    \begin{minipage}[t][2.35cm][t]{0.08\textwidth}
    		    \centering
    		    MS \\ 
                \includegraphics[width=0.85\textwidth]{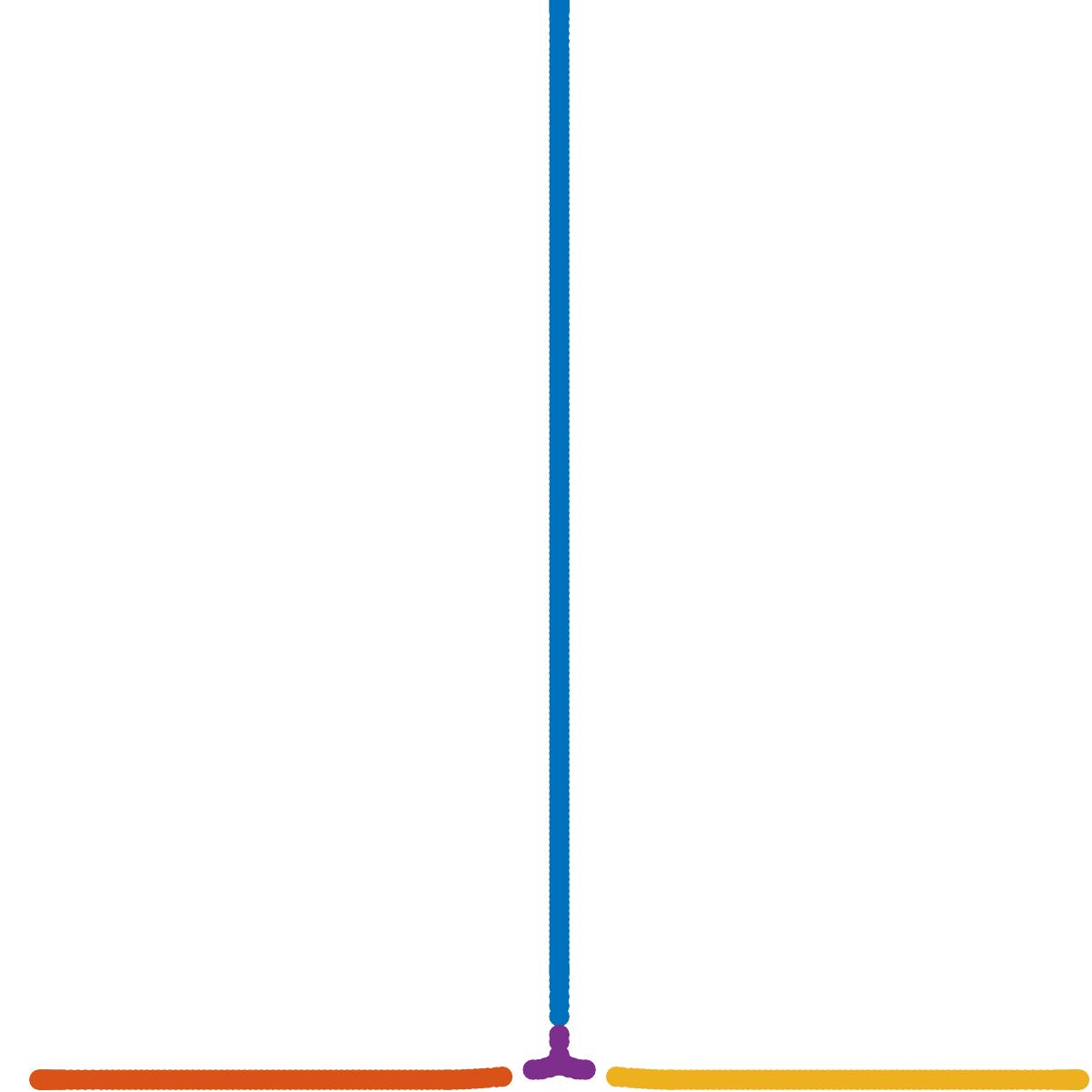} 
                \\ 
            \end{minipage}
            \begin{minipage}[t][2.35cm][t]{0.08\textwidth}
    		    \centering
    		    GT-$\lambda$-1 
    		    \\  \includegraphics[width=1\textwidth]{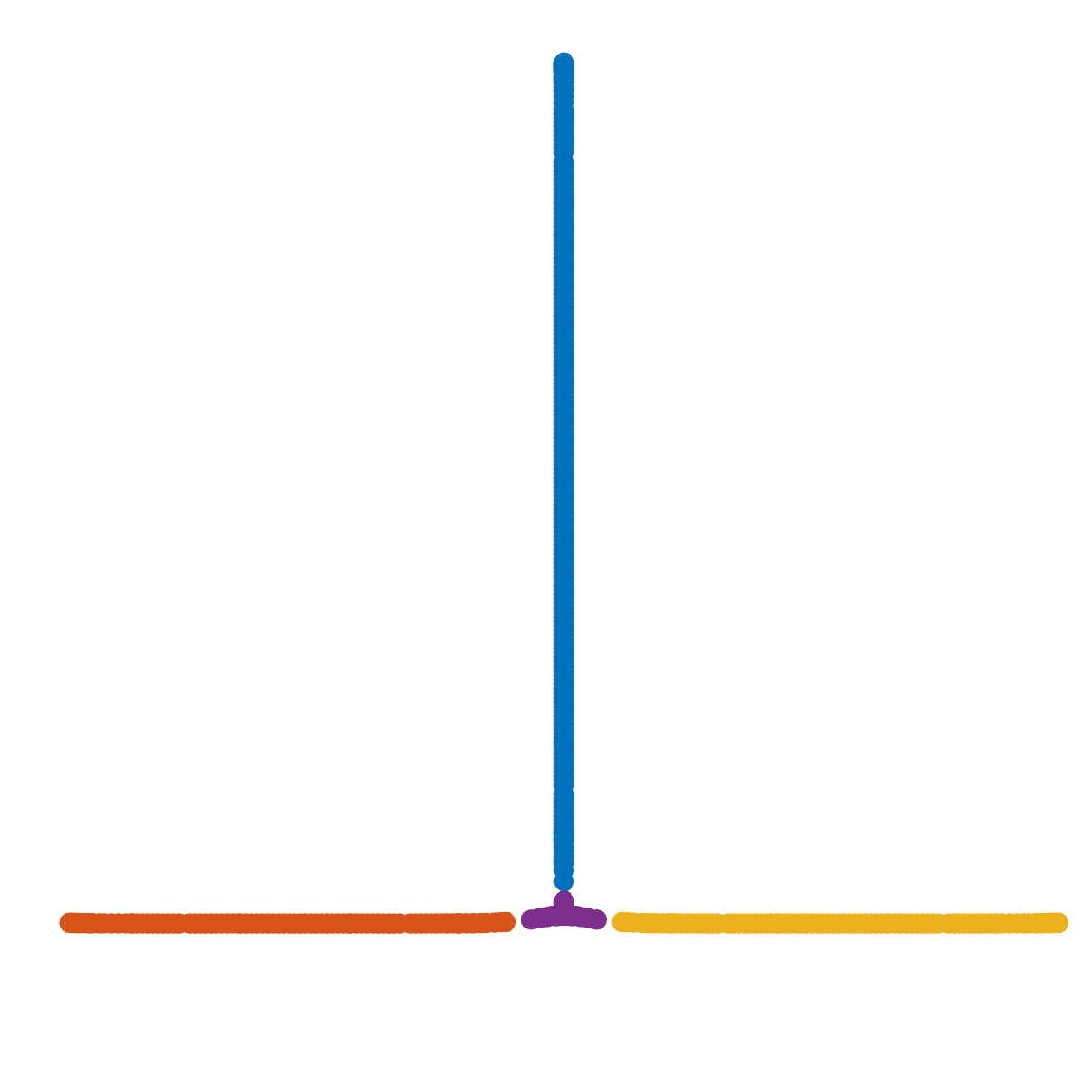}
    		    \\
    		   \includegraphics[width=1\textwidth]{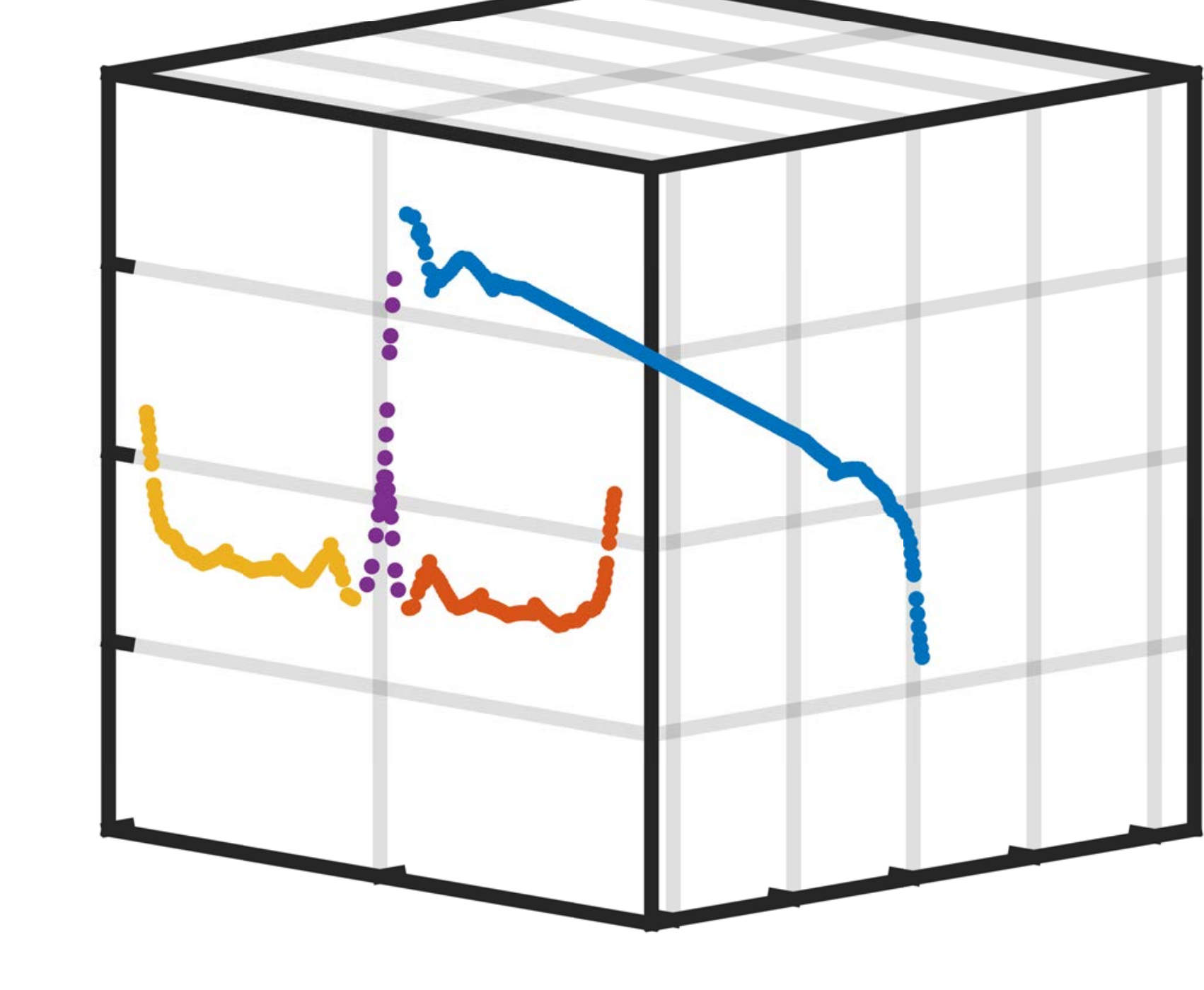} 
            \end{minipage}
            \begin{minipage}[t][2.35cm][t]{0.08\textwidth}
    		    \centering
    		    GT-$\lambda$-5 \\ 
                \includegraphics[width=1\textwidth]{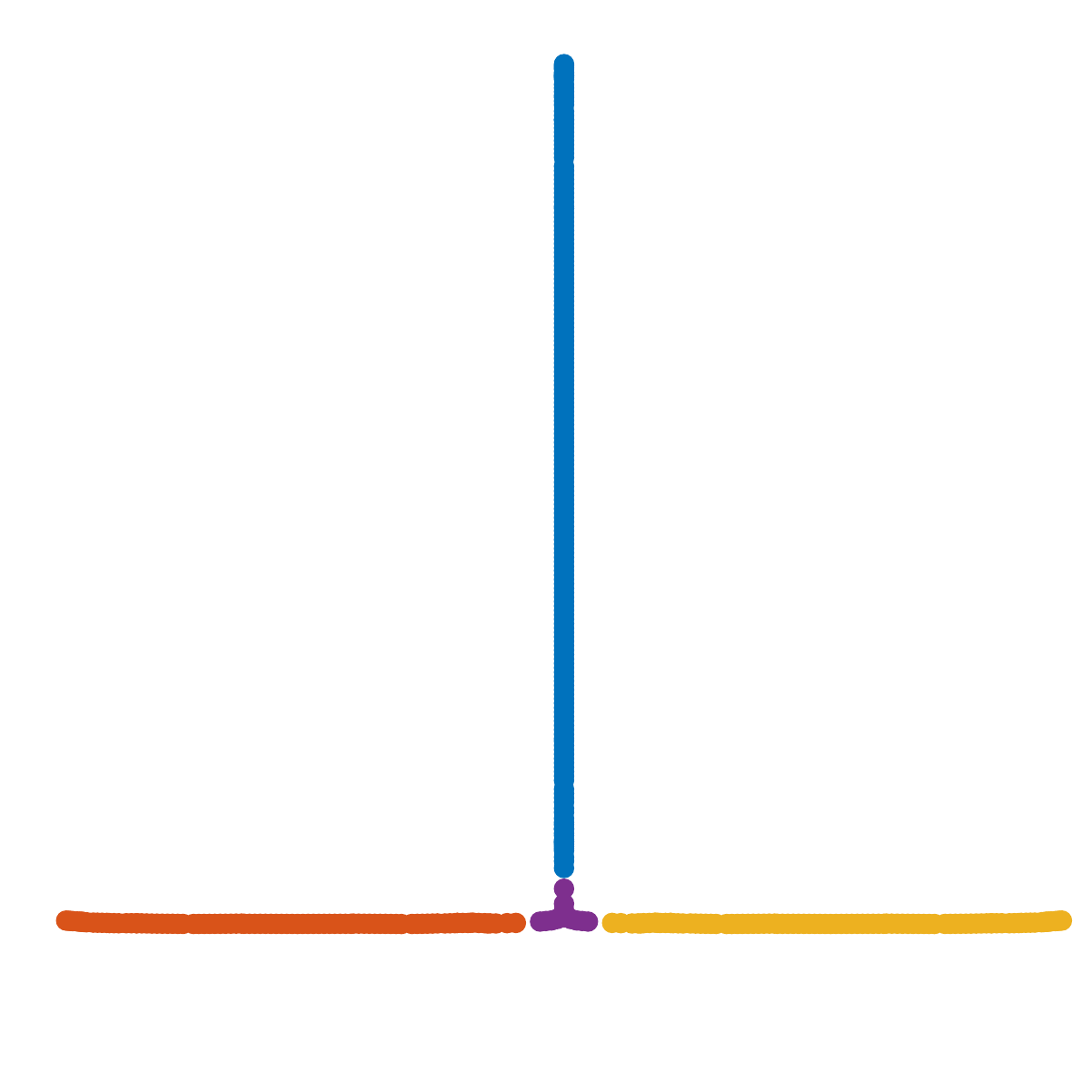}
    		    \\
    		   \includegraphics[width=1\textwidth]{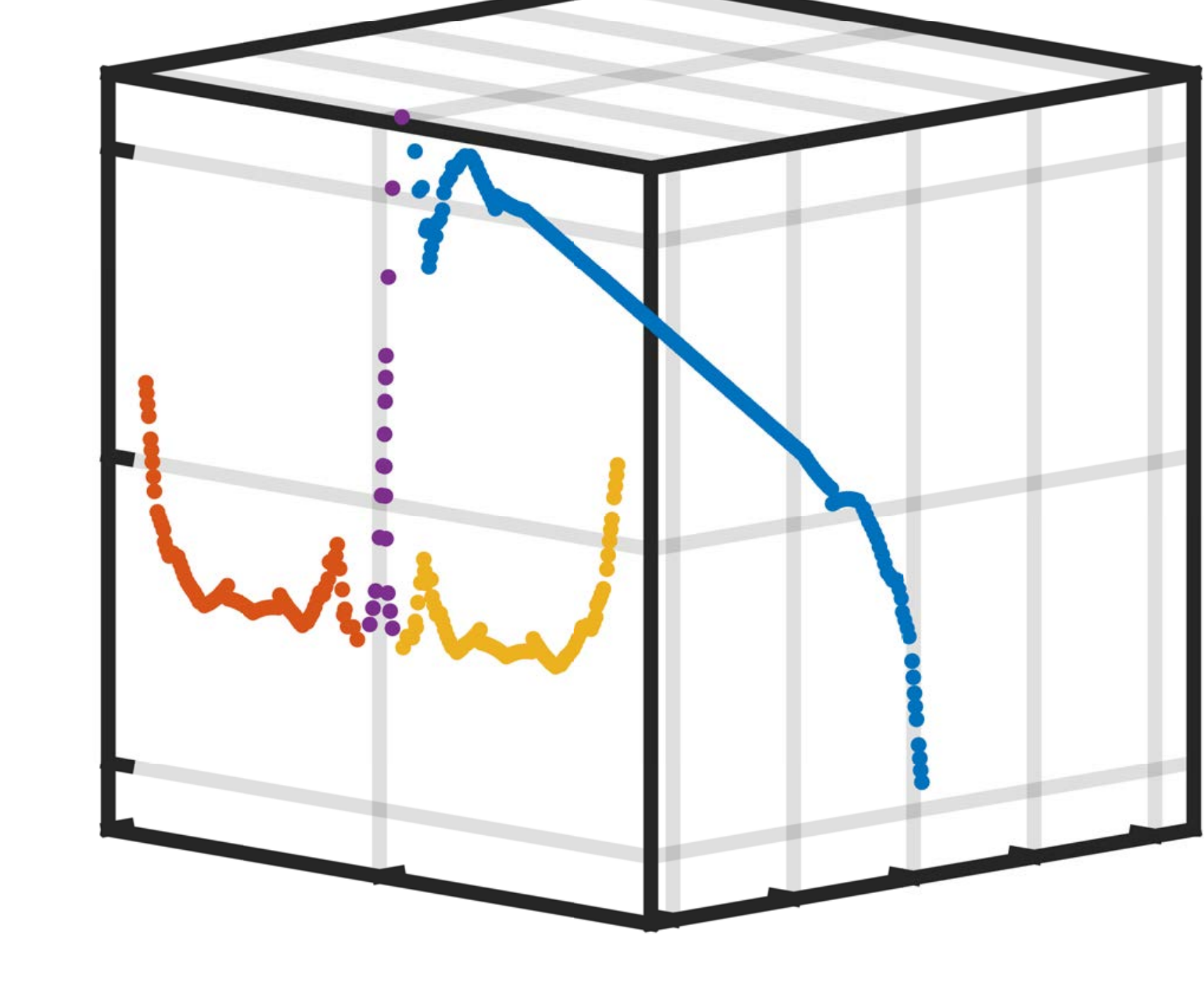} 
            \end{minipage}
            \begin{minipage}[t][2.35cm][t]{0.08\textwidth}
    		    \centering
    		    WT2 \\ 
    		    \includegraphics[width=1\textwidth]{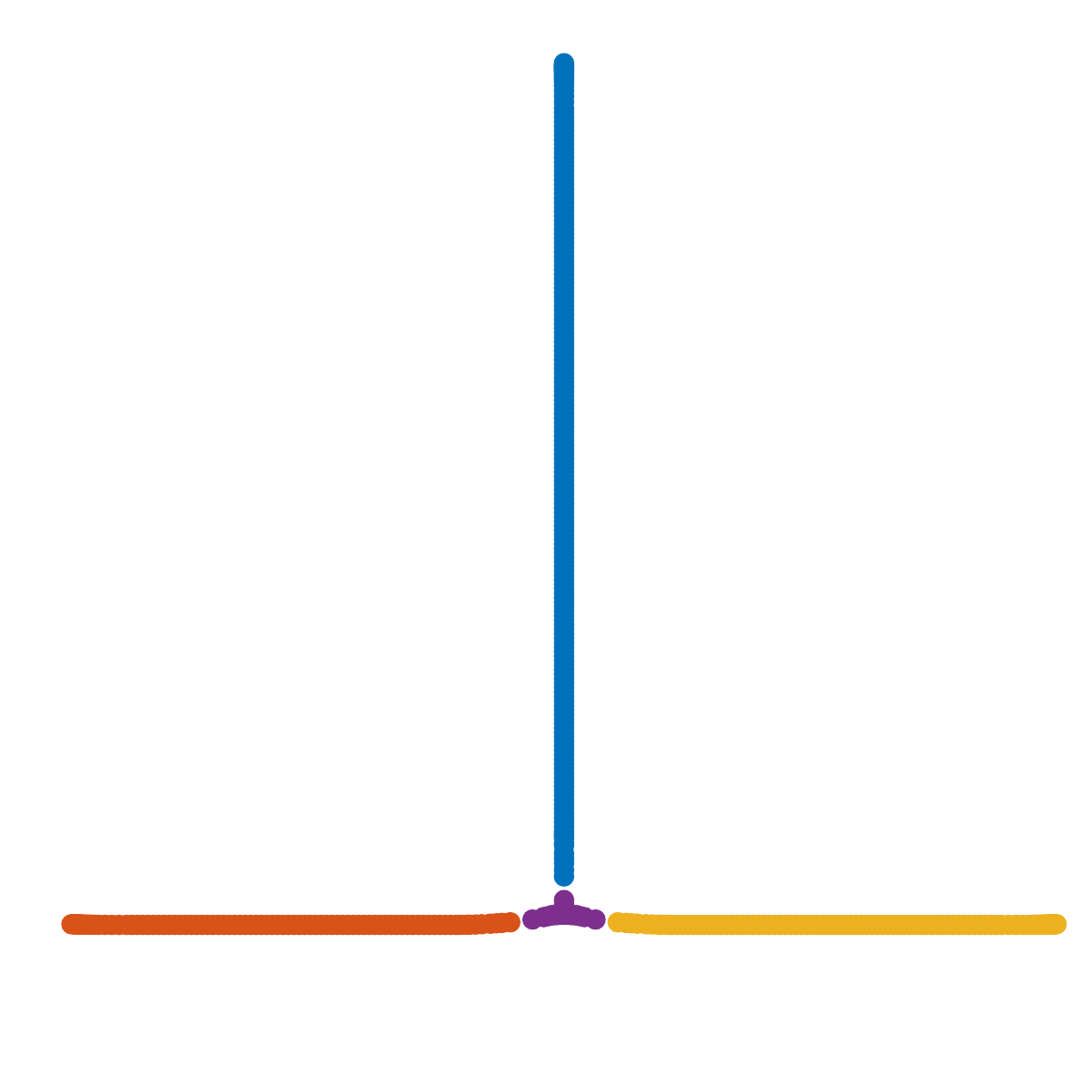}
                \\
               \includegraphics[width=1\textwidth]{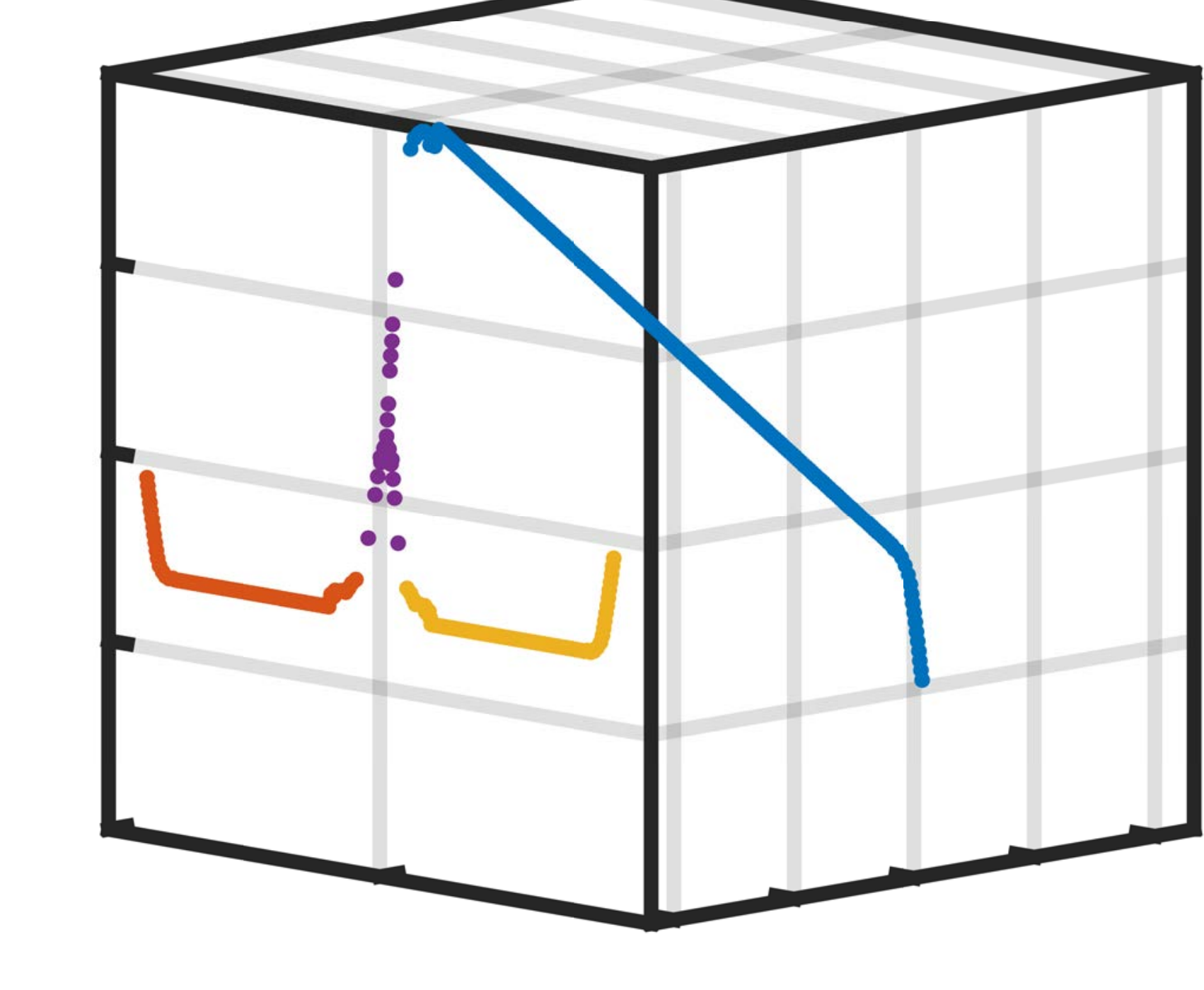} 
            \end{minipage}
        }
	}
	\subfloat[Dendrograms at $\tau=2$]{
	\label{fig:tline-dend}
        \fbox{
            \begin{minipage}[t][2.35cm][t]{0.08\textwidth}
    		    \centering
    		    MS \\ 
                \includegraphics[width=1\textwidth]{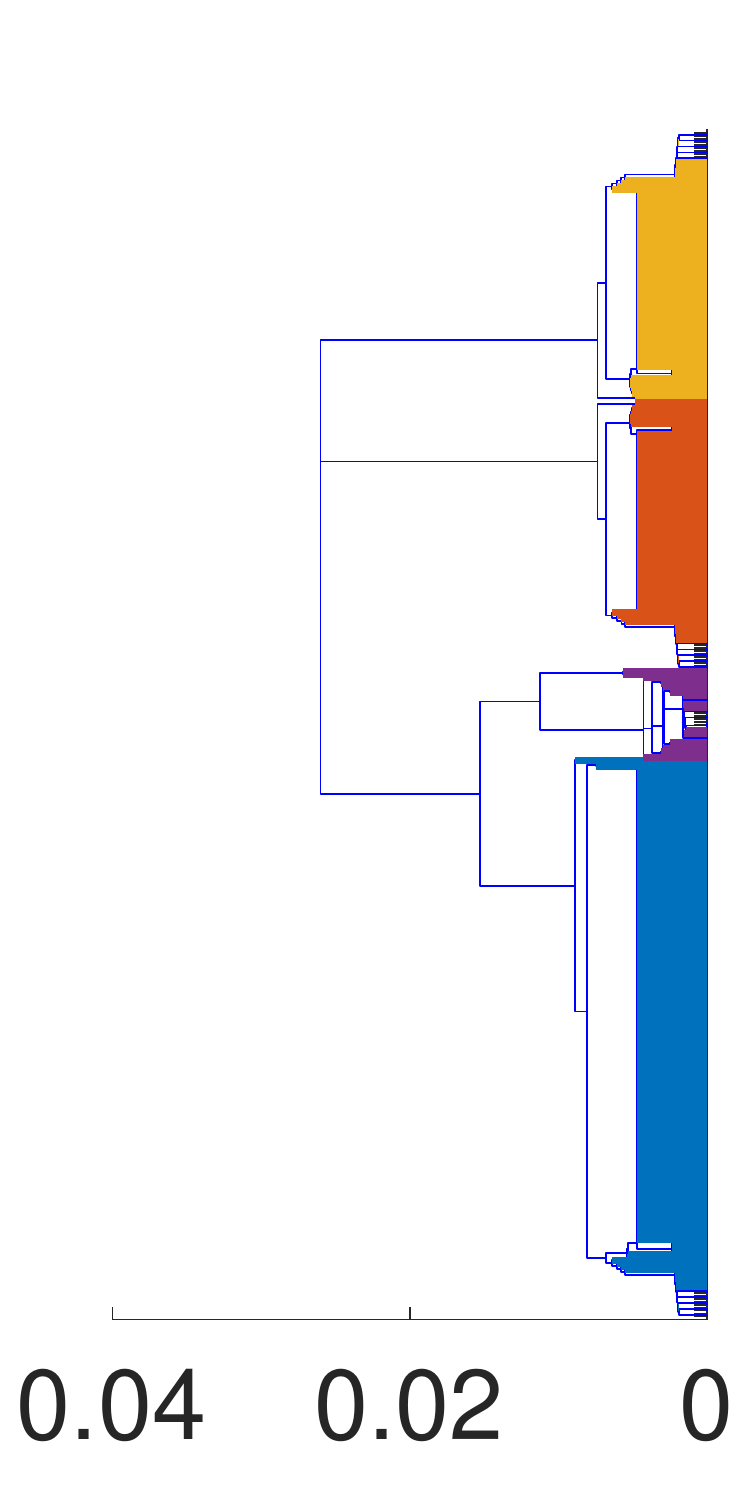}
            \end{minipage}
		    \begin{minipage}[t][2.35cm][t]{0.08\textwidth}
    		    \centering
    		    GT-$\lambda$-1 \\ 
                \includegraphics[width=1\textwidth]{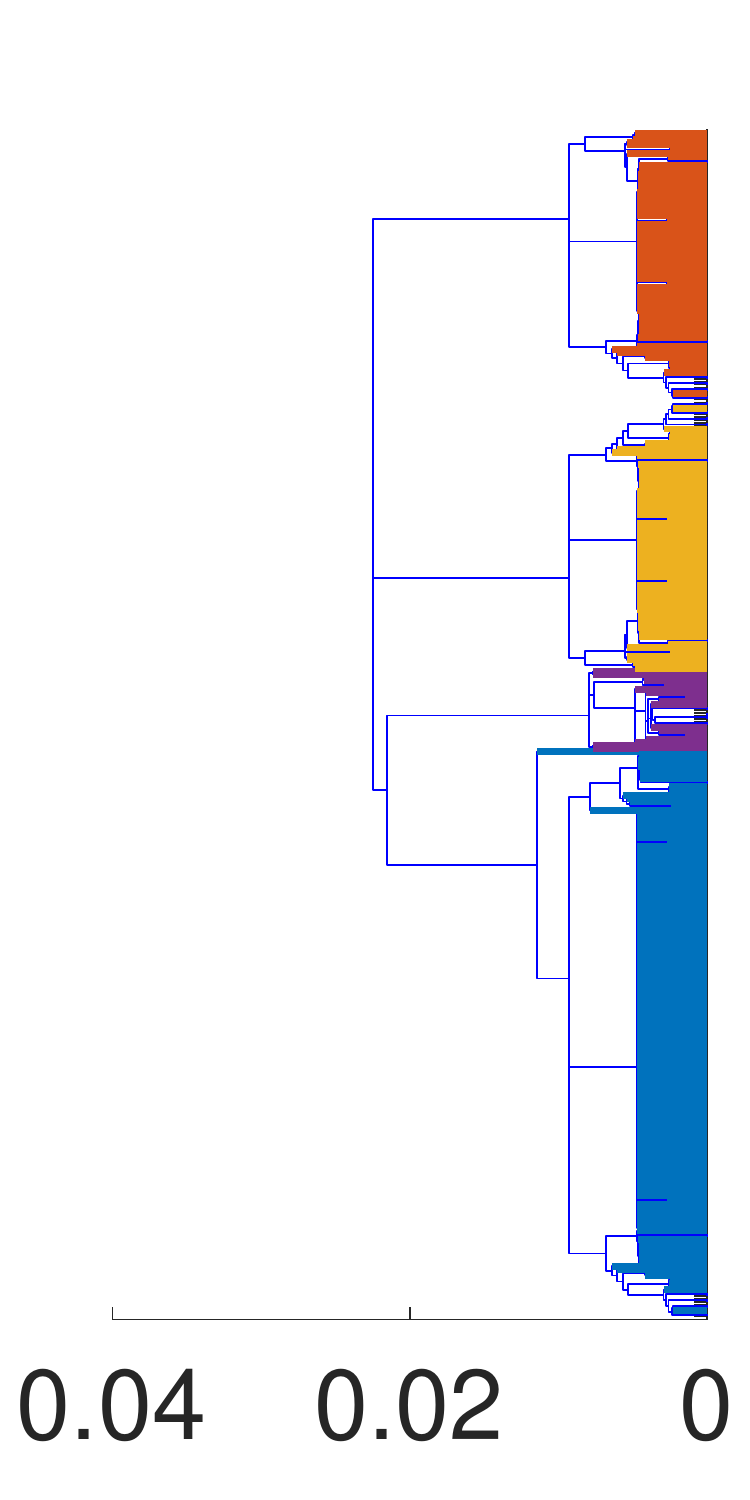}
            \end{minipage}
            \begin{minipage}[t][2.35cm][t]{0.08\textwidth}
    		    \centering
    		    GT-$\lambda$-5 \\ 
                \includegraphics[width=1\textwidth]{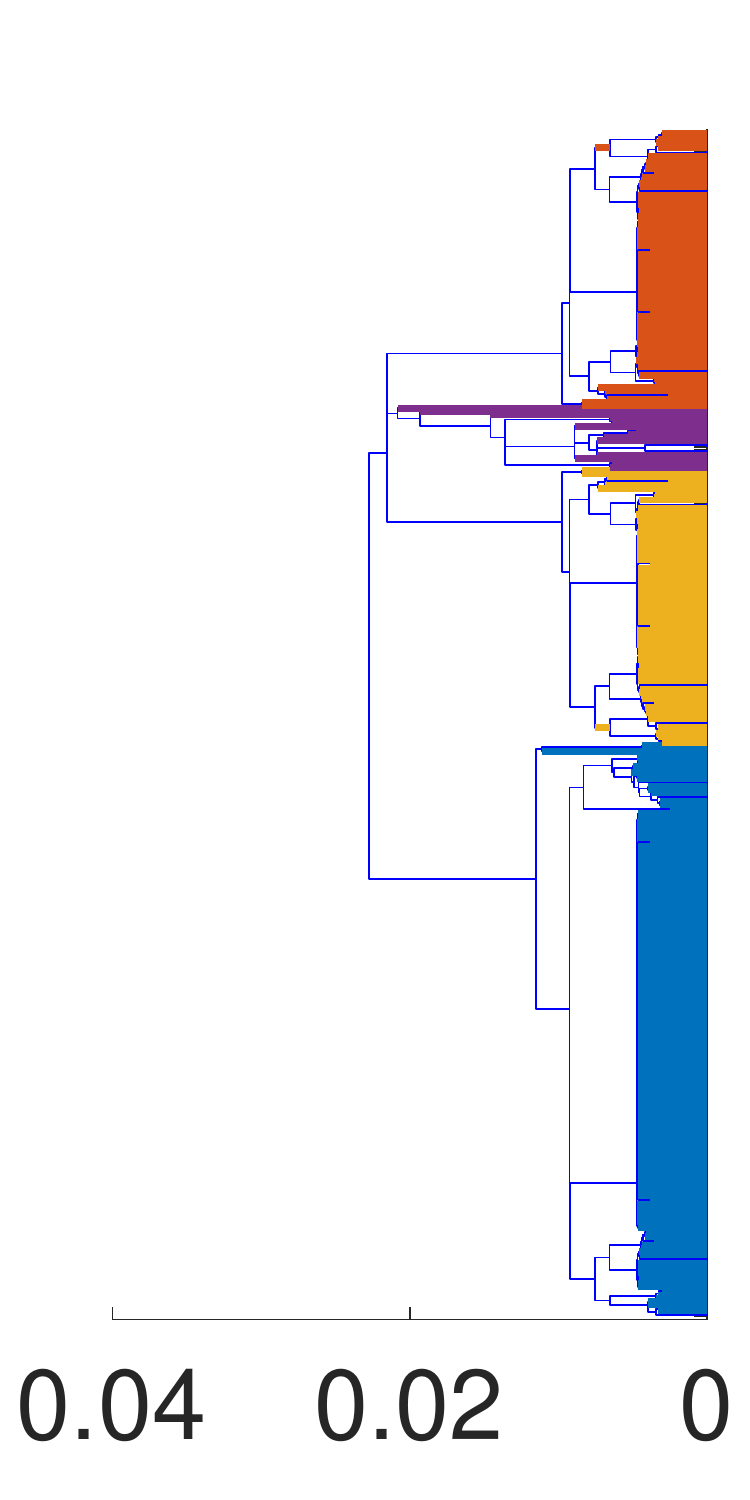}
            \end{minipage}  
            \begin{minipage}[t][2.35cm][t]{0.08\textwidth}
    		    \centering
    		    WT2 \\ 
                \includegraphics[width=1\textwidth]{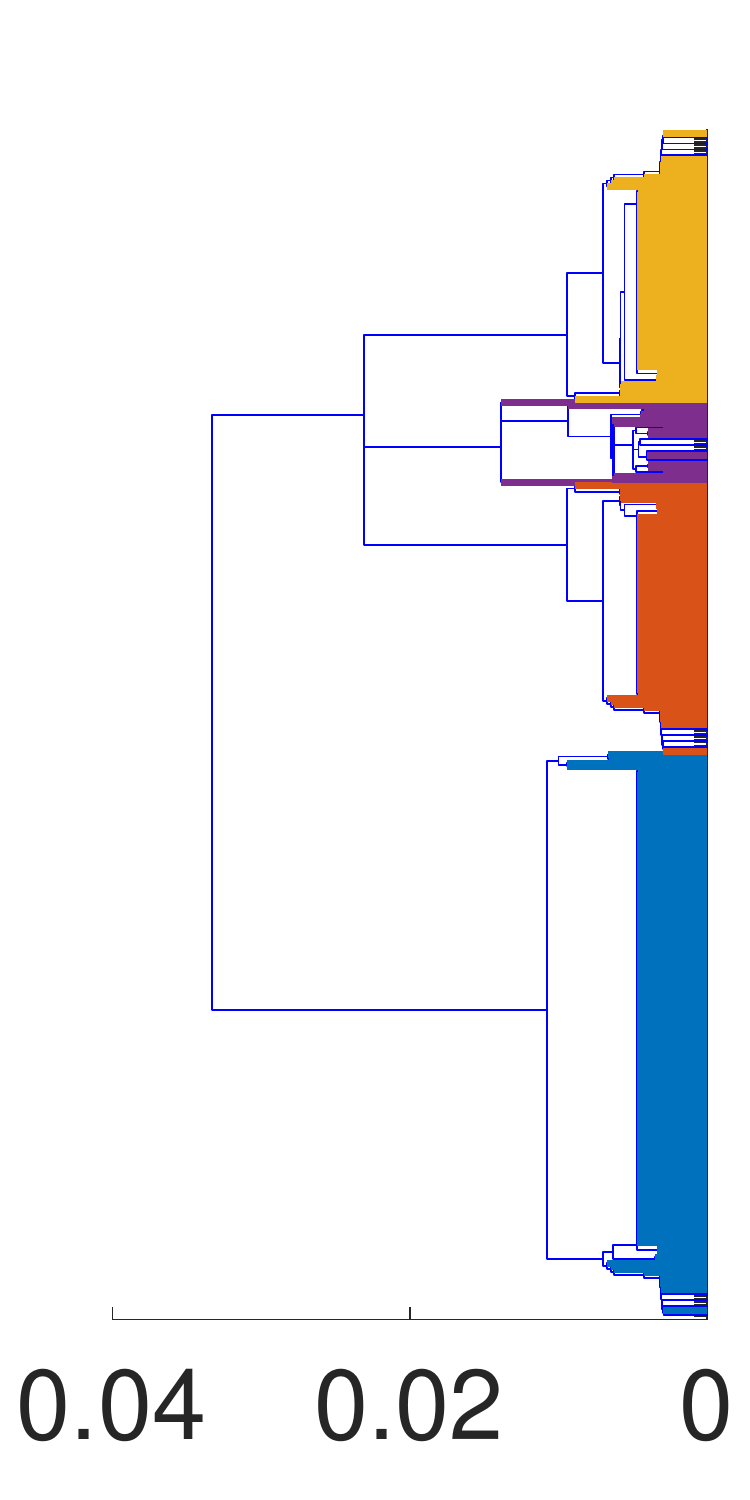}
            \end{minipage}
                
        } 
	}
    \caption{T-junction clustering. 
    (a): Original dataset. 
    (b): The first column shows the updated point cloud based on MS after 2 iterations; the next three columns shows the 2D and 3D MDS of distance matrix based on GT with $\lambda$=1, GT with $\lambda$=5 and WT2 after 2 iterations, respectively (results after 1 iteration is given in the supplementary material). Different colors in (b) represent different clusters of dataset, which are obtained by slicing the dendrograms illustrated in (c). 
    (c): The four columns demonstrate the clustering dendrograms using methods in (b).
    } 
    \label{fig:tlines}
\end{figure}

\textbf{Ameliorating the chaining effect.} 
In this application, we consider a clustering task on a dataset with two blobs connected by a chain. Each blob is composed of 300 uniform grid points in a circle of radius 1 and the chain is composed 200 uniform grid points with length 2. We set $\eps=0.2$ in this experiment. Standard single linkage hierarchical clustering fails to detect two such blobs -- a phenomenon known as the \emph{chaining effect}. However, GT helps separate the two blobs and improve the quality of dendrograms throughout iterations. 
See Figure~(\ref{fig:ellip-iters}) for the results. 
We further apply linear transformations  $T=\begin{psmallmatrix}
a_1 & 0\\0 & a_2
\end{psmallmatrix}$  to the dataset to examine how the geometry of datasets influences the clustering performance of MS, WT2 and GT. Define by $e\coloneqq a_2/a_1$ the \emph{eccentricity} of $T$. We apply the methods on transformed datasets with differen eccentricities for 1 iteration. Partial results are shown in Figure~(\ref{fig:ballchains-mswt}) and (\ref{fig:ballchains-gt}).We see from the dendrograms that {MS, WT2 and GT-$\lambda$-1 have similar performance in clustering. One noticeable observation is that under extreme condition of $e:=1/0.2$, GT-$\lambda$-5 outperforms MS, WT2 and GT-$\lambda$-1 in separating two blobs from the chain, which implies that $\lambda$ plays an important role for GT in enhancing geometric structure. Please refer to the supplementary material for more details and results. } 
 
\begin{figure}[htb]
    \centering
    \subfloat[GT and chaining effect]{
    \label{fig:ellip-iters}
    \fbox{
      \begin{minipage}[t][2.8cm][t]{0.07\textwidth}
          \centering
    		    $\tau=0$    
    		    \includegraphics[width=1\textwidth]{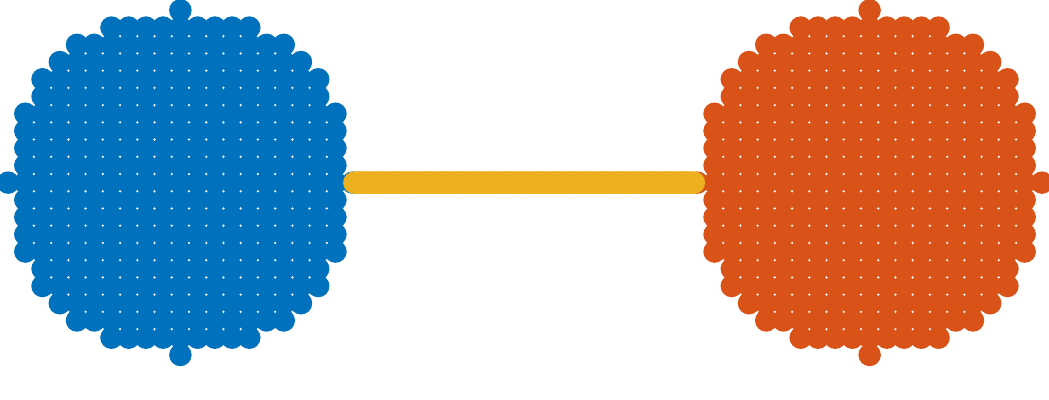} 
              \begin{minipage}[t][0.45cm][t]{0.07\textwidth}
              \centering
                 
              \end{minipage}
              
             \includegraphics[width=1\textwidth]{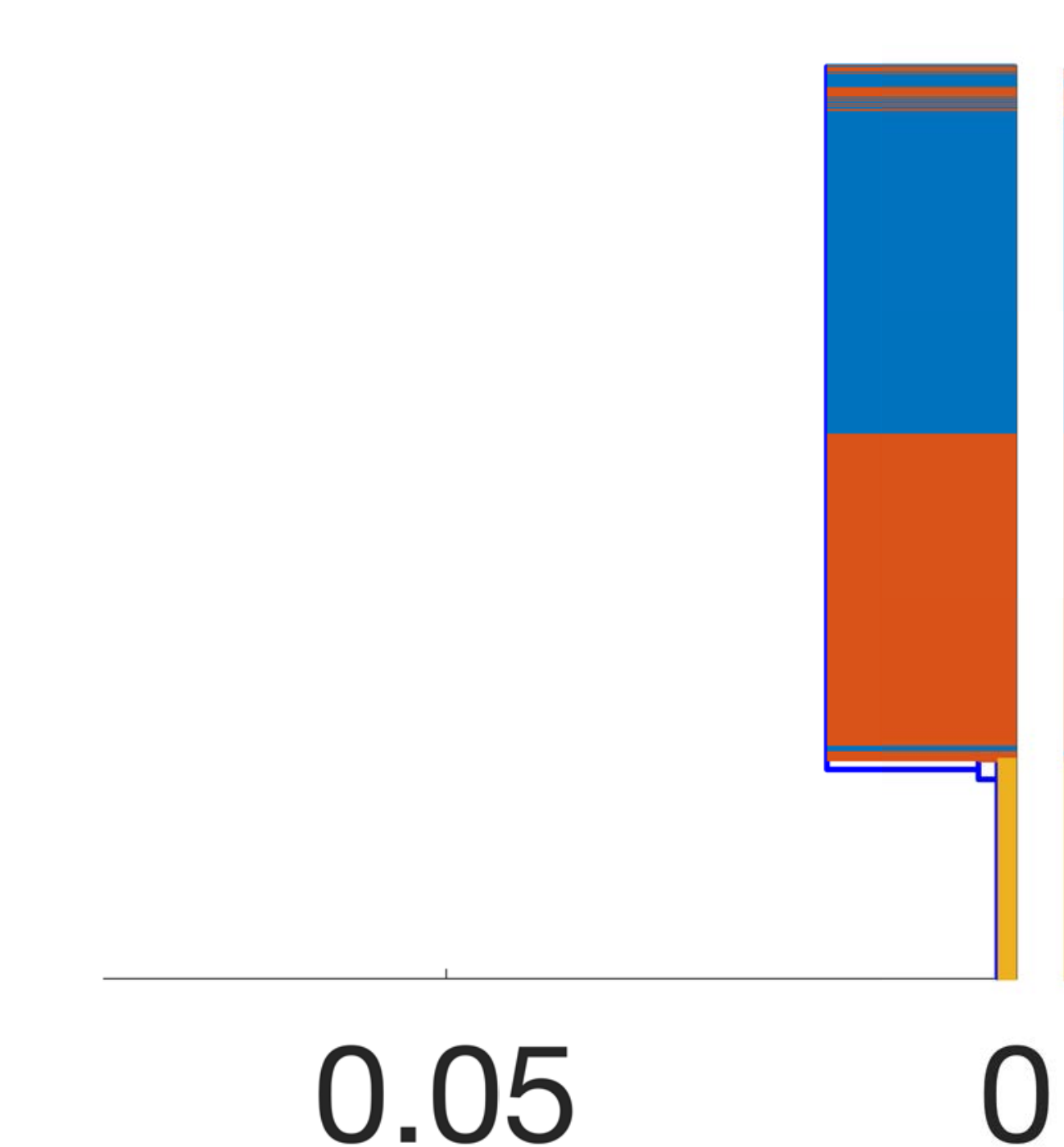}
          \end{minipage}
          \begin{minipage}[t]{0.07\textwidth}
          \centering
             $\tau=1$ 
            \includegraphics[width=1\textwidth]{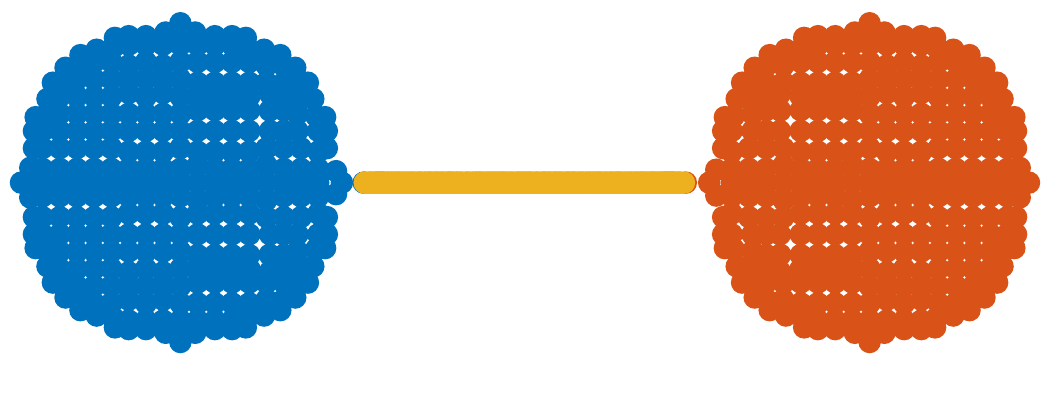}
               
             \includegraphics[width=1\textwidth]{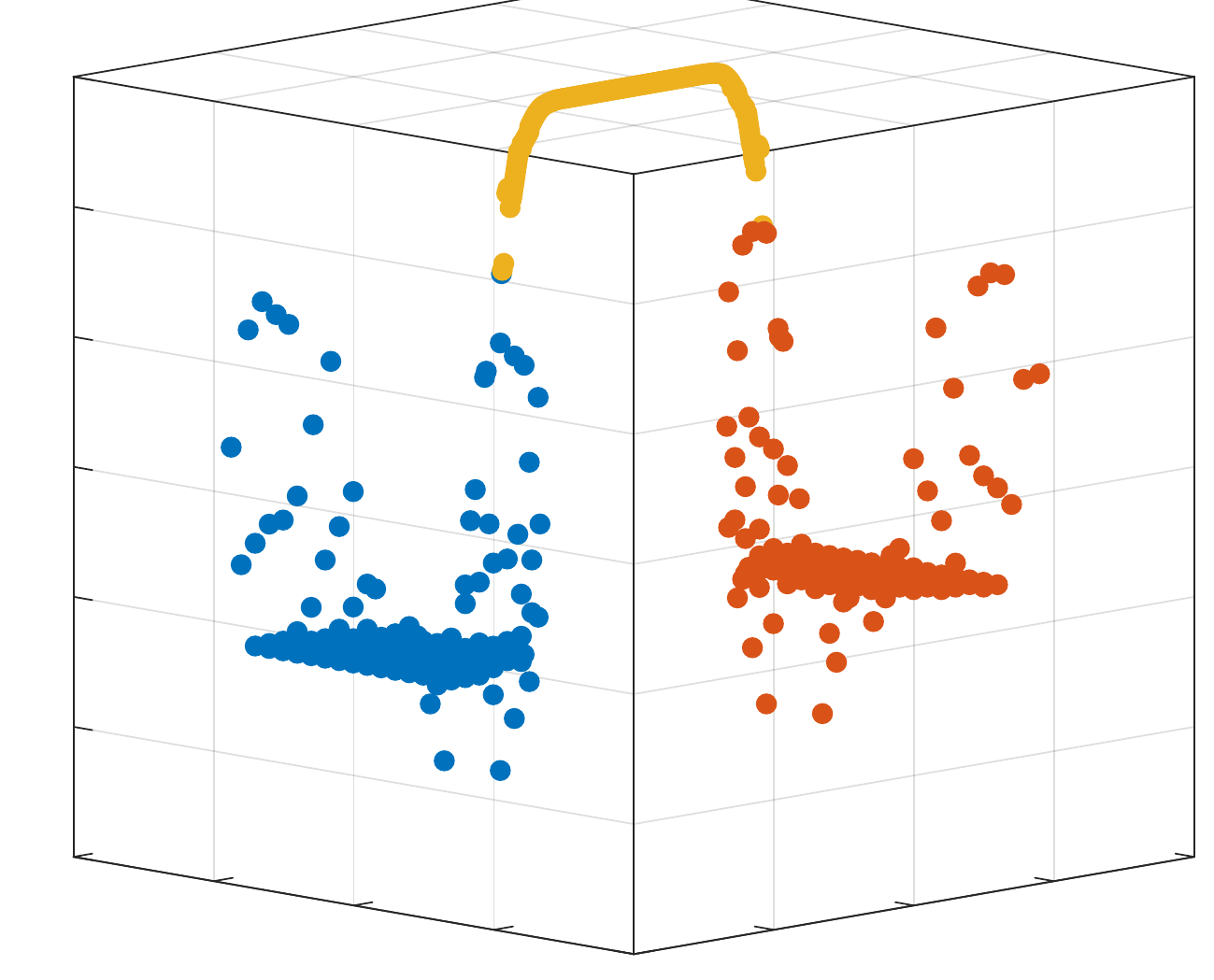}
              
             \includegraphics[width=1\textwidth]{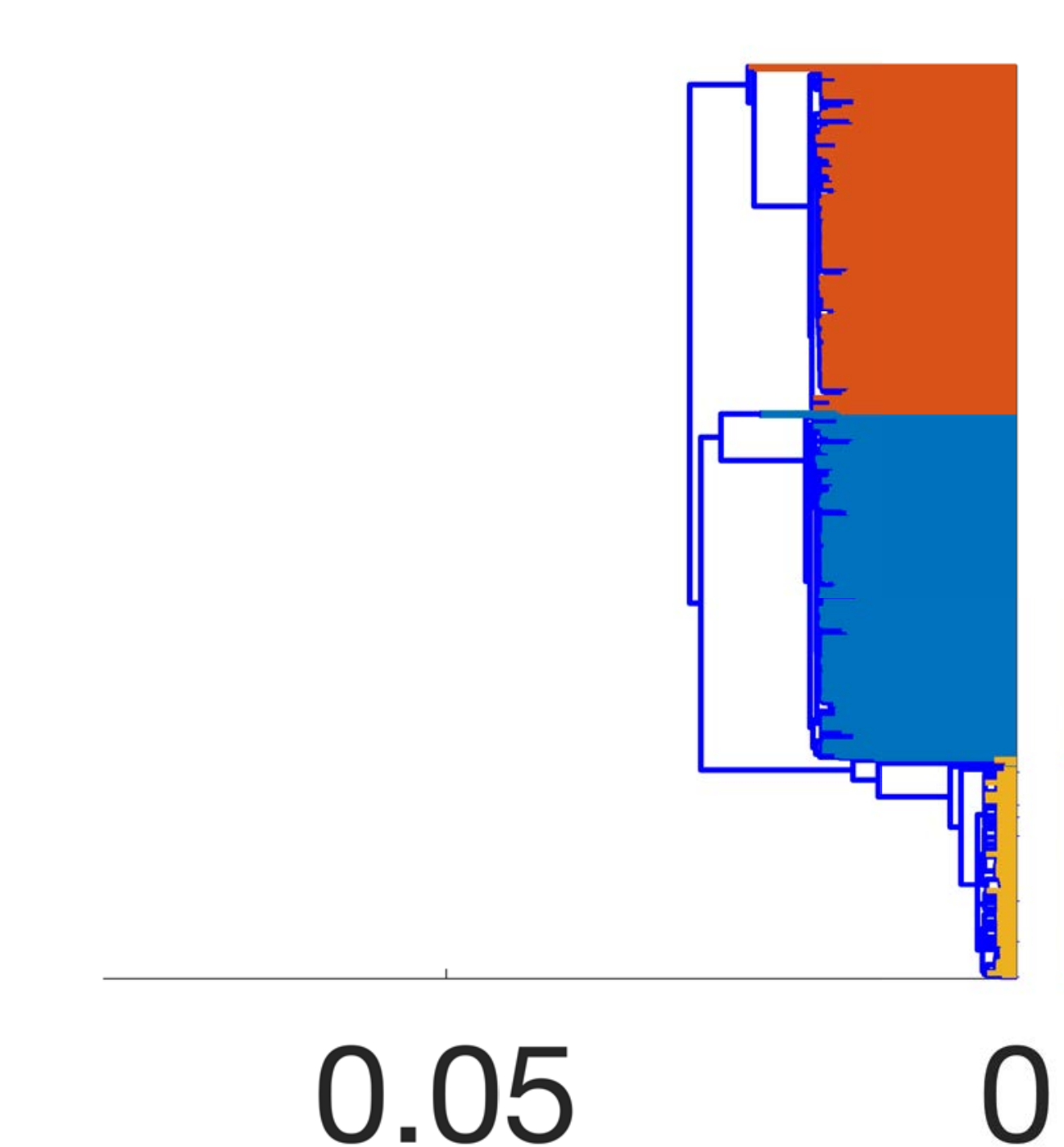}
          \end{minipage}
          \begin{minipage}[t]{0.07\textwidth}
          \centering
             $\tau=2$   
             \includegraphics[width=1\textwidth]{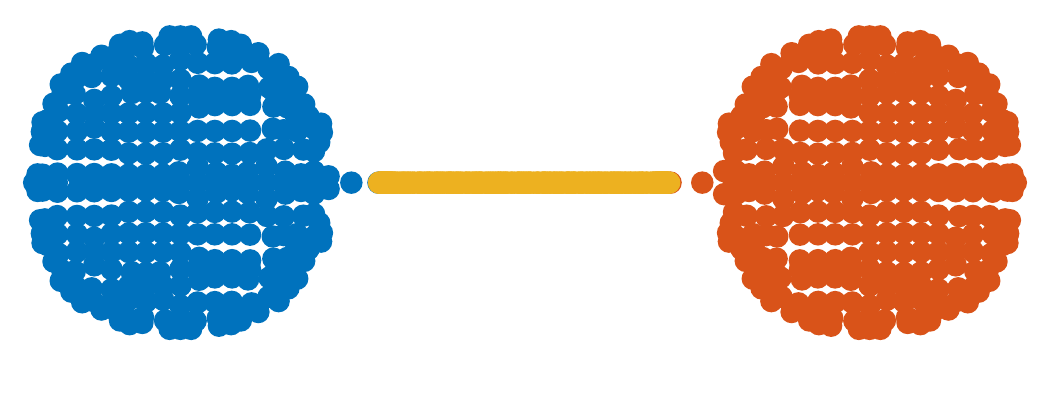}  
             \includegraphics[width=1\textwidth]{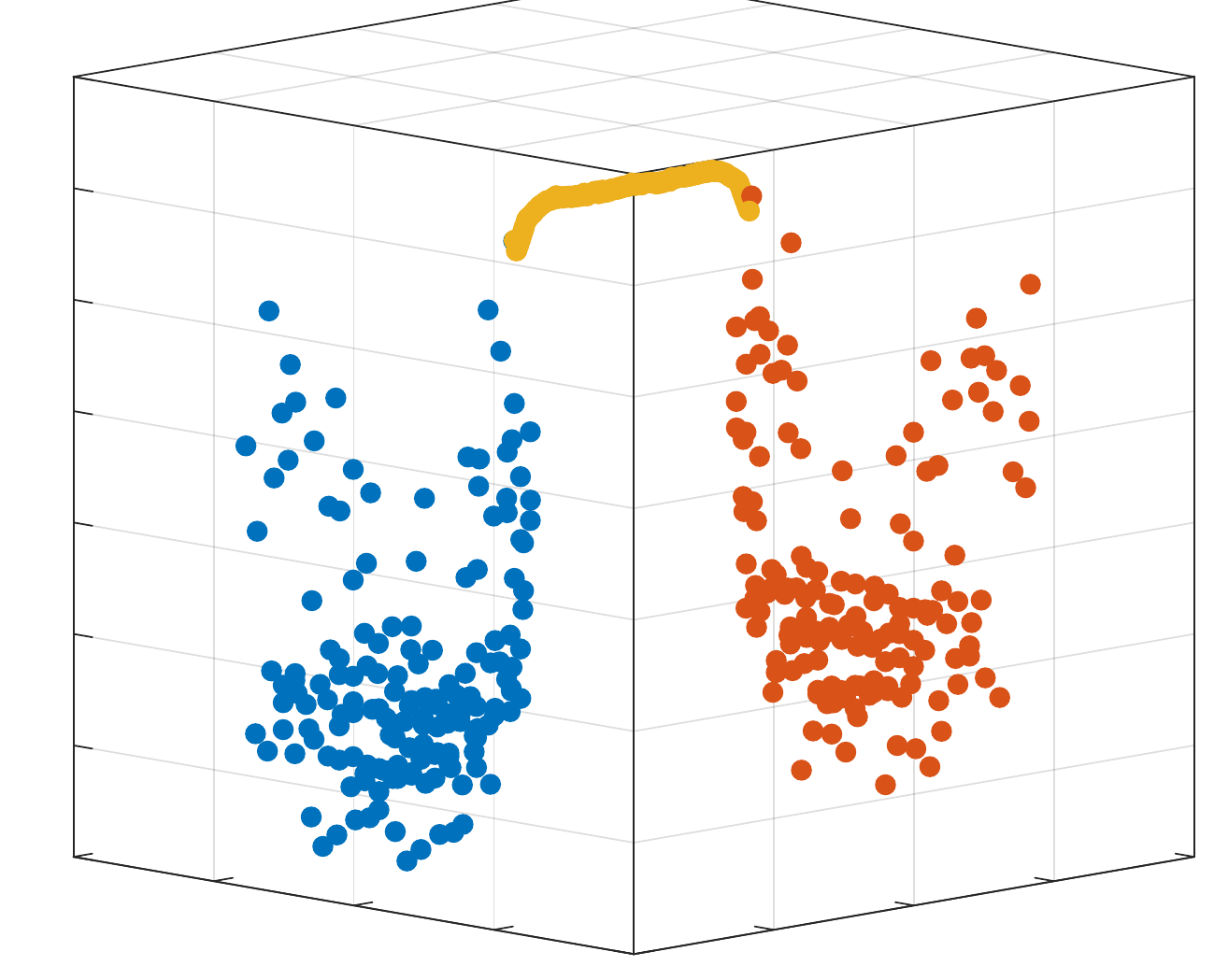} 
              
             \includegraphics[width=1\textwidth]{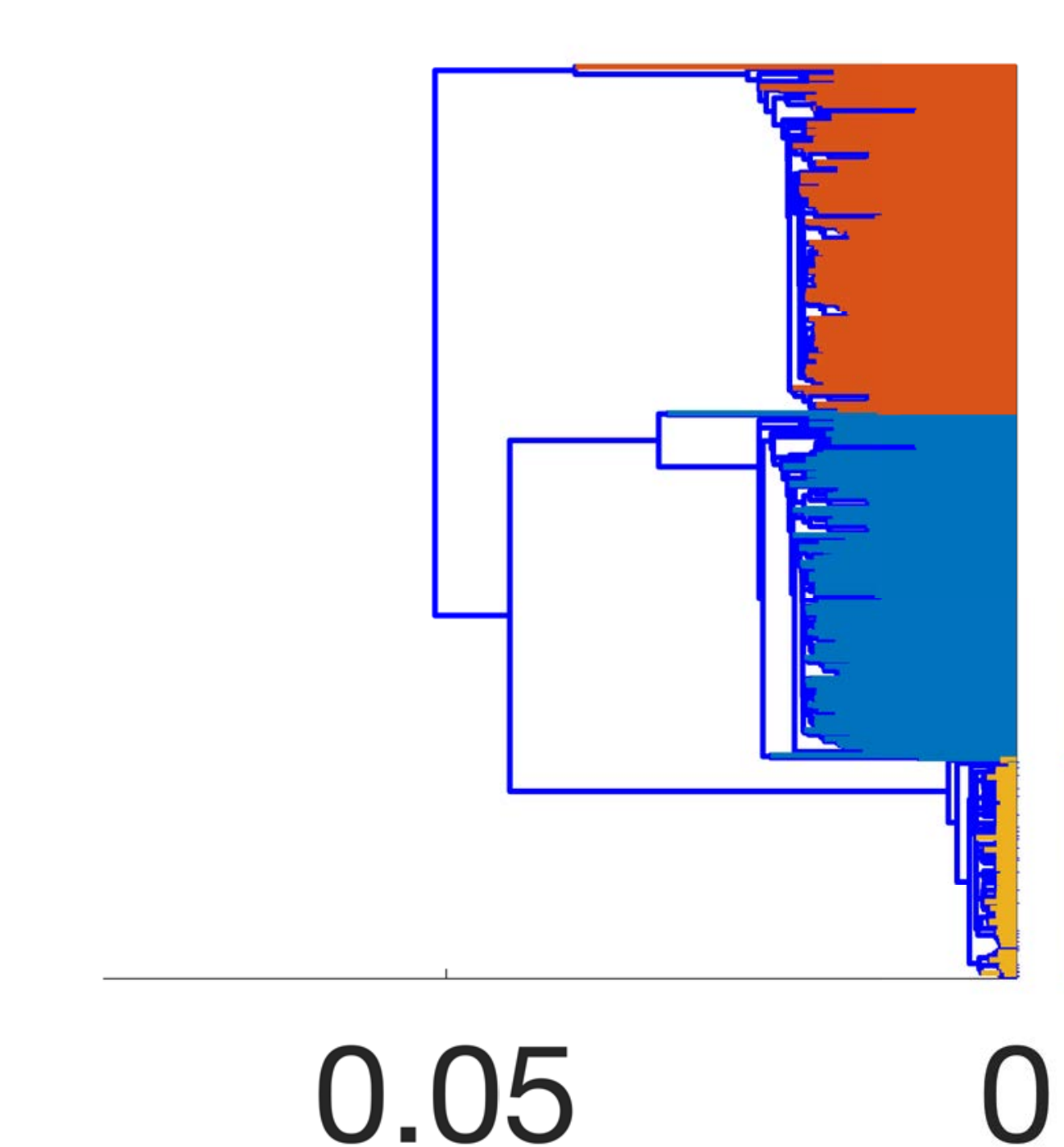}
          \end{minipage}
        \begin{minipage}[t]{0.07\textwidth}
          \centering
             $\tau=3$   
             \includegraphics[width=1\textwidth]{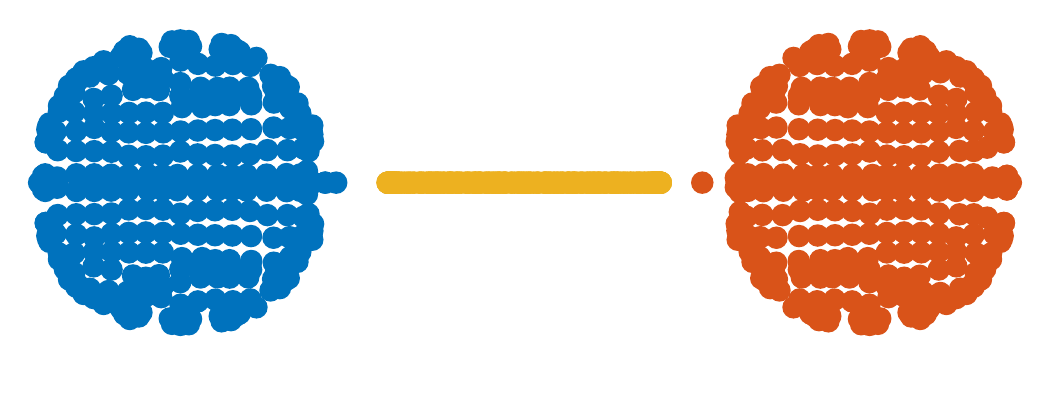}  
             
             \includegraphics[width=1\textwidth]{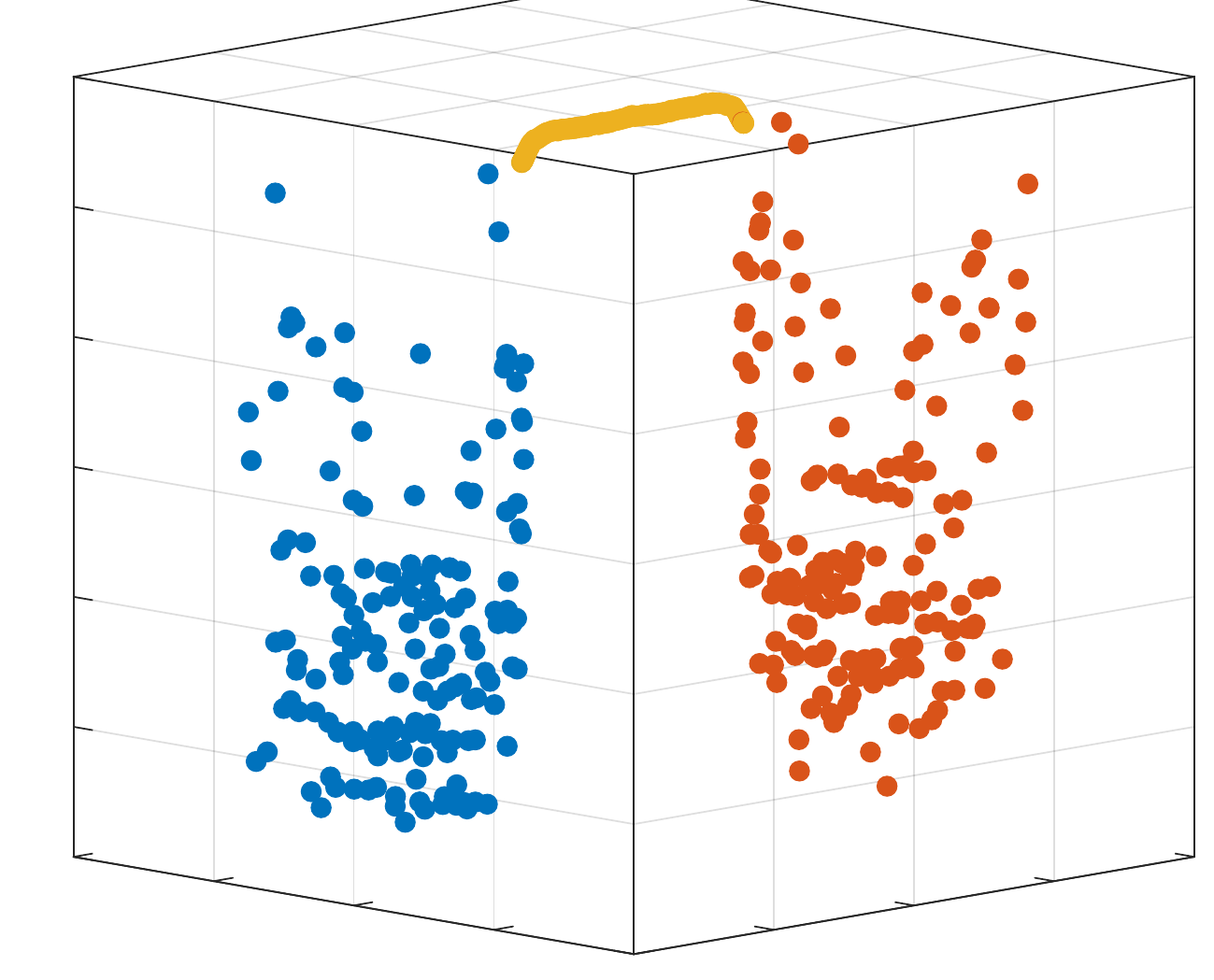}
             
             \includegraphics[width=1\textwidth]{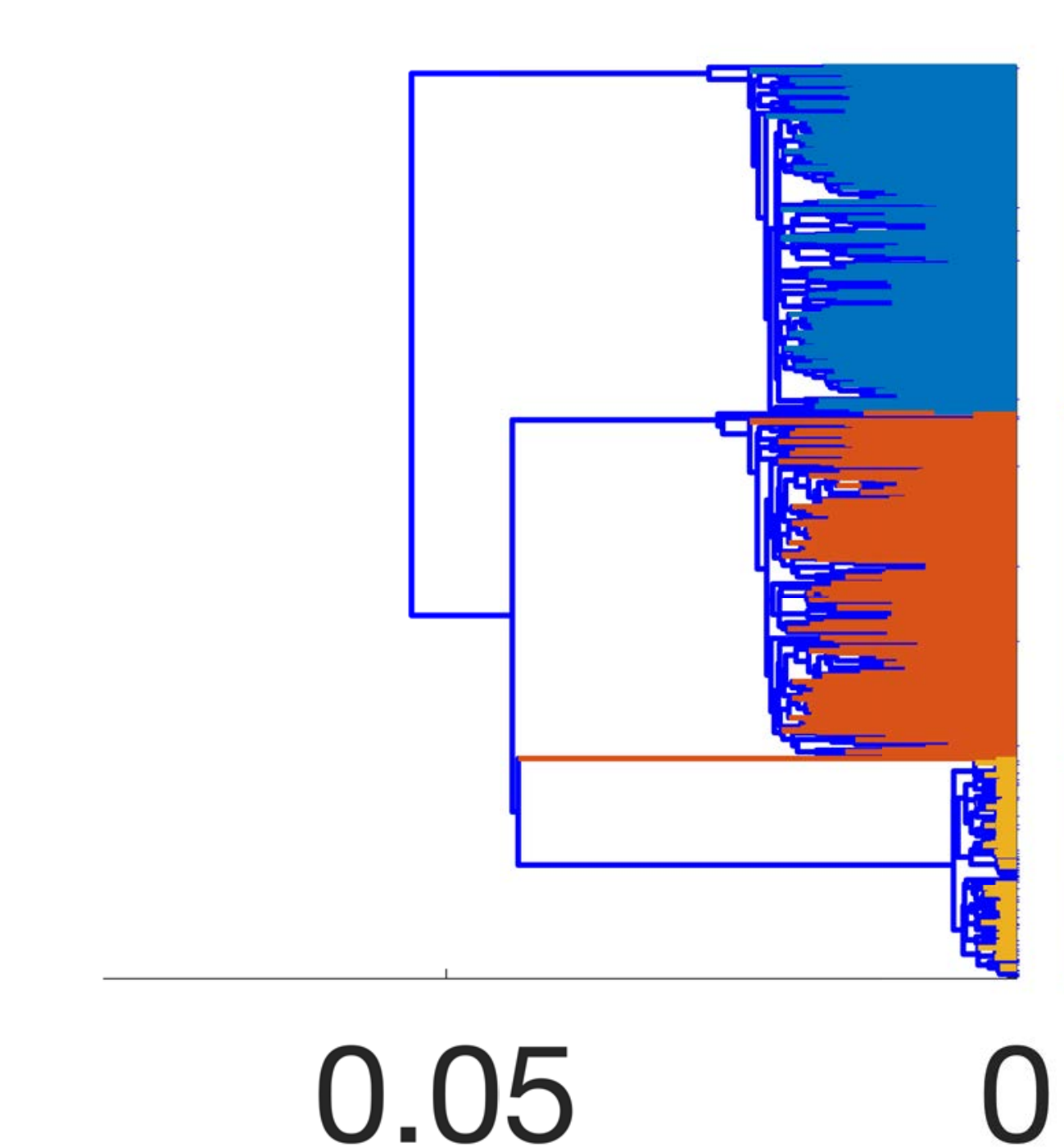}
          \end{minipage}
          \begin{minipage}[t]{0.07\textwidth}
          \centering
             $\tau=4$ 
             \includegraphics[width=1\textwidth]{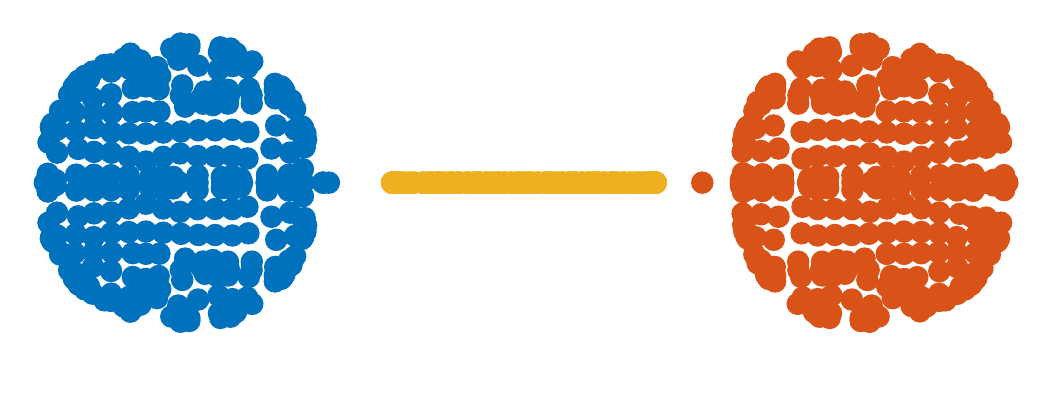}
             
             \includegraphics[width=1\textwidth]{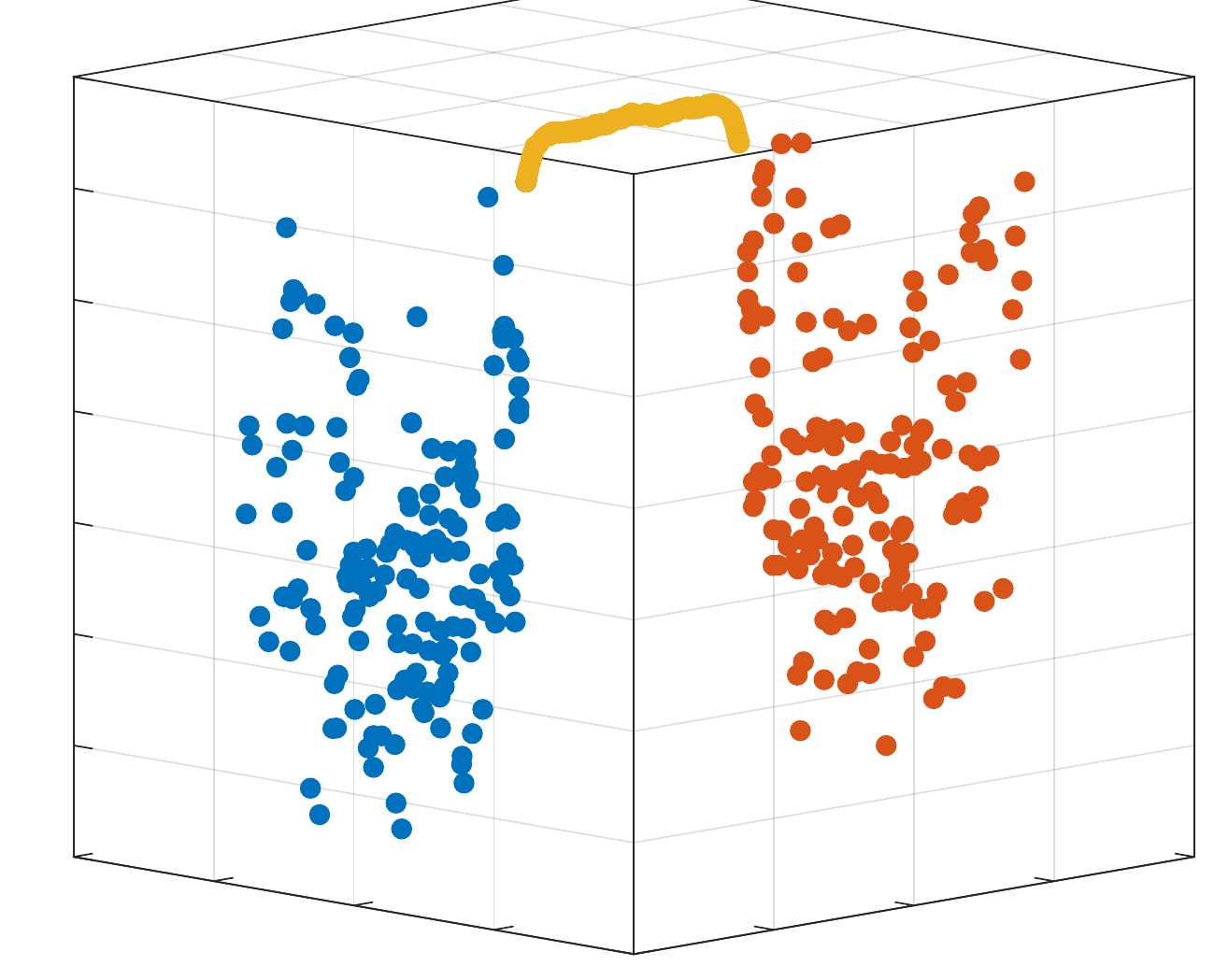}
             
             \includegraphics[width=1\textwidth]{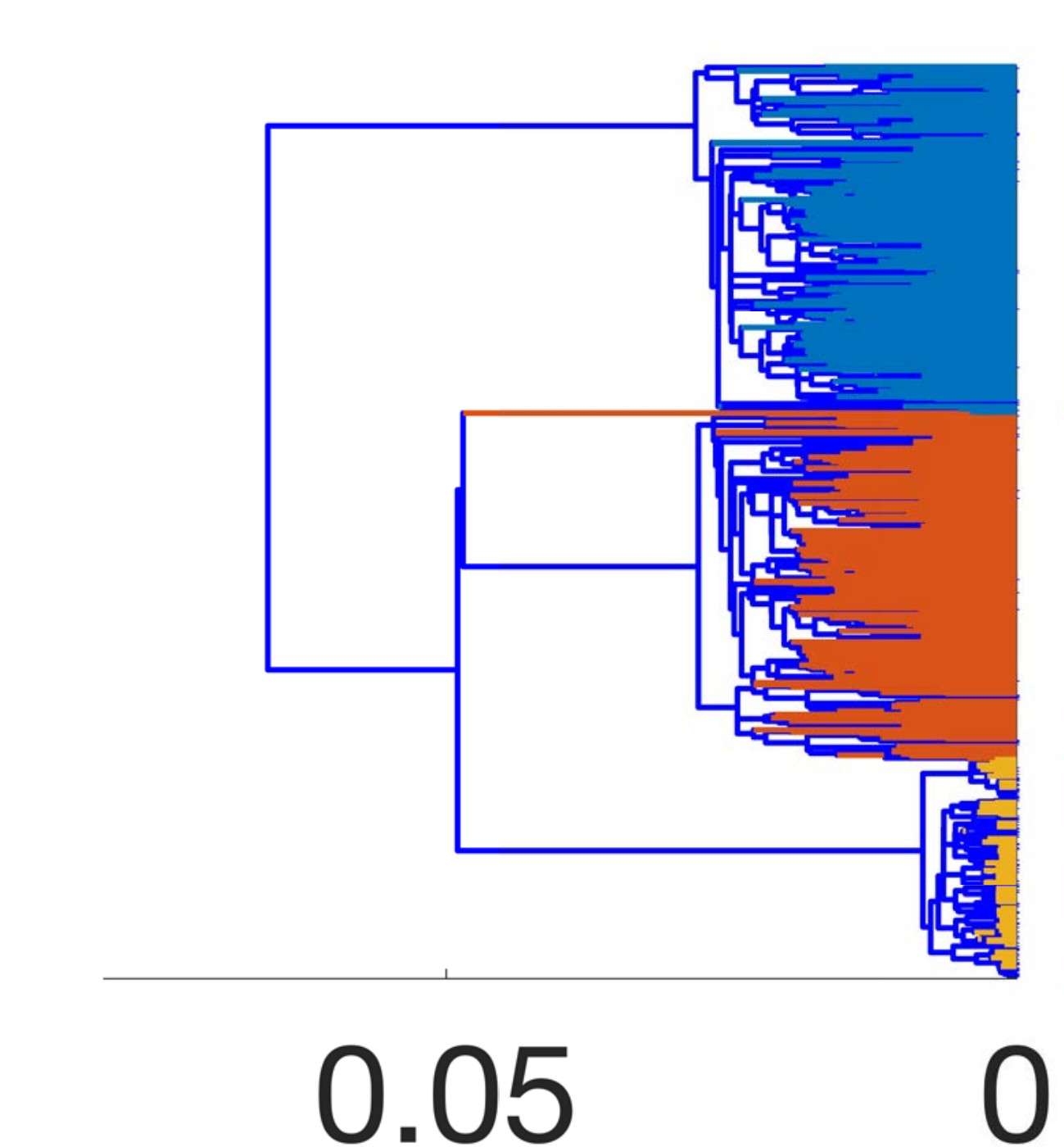}
          \end{minipage}
          }
    }
    \subfloat[Dataset, MS, WT2]{
    \label{fig:ballchains-mswt}
    \fbox{
		\begin{minipage}[t]{0.07\textwidth}
		    \centering
            \begin{minipage}[t]{1\textwidth}
    		    \centering
    		    $e$:1/1 \\ 
                \includegraphics[width=1\textwidth]{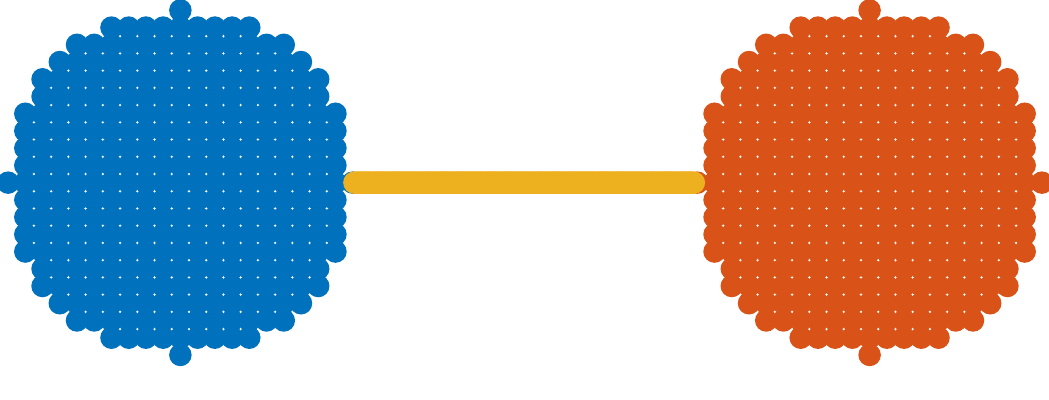} 
            \end{minipage}
            \begin{minipage}[t][0.55cm][t]{1\textwidth}
    		    \centering
                \includegraphics[width=1\textwidth]{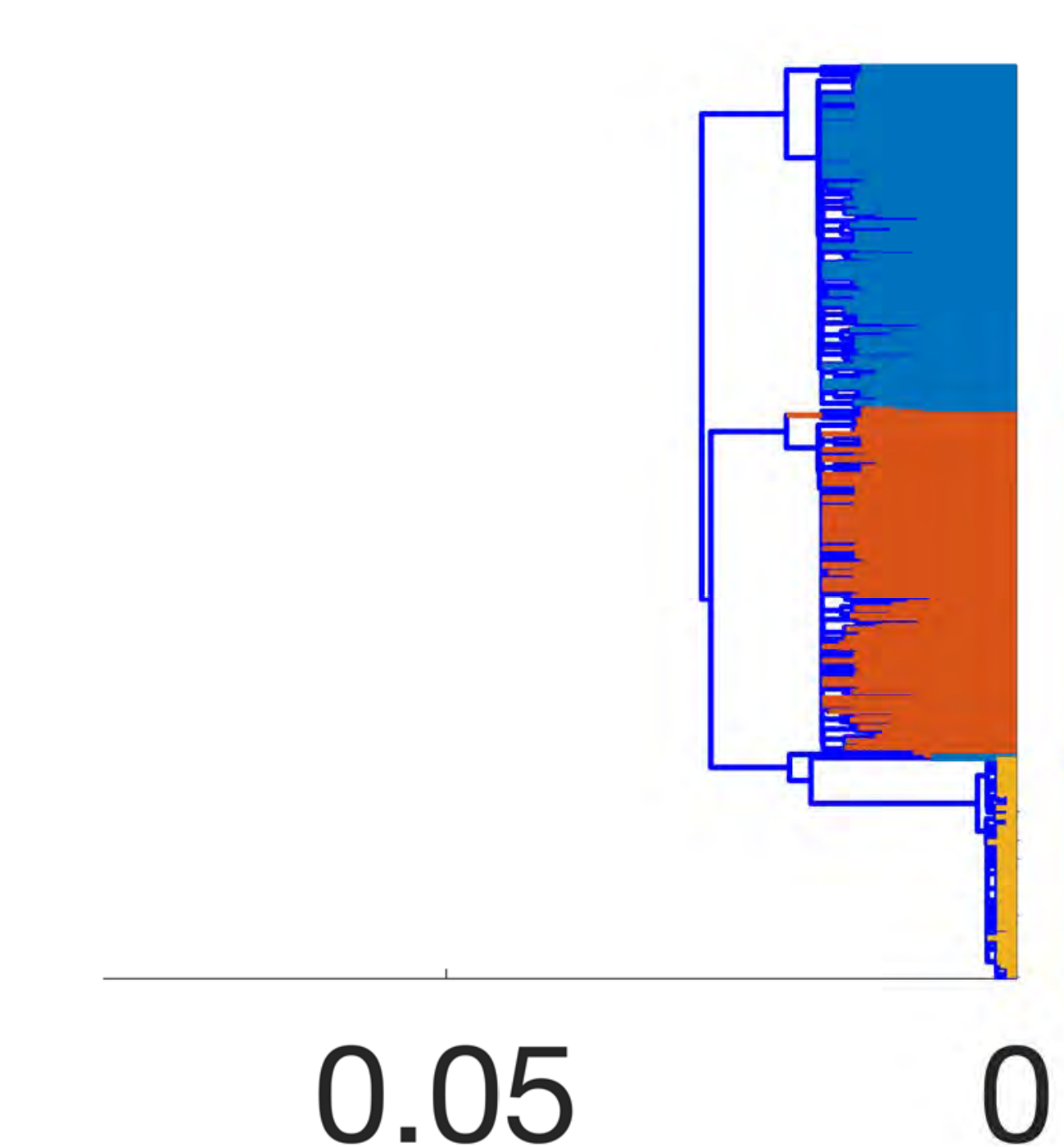}
            \end{minipage}
            \begin{minipage}[t][0.55cm][t]{1\textwidth}
    		    \centering
                \includegraphics[width=1\textwidth]{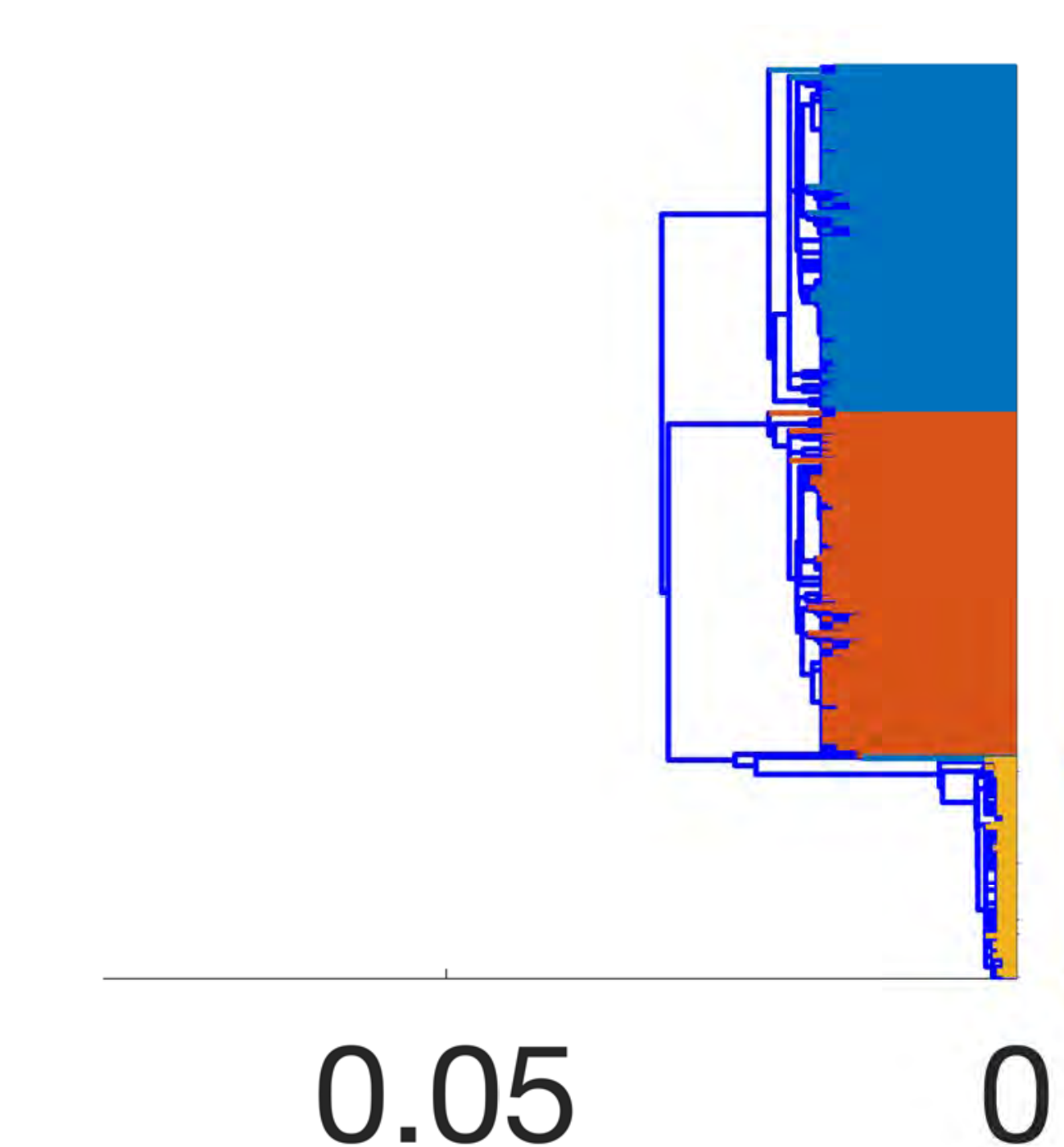}
            \end{minipage}
        \end{minipage}
        \begin{minipage}[t]{0.07\textwidth}
		    \centering
		    $e$:1/0.6 \\ 
            \includegraphics[width=1\textwidth]{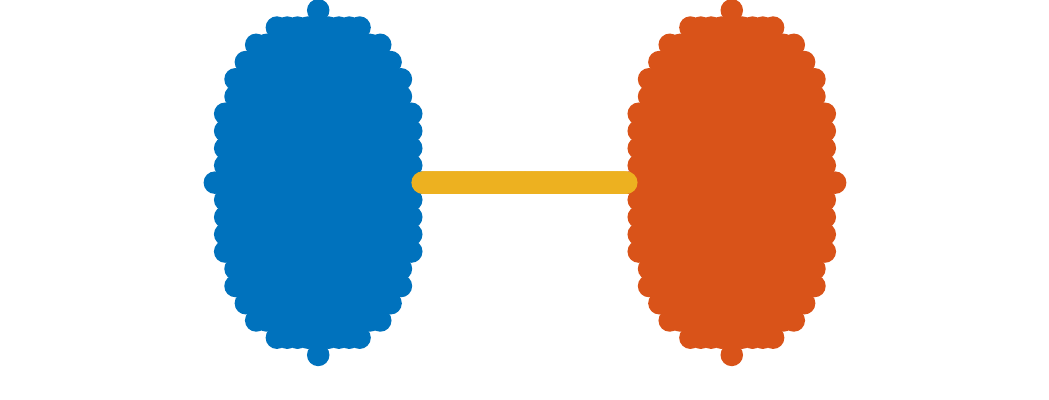} 

            \includegraphics[width=1\textwidth]{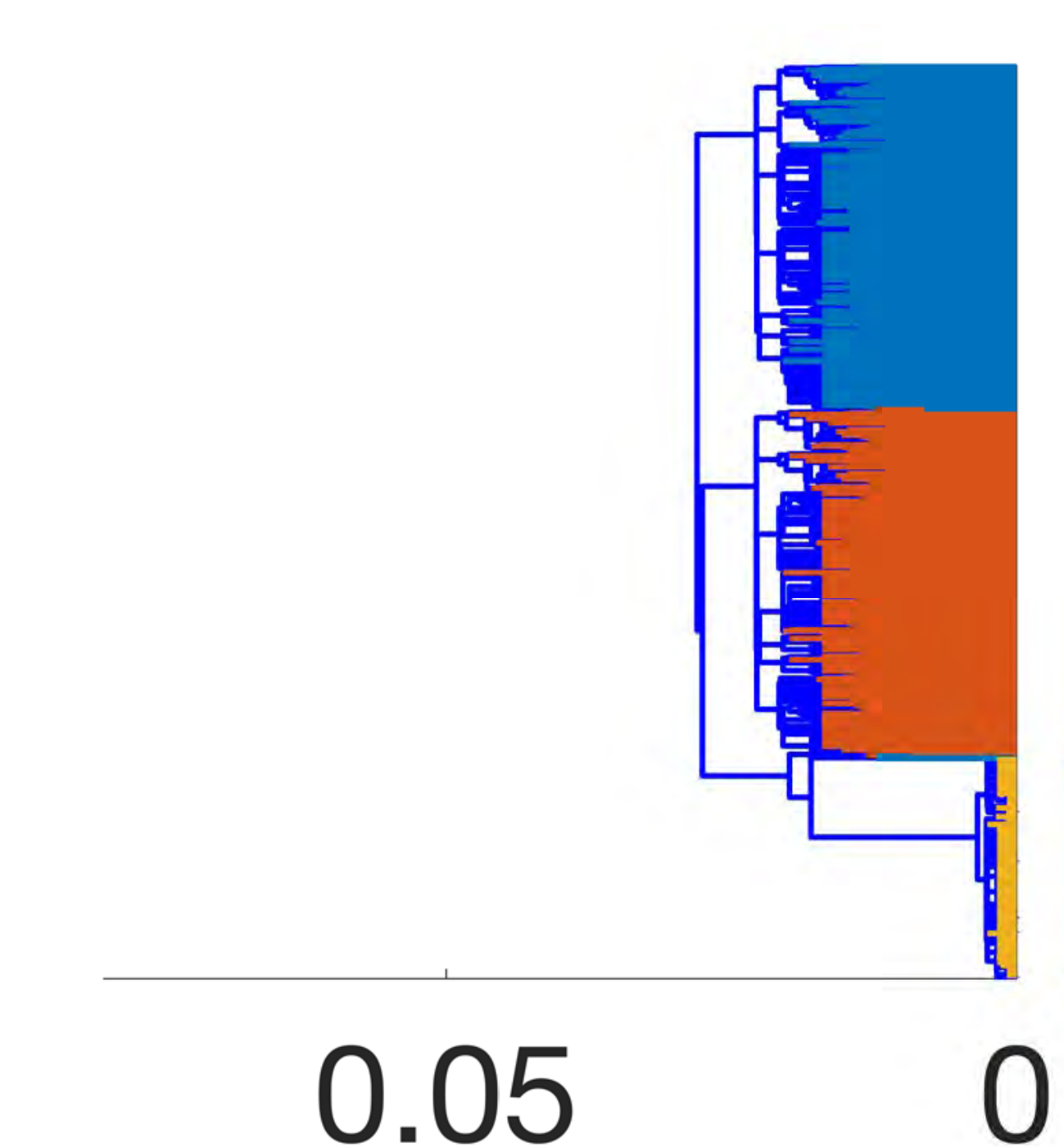}

            \includegraphics[width=1\textwidth]{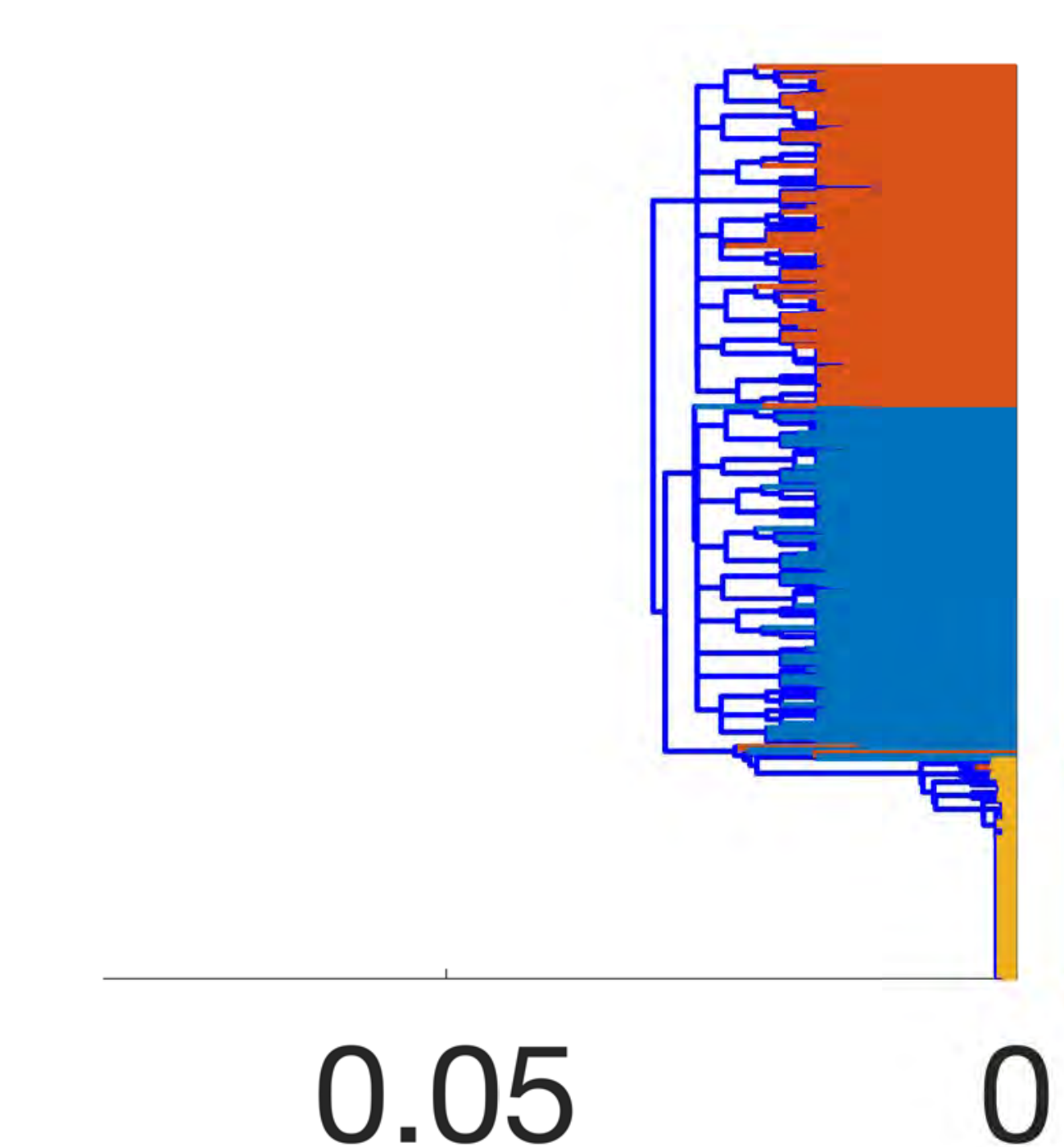}
        \end{minipage}
        \begin{minipage}[t]{0.07\textwidth}
		    \centering
		    $e$:1/0.2 \\ 
            \includegraphics[width=1\textwidth]{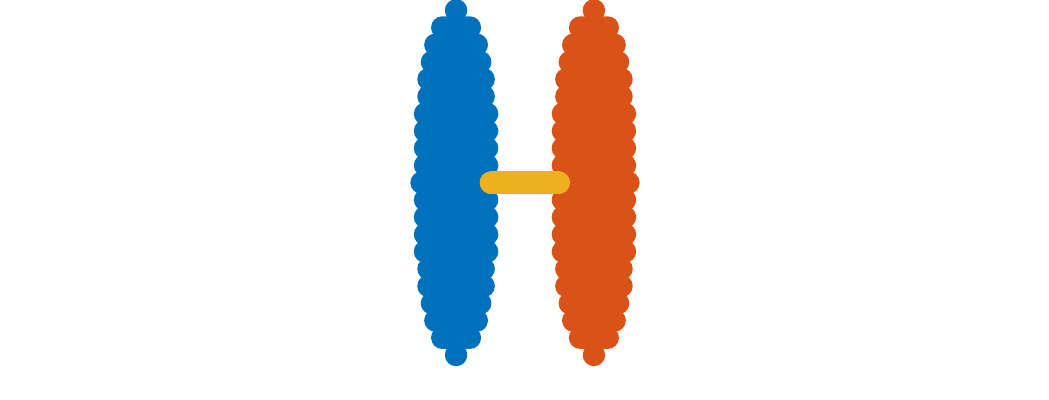} 

            \includegraphics[width=1\textwidth]{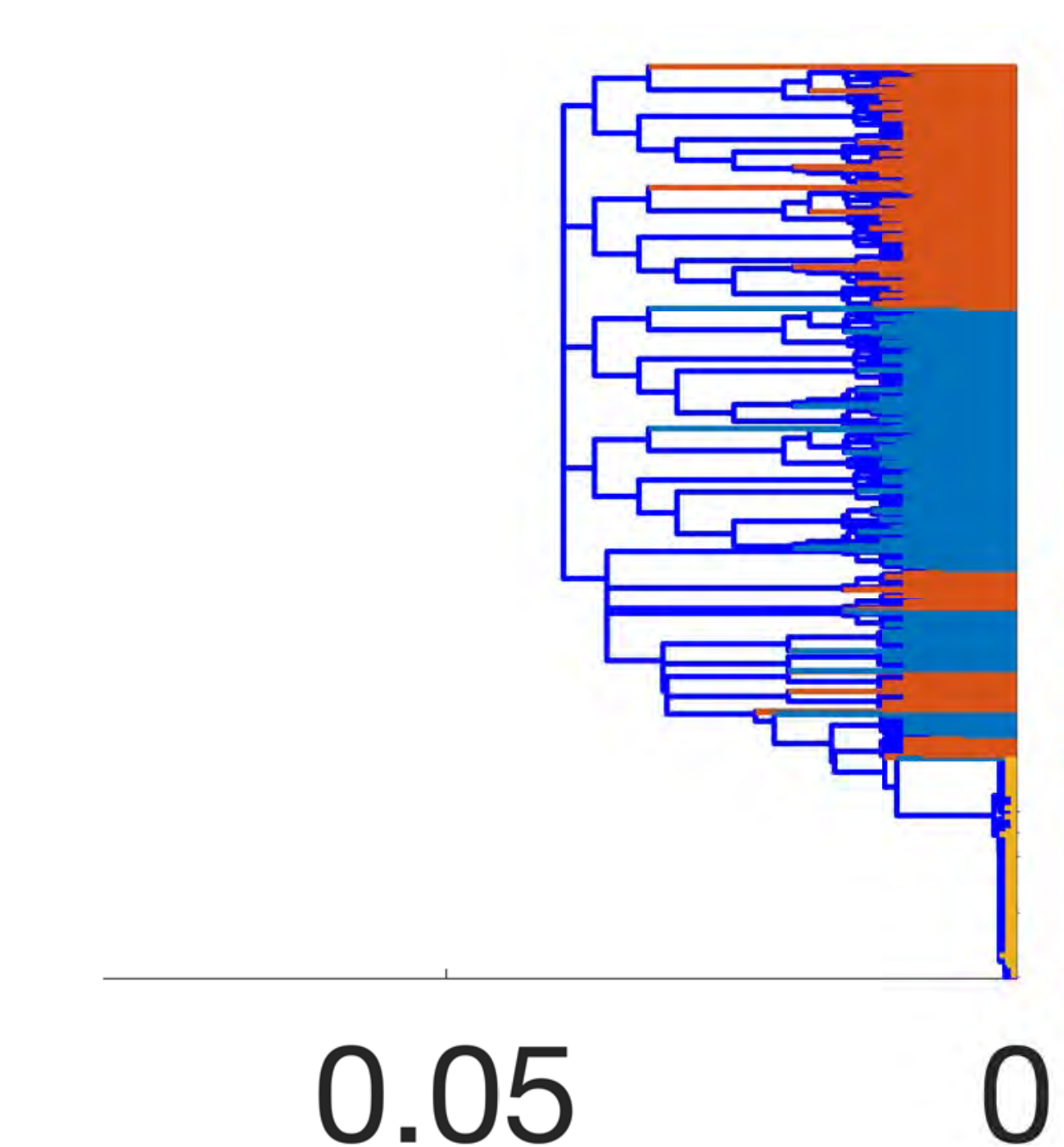}

            \includegraphics[width=1\textwidth]{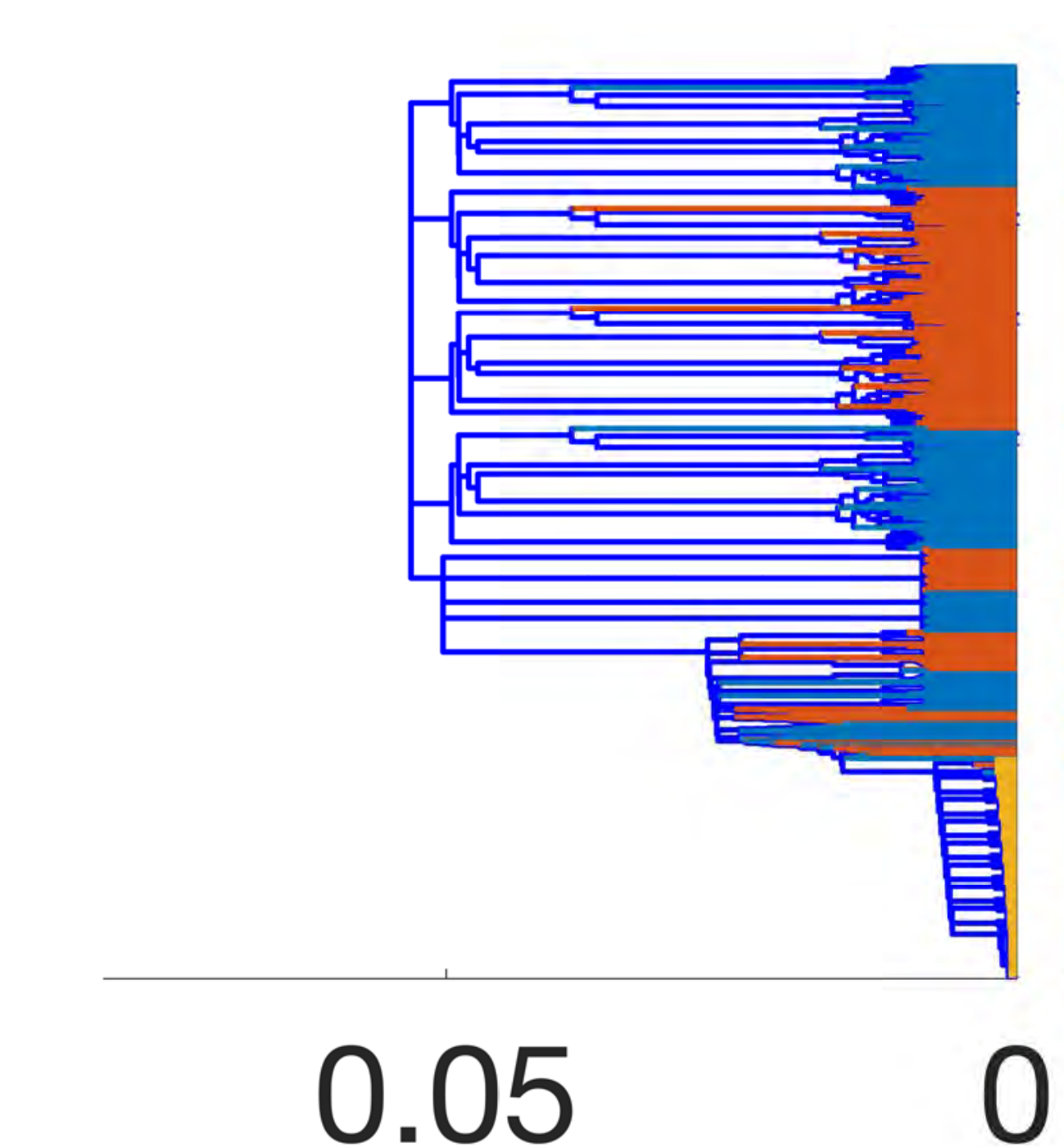}
        \end{minipage}
        }
	}
	\subfloat[Dataset, GT-$\lambda$-1, GT-$\lambda$-5]{
	\label{fig:ballchains-gt}
    \fbox{
        
		\begin{minipage}[t]{0.07\textwidth}
		    \centering
            $e$:1/1 \\  
            \includegraphics[width=1\textwidth]{figures/ellip/ellip-ori-seq-1-10.pdf} 

            \includegraphics[width=1\textwidth]{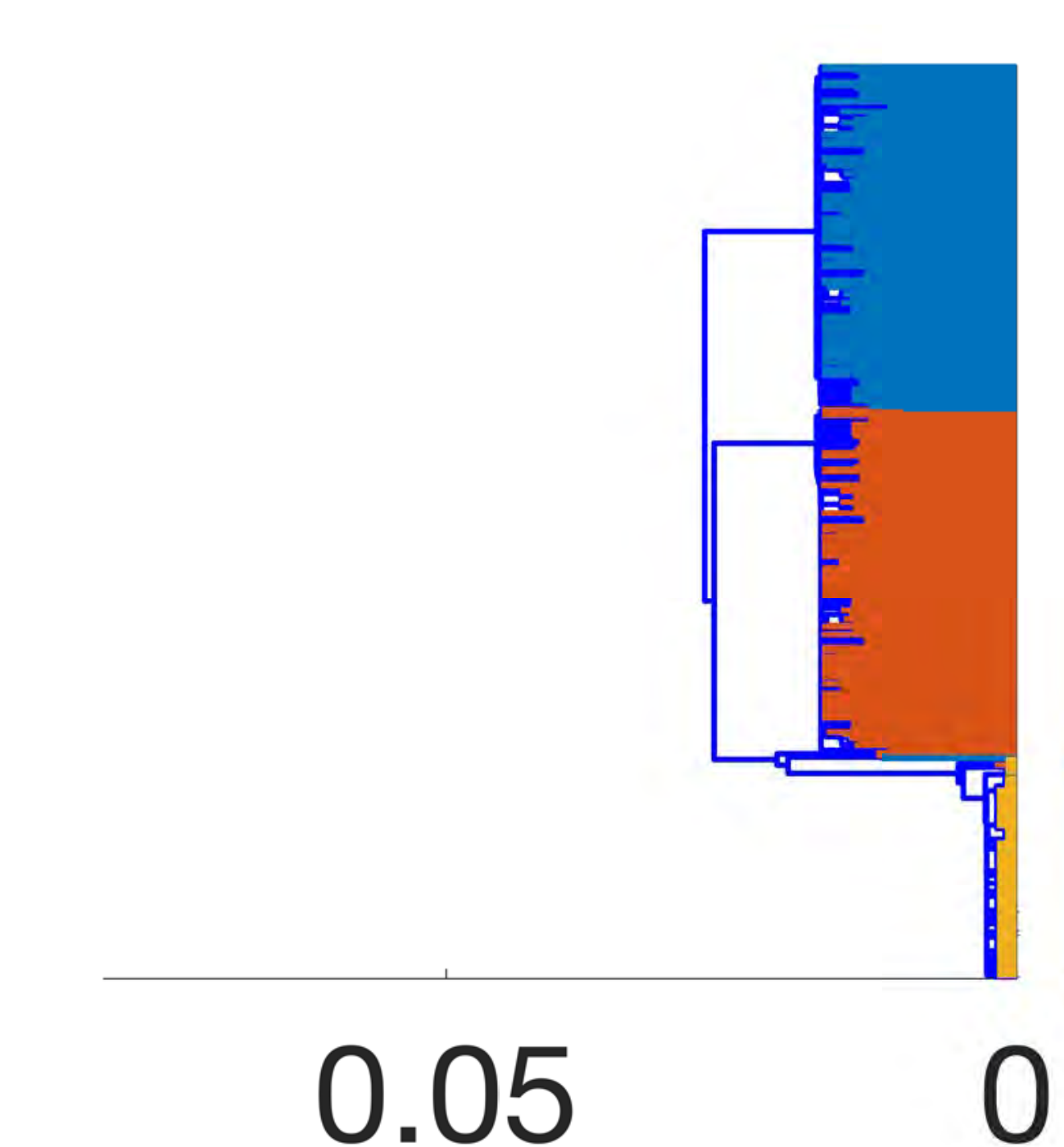}
            
            \includegraphics[width=1\textwidth]{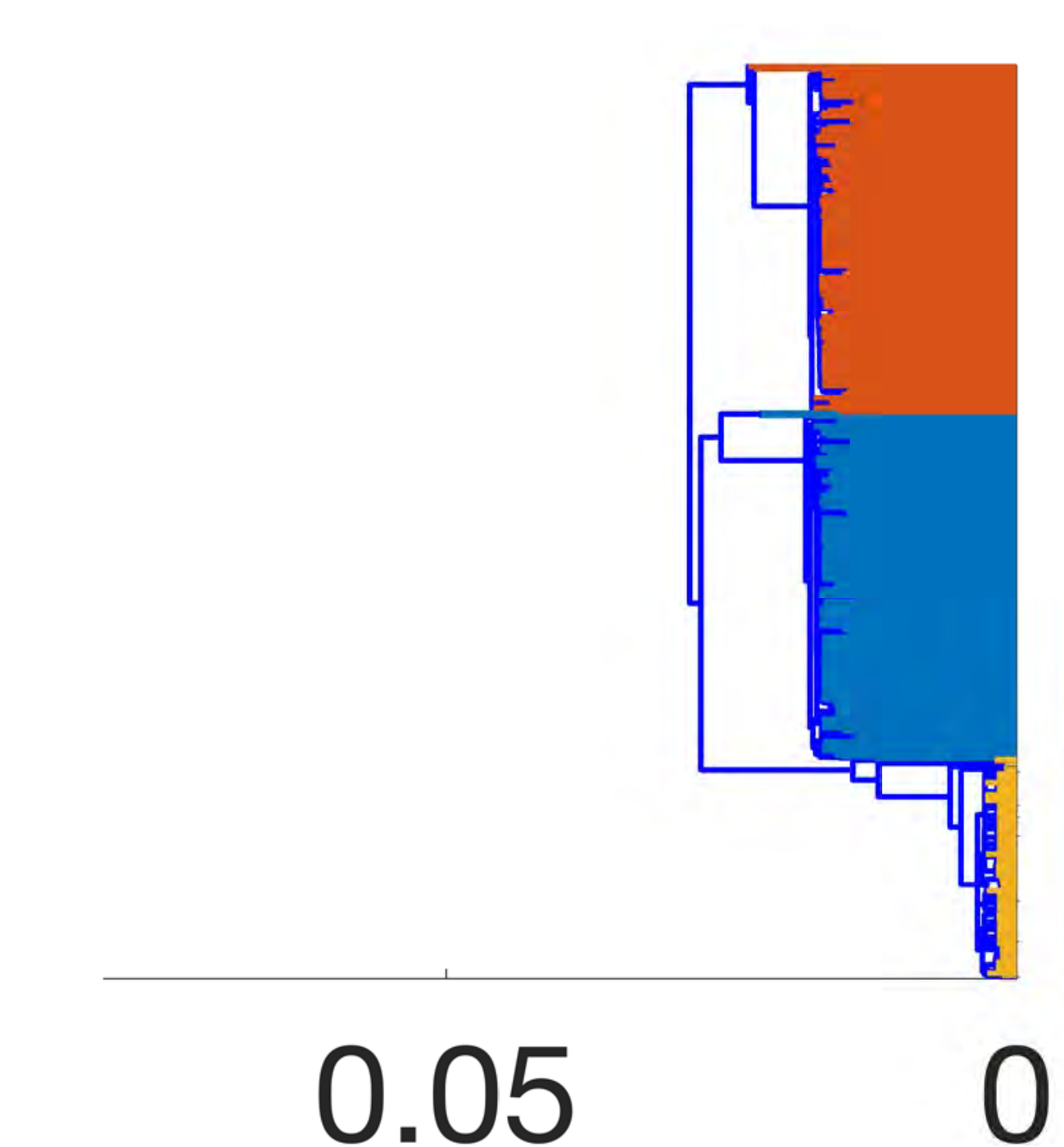}
        \end{minipage}
        \begin{minipage}[t]{0.07\textwidth}
		    \centering
            $e$:1/0.6 \\  
            \includegraphics[width=1\textwidth]{figures/ellip/ellip-ori-seq-1-6.pdf}

            \includegraphics[width=1\textwidth]{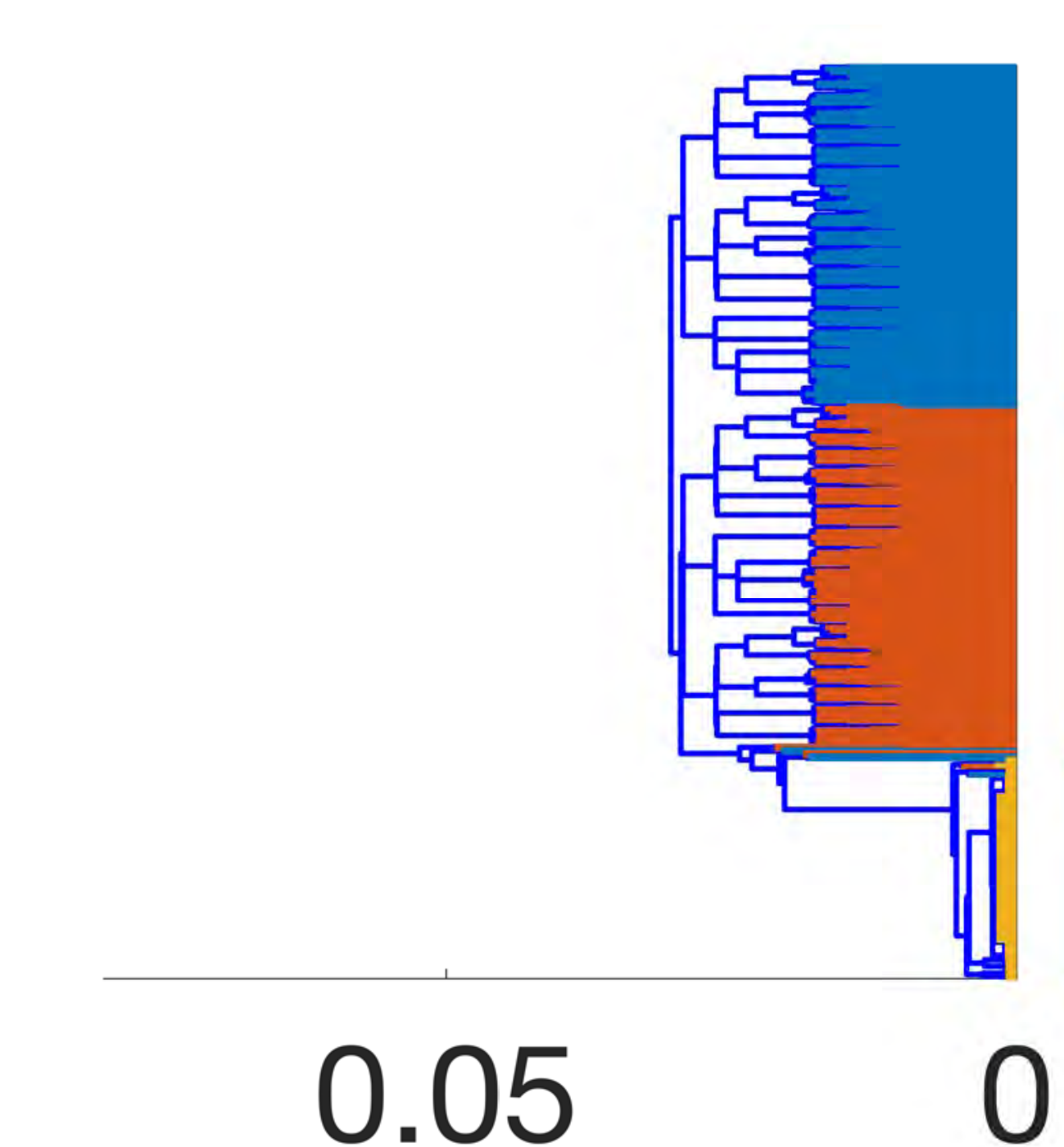}
            
            \includegraphics[width=1\textwidth]{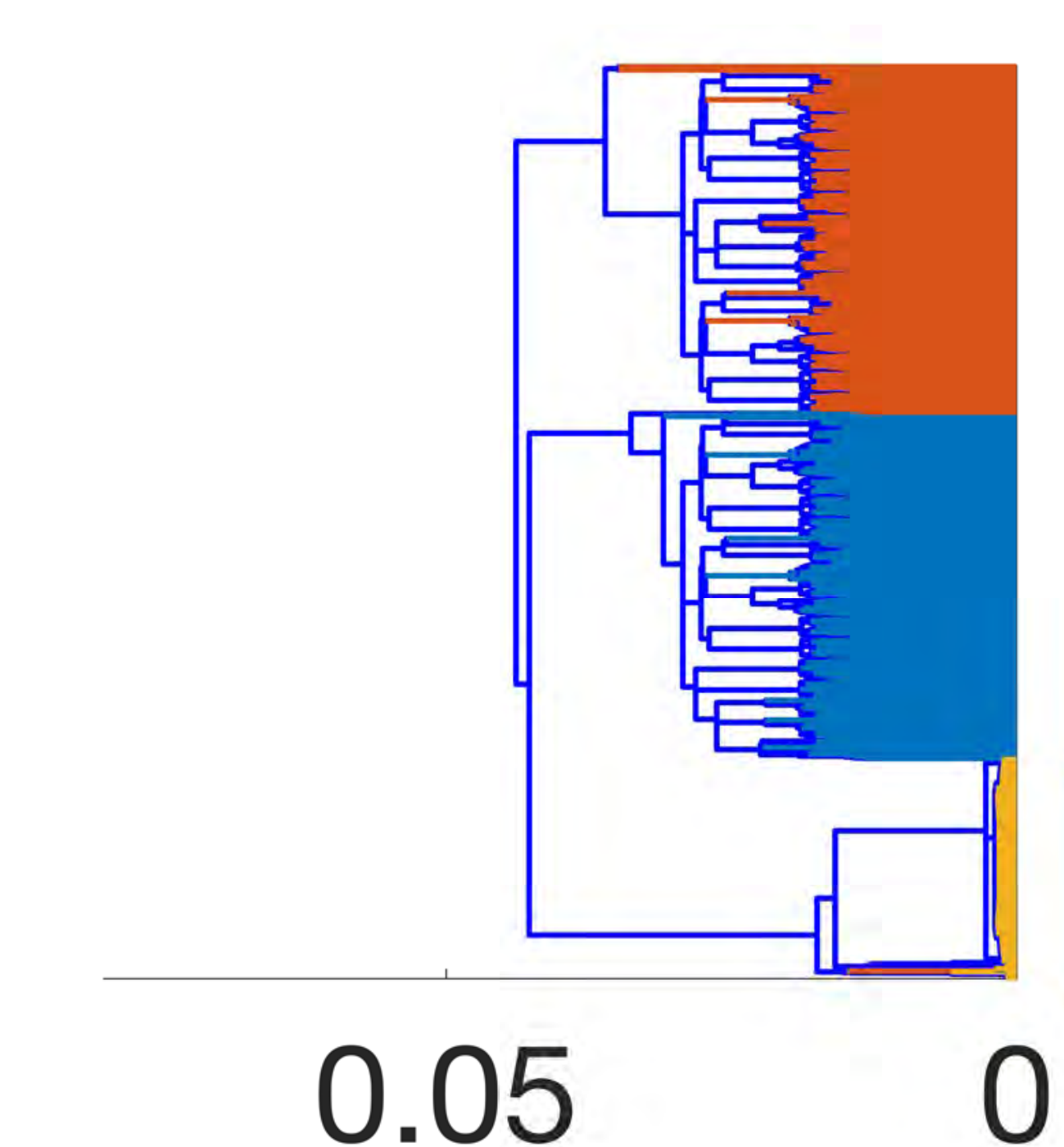}
        \end{minipage}
        \begin{minipage}[t]{0.07\textwidth}
		    \centering
            $e$:1/0.2 \\  
            \includegraphics[width=1\textwidth]{figures/ellip/ellip-ori-seq-1-2.pdf}

            \includegraphics[width=1\textwidth]{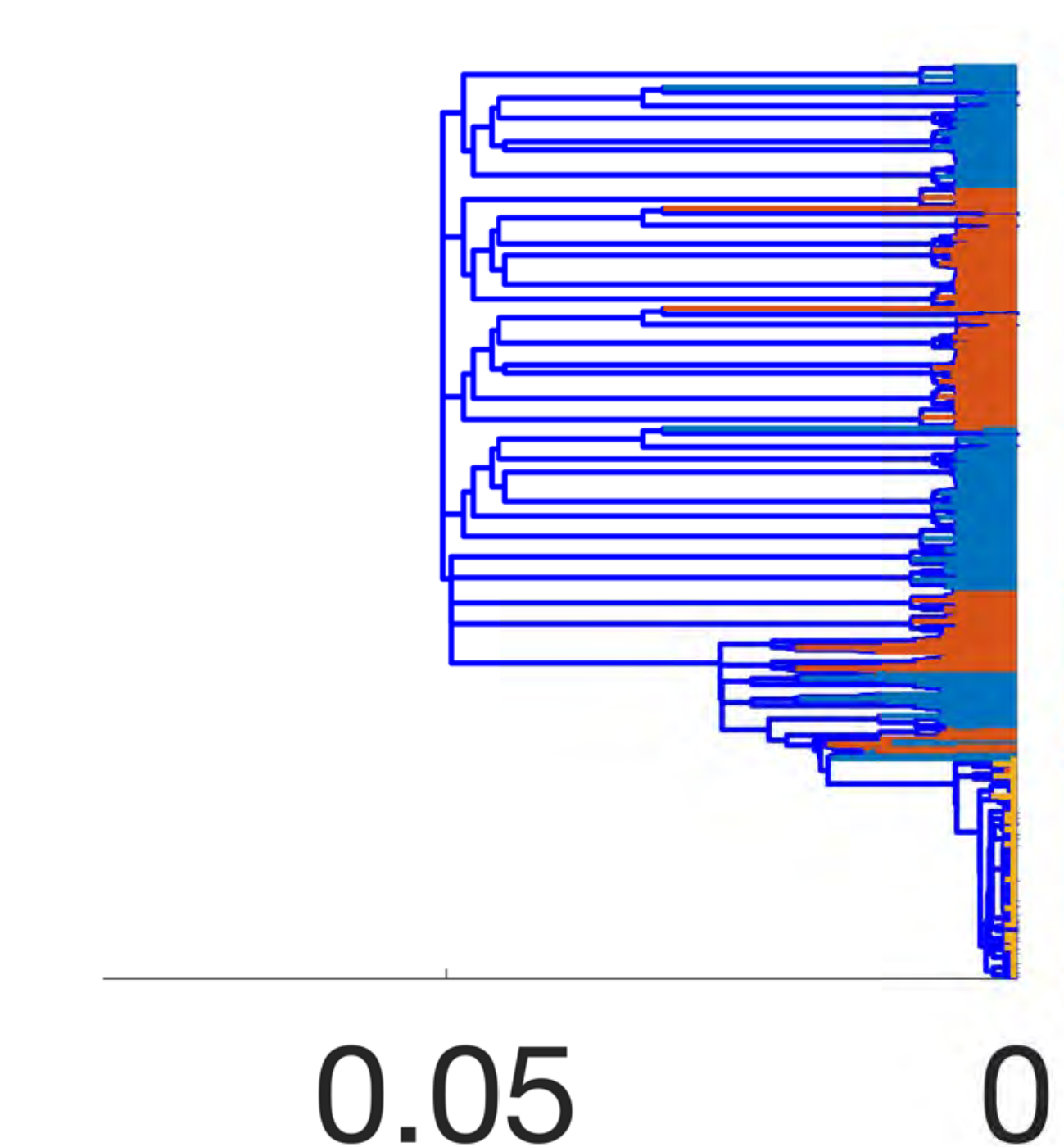}
            
            \includegraphics[width=1\textwidth]{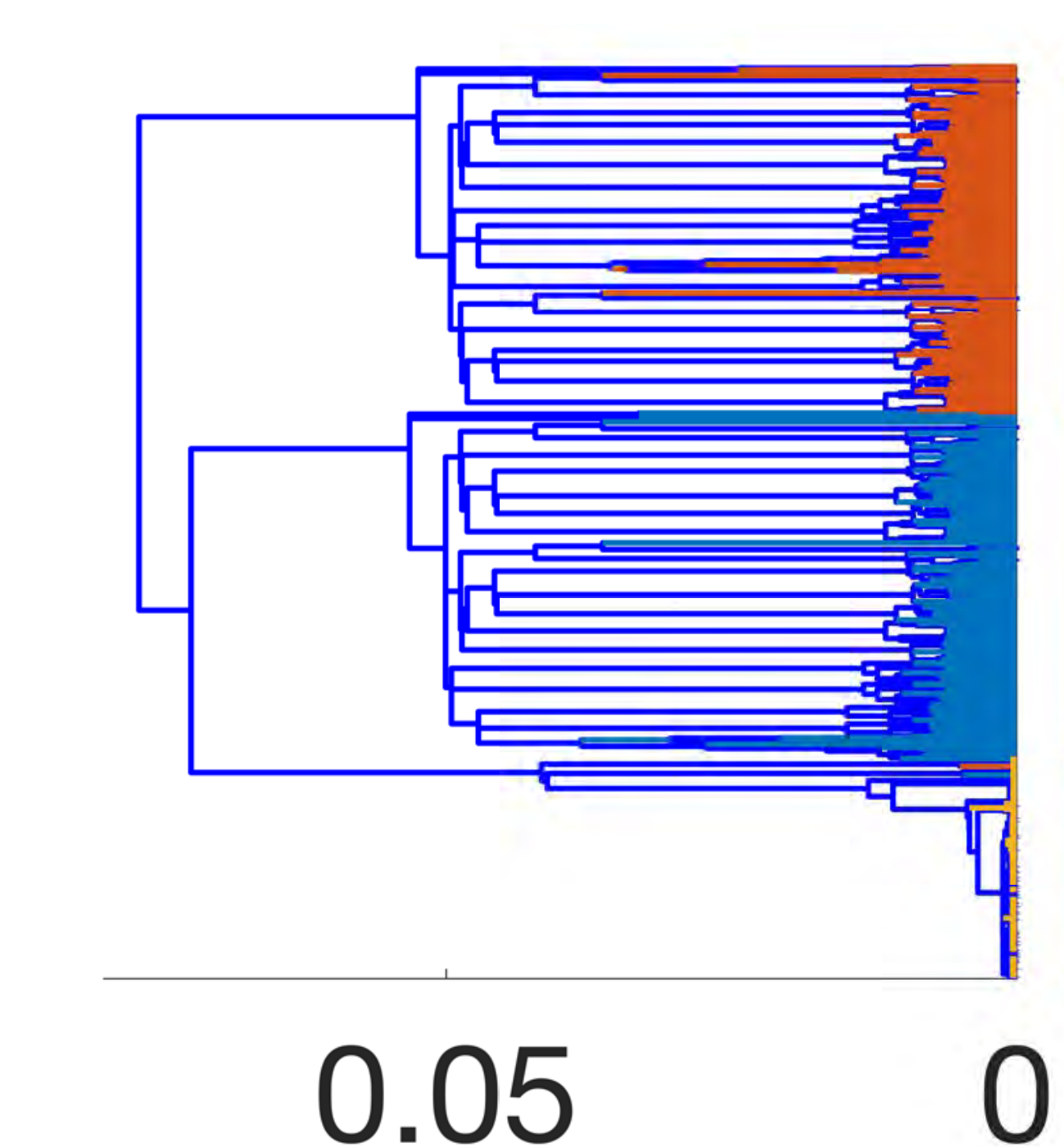}
        \end{minipage}
        }
	}
    \caption{Chaining effect. (a): Top and middle rows show 2D and 3D MDS plots after each iteration of applying GT, respectively. Bottom row present the corresponding dendrograms.
    (b): Top row represents datasets after applying the linear transformation $T$ with eccentricity $e$ to the original dataset ($e=1$). Middle and bottom row show the dendrograms of MS and WT2, respectively. 
    (c): Top row is the same as (b). Middle and bottom row show the respective dendrograms of GT-$\lambda$-1 and GT-$\lambda$-5. }
    \label{fig:ballchains} 
\end{figure}

\textbf{Noise removal.} 
We now analyze two datasets: the first (Figure~(\ref{fig:spiral})) is a spiral composed of 600 points lying in the square $[-30, 30]^2$ together with 150 outliers (uniformly generated); the other (Figure~(\ref{fig:concen-circ})) is composed of two concentric circles with random perturbations on points by small values. Each circle has 250 points lying in the square $[-2, 2]^2$. 
We compare the performance of MS, WT2 and GT after 2 iterations. 
Results are shown in Figure~(\ref{fig:denoise-spiral}) and Figure~(\ref{fig:denoise-concen-circ}). 
We see from Figure~(\ref{fig:denoise-spiral}) that GT generates cleaner spiral than those of
MS and WT2, and from Figure~(\ref{fig:denoise-concen-circ}) that GT and WT2 {both better absorb noisy points than MS.} We again emphasize the superior performance of GT compared with WT2 while being more computationally efficient. See the supplementary material for more results and a more demanding denoising experiment on a noisy circle.

\begin{figure}[htb]
    \centering
    \subfloat[Spiral]{
    \label{fig:spiral}
    \fbox{
        \begin{minipage}[t]{0.08\textwidth}
        \centering
        $\tau=0$ \\ 
            \includegraphics[width=1\textwidth]{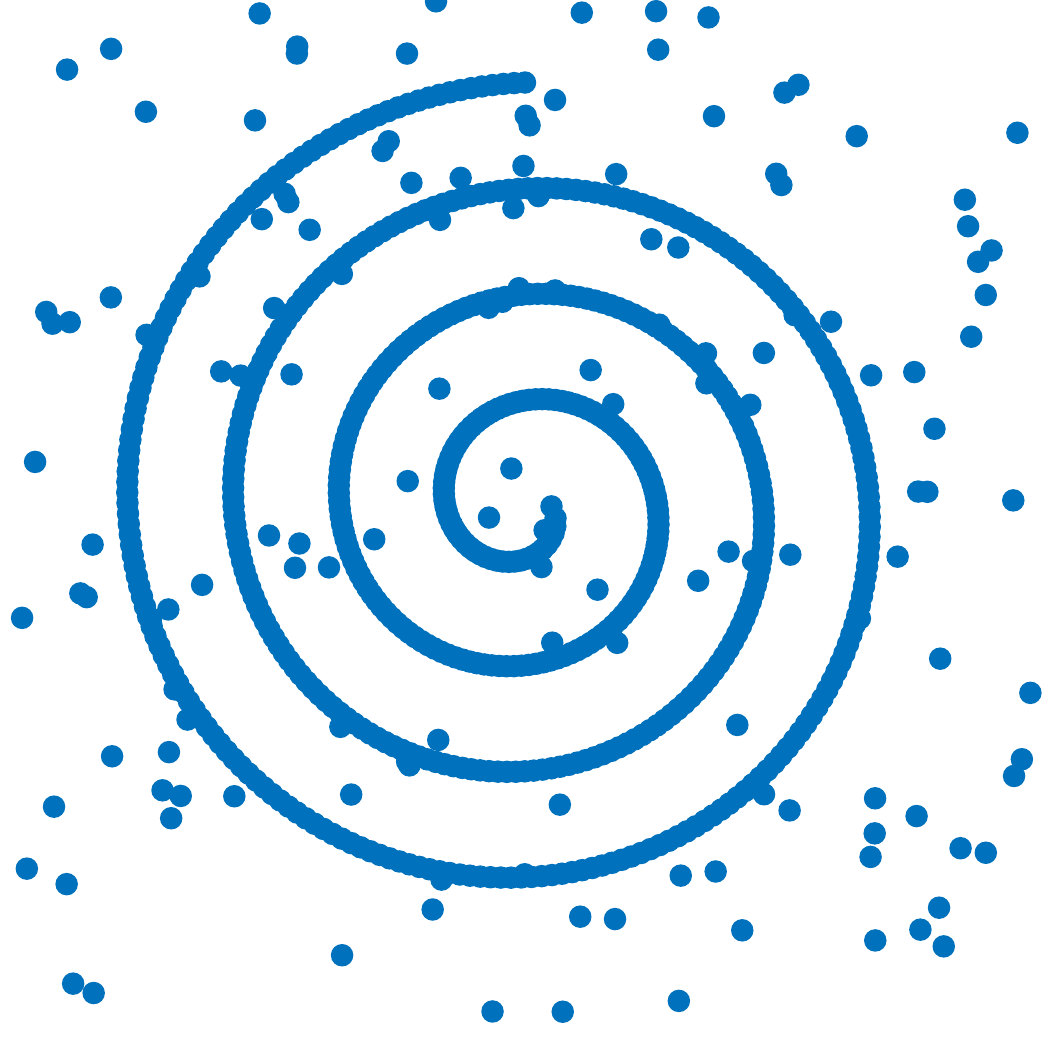}
        \end{minipage}
        }
    }
    \subfloat[Denoising of a spiral at $
    \tau=2$]{
    \label{fig:denoise-spiral}
    \fbox{
        \begin{minipage}[t]{0.08\textwidth}
        \centering
        MS \\ 
            \includegraphics[width=1\textwidth]{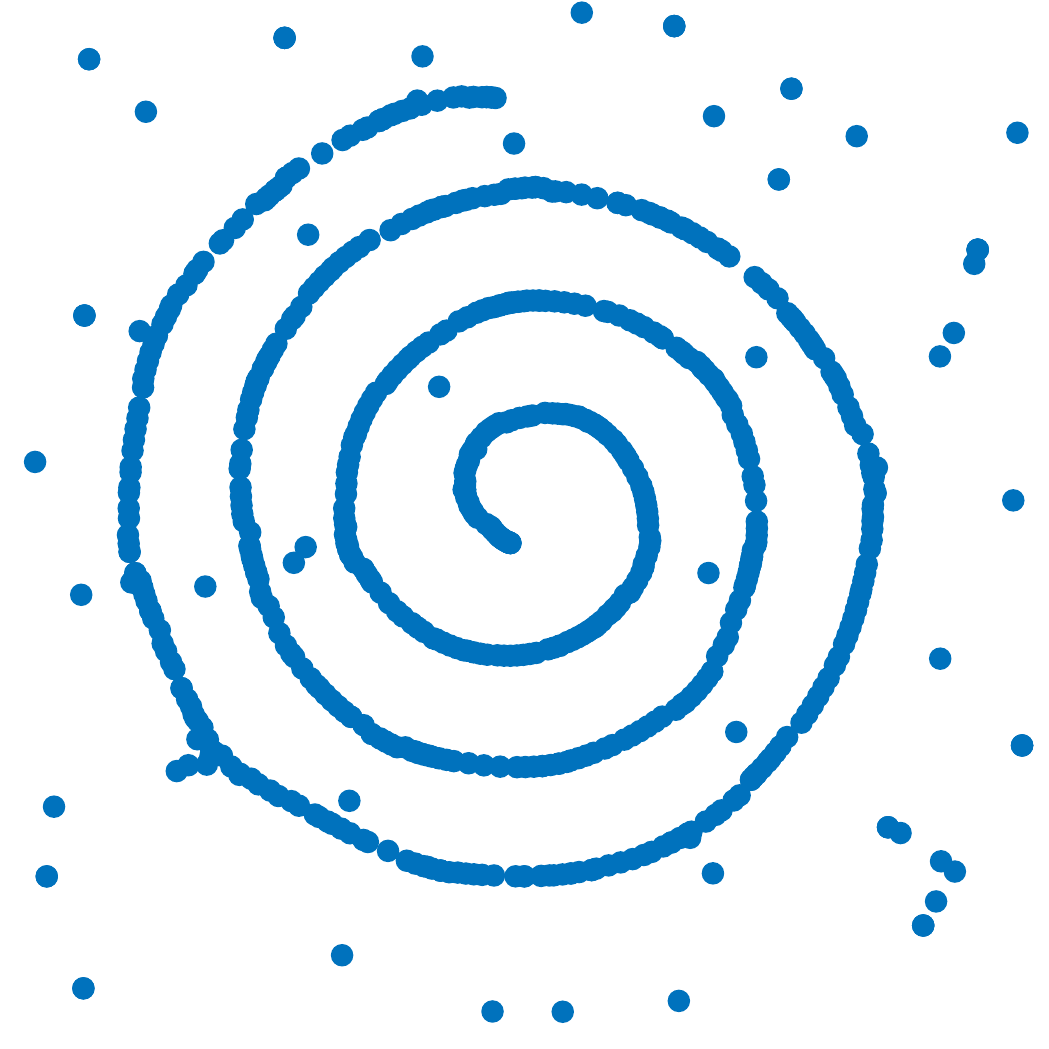}
        \end{minipage}
        
        \begin{minipage}[t]{0.08\textwidth}
        \centering
        GT \\
            \includegraphics[width=1\textwidth]{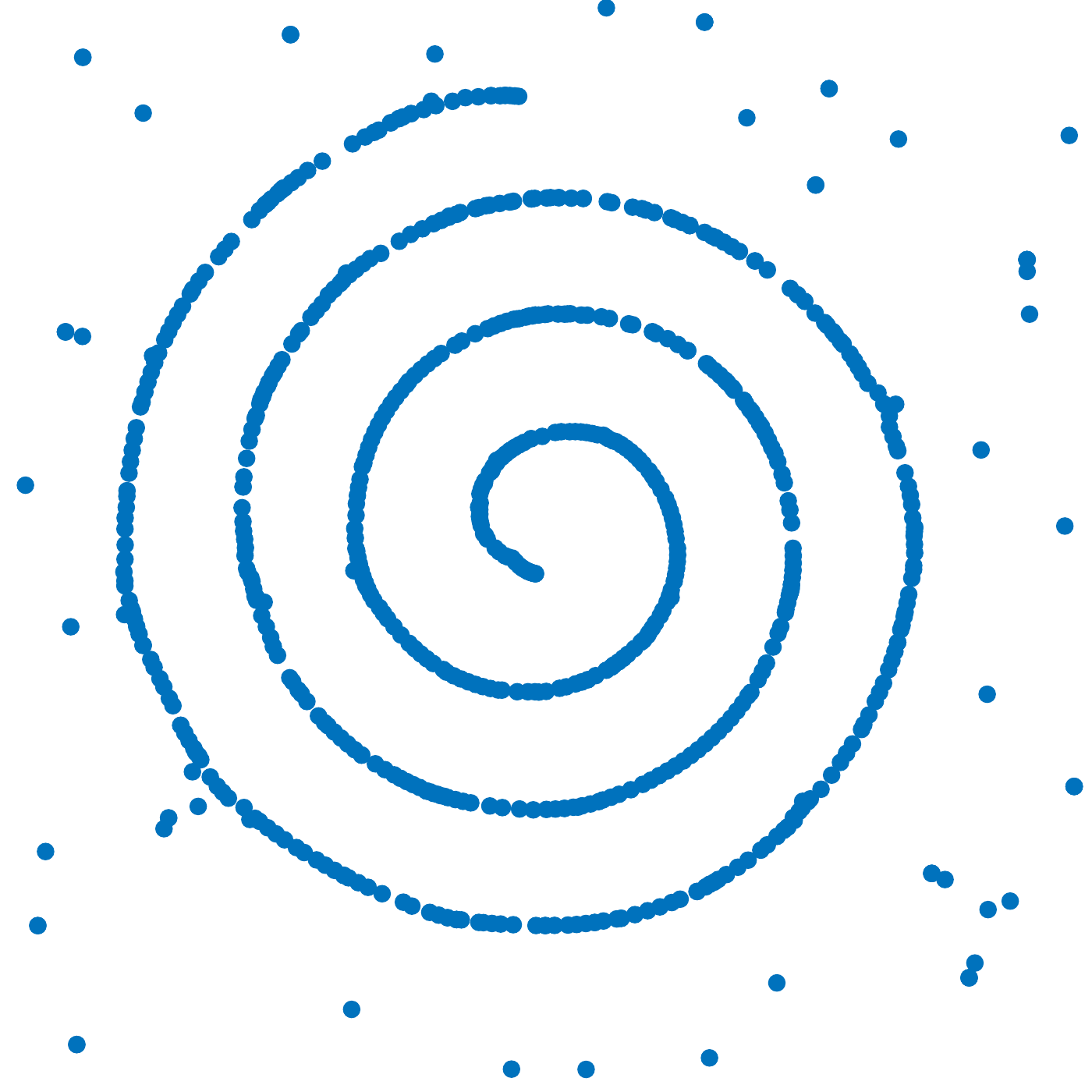}
        \end{minipage}
        \begin{minipage}[t]{0.08\textwidth}
        \centering
        WT2 \\ 
            \includegraphics[width=1\textwidth]{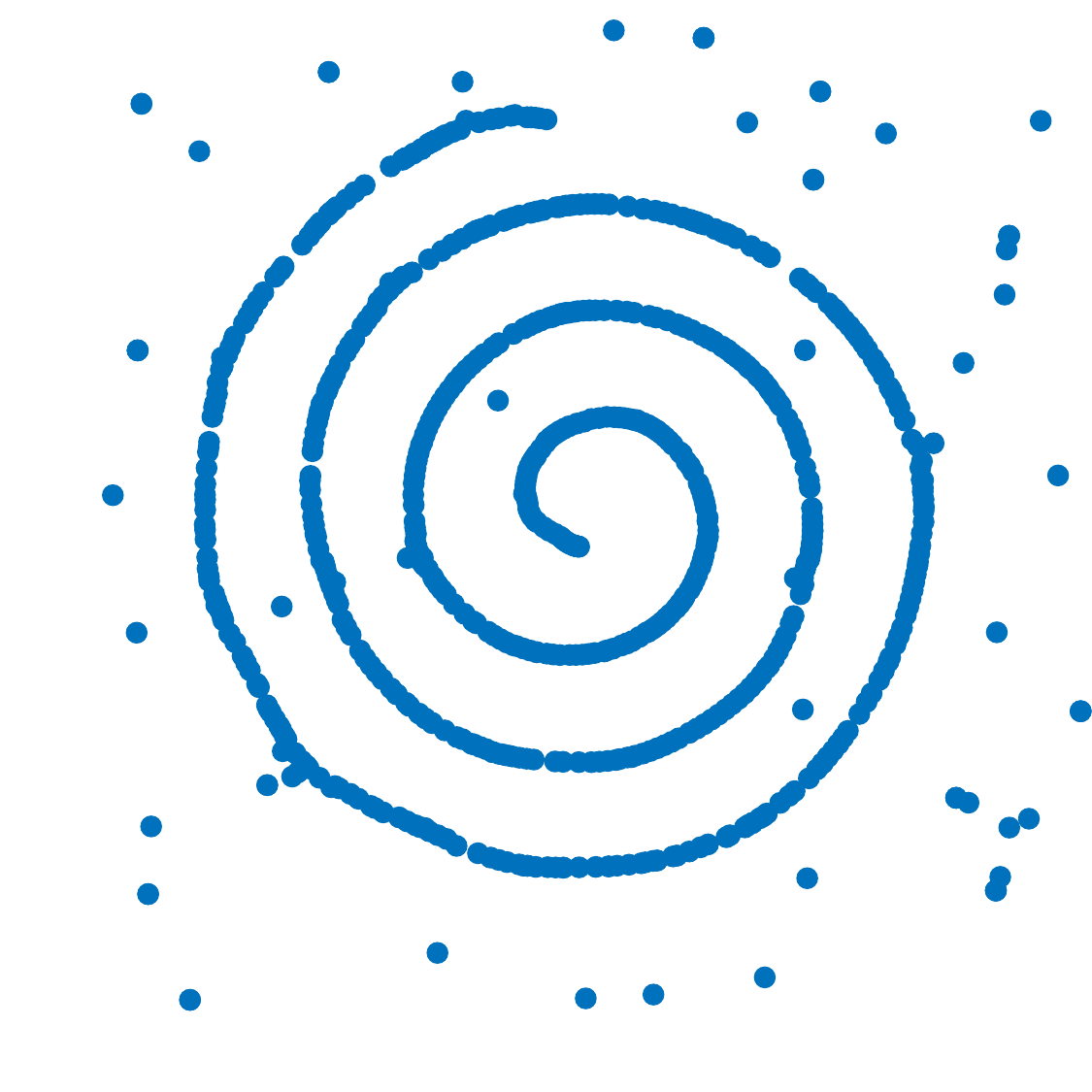}
        \end{minipage}
        }
        }
    \subfloat[Circles]{
    \label{fig:concen-circ}
    \fbox{
        \begin{minipage}[t]{0.08\textwidth}
        \centering
        $\tau=0$ \\ 
            \includegraphics[width=1\textwidth]{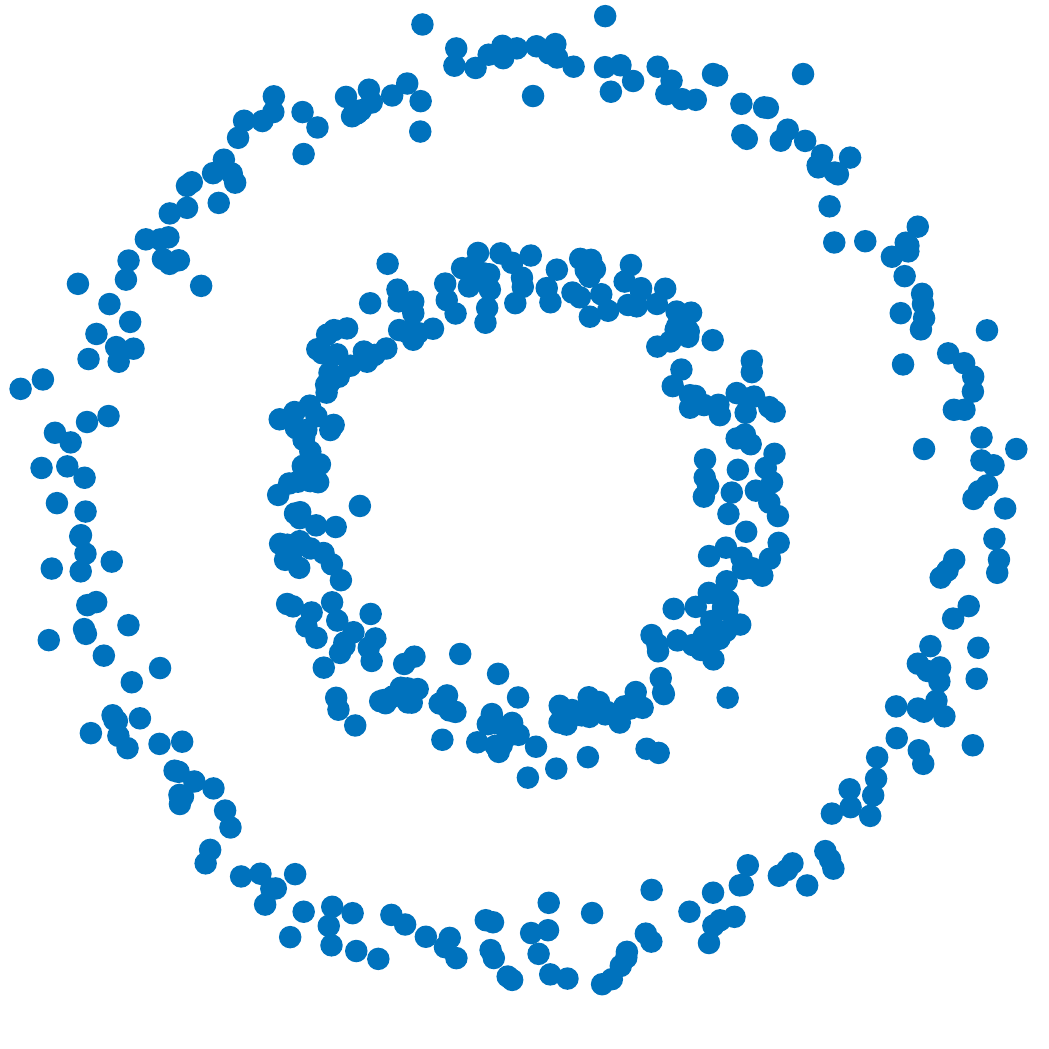}
        \end{minipage}
        }
    }
    \subfloat[Denoising of circles at $
    \tau=2$]{
    \label{fig:denoise-concen-circ}
    \fbox{
        \begin{minipage}[t]{0.08\textwidth}
        \centering
        MS \\ 
            \includegraphics[width=1\textwidth]{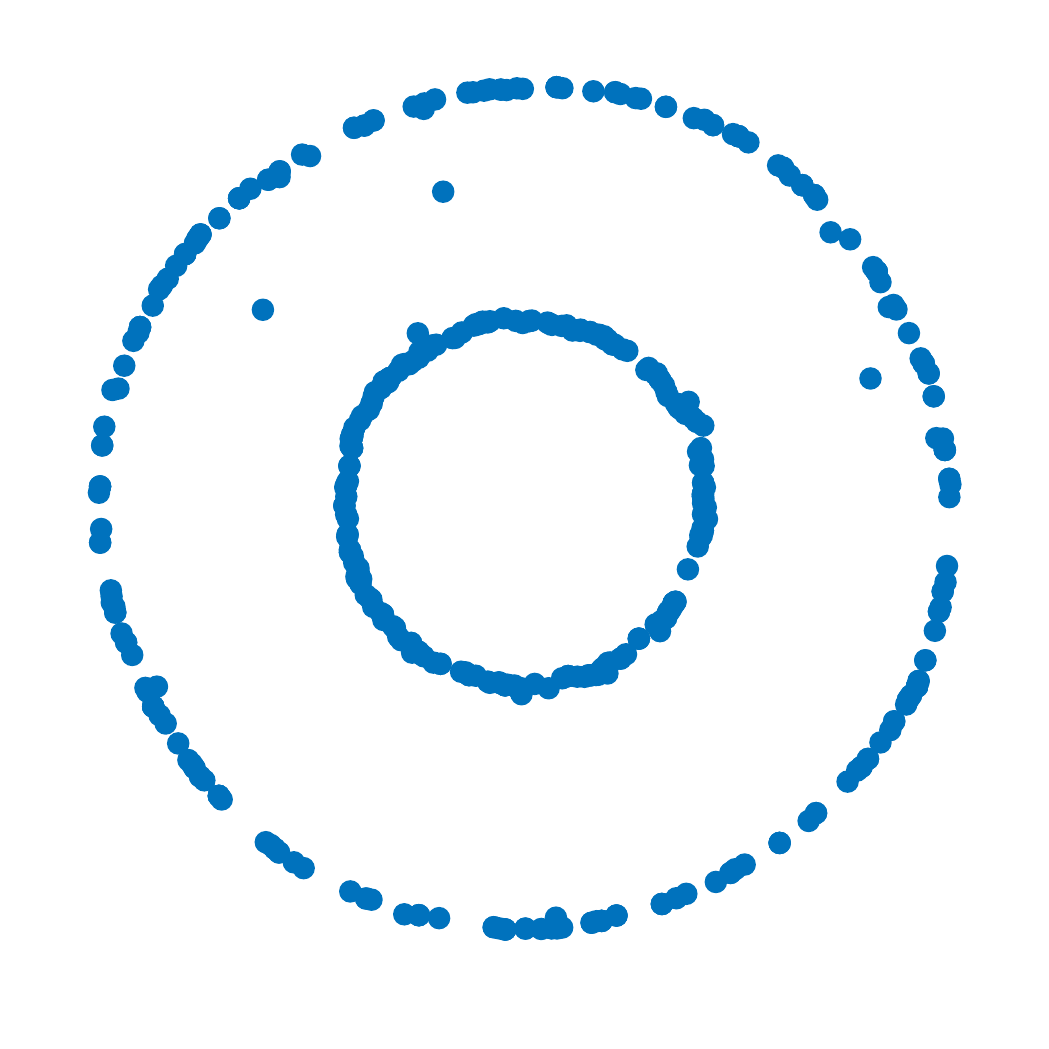}
        \end{minipage}
        \begin{minipage}[t]{0.08\textwidth}
        \centering
        GT \\
            \includegraphics[width=1\textwidth]{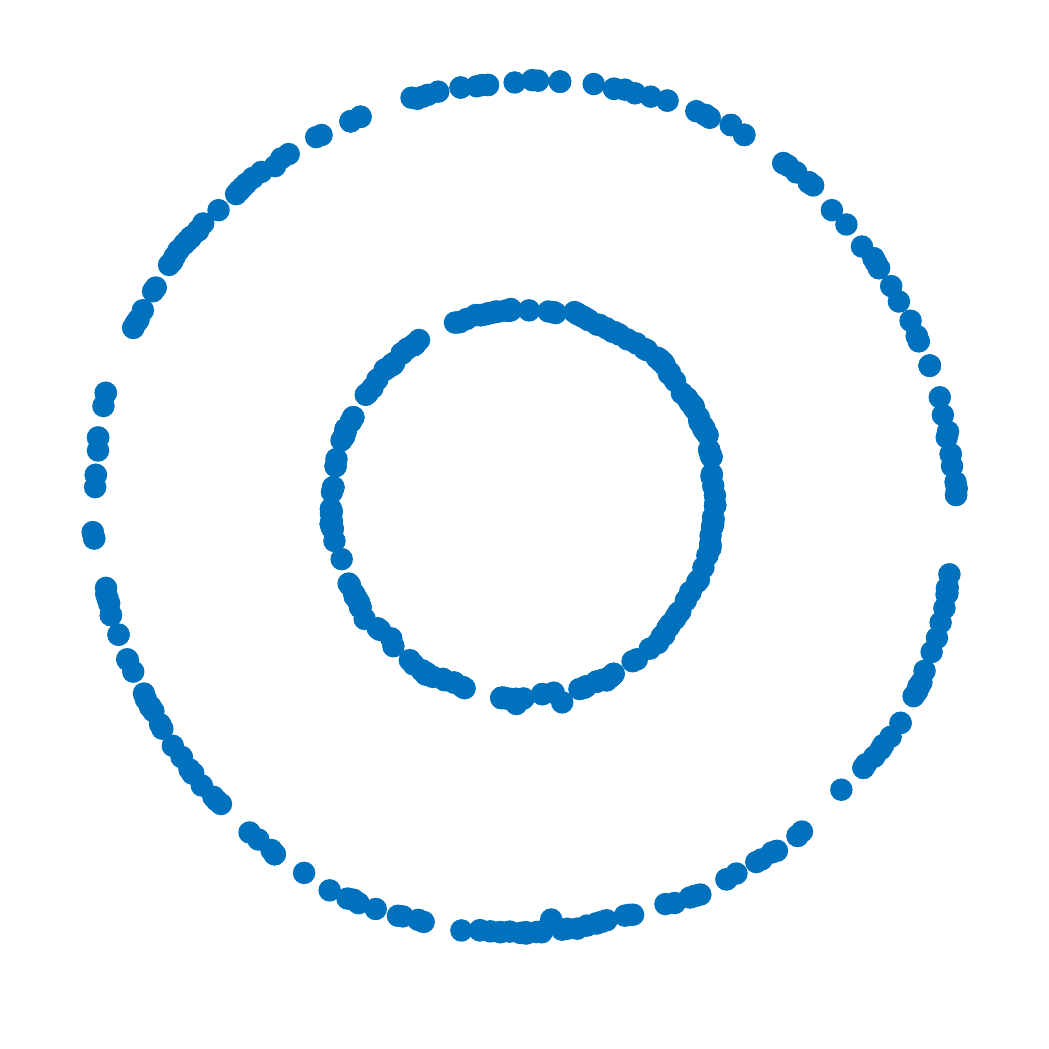}
        \end{minipage}
        \begin{minipage}[t]{0.08\textwidth}
        \centering
        WT2 \\ 
            \includegraphics[width=1\textwidth]{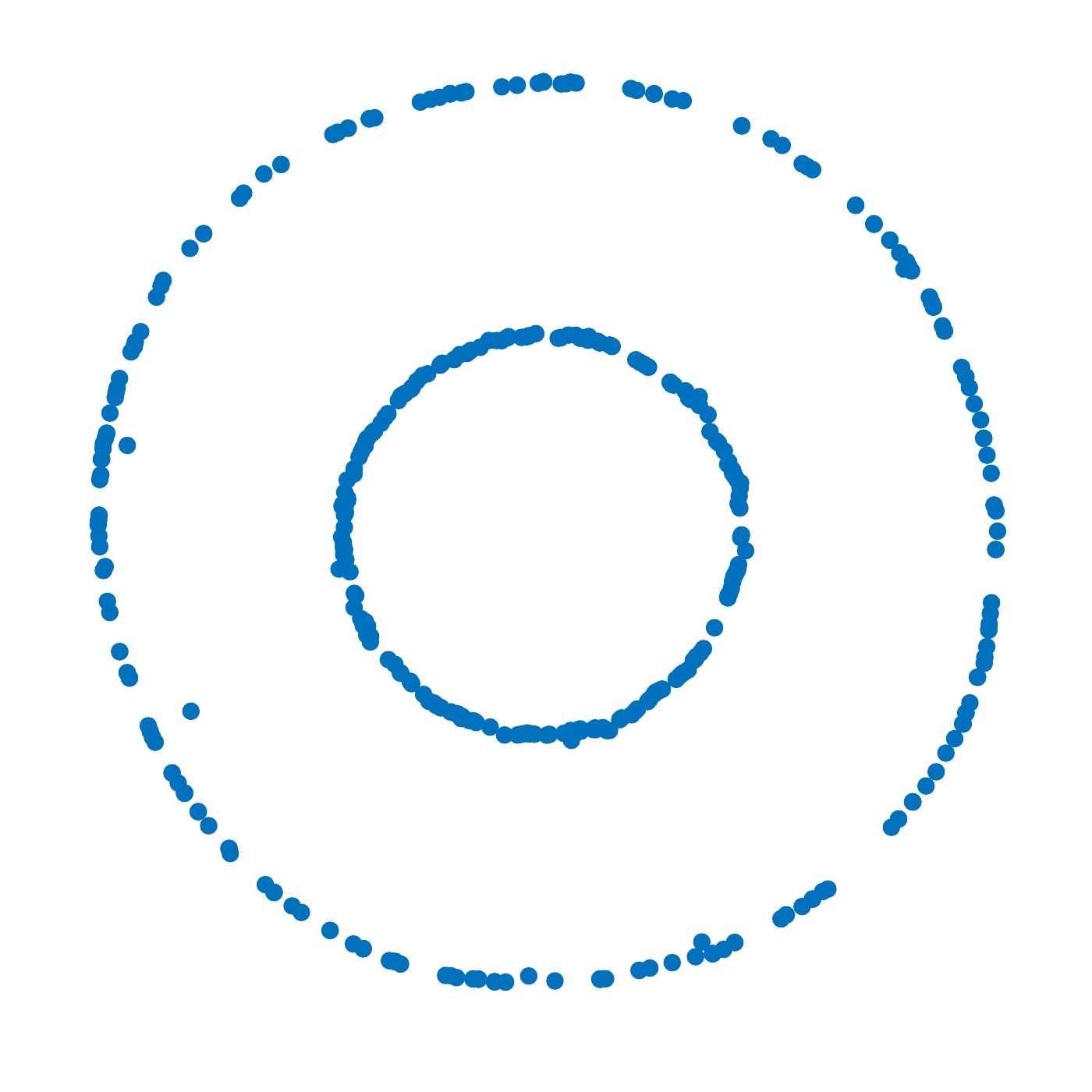}
        \end{minipage}
        }
        }
    \caption{Denoising of a spiral and concentric circles. 
    (a): Original spiral with outliers. 
    (b): Spiral denoising results after applying MS, GT-$\lambda$-1 and WT2 with $\eps=4$. 
    (c): Original perturbed concentric circles. 
    (d): Circles denoising results after applying MS, GT-$\lambda$-2.5 and WT2 with $\eps=0.7$.}
    \label{fig:spiral-concen} 
\end{figure}

\textbf{Image segmentation.}\label{sec:img-seg}
Image Segmentation is an important application domain in Computer 
Vision~\cite{szeliski2010computer} and 
MS~\cite{comaniciu2002mean,demirovic2019implementation} is an effective method in this area. We first review the MS as applied in Image Segmentation and then show how to comparatively apply GT. Note that WT2 is not applicable in this experiment since the process of Image Segmentation involves updating features whereas  WT2 only updates/retains distance matrices. Given an image, each pixel $x$ consists of two types of features: \textit{spatial features} and \textit{range features}, denoted by $x^s$ and $x^r$, respectively. $x^s$ is represented by a point in $\R^2$, whereas $x^r$ is represented by a point in $\R^3$ using L*u*v* values \cite{comaniciu2002mean}, in which the Euclidean distance approximates perceptual differences in color better than the RGB space. 
To apply MS with given bandwidth parameters $\eps_s$ and $\eps_r$, we define the $(\eps_s,\eps_r)$-neighborhood of a pixel $x$ to be the set of pixels $y=(y^s,y^r)$ such that $\| y^s - x^s \| \leq \eps_s$ and $\| y^r - x^r \| \leq \eps_r$. Then,  associate to each pixel $x$ one cluster point $T(x)\in\R^5$ which is initialized to coincide with $x$. MS will iteratively update $T(x)$ to the mean of its $(\eps_s,\eps_r)$-neighborhood until convergence. To apply GT in similar scenarios, we use spatial features to define the covariance because we only want to stretch the spatial distance instead of the range distance. We compute the GT distance between spatial features according to a variant of Equation (\ref{eq:gtd}) (see the supplementary material for a precise formula and an explanation). Then, we update the associated cluster point $T(x)$ similarly as in the case of MS.
We compare the performance of GT with MS on the grayscale cameraman image (resolution is 128$\times$128) in Figure~(\ref{fig:cameraman}a). We set $\eps_s=6$, $\eps_r=6$ and for GT, $\lambda=5$. The labels marked on the test image 
correspond to the major different segments that MS and GT recognize. We can see from Figure~(\ref{fig:cameraman}b) and (\ref{fig:cameraman}c) that GT generates a reasonably better segmentation than MS does. See the supplementary material for more results.

\begin{figure}
\centering
\begin{minipage}{.4\textwidth}
  \centering
  \begin{minipage}[t]{0.31\textwidth}
        \centering
        (a) Test \\ 
            \includegraphics[width=1\textwidth]{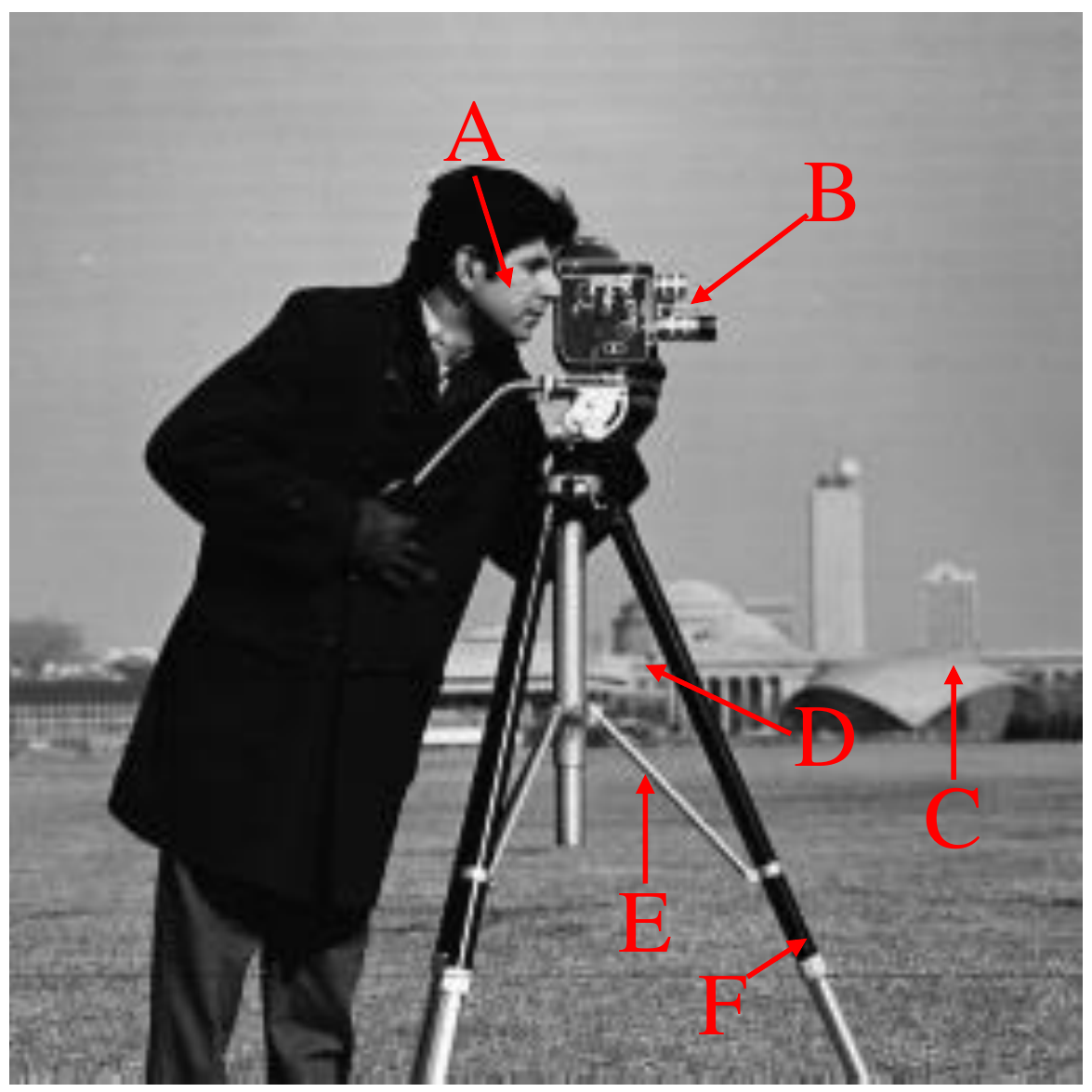}
        \end{minipage}
        \begin{minipage}[t]{0.31\textwidth}
        \centering
        (b) MS \\
            \includegraphics[width=1\textwidth]{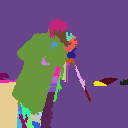}
        \end{minipage}
        \begin{minipage}[t]{0.31\textwidth}
        \centering
        (c) GT \\ 
            \includegraphics[width=1\textwidth]{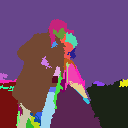}
        \end{minipage}
  \captionof{figure}{Image segmentation.} 
  \label{fig:cameraman}
\end{minipage}%
\begin{minipage}{.4\textwidth}
  \centering
    \begin{tabular}{|c|c|c|c|c|}
        \hline
        KNN & 1NN & 3NN & 5NN & 7NN  \\ \hline
        $ \mathrm{dT}$ & 3.29\% & 3.20\% & 3.24\% & 3.44\%  \\ \hline
        MS & 3.54\% & 3.54\% & 3.72\% & 3.84\% \\ \hline
        GT & 3.20\% & 3.12\% & 3.22\% & 3.37\% \\ \hline
        WT2 & 3.18\% & 3.14\% & 3.21\% & 3.39\% \\ \hline
    \end{tabular}
  \captionof{table}{Image classification.}
  \label{tab:imclassify}
\end{minipage}
\end{figure}

\textbf{Image classification.}
We perform KNN classification on MNIST images with MS, WT2 and GT as preprocessing methods. We choose 10k images from the dataset given in~\cite{lecun1998gradient}. 
We shuffle the whole 10k dataset 5 times and each time we choose the first 5k images as the training data and the last 5k as the test data.
We deskew the images and use the tangent distance $ \mathrm{dT}$ described in \cite{lecun1998gradient} to measure the dissimilarity between images. 
We compare the performance of MS, WT2, GT and the baseline $\mathrm{dT}$. Here, we run MS, WT2 and GT for 1 iteration based on $\mathrm{dT}$. We compute the mean classification error rate of the 5-time experiments. The results in Table \ref{tab:imclassify} show that 
GT and WT2 have similar performance and both exhibit lower classification error rates than both MS and $\mathrm{dT}$.

\textbf{Boosting  word embeddings in NLP.}
\label{sec:word-emb}
Word embedding methods are an important family of  techniques in Natural Language Processing~\cite{mikolov2013distributed,Vilnis2014WordRV,muzellec2018generalizing,pennington2014glove,Devlin2019BERTPO}. A basic instantiation of this idea is that one vectorizes each word in a given corpus by mapping it to a feature vector in a context sensitive way. 
Such ideas are applied widely in many NLP tasks, such as Machine Translation~\cite{zou2013bilingual}, Word Analogy~\cite{pmlr-v97-allen19a}, and Name Entity Recognition~\cite{das2017named}. 
However, training a word embedding layer for a specific large corpus $\mathcal{C}$ can be computationally intensive~\cite{anand2019asynchronous}. Instead of training such a layer from scratch, there are many freely available embeddings which have been pre-computed on extensive and rich corpora such as wikipedia. These embeddings could potentially be directly applied to a task on the corpus $\mathcal{C}$. 
However, the pre-trained embedding may not perform as well as an embedding layer specifically trained for $\mathcal{C}$. We study the possibility of applying GT to  pre-trained embeddings in order to improve their performance.  We consider a given pre-trained embedding as a map $\Omega:\mathrm{Dict}\rightarrow \R^m$ where $\mathrm{Dict}$ is the universe of all words under consideration. Given a certain corpus $\mathcal{C}$, and a word $w$ in $\mathcal{C}$, we regard the set of words in a suitably defined context $c_\mathcal{C}(w)$ of $w$ in $\mathcal{C}$ as the \emph{neighborhood} of $w$ (this is done by introducing a window size parameter $W$).  Then, we compute the $m\times m$ covariance matrix $\Sigma_w$  associated to  the vectors $\{\Omega(w');\,w'\in c_\mathcal{C}(w)\}$ corresponding to context words in $c_\mathcal{C}(w)$. This mechanism then permits augmenting the information provided by $\Omega$ by incorporating ideas related to the GT distance (Definition \ref{def:gtd}): instead of measuring dissimilarity between two words $w_1,w_2\in \mathcal{C}$  via the Euclidean distance $\|\Omega(w_1)-\Omega(w_2)\|$, we implement (a suitable version of) equation (\ref{eq:gtd}). See Table \ref{tab:words}  for our experimental results which show that this way of ``boosting" the embedding $\Omega$ via GT improves the performance of the pre-trained GloVe embedding from \cite{pennington2014glove} on a specified corpus (text8). 
Other procedures representing each word on a given corpus $\mathcal{C}$ by both a vector and a covariance matrix can be found in the literature~\cite{Vilnis2014WordRV,muzellec2018generalizing}.  However, these methods  perform training on the corpus $\mathcal{C}$ from scratch whereas our method is computationally much less demanding since it relies on the pre-trained embedding $\Omega$ and does not require any additional training. See the supplementary material for details and {more} comparison results. 

\begin{table}[htb]
\caption{Spearman rank correlation for word similarity datasets.  }
    \centering
    \begin{tabular}{| c| c| c|| c | c | c|}
    \hline
    \textbf{Dataset} & \textbf{GloVe} & \textbf{GloVe+GT} & \textbf{Dataset} & \textbf{GloVe} & \textbf{GloVe+GT} \\
    \hline
        MC-30 & 0.56  & \textbf{0.67} & SIMLEX-999 & 0.26 &  \textbf{0.27}\\ \hline 
        MEN-TR-3k & \textbf{0.65}  & \textbf{0.65} & SimVerb-3500 & \textbf{0.15} &  0.14 \\ \hline 
        MTurk-287 & 0.61 &  \textbf{0.62} & VERB-143 & \textbf{0.25} & 0.24 \\ \hline 
        MTurk-771 & 0.55 &  \textbf{0.56} & WS-353-ALL & 0.49 & \textbf{0.51} \\ \hline 
         RG-65 & 0.60 &  \textbf{0.62} & WS-353-REL & 0.46 & \textbf{0.47} \\ \hline 
         RW-STANFORD & 0.34 & \textbf{0.38} &  WS-353-SIM & 0.57  & \textbf{0.60} \\ \hline 
    \end{tabular}
        \label{tab:words}
\end{table}

%% file: discussion.tex
\section{Discussion}
The Gaussian transform is a method which takes as input a point cloud $X$ with a probability measure, and alters both the metric structure and point positions iteratively with the purpose of enhancing latent features and/or denoising $X$.  GT is in the same family of methods as WT and MS. GT is stable with respect to perturbations on the probability measure (under certain conditions) and it is amenable to many optimization strategies for accelerating its implementation. The intrinsic parameter $\lambda$ of GT provides flexibility in tuning the degree of magnification of the sensitivity of GT to anisotropic data features which makes GT comparable/superior to MS and WT in several experiments related to clustering, denoising, and classification. 
It seems interesting to generalize GT to non-Euclidean datasets such as manifolds. In our formulation of GT, $\lambda$ is a parameter which needs to be tuned for each different dataset. Thus, it would be useful to identify  adaptive ways to tune $\lambda$ automatically. Metric training ideas \cite{xing2003distance} are also eminently applicable to our setting. 

\subsubsection*{Acknowledgements} We acknowledge the support of NSF through grants  DMS-1723003 and CCF-1740761.

%% file: supp_algorithm.tex
\newpage

\begin{center}
    \Large{\textbf{Supplementary Material}}
\end{center}

\section{GT details and algorithm}
\label{app:alg}

\subsection{Remark on Bures metric}
\begin{remark}[Lower and upper bounds on Euclidean spaces]\label{rmk:bound}
Suppose $\alpha$ and $\beta$ are probability measures on $\R^m$. Denote by $\mu_\alpha,\mu_\beta$ the means of $\alpha,\beta$, respectively, and by $\Sigma_\alpha,\Sigma_\beta$ the covariance matrices of $\alpha,\beta$, respectively. Define two Gaussian distributions $\gamma_\alpha=\mathcal{N}(\mu_\alpha,\Sigma_\alpha)$ and $\gamma_\beta=\mathcal{N}(\mu_\beta,\Sigma_\beta)$. Then, we have the following relation:
$\norm{\mu_\alpha-\mu_\beta}\leq d_{\mathrm{W},2}(\gamma_\alpha,\gamma_\beta)\leq d_{\mathrm{W},2}(\alpha,\beta).$
The leftmost inequality follows directly from the formula of $d_{\mathrm{W},2}$ between two Gaussians mentioned in Section \ref{sec:bg}. The rightmost one was proved in \cite{gelbrich1990formula}.
In words, in Euclidean spaces, the $\ell^2$-Wasserstein distance between probability measures is bounded below by the $\ell^2$-Wasserstein distance between Gaussian distributions generated by the means and covariance matrices of the original probability measures.  
\end{remark}

\subsection{Discrete formulation of mean and covariance}
In the case when $X=\{x_1,\cdots,x_n\}$ is a finite space and let $\alpha_i\coloneqq \alpha(x_i)$, explicitly for $i=1,\cdots,n$, we have:
\begin{equation}
\me_{{\alpha,d_X}}(x_i) = \frac{1}{A_i}\sum_{j\in I_i^{(\eps)}}\alpha_j\,\delta_{x_j}, \,\, \mu^{\mathsmaller{(\eps)}}_{{\alpha,d_X}}(x_i) = \frac{1}{A_i} \sum_{j\in I_i^{(\eps)}}\alpha_j \, x_j,    
\end{equation}

\begin{equation}
\label{eq:supp-cov}
\Sigma_{{\alpha,d_X}}^{\mathsmaller{(\eps)}}(x_i) =
    \frac{1}{A_i}\sum_{j\in I_i^{(\eps)}}\alpha_j (x_j - \mu^{\mathsmaller{(\eps)}}_{{\alpha,d_X}}(x_i))(x_j - \mu^{\mathsmaller{(\eps)}}_{{\alpha,d_X}}(x_i))^\mathrm{T},
\end{equation}
where for each $i=1,2,\ldots,n$, the index set $I_i^{(\eps)}\coloneqq \{j:\,d_X(x_j,x_i)\leq\eps\}$ and $A_i\coloneqq \sum_{j\in I_i^{(\eps)}}\alpha_j$. Above $\delta_x$ denotes the Dirac delta at $x$.

\subsection{Iterative GT algorithm}
The iterative algorithm for GT is given in Algorithm \ref{alg:gt}. In line 4 of the algorithm, the measure $m_{\alpha,D^{k}}^{\mathsmaller{(\eps)}}(x_i^k)=\frac{\sum_{\left\{j:\,D^k\left(x_j^k,x_i^k\right)\leq\eps\right\}}\alpha_j\delta_{x_j^k}}{\sum_{\left\{j:\,D^k\left(x_j^k,x_i^k\right)\leq\eps\right\}}\alpha_j}$; in line 5, $\Sigma^{\mathsmaller{(\eps)}}_{\alpha,D^k}(x_i^{k+1})$ is the covariance matrix of $\frac{\sum_{\left\{j:\,D^k\left(x_j^k,x_i^k\right)\leq\eps\right\}}\alpha_j\delta_{x_j^{k+1}}}{\sum_{\left\{j:\,D^k\left(x_j^k,x_i^k\right)\leq\eps\right\}}\alpha_j}$; in line 6, $d^{\mathsmaller{(\eps,\lambda)}}_{{\alpha,D^k}}(x_i^{k+1},x_j^{k+1})$ is computed via
$\lc\norm{x_i^{k+1}-x_j^{k+1}}^2+\lambda\cdot \dcov^2\left(\Sigma_{\alpha,D^k}^{\mathsmaller{(\eps)}}(x_i^{k+1}),\Sigma_{\alpha,D^k}^{\mathsmaller{(\eps)}}(x_j^{k+1})\right)\rc^\frac{1}{2}. $ 

\begin{algorithm}[htb]
\caption{Iterative Gaussian transform}
\begin{algorithmic}[1]
\STATE \textbf{Input:} Points $X = \left\{ x_1, x_2, ..., x_n \right\} \in \mathbb{R}^{n\times d}$, probability measure $\alpha = \{ \alpha_1, \alpha_2, ..., \alpha_n\}$, distance matrix $D$
\STATE \textbf{Initialize:} $k=0$;
 $x^k_i=x_i$; $D^k(x_i^k,x_j^k)=d^{\mathsmaller{(\eps,\lambda)}}_{{\alpha,D}}(x_i^k,x_j^k)$; 
 \WHILE{ $k< \; \mathrm{max\_iter}$}

\STATE $x_i^{k+1}=\mean\left(m_{\alpha,D^{k}}^{\mathsmaller{(\eps)}}(x_i^k)\right)$, for $i \in [n]$

\STATE Compute covariance matrices $\Sigma^{\mathsmaller{(\eps)}}_{\alpha,D^k}(x_i^{k+1})$

\STATE Let $D^{k+1}(x_i^{k+1},x_j^{k+1})=d^{\mathsmaller{(\eps,\lambda)}}_{{\alpha,D^k}}(x_i^{k+1},x_j^{k+1})$.

\STATE $k = k + 1$
\ENDWHILE
\STATE \textbf{Output:} $X^k = \left\{ x_1^k, x_2^k, ..., x_n^k \right\} \in \mathbb{R}^{n\times d}$, $\alpha = \{ \alpha_1, \alpha_2, ..., \alpha_n\}$, $D^k$
\end{algorithmic}
\label{alg:gt}
\end{algorithm}

\subsection{Neighborhood mechanism and other acceleration methods}

In this section, we provide details about the neighborhood mechanism and introduce two more related acceleration methods.

\paragraph{Neighborhood mechanism.} The following proposition is a detailed restatement of Proposition. \ref{pro:nbh-mech}

\begin{proposition}\label{pro:nbh-trick}
In the $k$th iteration of the iterative GT algorithm (cf. {Algorithm \ref{alg:gt}}), we have for any point $x^{k+1}_i\in X^{k+1}$:
$B_\eps^{D^{k+1}}(x^{k+1}_i)\subset B_\eps(x^{k+1}_i),$
where $B_\eps^{D^{k+1}}(x^{k+1}_i)$ is the ball with respect to the distance matrix $D^{k+1}$ centered at $x^{k+1}_i$ with radius $\eps$ whereas $B_\eps(x^{k+1}_i)$ is the usual Euclidean ball.
\end{proposition}

\begin{proof}
This follows directly from $\lc D^{k+1}(x_i^{k+1},x_j^{k+1})\rc^2 = \norm{x_i^{k+1}-x_j^{k+1}}^2 + \dcov^2\lc\Sigma_{\alpha,D^{k+1}}^{\mathsmaller{(\eps)}}(x_i^{k+1}), \Sigma_{\alpha,D^{k+1}}^{\mathsmaller{(\eps)}}(x_j^{k+1})\rc\geq \norm{x_i^{k+1}-x_j^{k+1}}^2$.
\end{proof}

Hence, in order to determine $B_\eps^{D^{k+1}}(x)$ for updating points or computing the covariance matrices in the next iteration, we only need to compute $D^{k+1}(x,x')$ for pairs $(x,x')$ such that $\|x-x'\|\leq \eps$ by Proposition \ref{pro:nbh-trick}.

\paragraph{Neighborhood propagation.} Once we determine that $x'\in B_\eps^{D^k}(x)$, by symmetry of the GT distance, $x\in B_\eps^{D^k}(x')$. Hence, to determine $B_\eps^{D^k}(x_i)$ for $i=1,\cdots,n$, we only need to compute the GT distance between pairs $(x_i,x_j)$ with $j>i$, where pairs with $j<i$ are already computed for determining $B_\eps^{D^k}(x_j)$. This reduces the computation times of GT distance for determining neighborhood and  makes the GT algorithm more efficient for each iteration.

\paragraph{Merging collocated points.}
Empirically speaking, data points will usually converge to some modes of the dataset after several successive applications of GT, i.e., the GT distances between some pairs of points become 0. Equivalently, such pairs of points satisfy the following two conditions: 
\begin{enumerate}
    \item their coordinates are the same; 
    \item the neighborhood points w.r.t. GT distance coincide. 
\end{enumerate}
Then, we merge the collocated points into one new point. And the weight of the new point is the sum of weights of all these collocated points. 
Then, the point set is updated by substituting the collocated points with the new points. This process reduces the total number of data points through the iterations and thus accelerates the GT algorithm.

We verify in Table \ref{tab:compt} that the neighborhood mechanism and the other two methods indeed accelerate our implementation of GT algorithm. 

\begin{table}[hbt]
\caption{{Validation of acceleration methods.} Let $X=\left\{\lc\frac{i}{199},\frac{j}{199}\rc:\,i,j=0,\cdots,199\right\}\subset\mathbb{R}^2$ and $\alpha$ be the normalized empirical measure. Set $\lambda=1$ and $\eps=0.1$. 
$\tau$ denotes the current iteration number. 
Entries below show the running time of the GT algorithm with different combinations of neighborhood mechanisms in different iterations. 
The experiments are performed on a Unix Server which has 48 cores of CPU. We use C++ with the openMP (Open Multi-Processing) to implement GT with parallel computing. 
GT: full matrix computation of the GT distance; GT-v1: GT with the neighborhood mechanism mentioned in Section \ref{sec:nbh-mech}; GT-v2: GT-v1 with neighborhood propagation; GT-v3: GT-v1 with collocated points merged; GT-v4: GT-v2  with collocated points merged. }
    \centering
    \begin{tabular}{|c|c|c|c|c|c|}
        \hline
         & $\tau=1$ & $\tau=2$ & $\tau=3$ & 
        $\tau=4$ & $\tau=5$\\ \hline
        GT & 48.7s & - & - & - & - \\ \hline
        GT-v1 & 20.2s & 11.3s & 9s & 7.1s & 6.7s \\ \hline
        GT-v2 & 14.8s & 10s & 7.4s & 6.2s & 6.3s \\ \hline
        GT-v3 & 20.9s & 9.4s & 5.3s & 2.8s & 1.5s \\ \hline
        GT-v4 & 14.7s & 8.2s & 4.1s & 2.6s & 1.4s \\ \hline
    \end{tabular}
    \label{tab:compt}
\end{table}

\subsection{Worst complexity analysis}

We compare the computational complexity of MS, WT and GT for one iteration of each method. We denote the point cloud size by $n$, dimension by $m$ and assume the maximum $\eps$-neighborhood size of each point by $N$ for some $\eps>0$. For all methods, they need $O(n^2)$ operations to select $\eps$-neighborhoods of points. This cost is not dominating (as we will see in the sequel) and we ignore it in the following analysis. 

For MS, we need to first compute the Euclidean distance between each pair of points which costs $O(n^2m)$. The updating process for each point needs $O(Nm)$ operations and thus the point updating process for all points costs $O(nNm)$ in total. So the complexity of MS for one iteration is $O(n^2m+nNm)$.

As for WT, it computes the Wasserstein distance between all pairs of points' neighborhoods. For once distance computation, the complexity is $O(N^3\log N)$~\cite[p. 472, Th. 12.2]{ahujia1993network}, 
and there are $C_n^2 = \frac{n(n-1)}{2}$ pairs of points, leading to $O(n^2)$ times of such distance computation. Then in total, the complexity of WT for one iteration is $O(n^2(N^3\log N))$. 

Now, we derive the complexity of GT for one iteration. As in the case of WT, there are also $O(n^2)$ times of GT distance computation. For once distance computation, the determined cost lies in two parts: one is the computation of covariance matrices $\Sigma_1$ and $\Sigma_2$, whose complexity is $O(Nm^2)$; 
the other is the computation of 
$\tr\left(\left(\Sigma_1^\frac{1}{2}\Sigma_2\Sigma_1^\frac{1}{2}\right)^\frac{1}{2}\right)$. For matrix multiplication, the complexity is at most $O(m^{3})$. 
The computation of square root of covariance matrix is solved by eigen-decomposition, whose complexity is $O(m^3)$~\cite{pan1998complexity,demmel2007fast}. Then, the computational cost of $\tr\left(\left(\Sigma_1^\frac{1}{2}\Sigma_2\Sigma_1^\frac{1}{2}\right)^\frac{1}{2}\right)$ is $O(m^3)$. In total, the complexity of GT-Metric for one iteration is $O(n^2(Nm^2 + m^3))$.

In the end, we analyze the time complexity of GT with neighborhood mechanism, abbreviated by GT-Neighborhood. We need to first compute the Euclidean distance between each pair of points whose complexity is $O(n^2m)$. Since each point has at most $N$ Euclidean neighborhood points, there will be $N$ times of GT distance computation. From above we know once GT distance computation complexity is $O(Nm^2 + m^3)$. Then for all points, the total cost will be $O(n^2m + nN(Nm^2+m^3))$. 

The summary of complexity comparison is listed in Table~\ref{tab:compformula}. 
\begin{table}[htb]
\caption{Complexity comparison.}
    \centering
    \begin{tabular}{|c|c|c|c|c|}
        \hline
        & MS & WT & GT & GT-Neighborhood \\ \hline
        Cost &  $O(n^2(Nm^2 + m^3))$ & $O(n^2(N^3logN))$ & $O(n^2m + nNm)$ & $O(n^2m + nN(Nm^2+m^3))$  \\ \hline
    \end{tabular}
    \label{tab:compformula}
\end{table}

Note that from Table~\ref{tab:compformula}, when $N > m$, the complexity order of four methods is 
MS $<$ GT-Neighborhood $<$ GT $<$ WT. 

%% file: supp_proof.tex
\section{Additional theorems and their proofs}
\label{app:proof}
\subsection{Stability theorems}\label{app:proof-stable}
In this section, we always assume that {$(X,d_X)$ is a compact subspace of $\R^m$, i.e., $d_X$ is the underlying Euclidean distance between points.} We also assume that $\diam(X)\leq D$.
\paragraph{Explanation of the set $\mathcal{P}_f^{c,\Lambda}(X)$.} The set $\mathcal{P}_f^{c,\Lambda}(X)$ is actually the intersection of $\mathcal{P}_f^{c}(X)$ and $\mathcal{P}_f^{\Lambda}(X)$, where the former is the set of all $\alpha\in\pf(X)$ such that $\alpha(S)\leq c\cdot \mathcal{L}_m(S)$ for any measurable $S$ and the latter is the set of all $\alpha\in\pf(X)$ such that $\frac{\alpha\left(B^{d_X}_{r_1}(x)\right)}{\alpha\left(B^{d_X}_{r_2}(x)\right)}\leq\lc\frac{r_1}{r_2}\rc^\Lambda$ for any $x\in X$ and $r_1\geq r_2>0$. $\mathcal{P}_f^{c}(X)$ was used for proving a stability theorem for one type of local covariance matrices in \cite{martinez2020shape} (cf. Lemma \ref{lm:cmbound}) whereas $\mathcal{P}_f^{\Lambda}(X)$ was used in \cite{pmlr-v97-memoli19a} for establishing stability theorems for WT and MS (cf. Lemma \ref{lm:ms-stb}).

\begin{remark}\label{rmk:stb}
One drawback of the stability theorem (Theorem \ref{thm:GT-stable}) is that it does not apply to empirical measures, i.e., $\mathcal{P}_f^{c,\Lambda}(X)$ does not contain empirical measures. This fact is due to the discontinuity of the truncation kernel inherent in the definition of $\me_{\alpha,d_X}(\cdot)$. However, if we utilize a smooth kernel for computing local covariance matrices, we obtain a more general theorem (cf. Theorem \ref{thm:gt-stable-smooth}) which applies to empirical measures. 
\end{remark}

\paragraph{Proof of Theorem \ref{thm:GT-stable}.} The proof is based on the following series of lemmas. 
\begin{lemma}\label{lm:tr}
For symmetric positive semi-definite matrices $A,B$, we have
\[\mathrm{tr}\lc A+B-2\lc A^{ \frac{1}{2} }BA^{ \frac{1}{2} }\rc^{ \frac{1}{2} }\rc\leq\norm{A^{ \frac{1}{2} }-B^{ \frac{1}{2} }}_F^2,\]
where $\norm{\cdot}_F$ is the Frobenius norm of matrices.
\end{lemma}

\begin{proof}[Proof of Lemma \ref{lm:tr}]
Expand the right hand side of the inequality we obtain
\begin{align*}
\norm{A^{ \frac{1}{2} }-B^{ \frac{1}{2} }}_F^2 = & \mathrm{tr}\lc A+B-A^{ \frac{1}{2} }B^{ \frac{1}{2} }-B^{ \frac{1}{2} }A^{ \frac{1}{2} }\rc\\
=&\mathrm{tr}\lc A+B-2B^{ \frac{1}{2} }A^{ \frac{1}{2} }\rc
\end{align*}
Hence it suffices to prove 
\[\mathrm{tr}\lc \lc A^{ \frac{1}{2} }BA^{ \frac{1}{2} }\rc^{ \frac{1}{2} }\rc\geq\tr\lc B^{ \frac{1}{2} }A^{ \frac{1}{2} }\rc.\]
Let $X=B^{ \frac{1}{2} }A^{ \frac{1}{2} }$, then $\lc A^{ \frac{1}{2} }BA^{ \frac{1}{2} }\rc^{ \frac{1}{2} }=\lc X^\mathrm{T}X\rc^{ \frac{1}{2} }$. If we denote the singular values of $X$ as $\{\sigma_i\}_{i=1,\cdots,m}$ and the eigenvalues of $X$ as $\{\lambda_i\}_{i=1,\cdots,m}$, then 
\begin{align*}
\mathrm{tr}\lc \lc A^{ \frac{1}{2} }BA^{ \frac{1}{2} }\rc^{ \frac{1}{2} }\rc & =\tr\lc\lc X^\mathrm{T}X\rc^{ \frac{1}{2} }\rc\\
& =\sum_{i=1}^m\sigma_i\geq\sum_{i=1}^m|\lambda_i|\\
& \geq\tr\lc X\rc=\tr\lc B^{ \frac{1}{2} }A^{ \frac{1}{2} }\rc.
\end{align*}
The first inequality follows directly from Theorem 2.3.6 in \cite{bhatia2013matrix}.
\end{proof}

\begin{lemma}\label{lm:sqrn}
For symmetric positive semi-definite matrices $A,B$ with dimension $m$, we have
\[\norm{A^{ \frac{1}{2} }-B^{ \frac{1}{2} }}_F^2\leq m\norm{A-B}_F.\]
\end{lemma}

\begin{proof}[Proof of Lemma \ref{lm:sqrn}]
It's shown in page 290 of \cite{bhatia2013matrix} that by using the operator norm $\norm{\cdot}$ of matrices one has
\[\norm{A^{ \frac{1}{2} }-B^{ \frac{1}{2} }}^2\leq \norm{A-B}.\]
By using the following relation for any $m$-dimensional matrix $M$ (see page 7 of \cite{bhatia2013matrix})
\[\norm{M}\leq \norm{M}_F\leq\sqrt{m}\norm{M},\]
we obtain
\[\norm{A^{ \frac{1}{2} }-B^{ \frac{1}{2} }}_F^2\leq m\norm{A-B}_F.\]
\end{proof}
Denote by $\Tilde{\Sigma}_{\alpha,d_X}^{\mathsmaller{(\eps)}}(x)$ the matrix defined as follows:
\begin{equation}\label{eq:til-cov}
    \frac{1}{\eps^m\,\nu_m}\int_{B_\eps^{d_X}(x)}\lc{y}-x\rc\otimes \lc {y}- x\rc\,\alpha(d {y}), 
\end{equation}
where {$\nu_m$ is the volume of the unit ball in $\R^m$.}

Then, we have the following result:

\begin{lemma}[Theorem 3 in \cite{martinez2020shape}]\label{lm:cmbound}
Assume $\alpha,\beta\in\pf^c(X)$. Then, there is a constant $A=A(\eps,m,D)$ such that 
\[\sup_{x\in\R^m}\norm{\tilde{\Sigma}_{\alpha,d_X}^{\mathsmaller{(\eps)}}(x)-\tilde{\Sigma}_{\beta,d_X}^{\mathsmaller{(\eps)}}(x)}_F\leq c\, A\cdot d_{\mathrm{W},\infty}(\alpha,\beta).\]
\end{lemma}

Note that $\Tilde{\Sigma}_{\alpha,d_X}^{\mathsmaller{(\eps)}}(x)$ is different from the local covariance matrix ${\Sigma}_{\alpha,d_X}^{\mathsmaller{(\eps)}}(x)$ defined in Section \ref{sec:gt} of the paper. To make use of Lemma \ref{lm:cmbound}, we define another matrix as follows to mediate between the two different matrices:
\[\hat{\Sigma}_{\alpha,d_X}^{\mathsmaller{(\eps)}}(x)\coloneqq\int_{\mathbb{R}^d}\lc{y}-x\rc\otimes \lc {y}- x\rc\,\me_{{\alpha,d_X}}(x)(d {y}). \]
Note that $\hat{\Sigma}_{\alpha,d_X}^{\mathsmaller{(\eps)}}(x)=\frac{\eps^m\,\nu_m}{\alpha(B_\eps^{d_X}(x))}\tilde{\Sigma}_{\alpha,d_X}^{\mathsmaller{(\eps)}}(x)$ and $\hat{\Sigma}_{\alpha,d_X}^{\mathsmaller{(\eps)}}(x)={\Sigma}_{\alpha,d_X}^{\mathsmaller{(\eps)}}(x)+(\mu_{\alpha,d_X}^{\mathsmaller{(\eps)}}(x)-x)\otimes(\mu_{\alpha,d_X}^{\mathsmaller{(\eps)}}(x)-x)$.

Denote by $\psi_{\Lambda,D}(\eps) \coloneqq\min\lc1,\lc\frac{\eps}{D}\rc^\Lambda\rc$, $\Phi_{\Lambda,\eps}(t)\coloneqq\lc\lc1+\frac{t}{\eps}\rc^\Lambda-1\rc+t$ and $\Phi_{\Lambda,D,\eps}^{c,A}(t)\coloneqq \frac{\eps^m\,\nu_m\,c\,A}{\psi_{\Lambda,D}(\eps)}\,t+\frac{\eps^2}{\psi_{\Lambda,D}^2(\eps)}\Phi_{\Lambda,\eps}\left(\sqrt{t}\right)$ for $t\geq 0.$ Note that both $\Phi_{\Lambda,\eps}$ and $\Phi_{\Lambda,D,\eps}^{c,A}$ are increasing functions with value 0 when the argument is 0.

\begin{lemma}\label{lm:stb-hat-sigma}
Under the same assumptions as in Lemma \ref{lm:cmbound}, we have that
\[\sup_{x\in\R^m}\norm{\hat{\Sigma}_{\alpha,d_X}^{\mathsmaller{(\eps)}}(x)-\hat{\Sigma}_{\beta,d_X}^{\mathsmaller{(\eps)}}(x)}_F\leq \Phi_{\Lambda,D,\eps}^{c,A}(d_{\mathrm{W},\infty}(\alpha,\beta)).\]
\end{lemma}

\begin{proof}[Proof of Lemma \ref{lm:stb-hat-sigma}]
For simplicity of notation, we denote $\alpha_x\coloneqq\alpha(\bxe)$ and $\beta_x\coloneqq\beta(\bxe)$.
\begin{align*}
    &\norm{\hat{\Sigma}_{\alpha,d_X}^{\mathsmaller{(\eps)}}(x)-\hat{\Sigma}_{\beta,d_X}^{\mathsmaller{(\eps)}}(x)}_F\\
    =&{\eps^m\,\nu_m}\norm{\frac{\tilde{\Sigma}_{\alpha,d_X}^{\mathsmaller{(\eps)}}(x)}{\alpha(B_\eps(x))}-\frac{\tilde{\Sigma}_{\beta,d_X}^{\mathsmaller{(\eps)}}(x)}{\beta(B_\eps(x))}}_F\\
    =&\frac{\eps^m\,\nu_m}{\alpha_x\,\beta_x}\norm{{\tilde{\Sigma}_{\alpha,d_X}^{\mathsmaller{(\eps)}}(x)}{\beta_x}-{\tilde{\Sigma}_{\beta,d_X}^{\mathsmaller{(\eps)}}(x)}{\alpha_x}}_F\\
    \leq&\underbrace{\frac{\eps^m\,\nu_m}{\alpha_x}\norm{{\tilde{\Sigma}_{\alpha,d_X}^{\mathsmaller{(\eps)}}(x)}-{\tilde{\Sigma}_{\beta,d_X}^{\mathsmaller{(\eps)}}(x)}}_F}_{T_1}\\
    +&\underbrace{\frac{\eps^m\,\nu_m}{\alpha_x\,\beta_x}|\alpha_x-\beta_x|\norm{{\tilde{\Sigma}_{\beta,d_X}^{\mathsmaller{(\eps)}}(x)}}_F}_{T_2}
\end{align*}
By Remark 4.1 in \cite{pmlr-v97-memoli19a}, we have that $\alpha_x,\beta_x\geq\psi_{\Lambda,D}(\eps)$. Hence, together with Lemma \ref{lm:cmbound}, we have
\[T_1\leq \frac{\eps^m\,\nu_m\,c\,A}{\psi_{\Lambda,D}(\eps)}d_{\mathrm{W},\infty}(\alpha,\beta). \]

To estimate $|\alpha_x-\beta_x|$, we introduce the so-called \emph{Prokhorov distance} $d_\mathrm{P}$ \cite{givens1984class} between probability measures, which is defined by 
\[d_\mathrm{P}(\alpha,\beta)\coloneqq\inf\{\eta:\,\alpha(A)\leq \beta(A^\eta)+\eta\}. \]
Though seemingly asymmetric, $d_\mathrm{P}$ is a symmetric metric on $\pf(X)$ and as a consequence, the roles of $\alpha$ and $\beta$ in the definition are interchangeable. 

Without loss of generality, we assume that $\beta_x\geq \alpha_x$. Let $\xi\coloneqq d_\mathrm{P}(\alpha,\beta)$. Then,
\begin{align*}
    &\beta(\bxe)-\alpha(\bxe)\\
    \leq&\alpha\left((\bxe)^\xi\right)+\xi-\alpha(\bxe)\\
    \leq &\alpha(\bxe)\lc\frac{\alpha\left(B_{\eps+\xi}^{d_X}(x)\right)}{\alpha(\bxe)}-1\rc+\xi\\
    \leq& \lc\lc1+\frac{\xi}{\eps}\rc^\Lambda-1\rc+\xi=\Phi_{\Lambda,\eps}(\xi)\\
    \leq &\Phi_{\Lambda,\eps}\lc\sqrt{d_{\mathrm{W},1}(\alpha,\beta)}\rc
    \leq\Phi_{\Lambda,\eps}\lc\sqrt{d_{\mathrm{W},\infty}(\alpha,\beta)}\rc.
\end{align*}
Since $\Phi_{\Lambda,\eps}$ is increasing, the second to last inequality follows from the fact that $(d_\mathrm{P})^2\leq d_{\mathrm{W},1}$ \cite{gibbs2002choosing} and the last inequality follows from the fact that $d_{\mathrm{W},p}\leq d_{\mathrm{W},q}$ whenever $1\leq p\leq q\leq \infty$ \cite{givens1984class}.

Since $y$ is constructed in $\bxe$ in Equation (\ref{eq:til-cov}), we have $\norm{y-x}\leq \eps$ and thus $\norm{\tilde{\Sigma}_{\beta,d_X}^{\mathsmaller{(\eps)}}(x)}_F\leq \frac{\eps^2}{{\eps^{m}\nu_m}}$. Therefore,
\[T_2\leq  \frac{\eps^2}{\psi_{\Lambda,D}^2(\eps)}\Phi_{\Lambda,\eps}\left(\sqrt{d_{\mathrm{W},\infty}(\alpha,\beta)}\right).\]
Hence, $T_1+T_2\leq \Phi_{\Lambda,D,\eps}^{c,A}(d_{\mathrm{W},\infty}(\alpha,\beta))$
\end{proof}

In \cite{pmlr-v97-memoli19a}, the authors provide a stability theorem for MS with respect to probability measures in $\pf^\Lambda(X)$. Denote $\Psi_{\Lambda,D,\eps}(t)\coloneqq \frac{t}{\psi_{\Lambda,D}(\eps)}+\left[\lc1+\frac{t}{\eps}\rc^\Lambda-1\right]$ for $t\geq 0$.

\begin{lemma}[Theorem 4.6 in \cite{pmlr-v97-memoli19a}]\label{lm:ms-stb}
Assume $\alpha,\beta\in\pf^\Lambda(X)$. Then, 
\begin{align*}
    &\sup_{x\in X}\norm{\mu^{\mathsmaller{(\eps)}}_{\alpha,d_X}(x)-\mu^{\mathsmaller{(\eps)}}_{\beta,d_X}(x)} \\
    \leq &(1+2\eps)\,\Phi_{\Lambda,D,\eps}\lc\sqrt{d_{\mathrm{W},1}(\alpha,\beta)}\rc.
\end{align*}
\end{lemma}

Now we are ready to establish a key lemma for proving Theorem \ref{thm:GT-stable}. Denote by $\Psi_{\Lambda,D,\eps}^{c,A}(t)\coloneqq \Phi_{\Lambda,D,\eps}^{c,A}(t)+2\eps(1+2\eps)\Phi_{\Lambda,D,\eps}\left(\sqrt{t}\right)$ for $t\geq 0$. It is easy to see that $\Psi_{\Lambda,D,\eps}^{c,A}$ is an increasing function such that $\Psi_{\Lambda,D,\eps}^{c,A}(0)=0$. 

\begin{lemma}\label{lm:key-cov-est}
Assume $\alpha,\beta\in\pf^{c,\Lambda}(X)$. Then,

\[\sup_{x\in\R^m}\norm{{\Sigma}_{\alpha,d_X}^{\mathsmaller{(\eps)}}(x)-{\Sigma}_{\beta,d_X}^{\mathsmaller{(\eps)}}(x)}_F\leq \Psi_{\Lambda,D,\eps}^{c,A}(d_{\mathrm{W},\infty}(\alpha,\beta)).\]
\end{lemma}

\begin{proof}
For simplicity of notation, we let $\mu_\alpha\coloneqq\mu^{\mathsmaller{(\eps)}}_{\alpha,d_X}(x)$ and $\mu_\beta\coloneqq\mu^{\mathsmaller{(\eps)}}_{\beta,d_X}(x)$.
\begin{align*}
    &\norm{{\Sigma}_{\alpha,d_X}^{\mathsmaller{(\eps)}}(x)-{\Sigma}_{\beta,d_X}^{\mathsmaller{(\eps)}}(x)}_F\\
    \leq &\norm{\hat{\Sigma}_{\alpha,d_X}^{\mathsmaller{(\eps)}}(x)-\hat{\Sigma}_{\beta,d_X}^{\mathsmaller{(\eps)}}(x)}_F\\
    +&\norm{(\mu_\alpha-x)^{\otimes^2}-(\mu_\beta-x)^{\otimes^2}}_F\\
    \leq &\Phi_{\Lambda,D,\eps}^{c,A}(d_{\mathrm{W},\infty}(\alpha,\beta))+\norm{(\mu_\alpha-x)\otimes(\mu_\alpha-\mu_\beta)}_F\\
    +&\norm{(\mu_\alpha-\mu_\beta)\otimes(\mu_\beta-x)}_F\\
    \leq &\Phi_{\Lambda,D,\eps}^{c,A}(d_{\mathrm{W},\infty}(\alpha,\beta))+2\eps\norm{\mu_\alpha-\mu_\beta}.
\end{align*}
By Lemma \ref{lm:ms-stb}, we obtain
\begin{align*}
    &\norm{{\Sigma}_{\alpha,d_X}^{\mathsmaller{(\eps)}}(x)-{\Sigma}_{\beta,d_X}^{\mathsmaller{(\eps)}}(x)}_F\leq \Phi_{\Lambda,D,\eps}^{c,A}(d_{\mathrm{W},\infty}(\alpha,\beta))\\
    +&2\eps(1+2\eps)\Phi_{\Lambda,D,\eps}\left(\sqrt{d_{\mathrm{W},1}(\alpha,\beta)}\right)\\
    \leq &\Psi_{\Lambda,D,\eps}^{c,A}(d_{\mathrm{W},\infty}(\alpha,\beta)).
\end{align*}
We use the fact that $d_{\mathrm{W},1}\leq d_{\mathrm{W},\infty}$ again in the last inequality.
\end{proof}

\begin{proof}[Proof of Theorem \ref{thm:GT-stable}]
For any $x\in X$, one has
\begin{align*}
    &\dw\left(\gamma_{\alpha,d_X}^{\mathsmaller{(\eps,\lambda)}}(x),\gamma_{\beta,d_X}^{\mathsmaller{(\eps,\lambda)}}(x)\right)\\
    =&\sqrt{\lambda\,\dcov^2\left(\Sigma_{\alpha,d_X}^{\mathsmaller{(\eps)}}(x),\Sigma_{\beta,d_X}^{\mathsmaller{(\eps)}}(x)\right)}\\
    \leq&\sqrt{\lambda} \norm{\left(\Sigma_{\alpha,d_X}^{\mathsmaller{(\eps)}}(x)\right)^\frac{1}{2}-\left(\Sigma_{\beta,d_X}^{\mathsmaller{(\eps)}}(x)\right)^\frac{1}{2}}_F\\
    \leq& \sqrt{m\,\lambda} \norm{\Sigma_{\alpha,d_X}^{\mathsmaller{(\eps)}}(x)-\Sigma_{\beta,d_X}^{\mathsmaller{(\eps)}}(x)}_F^\frac{1}{2}\\
    \leq& \sqrt{m\,\lambda\,\Psi_{\Lambda,D,\eps}^{c,A}(d_{\mathrm{W},\infty}(\alpha,\beta))}.
\end{align*}
The first inequality follows from Lemma \ref{lm:tr}. The second inequality follows from Lemma \ref{lm:sqrn}. The last inequality follows from Lemma \ref{lm:key-cov-est}.

Now, for any $x,x'\in X$, one has
\begin{align*}
    &\left|\dgt(x,x')-d^{\mathsmaller{(\eps,\lambda)}}_{{\beta,d_X}}(x,x')\right|\\
    = & \big|\dw(\gamma_{\alpha,d_X}^{\mathsmaller{(\eps,\lambda)}}(x),\gamma_{\alpha,d_X}^{\mathsmaller{(\eps,\lambda)}}(x'))-\dw(\gamma_{\alpha,d_X}^{\mathsmaller{(\eps,\lambda)}}(x),\gamma_{\beta,d_X}^{\mathsmaller{(\eps,\lambda)}}(x'))\\
    +&\dw(\gamma_{\alpha,d_X}^{\mathsmaller{(\eps,\lambda)}}(x),\gamma_{\beta,d_X}^{\mathsmaller{(\eps,\lambda)}}(x'))-\dw(\gamma_{\beta,d_X}^{\mathsmaller{(\eps,\lambda)}}(x),\gamma_{\beta,d_X}^{\mathsmaller{(\eps,\lambda)}}(x'))\big|\\
    \leq &\dw(\gamma_{\alpha,d_X}^{\mathsmaller{(\eps,\lambda)}}(x'),\gamma_{\beta,d_X}^{\mathsmaller{(\eps,\lambda)}}(x))+\dw(\gamma_{\alpha,d_X}^{\mathsmaller{(\eps,\lambda)}}(x),\gamma_{\beta,d_X}^{\mathsmaller{(\eps,\lambda)}}(x))\\
    \leq & 2\sqrt{m\,\lambda\,\Psi_{\Lambda,D,\eps}^{c,A}(d_{\mathrm{W},\infty}(\alpha,\beta))}.
\end{align*}
\end{proof}

\paragraph{Smooth kernels.} As mentioned in Remark \ref{rmk:stb}, if we compute local covariance matrices via a smooth kernel, we would obtain a more general stability theorem. The following definition characterizes the requirements of a smooth kernel.

\begin{definition}
Let $f:[0,\infty)\rightarrow(0,\infty)$ be a bounded and differentiable function such that:
\begin{enumerate}
    \item $M_m\coloneqq\int_0^\infty r^{\frac{m}{2}-1}f(r)\,dr<\infty$.
    \item There exists $C>0$ such that $rf(r)\leq C,\,\forall r\in[0,\infty)$.
    \item There exists $L>0$ such that $|f'(r)|\leq L$ and $r^\frac{3}{2}|f'(r)|\leq L$ for $r\in[0,\infty)$.
\end{enumerate}
Then, we define the multiscale smooth kernel $K:\R^m\times \R^m\times(0,\infty)\rightarrow \R$ associated with $f$ by 
\[K(x,y,\eps)\coloneqq\frac{1}{C_m(\eps)}f\lc\frac{\norm{y-x}^2}{\eps^2}\rc, \]
where $C_m(\eps)\coloneqq\frac{1}{2}\eps^mM_m\omega_{m-1}$ and $\omega_{m-1}$ is the surface area of the unit sphere $\mathbb{S}^{m-1}$.
\end{definition}

\begin{remark}
The definition is a combination of Definition 2, assumptions of Theorem 1 in \cite{martinez2020shape} and assumptions of Remark 4.4 in \cite{pmlr-v97-memoli19a}.
\end{remark}

Now we define the GT distance with respect to the smooth kernels and state our main result as follows.

The mean of $\alpha$ at $x\in X$ with respect to $K$ is defined as follows:
\begin{equation}\label{eq:smooth-mean}
    \mu^{\mathsmaller{(\eps)}}_{\alpha,K}(x)\coloneqq \frac{\int_{\R^m}y\,K(x,y,\eps)\,\alpha(dy)}{\int_{\R^m}K(x,y,\eps)\,\alpha(dy)}.
\end{equation}
The local covariance ${\Sigma}_{\alpha,K}^{\mathsmaller{(\eps)}}(x)$ of $\alpha$ generated through $K$ is defined by the following matrix:
\[ \frac{\int_{\R^m}\lc{y}-\mu^{\mathsmaller{(\eps)}}_{\alpha,K}(x)\rc\otimes \lc {y}- \mu^{\mathsmaller{(\eps)}}_{\alpha,K}(x)\rc K(x,y,\eps)\,\alpha(d {y})}{\int_{\R^m}K(x,y,\eps)\,\alpha(dy)}.\]

\begin{remark}\label{rmk:x-replace}
In all the integrals above, the domain of integration $\R^m$ can be replaced by $X$ since $\alpha$ is supported on $X$.
\end{remark}

Then, with respect to a smooth kernel $K$ we define the GT distance $d_{\alpha,K}^{\mathsmaller{(\eps,\lambda)}}(x,x')$ between $x,x'\in X$ by the following quantity 
\begin{equation}\label{eq:gt-smooth}
    \left(\norm{x-x'}^2+\lambda\cdot \dcov^2\lc\Sigma_{{\alpha,K}}^{\mathsmaller{(\eps)}}(x),\Sigma_{{\alpha,K}}^{\mathsmaller{(\eps)}}(x')\rc\right)^\frac{1}{2}.
\end{equation}

\begin{theorem}[Stability of GT for smooth kernels]\label{thm:gt-stable-smooth}
There exists a positive constant $Z>0$ such that for $\alpha,\beta\in \pf(X)$, we have
\[\norm{d_{\alpha,K}^{\mathsmaller{(\eps,\lambda)}}-d_{\beta,K}^{\mathsmaller{(\eps,\lambda)}}}_\infty\leq 2\sqrt{ m\,\lambda\,Z\,d_{\mathrm{W},1}(\alpha,\beta)}. \]
\end{theorem}

The proof of the theorem is based on the following series of lemmas.

\begin{lemma}[Remark 4.4 in \cite{pmlr-v97-memoli19a}]\label{lm:smooth-wt}
Fix any compact metric space $Y$ (not necessarily Euclidean) and $C>0$. If $g:[0,\infty)\rightarrow(0,\infty)$ is $C$-Lipschitz, then there exist positive constants $U$ and $V$ depending only on $g$ and $Y$ such that 
\[d_{\mathrm{W},1}(m_\alpha^g(y),m_\beta^g(y))\leq\frac{2\,C\,\diam(Y)+V}{U}d_{\mathrm{W},1}(\alpha,\beta), \]
where $y\in Y$ and the probability measures at the left hand side are defined by \[m_\alpha^g(y)(A)\coloneqq\frac{\int_Ag(d(y,z))\alpha(dz)}{\int_Y g(d(y,z))\alpha(dz)},\] 
for any measurable set $A\subset Y$.
\end{lemma}
\begin{remark}\label{rmk:rst-domain}
Notice that in the definition of $m_\alpha^g(\cdot)$, we only used the restriction of $g$ on $[0,\diam(Y)]$. So the result still holds true assuming $g$ is a Lipschitz function from $[0,\diam(Y)]$ to $(0,\infty)$.
\end{remark}

\begin{lemma}\label{lm:ms-smooth}
Let $C=C(\eps,L,D)\coloneqq\frac{2L\,D}{\eps^2}$. There exist positive constants $U$ and $V$ depending only on $f$ and $X$ such that 
\[\norm{\mu^{\mathsmaller{(\eps)}}_{\alpha,K}(x)-\mu^{\mathsmaller{(\eps)}}_{\beta,K}(x)}\leq\frac{2\,CD+V}{U}d_{\mathrm{W},1}(\alpha,\beta). \]
\end{lemma}
\begin{proof}
Let $g(t)\coloneqq f\lc\frac{t^2}{\eps^2}\rc$ and $Y=X$. Then, we have that $\mu^{\mathsmaller{(\eps)}}_{\alpha,K}(x)=\mean\lc m_\alpha^g(x)\rc$. Since $g'(t)=\frac{2t}{\eps^2}f'\lc\frac{t^2}{\eps^2}\rc$, we have that $|g'(t)|\leq \frac{2\,L\,D}{\eps^2}$ for all $t\in[0,D]$, which implies that $g$ is $\frac{2\,L\,D}{\eps^2}$-Lipschitz on $[0,D]$. Then, by Lemma \ref{lm:smooth-wt} and Remark \ref{rmk:rst-domain}, we have that 
\[d_{\mathrm{W},1}\left(m_\alpha^g(x),m_\beta^g(x)\right)\leq\frac{2\,CD+V}{U}d_{\mathrm{W},1}(\alpha,\beta), \]
where $U,V$ are positive constants depending on $f$ and $X$.
Then, by a standard result in Euclidean space \cite{rubner1998metric}, we have that
\begin{align*}
    \norm{\mu^{\mathsmaller{(\eps)}}_{\alpha,K}(x)-\mu^{\mathsmaller{(\eps)}}_{\beta,K}(x)}&\leq d_{\mathrm{W},1}\left(m_\alpha^g(x),m_\beta^g(x)\right)\\
    &\leq \frac{2\,CD+V}{U}\,d_{\mathrm{W},1}(\alpha,\beta).
\end{align*}
\end{proof}

Denote by $\Tilde{\Sigma}_{\alpha,K}^{\mathsmaller{(\eps)}}(x)$ the following matrix:
\begin{equation}\label{eq:til-cov-smooth}
    \int_{\R^m}\lc{y}-x\rc\otimes \lc {y}- x\rc K(x,y,\eps)\,\alpha(d {y}). 
\end{equation}

\begin{lemma}[Theorem 1 in \cite{martinez2020shape}]\label{thm:tilde-sigma}
There exists a constant $A_f>0$ only depending on $f$ such that for any $\alpha,\beta\in\mathcal{P}_f(X)$, we have
\[\norm{\Tilde{\Sigma}_{\alpha,K}^{\mathsmaller{(\eps)}}(x)-\Tilde{\Sigma}_{\beta,K}^{\mathsmaller{(\eps)}}(x)}_F\leq A_f\,d_{\mathrm{W},1}(\alpha,\beta).\]
\end{lemma}

\begin{lemma}\label{lm:key-smooth-stb}
There exists a positive constant $Z$ depending on $f,\eps,D$ and $X$ such that for any $\alpha,\beta\in\pf(X)$
\[\norm{{\Sigma}_{\alpha,K}^{\mathsmaller{(\eps)}}(x)-{\Sigma}_{\beta,K}^{\mathsmaller{(\eps)}}(x)}_F\leq Z\,d_{\mathrm{W},1}(\alpha,\beta).\]
\end{lemma}

\begin{proof}
To simplify our notations, denote $M_\alpha\coloneqq\int_XK(x,y,\eps)\alpha(dy)$, $\mu_\alpha\coloneqq\mu^{\mathsmaller{\eps}}_{\alpha,K}(x)$, $\Sigma_\alpha\coloneqq{\Sigma}_{\alpha,K}^{\mathsmaller{(\eps)}}(x)$ and $\tilde{\Sigma}_\alpha\coloneqq\tilde{\Sigma}_{\alpha,K}^{\mathsmaller{(\eps)}}(x)$.
Note that $\Sigma_\alpha=\frac{\tilde{\Sigma}_\alpha}{M_\alpha}-(\mu_\alpha-x)\otimes(\mu_\alpha-x)$. Then,
\begin{align*}
    \norm{\Sigma_\alpha-\Sigma_\beta}_F&\leq\norm{\frac{1}{M_\alpha}\tilde{\Sigma}_\alpha-\frac{1}{M_\beta}\tilde{\Sigma}_\beta}_F\\
    &+\norm{(\mu_\alpha-x)^{\otimes^2}-(\mu_\beta-x)^{\otimes^2}}\\
    &\leq \underbrace{\frac{1}{M_\alpha}\norm{\tilde{\Sigma}_\alpha-\tilde{\Sigma}_\beta}_F}_{T_1}+\underbrace{\frac{|M_\alpha-M_\beta|}{M_\alpha\,M_\beta}\norm{\tilde{\Sigma}_\beta}_F}_{T_2}\\
    &+\underbrace{\norm{\mu_\alpha-\mu_\beta}\lc\norm{\mu_\alpha-x}+\norm{\mu_\beta-x}\rc}_{T_3}
\end{align*}

Since $f$ is continuous and positive by assumption, there exists $W=W(\eps,D)>0$ such that for any $t\in\left[0,\frac{D^2}{\eps^2}\right]$, $f(t)\geq W$. Hence, $M_\alpha\geq W$. Then, by Lemma \ref{thm:tilde-sigma}, we have that $T_1\leq \frac{A_f}{W}d_{\mathrm{W},1}(\alpha,\beta).$

Since $f$ is $L$-Lipschitz, we have that for a $y_1,y_2\in X$,
\begin{align*}
    &|K(x,y_1,\eps)-K(x,y_2,\eps)|\\
    =&\frac{1}{C_m(\eps)}\left|f\lc\frac{\norm{y_1-x}^2}{\eps^2}\rc-f\lc\frac{\norm{y_2-x}^2}{\eps^2}\rc\right|\\
    \leq &\frac{L}{\eps^2\,C_m(\eps)}\left|\norm{y_1-x}^2-\norm{y_2-x}^2\right|\\
    =&\frac{L}{\eps^2\,C_m(\eps)}\norm{y_1-y_2}\norm{2x-y_1-y_2}\\
    \leq &\frac{2\,L\,D}{\eps^2\,C_m(\eps)}\norm{y_1-y_2}.
\end{align*}
So $K(x,\cdot,\eps)$ is a $\frac{2\,LD}{\eps^2\,C_m(\eps)}$-Lipschitz function on $X$. Then, by the Kantorovich duality (see for example  Remark 6.5 in \cite{villani2008optimal}), we have
\begin{align*}
    |M_\alpha-M_\beta|&=\left|\int_XK(x,y,\eps)\alpha(dy)-\int_XK(x,y,\eps)\beta(dy)\right|\\
    &\leq \frac{2\,LD}{\eps^2\,C_m(\eps)}\,d_{\mathrm{W},1}(\alpha,\beta).
\end{align*}
As for $\frac{1}{M_\beta}\norm{\tilde{\Sigma}_\beta}_F$, we know from Equation (\ref{eq:til-cov-smooth}) and Remark \ref{rmk:x-replace} that
\begin{align*}
    \frac{\norm{\tilde{\Sigma}_\beta}_F}{M_\beta}&\leq \frac{\int_{X}\norm{\lc{y}-x\rc^{\otimes^2}}_F K(x,y,\eps)\,\beta(d {y})}{\int_XK(x,y,\eps)\,\beta(dy)}\\
    &\leq D^2.
\end{align*}
Thus, $T_2\leq \frac{2\,L\,D^3}{\eps^2\,C_m(\eps)\,W}\,d_{\mathrm{W},1}(\alpha,\beta).$

As for $T_3$, we first have by Lemma \ref{lm:ms-smooth} that
\[\norm{\mu_\alpha-\mu_\beta}\leq\frac{2\,CD+V}{U}d_{\mathrm{W},1}(\alpha,\beta). \]
Then, by Equation (\ref{eq:smooth-mean}) and Remark \ref{rmk:x-replace}, we have that
\begin{align*}
    \norm{\mu_\alpha-x}&=\norm{\frac{\int_{X}(y-x)\,K(x,y,\eps)\,\alpha(dy)}{\int_{X}K(x,y,\eps)\,\alpha(dy)}}\\
    &\leq \frac{\int_{X}\norm{y-x}\,K(x,y,\eps)\,\alpha(dy)}{\int_{X}K(x,y,\eps)\,\alpha(dy)}\leq D.
\end{align*}
Hence, $T_3\leq \frac{2D(2CD+V)}{U}\,d_{\mathrm{W},1}(\alpha,\beta).$

By adding up the three inequalities regarding upper bounds of $T_1,T_2$ and $T_3$, we conclude the proof.
\end{proof}

Then, based on Lemma \ref{lm:key-smooth-stb}, with a similar proof as the one for Theorem \ref{thm:GT-stable}, we obtain the stability theorem (Theorem \ref{thm:gt-stable-smooth})
\subsection{Comparing WT and GT on line segments (proof of Proposition \ref{prop:line})}\label{app:proof-prop-line}

\begin{figure}[htb]
    \centering
    \includegraphics[width=0.5\linewidth]{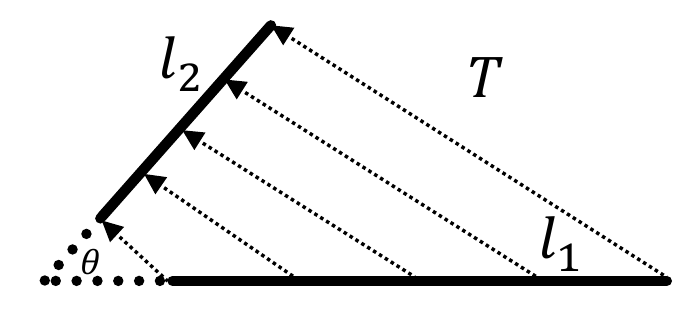} 
    \caption{Illustration of the linear map $T$ from $l_1$ to $l_2$. 
    }
    \label{fig:GT}
\end{figure}

\begin{proof}
Since $x_i$ is the mean of $\alpha_i$ for $i=1,2$, the leftmost inequality in the statement of the proposition follows directly from Remark \ref{rmk:bound}.

We now compute $\dw(\gamma_1,\gamma_2)$ explicitly. Without loss of generality, we assume $l_1$ is parametrized by $l_1(t)=(t\cdot s_1,0)$ and $l_2$ is parametrized by $l_2(t)=(a_0+t\cdot s_2\cos{\theta},b_0+t\cdot s_2\sin{\theta})$ for $t\in[0,1]$. Then, $x_1=\left(\frac{s_1}{2},0\right)$, $x_2=\left(a_0+\frac{s_2\cos{\theta}}{2},b_0+\frac{s_2\sin{\theta}}{2}\right)$, $\Sigma_1=\begin{pmatrix}\frac{s_1^2}{12}&0\\0&0\end{pmatrix}$ and $\Sigma_2=\begin{pmatrix}\frac{s_2^2}{12}(\cos{\theta})^2&\frac{s_2^2}{12}\sin{\theta}\cos{\theta}\\\frac{s_2^2}{12}\sin{\theta}\cos{\theta}&\frac{s_2^2}{12}(\sin{\theta})^2\end{pmatrix}$.
Then, by definition of $\dcov$ in Section \ref{sec:bg}, it is easy to check that 
\[\dcov^2(\Sigma_1,\Sigma_2)=\frac{s_1^2}{12}+\frac{s_2^2}{12}-\frac{s_1s_2}{6}\cos{\theta}. \]
Therefore, 
\begin{align*}
    \dw^2(\gamma_1,\gamma_2)=&\norm{x_1-x_2}^2+\dcov^2(\Sigma_1,\Sigma_2)\\
    =&\frac{s_1^2}{3}+\frac{s_2^2}{3}-\frac{2s_1s_2\cos{\theta}}{3}\\
    +&a_0^2+b_0^2+a_0s_2\cos{\theta}+b_0s_2\sin{\theta}-a_0s_1.
\end{align*}

Next we compute $\dw(\alpha_1,\alpha_2)$. Consider a linear map $T:l_1\rightarrow l_2$ defined by taking $ l_1(t)$ to $ l_2(t)$. See Figure \ref{fig:GT} for an illustration. Then, it is easy to check that $T_\#\alpha_1=\alpha_2$. This gives rise to a transport plan $\pi=(\mathrm{Id}\times T)_\#\alpha_1\in\Pi(\alpha_1,\alpha_2)$, where $\mathrm{Id}: l_1\rightarrow l_1$ is the identity map on $ l_1$. Then, 
\begin{align*}
   &\dw^2(\alpha_1,\alpha_2)\leq\int_{(x,x')\in l_1\times l_2}\norm{x-x'}^2d\pi(x\times x')\\
   &=\int_{x\in l_1}\norm{x-T(x)}^2d\alpha_1(x)\\
   &=\int_0^1 \left|a_0+ts_2\cos{\theta}-ts_1\right|^2+\left|b_0+ts_2\sin{\theta}\right|^2dt\\
   &=\frac{s_1^2}{3}+\frac{s_2^2}{3}-\frac{2s_1s_2\cos{\theta}}{3}\\
    &+a_0^2+b_0^2+a_0s_2\cos{\theta}+b_0s_2\sin{\theta}-a_0s_1\\
    &=\dw^2(\gamma_1,\gamma_2).
\end{align*}
By Remark \ref{rmk:bound}, we know $\dw(\alpha_1,\alpha_2)\geq\dw(\gamma_1,\gamma_2)$ and thus $\dw(\alpha_1,\alpha_2)=\dw(\gamma_1,\gamma_2)$
\end{proof}

\subsection{Anisotropic neighborhood (proof of Theorem \ref{thm:ellipsoid})}\label{app:proof-ellip}
\begin{proof}
When the dimension $m=1$, we have
\[\mu^{\mathsmaller{(\eps)}}_{{\alpha,d_X}}(x_0)=\frac{\int_{-\eps}^\eps(x_0+z)f(x_0+z)dz}{\int_{-\eps}^\eps f(x_0+z)dz}, \]
and 
\[\Sigma_\alpha^{\mathsmaller{(\eps)}}(x_0)=\frac{\int_{-\eps}^{\eps}\left(z+x_0-\mu^{\mathsmaller{(\eps)}}_{{\alpha,d_X}}(x_0)\right)^2f(z+x_0)dz}{\int_{-\eps}^\eps f(x_0+z)dz}. \]
By replacing $f(x_0+z)$ with its Taylor expansion around $x_0$, we obtain
\begin{align*}
  \Sigma_\alpha^{\mathsmaller{(\eps)}}(x_0)&=\frac{1}{3}\eps^2+\frac{-5(f'(x_0))^2+2f(x_0)f''(x_0)}{45f(x_0)^2}\eps^4+O(\eps^6)\\
  &=\frac{\eps^2}{3}(1+h(x_0)\eps^2+O(\eps^4)),
\end{align*}
where $h(x_0)=\frac{-5(f'(x_0))^2+2f(x_0)f''(x_0)}{9f(x_0)^2}$. Since $d=1$, for any $x_1\in B_\eps(x_0)$, we have
\begin{align}
    &\dcov^2(\Sigma_\alpha^{\mathsmaller{(\eps)}}(x_0),\Sigma_\alpha^{\mathsmaller{(\eps)}}(x_1))=\left|\sqrt{\Sigma_\alpha^{\mathsmaller{(\eps)}}(x_0)}-\sqrt{\Sigma_\alpha^{\mathsmaller{(\eps)}}(x_1)}\right|^2\label{eq:1d_trick}\\
    &=\lc\frac{h(x_0)-h(x_1)}{2\sqrt{3}}\eps^3+O(\eps^5)\rc^2\\
    &=\frac{(h'(x_0))^2}{12}\eps^6\norm{x_0-x_1}^2+O(\eps^9),
\end{align}
where the first equality holds since $m=1$ and the last equality follows from the Taylor expansion of $h$ at $x_0$ and $\norm{x_0-x_1}\leq\eps$.

So, if $x_1\in B_\eps^{\lambda,\alpha}(x_0)$, we have
\begin{align}
    \eps^2&\geq \norm{x_0-x_1}^2+\eps^{-6}\dcov^2(\Sigma_\alpha^{\mathsmaller{(\eps)}}(x_0),\Sigma_\alpha^{\mathsmaller{(\eps)}}(x_1))\\
    &=\lc1+\frac{(h'(x_0))^2}{12}\rc\norm{x_0-x_1}^2+O(\eps^3),\label{eq:1d-ell}
\end{align}
Let $x_1=x_0+a\eps$, then we have $|a|\leq\sqrt{\frac{12}{12+(h'(x_0))^2}}=:a_0$ by discarding the higher order term. This implies that $B_\eps^{\lambda,\alpha}(x_0)$ is approximately a Euclidean ball $B_{a_0\eps}(x_0)$. More precisely, consider any decreasing sequence $\{\eps_n\}_{n=1}^\infty$ approaching $0$ and $a\in\R$ such that $|a|< a_0$. Define $x_n\coloneqq x_0+a\eps_n$. Then, when $n$ is large enough, we have
\[|a|^2+O(\eps_n)\leq a_0^2. \]
This implies by inequality (\ref{eq:1d-ell}) that $x_n\in B_{\eps_n}^{\lambda,\alpha}(x_0)$ when $n$ is large enough, which shows $x_0+a\in\frac{1}{\eps_n}B_{\eps_n}^{\lambda,\alpha}(x_0)$. Hence, $(B_{a_0}(x_0))^\circ\subset\limsup_{n\rightarrow\infty}\frac{1}{\eps_n}B_{\eps_n}^{\lambda,\alpha}(x_0)$. Thus, $B_{a_0}(x_0)\subset\overline{\limsup_{n\rightarrow\infty}\frac{1}{\eps_n}B_{\eps_n}^{\lambda,\alpha}(x_0)}$. Conversely, suppose $x\in \limsup_{n\rightarrow\infty}\frac{1}{\eps_n}B_{\eps_n}^{\lambda,\alpha}(x_0)$, then for $n$ large enough, there exists $x_n\in B_{\eps_n}^{\lambda,\alpha}(x_0) $ such that $x = x_0 +\frac{1}{\eps_n}(x_n-x_0)$. By inequality (\ref{eq:1d-ell}), we have that $\norm{\frac{1}{\eps_n}(x_n-x_0)}^2+O(\eps_n)\leq a_0^2$. Therefore, \[\norm{x-x_0}=\norm{\frac{1}{\eps_n}(x_n-x_0)}=\lim_{n\rightarrow \infty}\norm{\frac{1}{\eps_n}(x_n-x_0)}\leq a_0.\]
Thus, $x\in B_{a_0}(x_0)$ and $\limsup_{n\rightarrow\infty}\frac{1}{\eps_n}B_{\eps_n}^{\lambda,\alpha}(x_0)\subset B_{a_0}(x_0)$. Thus $\overline{\limsup_{n\rightarrow\infty}\frac{1}{\eps_n}B_{\eps_n}^{\lambda,\alpha}(x_0)}= B_{a_0}(x_0)$. Since the sequence $\{\eps_n\}_{n=1}^\infty$ is arbitrary, we conclude that $\overline{\limsup_{\eps\rightarrow 0}\frac{1}{\eps}B_{\eps}^{\lambda,\alpha}(x_0)}= B_{a_0}(x_0)$.

When $m>1$, there is no formula analogous to Equation (\ref{eq:1d_trick}) that helps simplify the computation of the Taylor expansion of $\dcov$, yet through a direct and tedious calculation, we are able to compute the Taylor expansion of $\dcov$ around $\mathbf{x_0}$ and show that there exists an $m$-dimensional PSD matrix function $H(\mathbf{x_0})$ (which boils down to $\frac{(h'(x_0))^2}{12}$ when $m=1$) depending only on $f$ such that for $\mathbf{x_1}\in B_\eps^{\lambda,\alpha}(\mathbf{x_0})$ we have
\[(\mathbf{x_1}-\mathbf{x_0})^\mathrm{T}(I_m+H(\mathbf{x_0}))(\mathbf{x_1}-\mathbf{x_0})+O(\eps^3)\leq\eps^2, \]
where $I_m$ is the $m$-dimensional identity matrix. Write again $\mathbf{x_1}=\mathbf{x_0}+\mathbf{a}\eps$ for some vector $\mathbf{a}\in\R^m$. By discarding the higher order term, we have 
\[\mathbf{a}^\mathrm{T}(I_m+H(\mathbf{x_0}))\mathbf{a}\leq 1.\]
A similar argument as in the 1-dimensional case indicates that $\overline{\limsup_{\eps\rightarrow 0}\frac{1}{\eps}B_{\eps}^{\lambda,\alpha}(\mathbf{x}_0)}= \{\mathbf{x}_0+\mathbf{a}:\,\mathbf{a}^\mathrm{T}(I_m+H(\mathbf{x_0}))\,\mathbf{a}\leq 1\}$, which is an ellipsoid centered at $\mathbf{x}_0$.
\end{proof}

\subsection{A new trace formula (proof of Theorem \ref{thm:redux})}\label{app:proof-thm-redux}

\begin{proof}[Proof of Theorem \ref{thm:redux}]
The main idea is to prove that $AB$ and $A^\frac{1}{2}BA^\frac{1}{2}$ share the same spectrum. In the course of proving the theorem, we found a discussion website \cite{eigen} where user Ahmad Bazzi proved the fact for another purpose. In the following, we present our original proof which is different from the one by Ahmad Bazzi.

The case when one of the matrices is invertible is trivial and we found it mentioned in \cite{bhatia2019matrix}. Without loss of generality assume $A$ is invertible. Then, we have $AB= A^{\frac{1}{2}}\cdot A^\frac{1}{2}BA^\frac{1}{2}\cdot A^{-\frac{1}{2}},$ which implies that $AB$ and $A^\frac{1}{2}BA^\frac{1}{2}$ are similar to each other and thus they share the same spectrum. Thus, the sum of square root of eigenvalues of $AB$ counted with multiplicity is the same as the sum of square root of eigenvalues of $A^\frac{1}{2}BA^\frac{1}{2}$ counted with multiplicity, which is exactly $\mathrm{tr}\lc\left(A^\frac{1}{2}BA^\frac{1}{2}\right)^\frac{1}{2}\rc$.

Now suppose $A$ is singular. If $B$ is invertible, then similarly we have that $B^\frac{1}{2}AB^\frac{1}{2}$ and $AB$ are similar and everything else follows from the fact that $\mathrm{tr}\lc\left(A^\frac{1}{2}BA^\frac{1}{2}\right)^\frac{1}{2}\rc=\mathrm{tr}\lc\left(B^\frac{1}{2}AB^\frac{1}{2}\right)^\frac{1}{2}\rc$.

If $B$ is also singular, let $B_t=tI+B$ where $t\geq 0$ and $I$ is the $n$-dimensional identity matrix. $B_t$ is then positive definite and thus invertible when $t>0$. Then, by previous analysis, $A^\frac{1}{2}B_tA^\frac{1}{2}$ and $AB_t$ share the same spectrum for all $t>0$. Since $A^\frac{1}{2}B_tA^\frac{1}{2}=tA+A^\frac{1}{2}BA^\frac{1}{2}$ and $AB_t=tA+AB$, by continuity of eigenvalues, we conclude that $A^\frac{1}{2}BA^\frac{1}{2}$ and $AB$ share the same spectrum by letting $t$ going to 0.
Therefore, $\mathrm{tr}\lc\left(A^\frac{1}{2}BA^\frac{1}{2}\right)^\frac{1}{2}\rc=\mathrm{tr}\lc(AB)^\frac{1}{2}\rc. $ 
\end{proof}

%% file: supp_experiments.tex
\section{Details about implementations}
\label{app:experiments}

\subsection{T-junction clustering}
In this experiment, we compare the clustering results for the first 2 iterations based on GT with those of MS, WT2 and WT1 on the T-junction dataset. The results are shown in Figure~(\ref{fig:tlines-supp}). 
\begin{figure}[htb!]
    \centering
    \subfloat[2D and 3D MDS at $\tau=1$]{
    \label{fig:tline-mds-1-supp}
        \fbox{
            \begin{minipage}[t][2.35cm][t]{0.08\textwidth}
                \centering
                MS \\ 
                \includegraphics[width=0.85\textwidth]{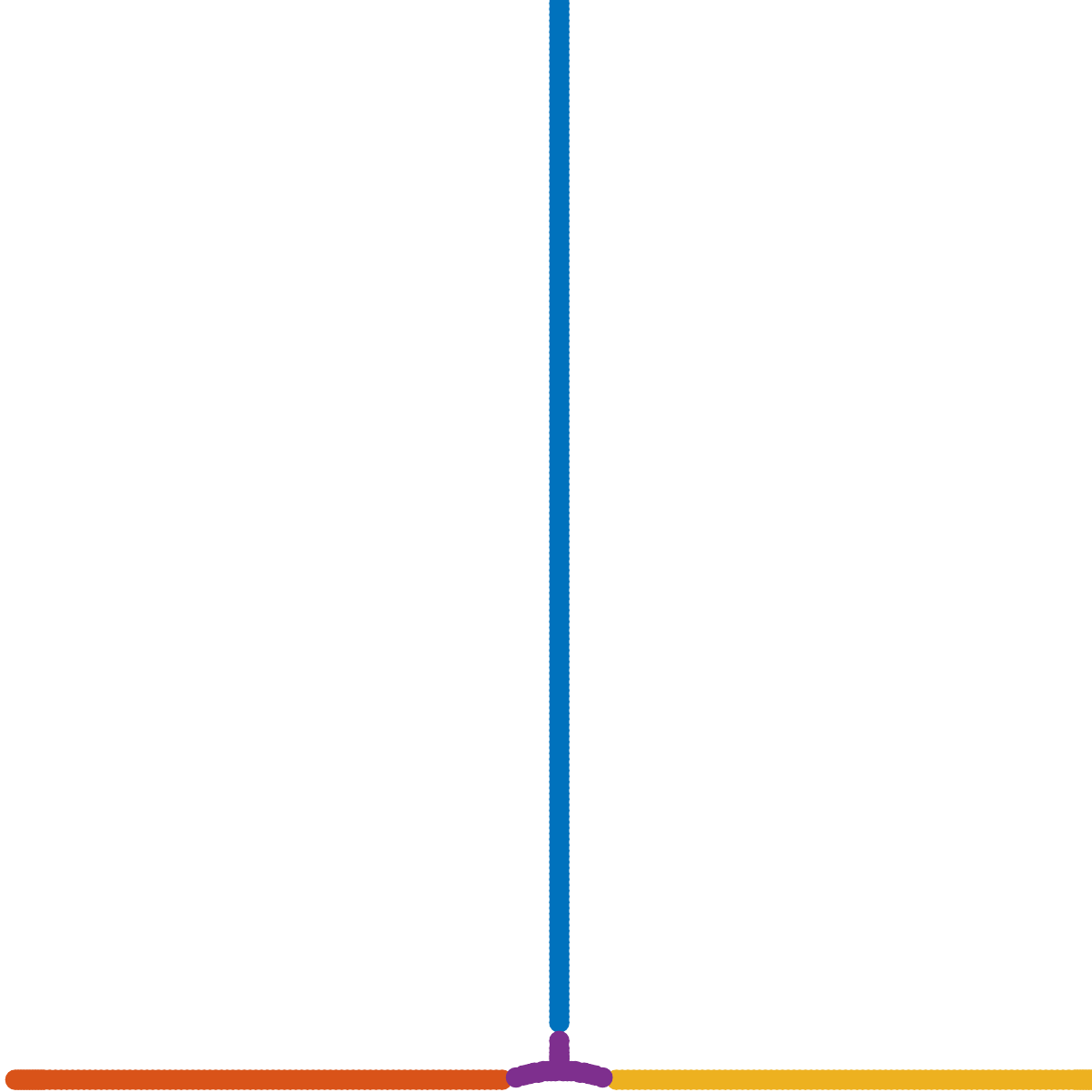} 
                \\ 
            \end{minipage}
            \begin{minipage}[t][2.35cm][t]{0.08\textwidth}
                \centering
                GT-$\lambda$-1 
                \\  \includegraphics[width=1\textwidth]{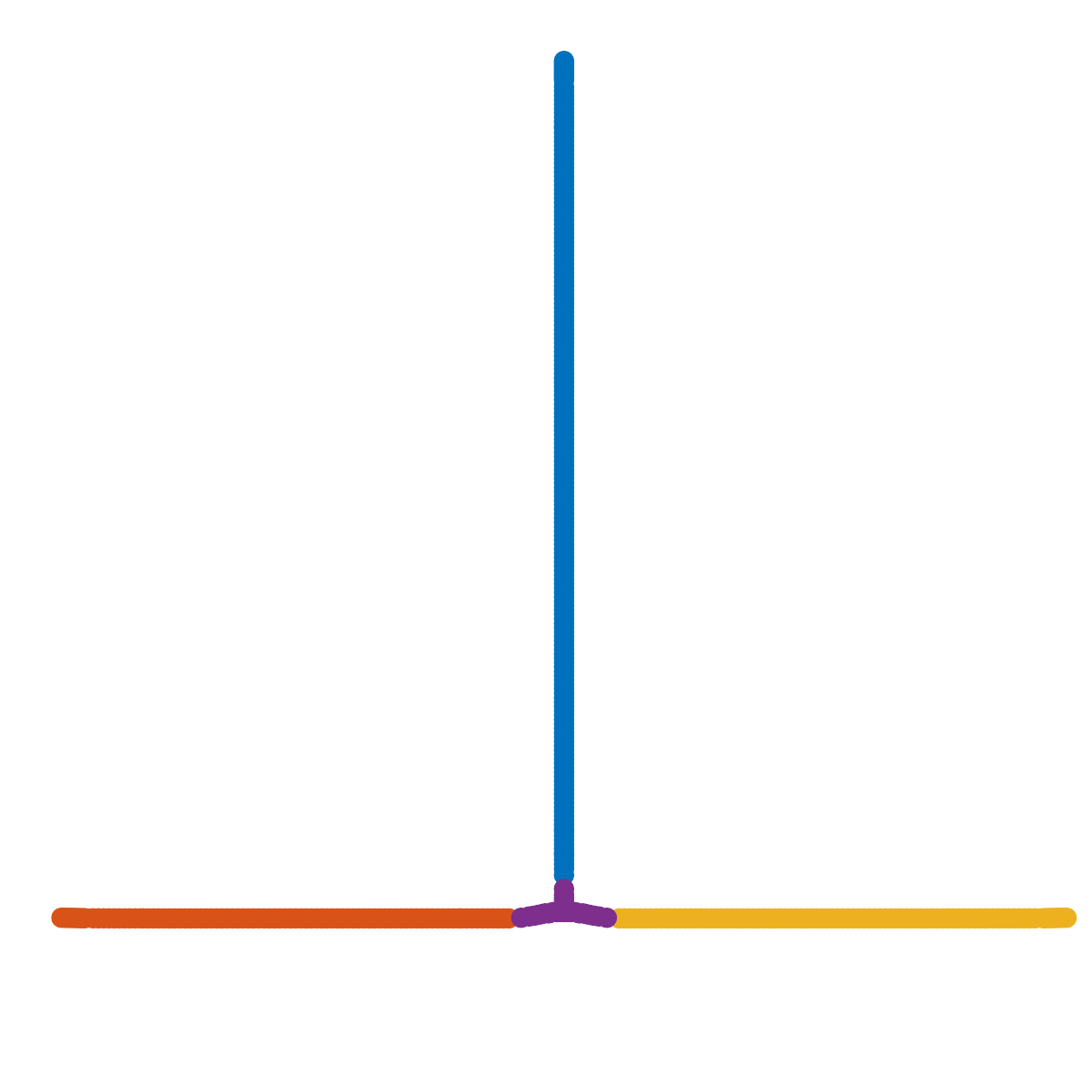}
                \\
               \includegraphics[width=1\textwidth]{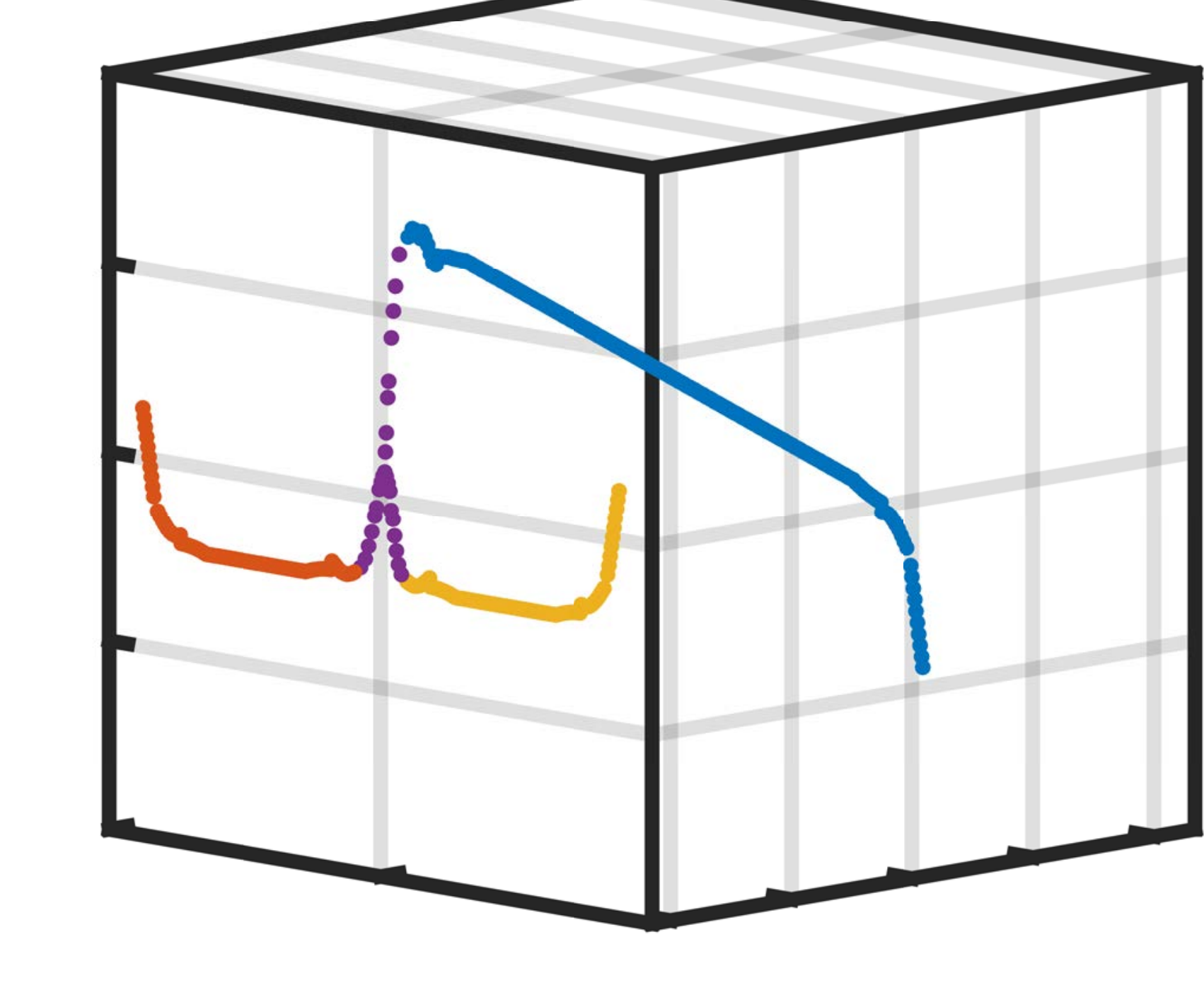} 
            \end{minipage}
            \begin{minipage}[t][2.35cm][t]{0.08\textwidth}
                \centering
                GT-$\lambda$-5 \\ 
                \includegraphics[width=1\textwidth]{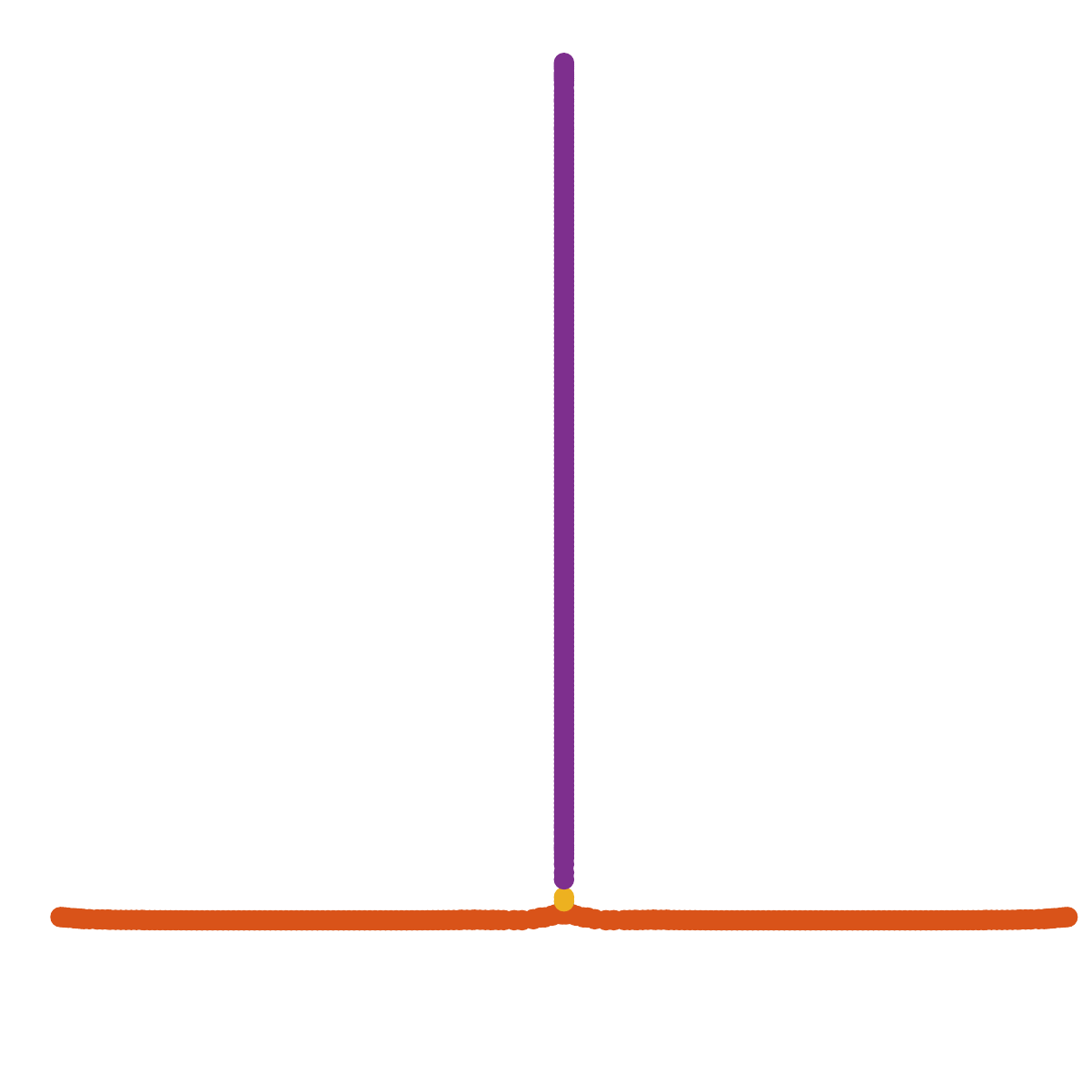}
                \\
               \includegraphics[width=1\textwidth]{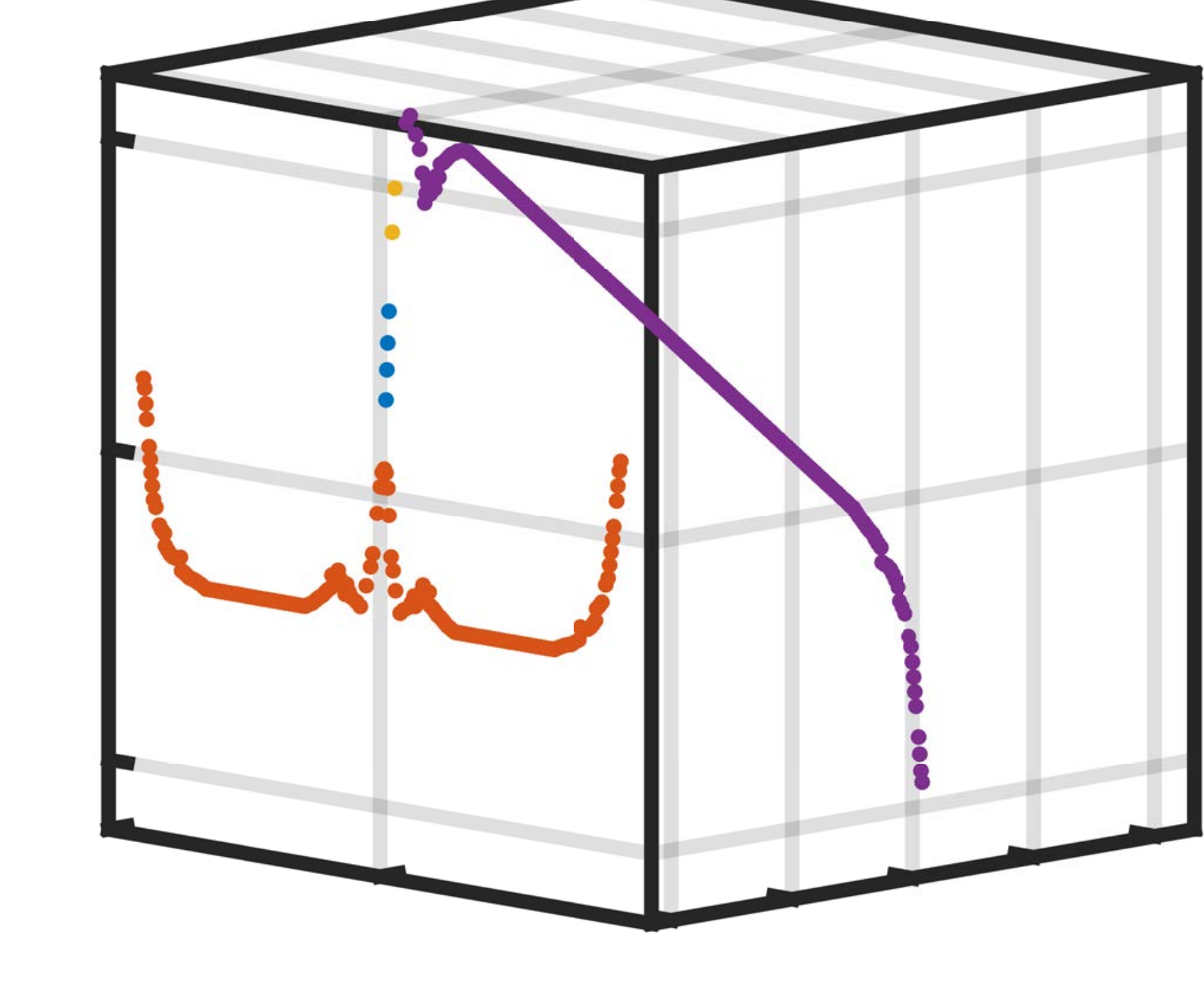} 
            \end{minipage}
            \begin{minipage}[t][2.35cm][t]{0.08\textwidth}
                \centering
                WT2 \\ 
                \includegraphics[width=1\textwidth]{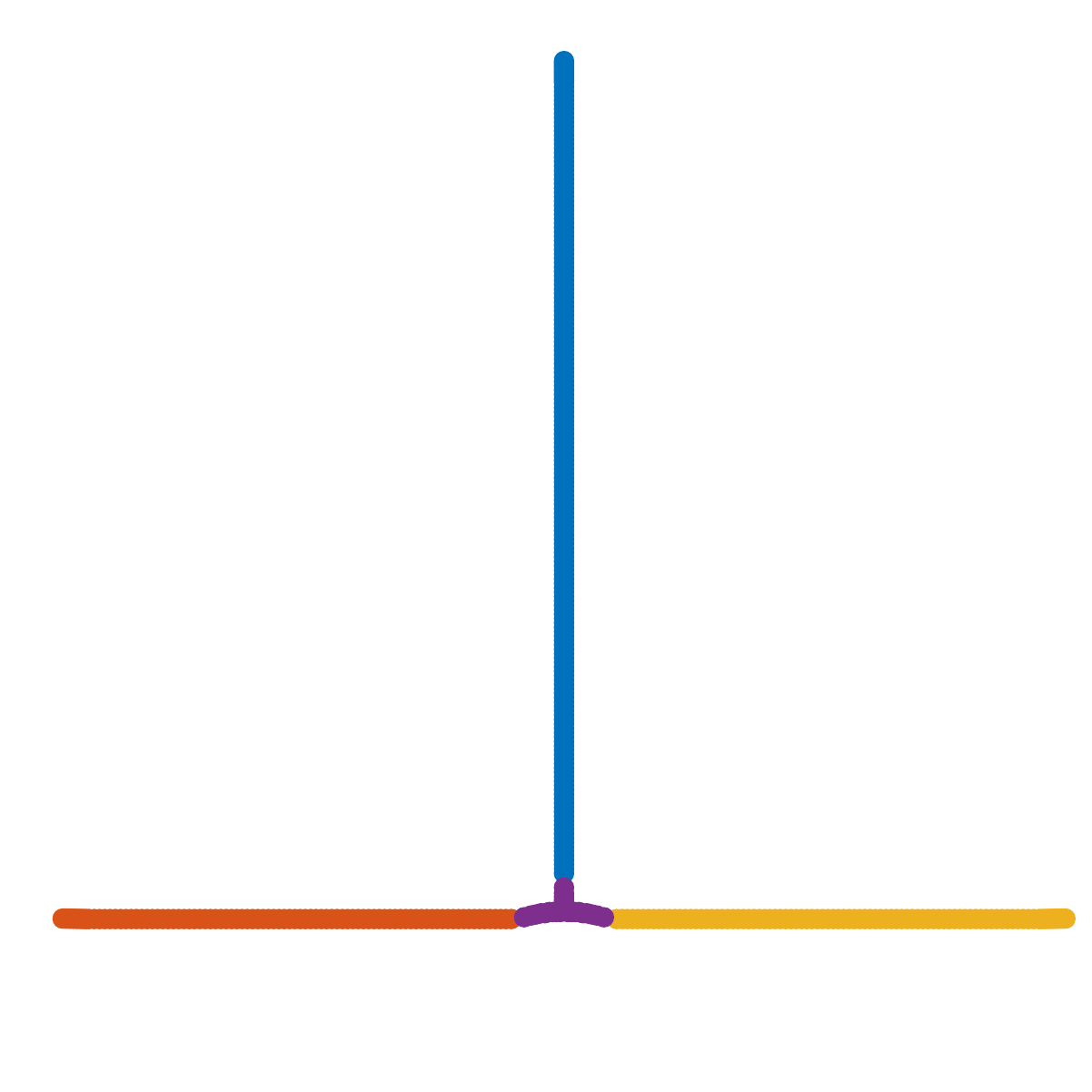}
                \\
               \includegraphics[width=1\textwidth]{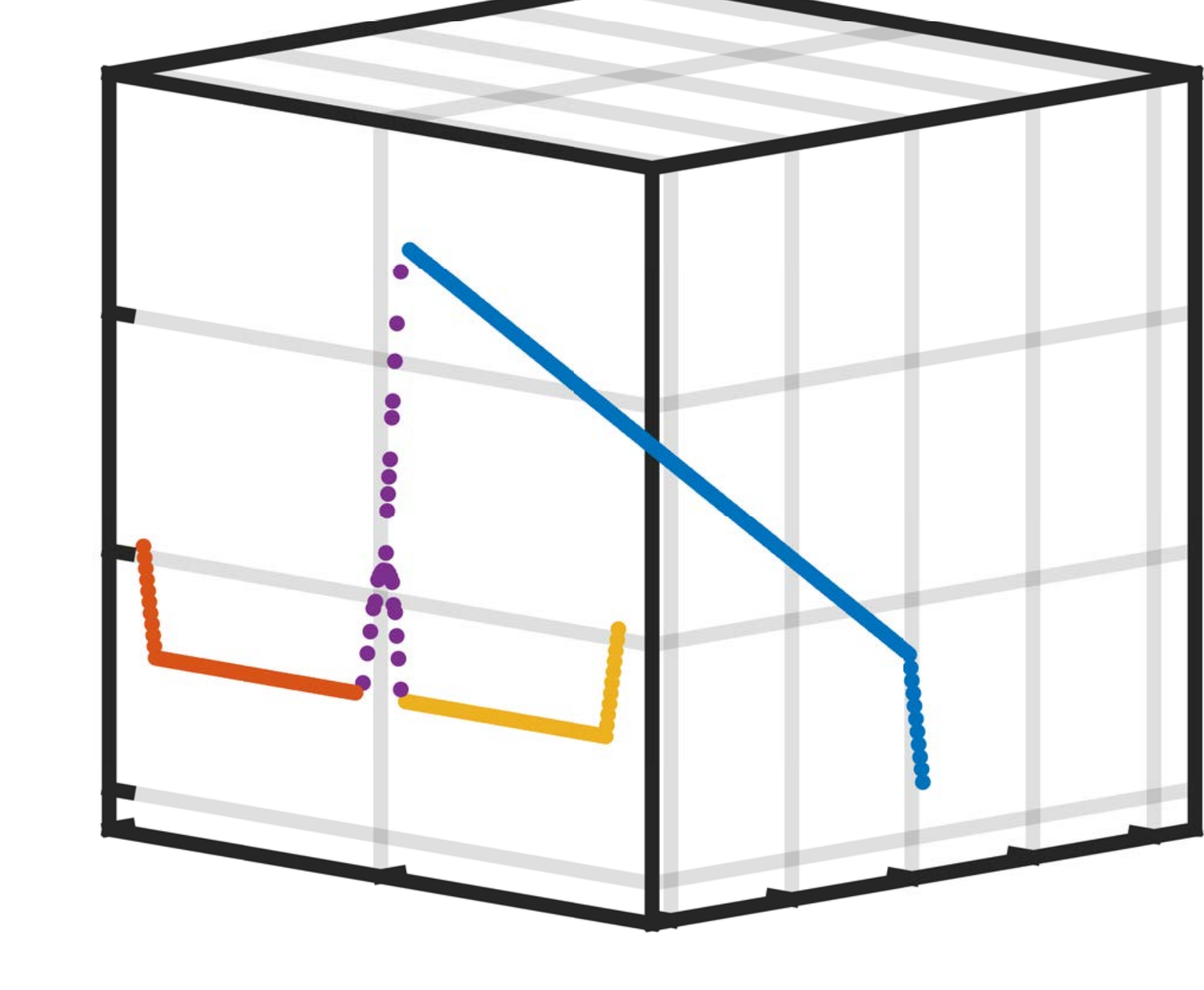} 
            \end{minipage}
            \begin{minipage}[t][2.35cm][t]{0.08\textwidth}
                \centering
                WT1 \\ 
                \includegraphics[width=1\textwidth]{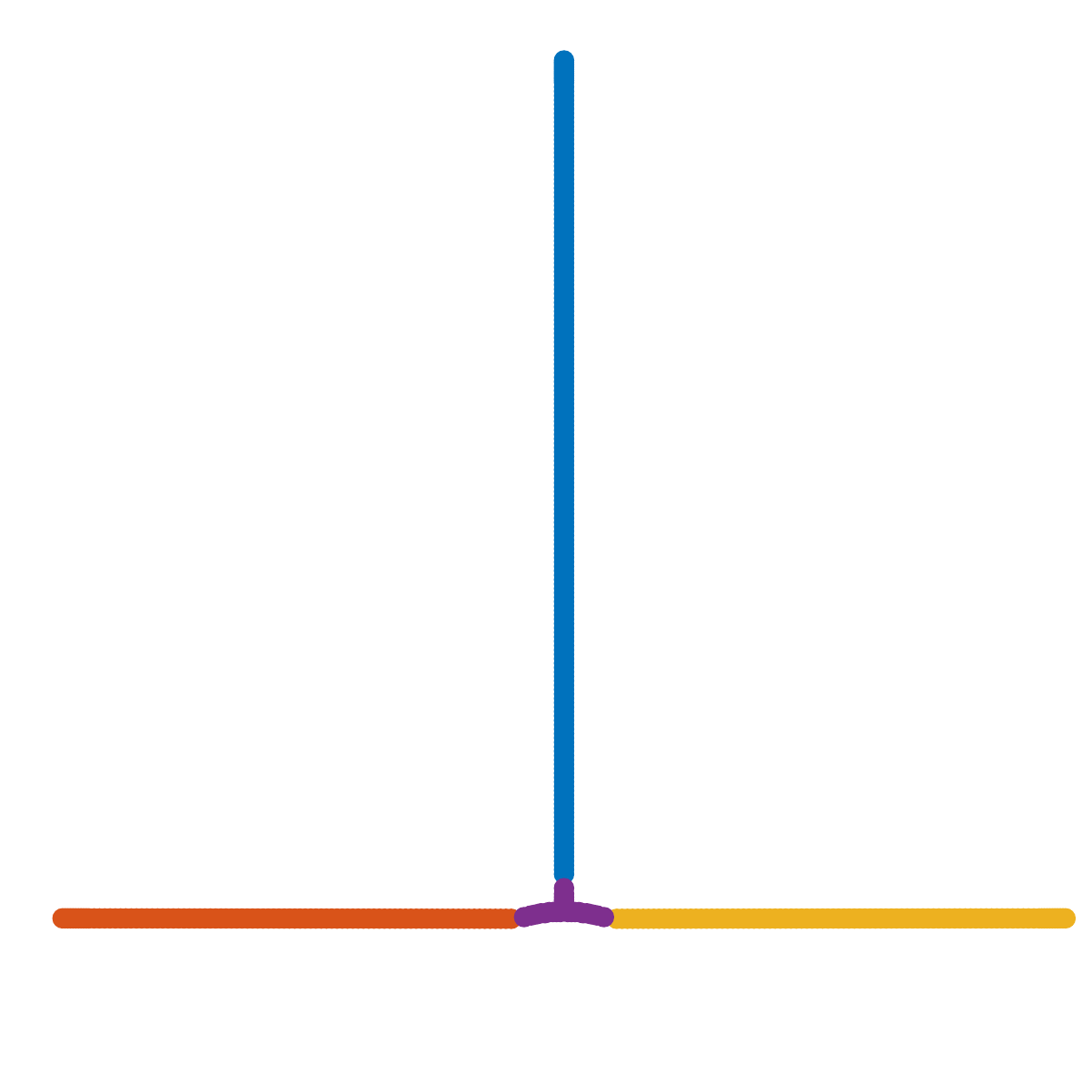}
                \\
               \includegraphics[width=1\textwidth]{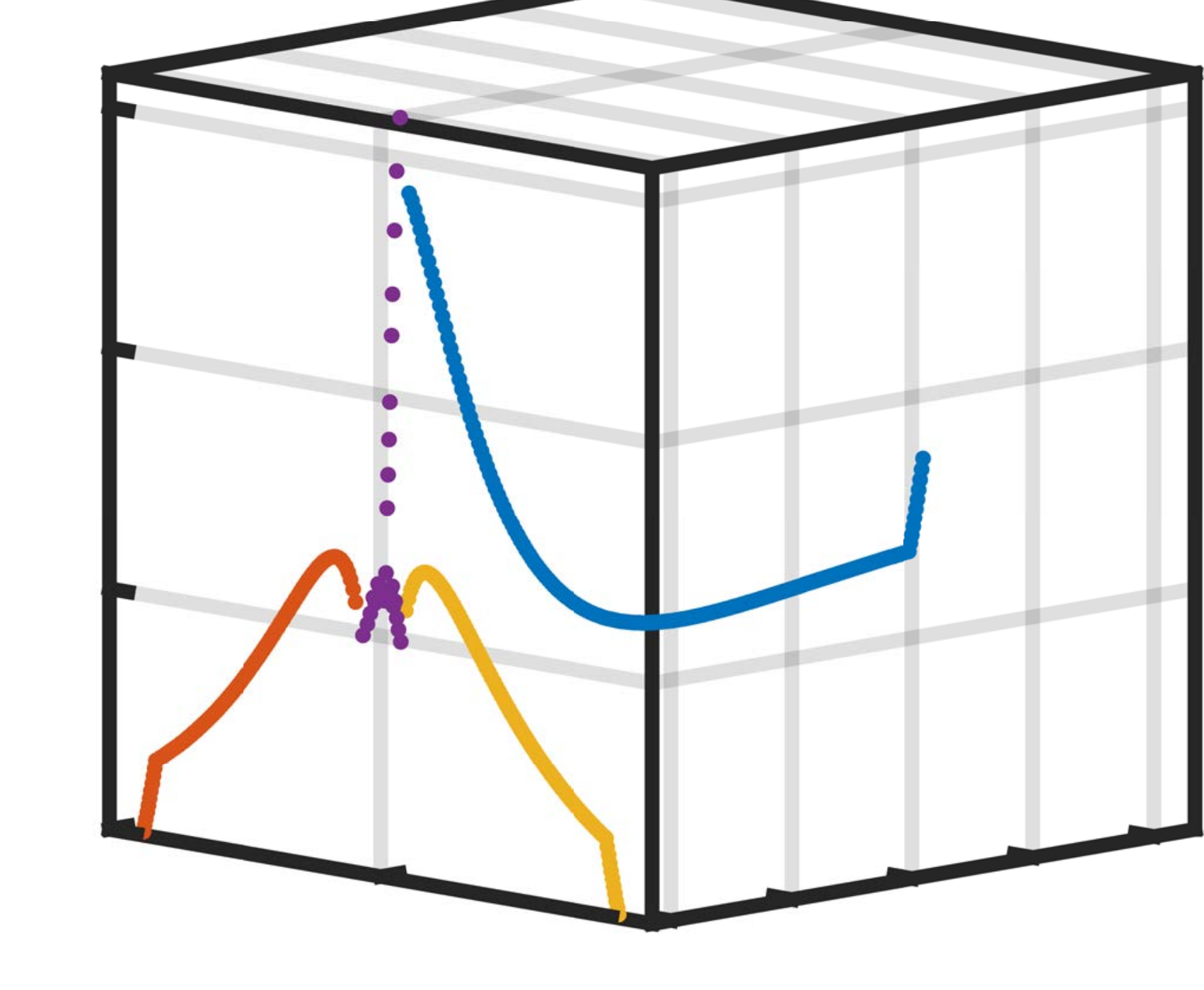} 
            \end{minipage}
        }
    }
    \subfloat[Dendrograms at $\tau=1$]{
    \label{fig:tline-dend-1-supp}
        \fbox{
            \begin{minipage}[t][2.35cm][t]{0.08\textwidth}
                \centering
                MS \\ 
                \includegraphics[width=1\textwidth]{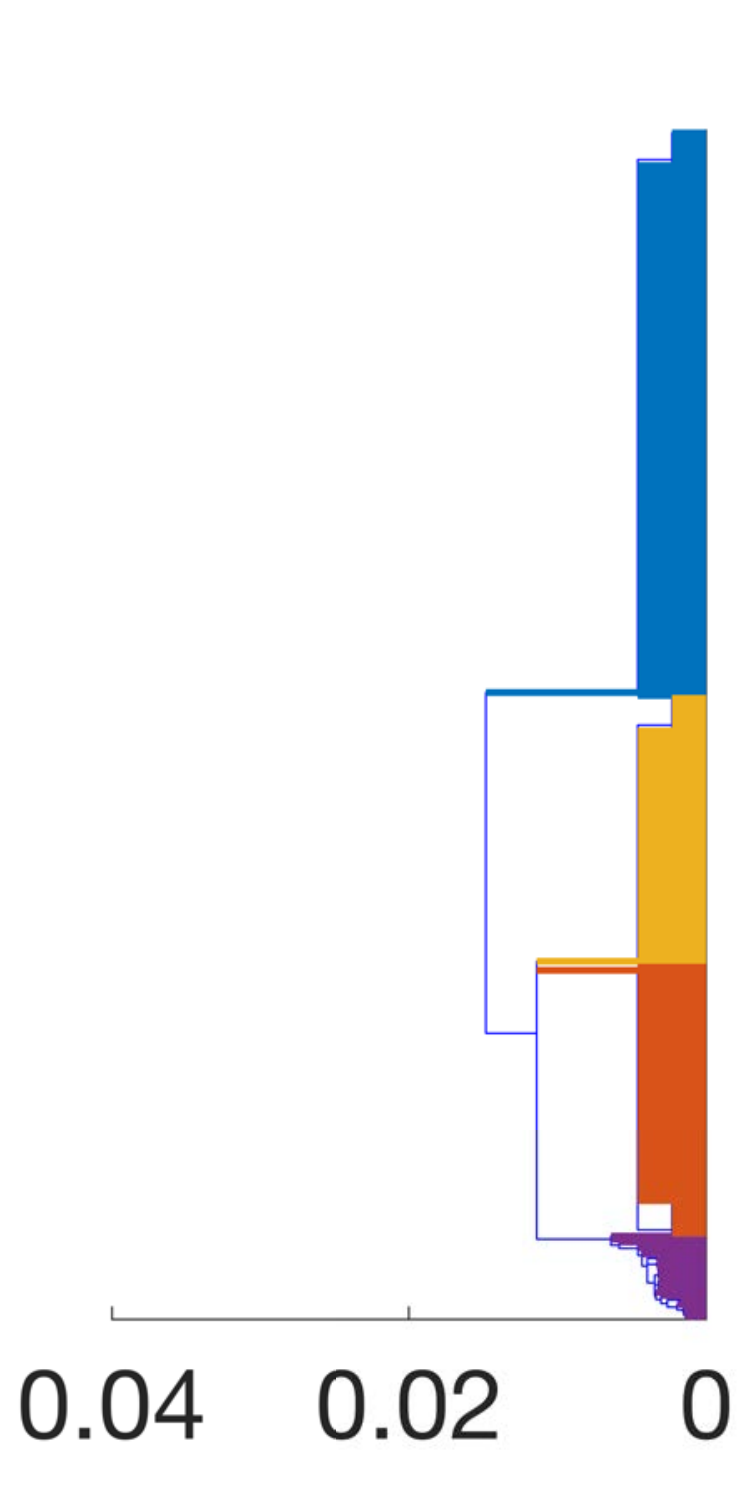}
            \end{minipage}
            \begin{minipage}[t][2.35cm][t]{0.08\textwidth}
                \centering
                GT-$\lambda$-1 \\ 
                \includegraphics[width=1\textwidth]{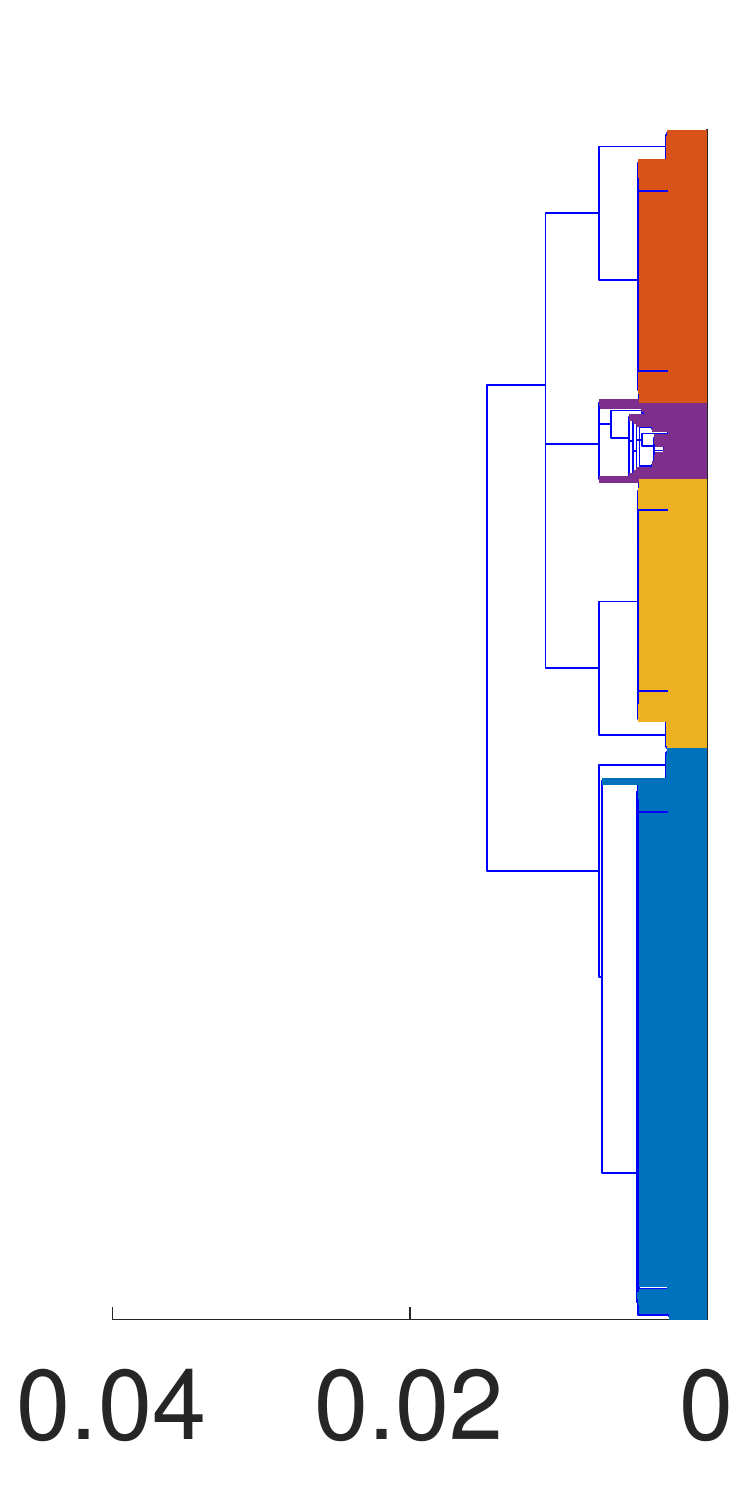}
            \end{minipage}
            \begin{minipage}[t][2.35cm][t]{0.08\textwidth}
                \centering
                GT-$\lambda$-5 \\ 
                \includegraphics[width=1\textwidth]{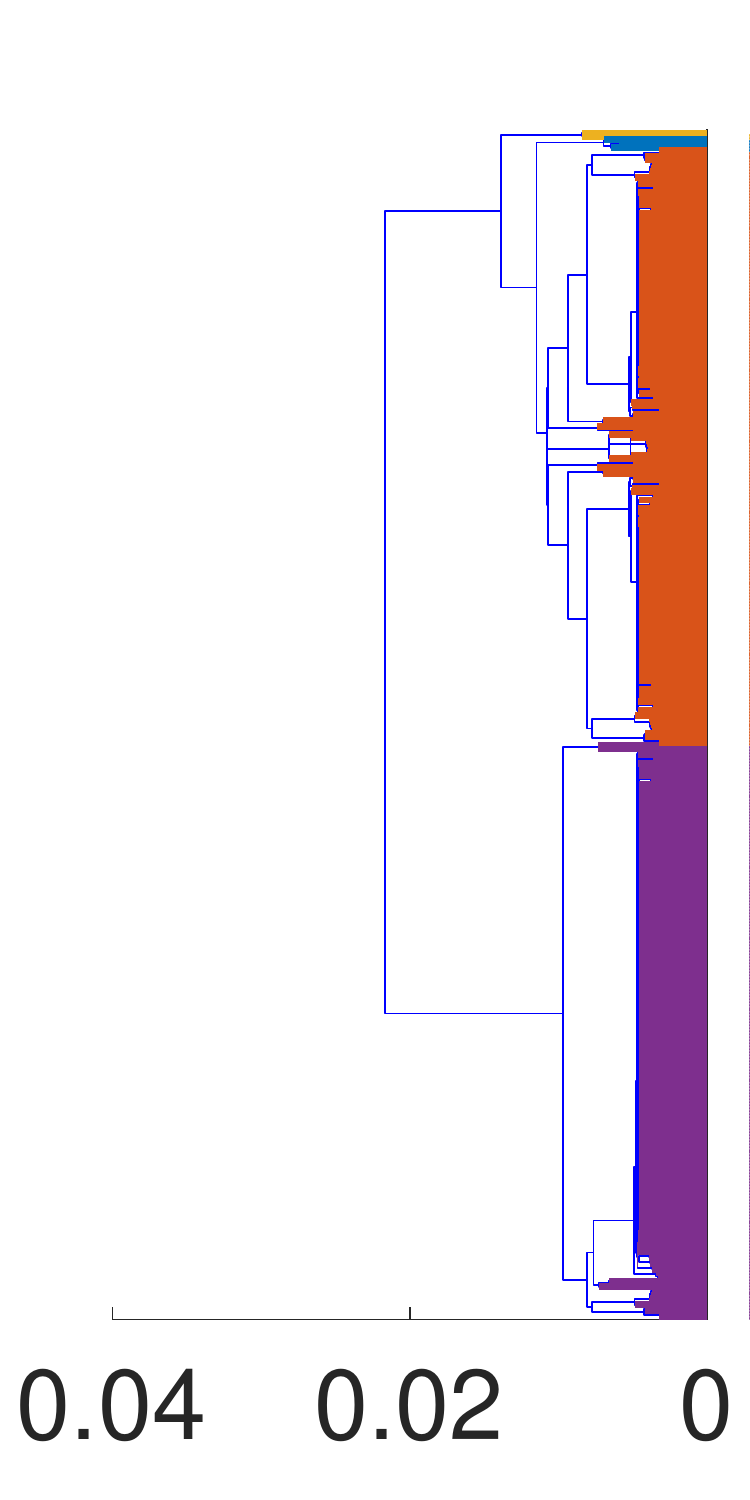}
            \end{minipage}  
            \begin{minipage}[t][2.35cm][t]{0.08\textwidth}
                \centering
                WT2 \\ 
                \includegraphics[width=1\textwidth]{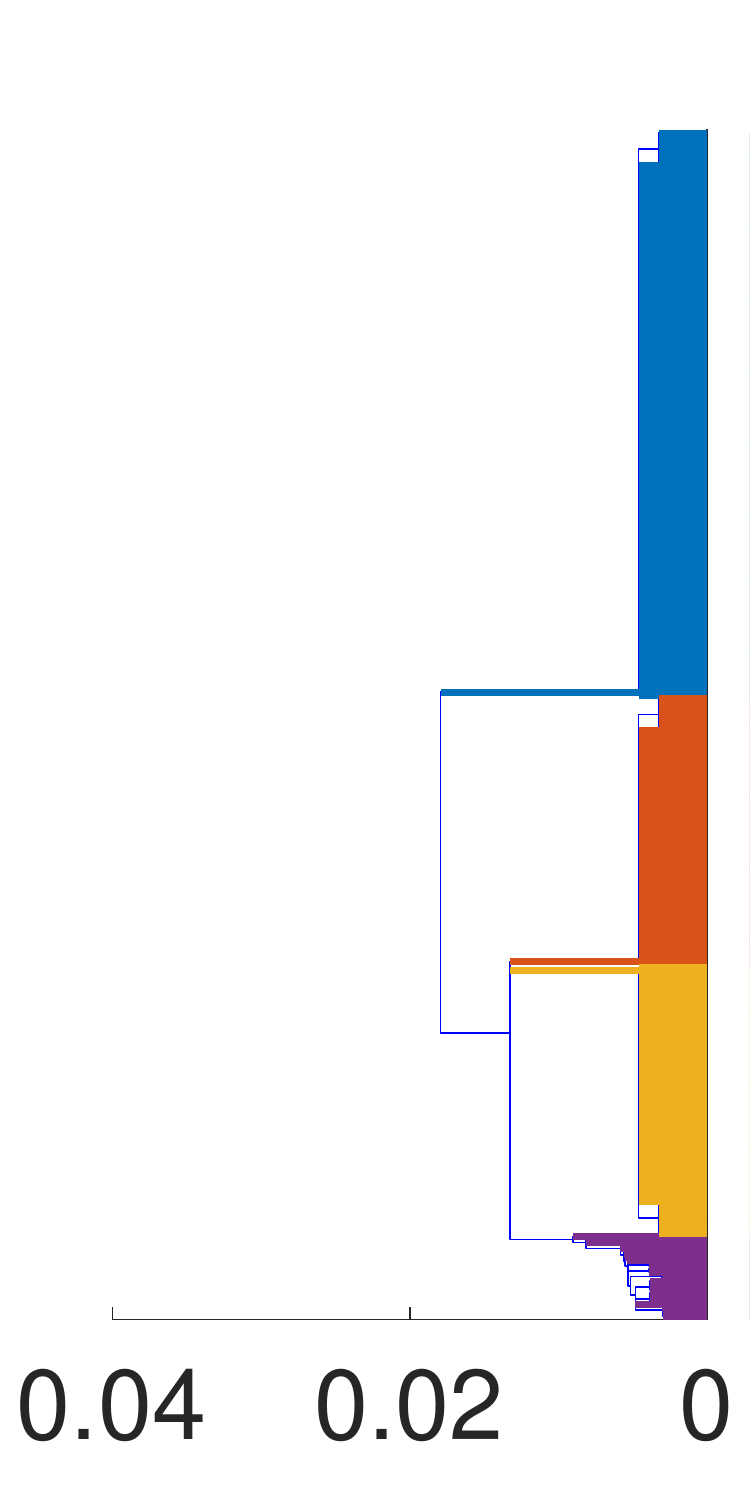}
            \end{minipage}
            \begin{minipage}[t][2.35cm][t]{0.08\textwidth}
                \centering
                WT1 \\ 
                \includegraphics[width=1\textwidth]{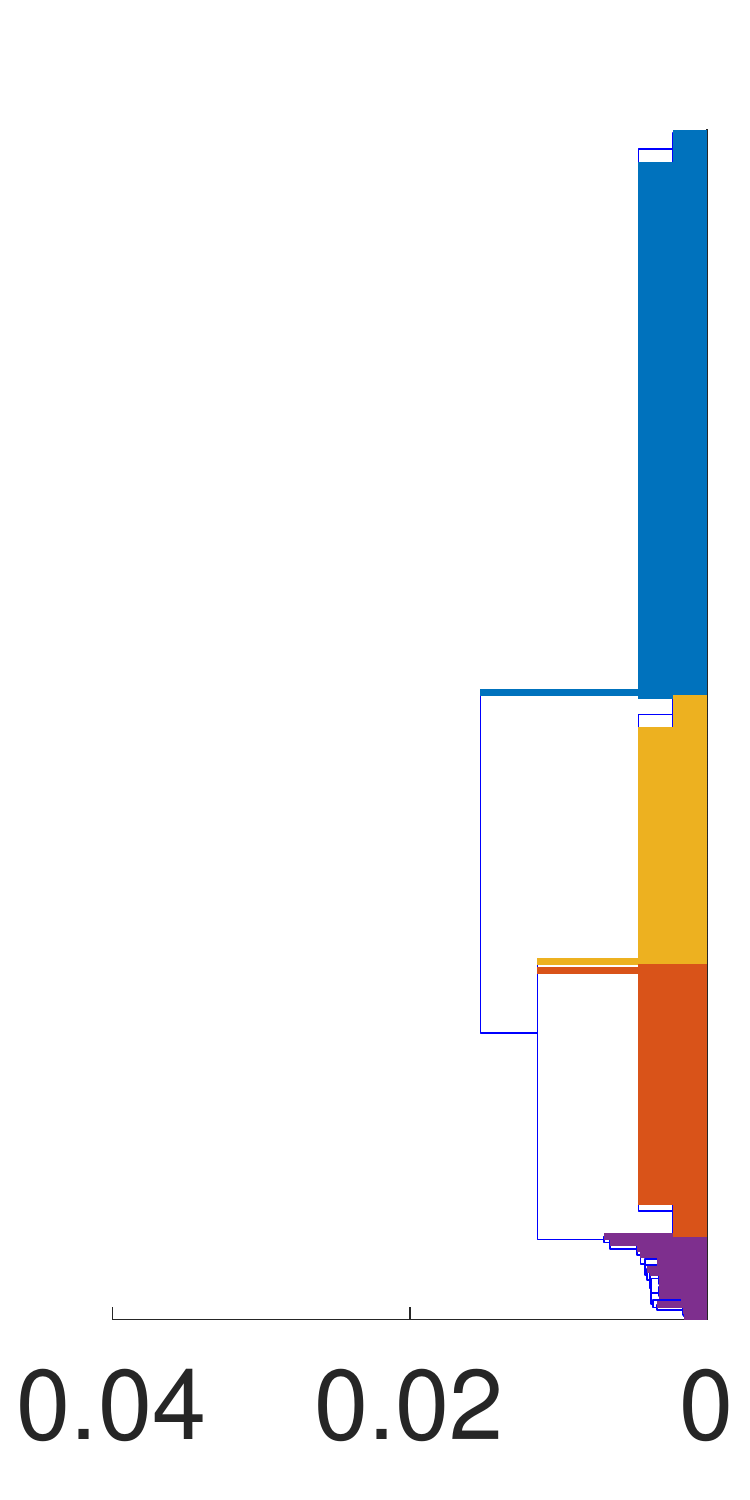}
            \end{minipage}    
        } 
    }

    \subfloat[2D and 3D MDS at $\tau=2$]{
    \label{fig:tline-mds-2-supp}
        \fbox{
            \begin{minipage}[t][2.35cm][t]{0.08\textwidth}
                \centering
                MS \\ 
                \includegraphics[width=0.85\textwidth]{figures/tline/ms-tline-2d-axis-2.pdf} 
                \\ 
            \end{minipage}
            \begin{minipage}[t][2.35cm][t]{0.08\textwidth}
                \centering
                GT-$\lambda$-1 
                \\  \includegraphics[width=1\textwidth]{figures/tline/gtv-tline-lamb-1-2d-axis-2.pdf}
                \\
               \includegraphics[width=1\textwidth]{figures/tline/gtv-tline-lamb-1-3d-axis-2.pdf} 
            \end{minipage}
            \begin{minipage}[t][2.35cm][t]{0.08\textwidth}
                \centering
                GT-$\lambda$-5 \\ 
                \includegraphics[width=1\textwidth]{figures/tline/gtv-tline-lamb-5-2d-axis-2.pdf}
                \\
               \includegraphics[width=1\textwidth]{figures/tline/gtv-tline-lamb-5-3d-axis-2.pdf} 
            \end{minipage}
            \begin{minipage}[t][2.35cm][t]{0.08\textwidth}
                \centering
                WT2 \\ 
                \includegraphics[width=1\textwidth]{figures/tline/wt2-fix-emd-tline-2d-axis-2.pdf}
                \\
               \includegraphics[width=1\textwidth]{figures/tline/wt2-fix-emd-tline-3d-axis-2.pdf} 
            \end{minipage}
            \begin{minipage}[t][2.35cm][t]{0.08\textwidth}
                \centering
                WT1 \\ 
                \includegraphics[width=1\textwidth]{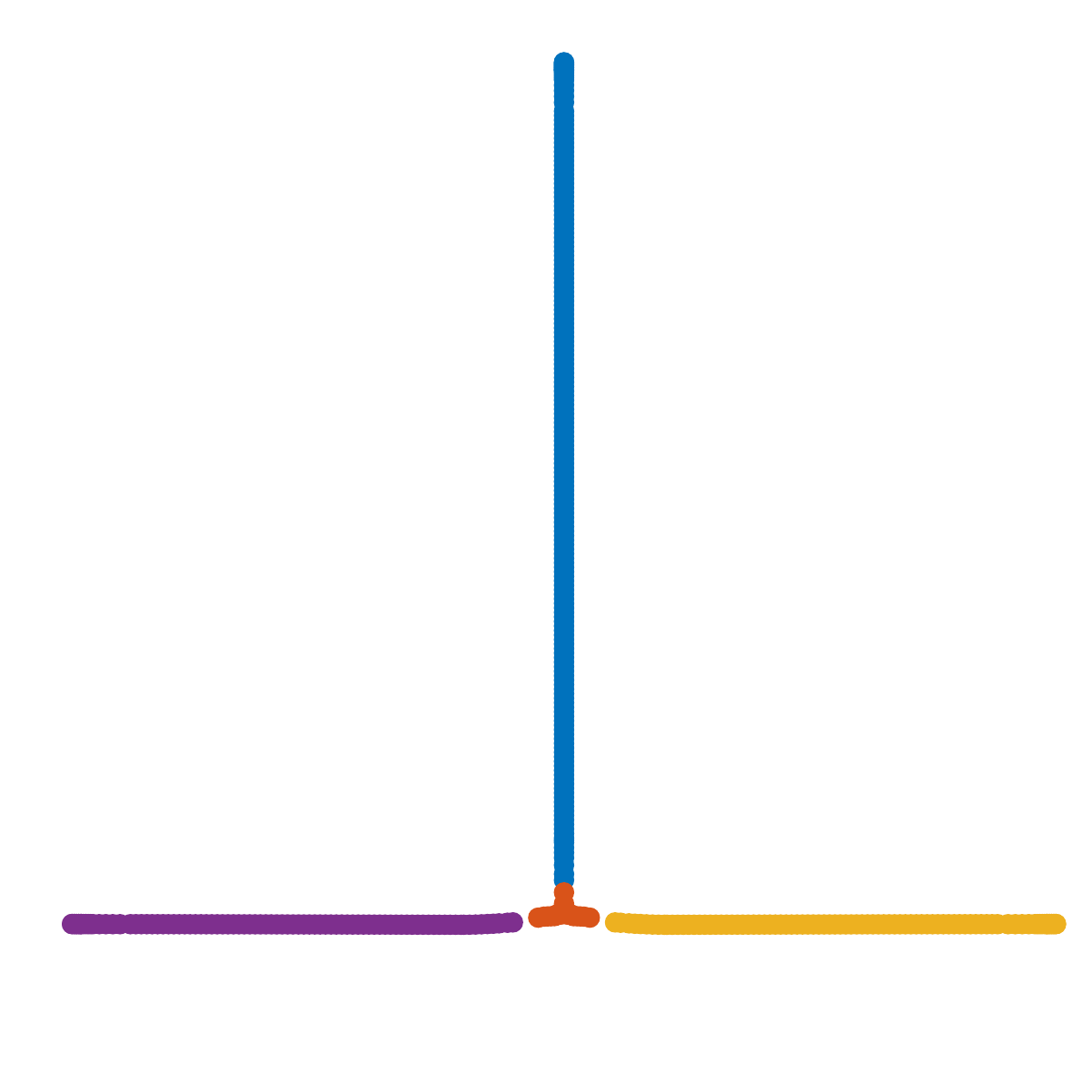}
                \\
               \includegraphics[width=1\textwidth]{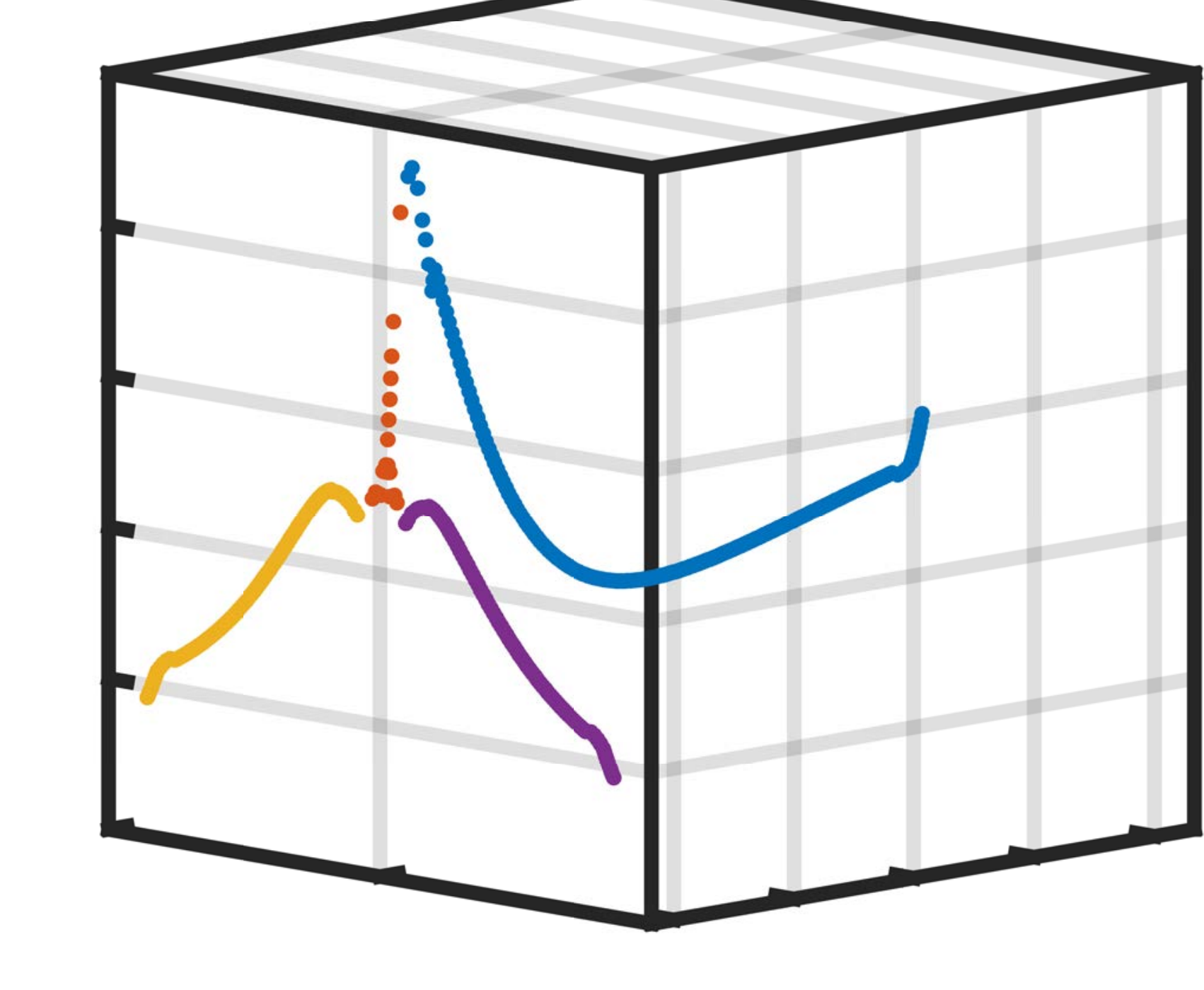} 
            \end{minipage}
        }
    }
    \subfloat[Dendrograms at $\tau=2$]{
    \label{fig:tline-dend-2-supp}
        \fbox{
            \begin{minipage}[t][2.35cm][t]{0.08\textwidth}
                \centering
                MS \\ 
                \includegraphics[width=1\textwidth]{figures/tline/ms-tline-dend-2.pdf}
            \end{minipage}
            \begin{minipage}[t][2.35cm][t]{0.08\textwidth}
                \centering
                GT-$\lambda$-1 \\ 
                \includegraphics[width=1\textwidth]{figures/tline/gtv-tline-lamb-1-dend-2.pdf}
            \end{minipage}
            \begin{minipage}[t][2.35cm][t]{0.08\textwidth}
                \centering
                GT-$\lambda$-5 \\ 
                \includegraphics[width=1\textwidth]{figures/tline/gtv-tline-lamb-5-dend-2.pdf}
            \end{minipage}  
            \begin{minipage}[t][2.35cm][t]{0.08\textwidth}
                \centering
                WT2 \\ 
                \includegraphics[width=1\textwidth]{figures/tline/wt2-fix-emd-tline-dend-2.pdf}
            \end{minipage}
            \begin{minipage}[t][2.35cm][t]{0.08\textwidth}
                \centering
                WT1 \\ 
                \includegraphics[width=1\textwidth]{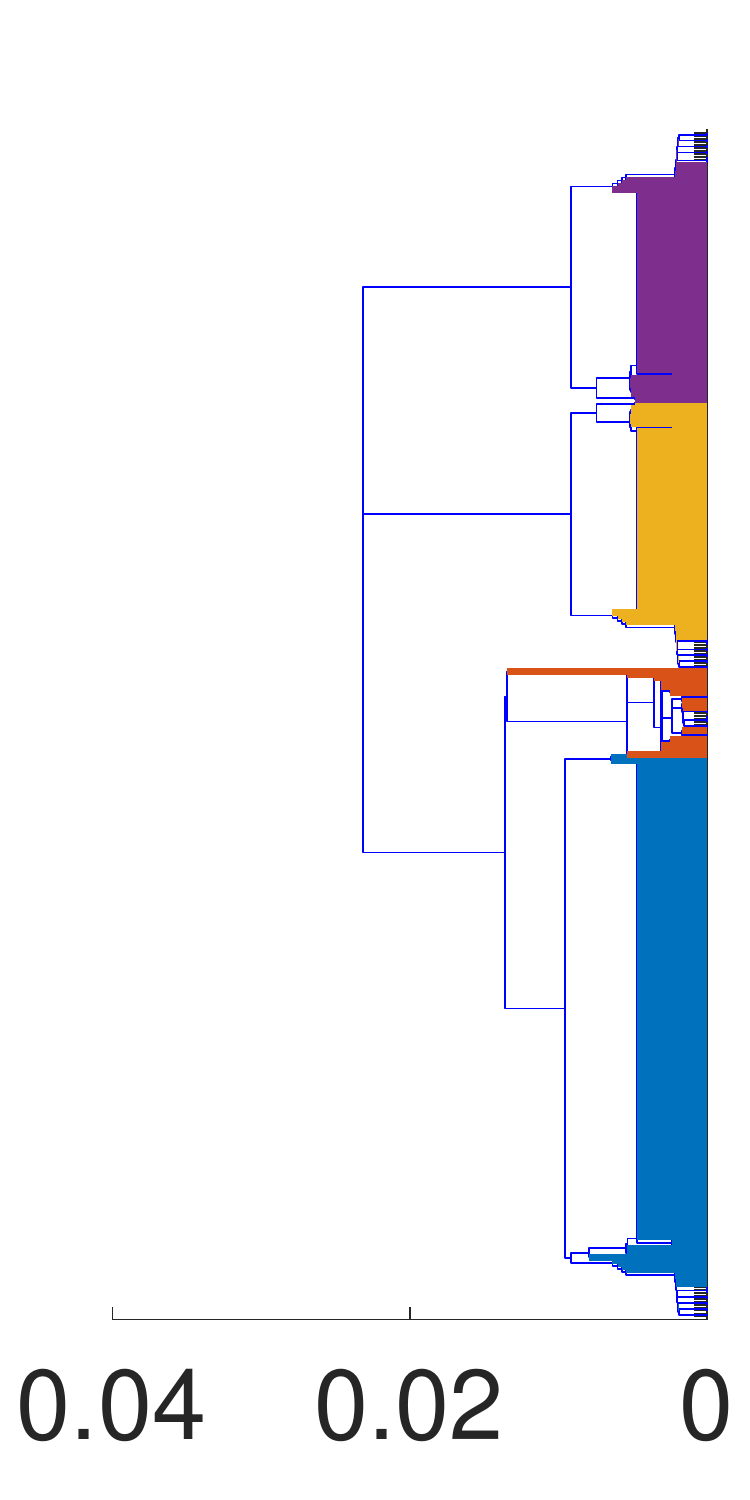}
            \end{minipage}    
        } 
    }
    \caption{T-junction clustering. 
    {
    (a): The first column shows the updated point cloud based on MS after 1 iteration; the next four columns shows the 2D and 3D MDS plots of distance matrices based on GT with $\lambda$=1, GT with $\lambda$=5, WT2 and WT1 after 1 iteration, respectively. Different colors in (a) represent different clusters, which are obtained by slicing the dendrograms illustrated in (b). 
    (b): The five columns demonstrate the clustering dendrograms using methods in (a).
    (c): The first column shows the updated point cloud based on MS after 2 iterations; the next four columns shows the 2D and 3D MDS plots of distance matrices based on GT with $\lambda$=1, GT with $\lambda$=5, WT2 and WT1 after 2 iterations, respectively. Different colors in (c) represent different clusters, which are obtained by slicing the dendrograms illustrated in (d). 
    (d): The five columns demonstrate the clustering dendrograms using methods in (c). 
    } }
    \label{fig:tlines-supp}
\end{figure}

\subsection{Ameliorating the chaining effect}
In this experiment, we examine how data geometry influences the performance of GT, MS, WT2 and WT1 on ameliorating the chaining effect. The results are shown in Figure~(\ref{fig:supp-ballchains}). Note that GT with $\lambda=5$ generates clearly better clustering results when $e=1/0.2$ than other methods.
\begin{figure}[htb]
    \centering
    \subfloat[MS]{
    \fbox{
        \begin{minipage}[t]{0.09\textwidth}
            \centering
            \begin{minipage}[t]{1\textwidth}
                \centering
                $e$:1/1 \\ 
                \includegraphics[width=1\textwidth]{figures/ellip/ellip-ori-seq-1-10.pdf} 
            \end{minipage}
            \begin{minipage}[t][0.55cm][t]{1\textwidth}
                \centering
                \includegraphics[width=1\textwidth]{figures/ellip/ms-ellip-seq-dend-1-10.pdf}
            \end{minipage}
            \begin{minipage}[t]{1\textwidth}
                \centering
                $e$:1/1 \\ 
                \includegraphics[width=1\textwidth]{figures/ellip/ellip-ori-seq-1-10.pdf} 
            \end{minipage}
            \begin{minipage}[t][0.55cm][t]{1\textwidth}
                \centering
                \includegraphics[width=1\textwidth]{figures/ellip/ms-ellip-seq-dend-1-10.pdf}
            \end{minipage}
        \end{minipage}
        \begin{minipage}[t]{0.09\textwidth}
            \centering
            \begin{minipage}[t]{1\textwidth}
                \centering
                $e$:1/0.8 \\ 
                \includegraphics[width=1\textwidth]{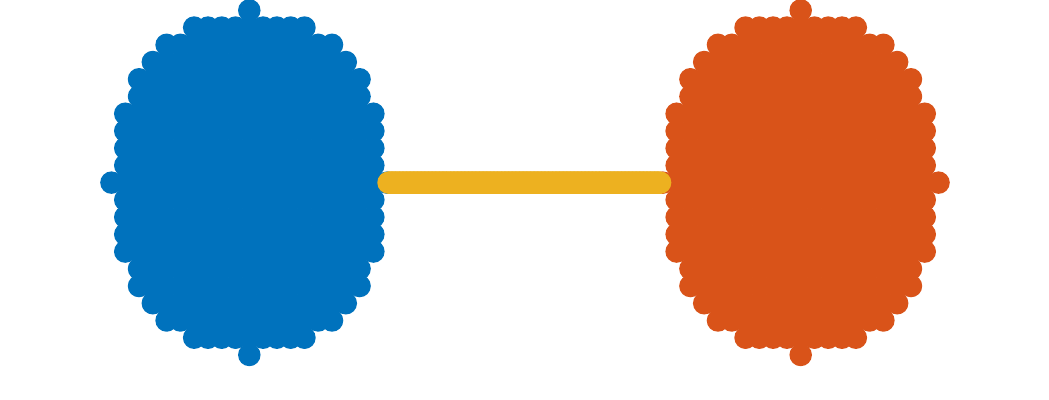} 
            \end{minipage}
            \begin{minipage}[t][0.55cm][t]{1\textwidth}
                \centering
                \includegraphics[width=1\textwidth]{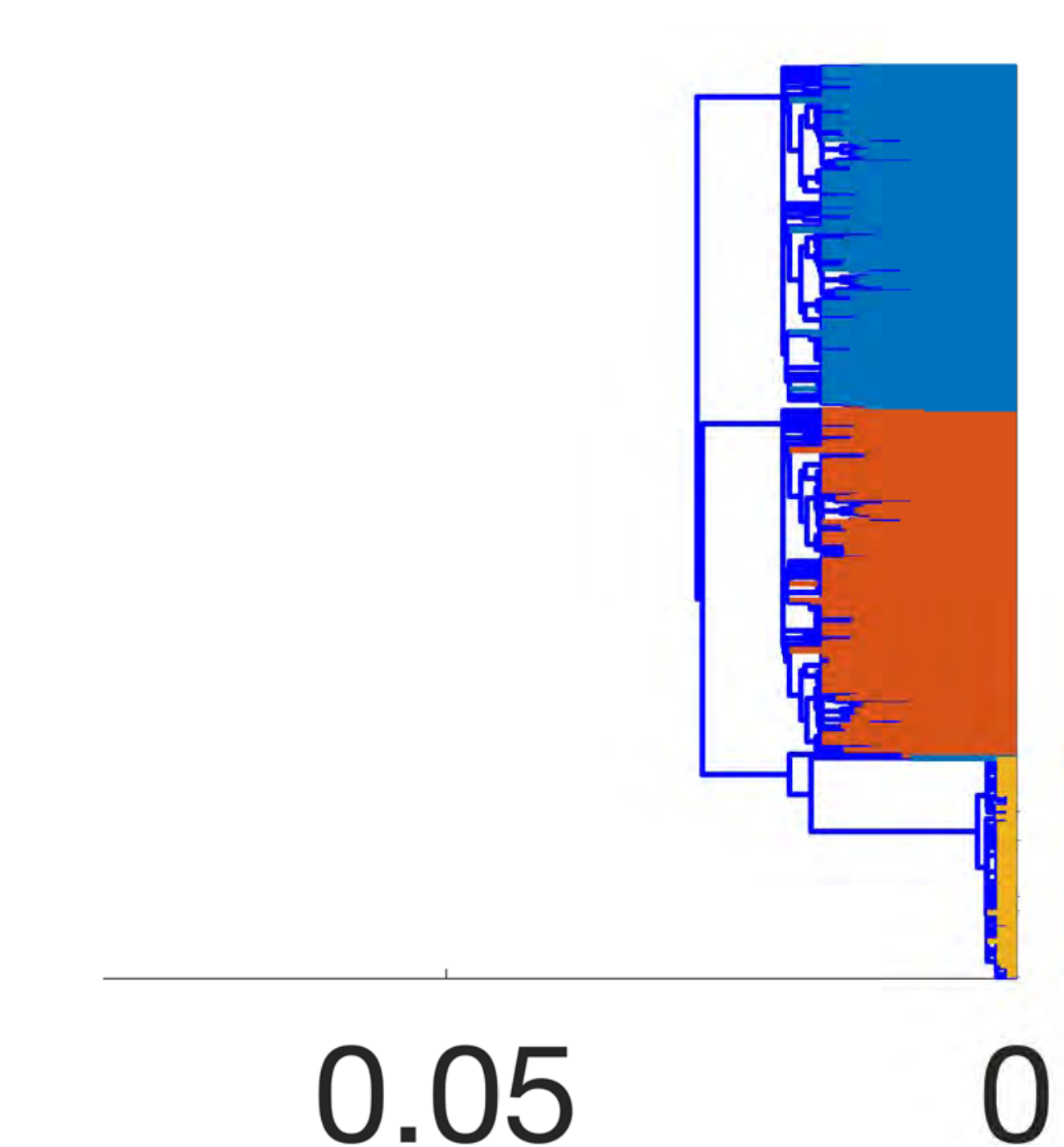}
            \end{minipage}
            \begin{minipage}[t]{1\textwidth}
                \centering
                $e$:0.8/1 \\ 
                \includegraphics[width=1\textwidth]{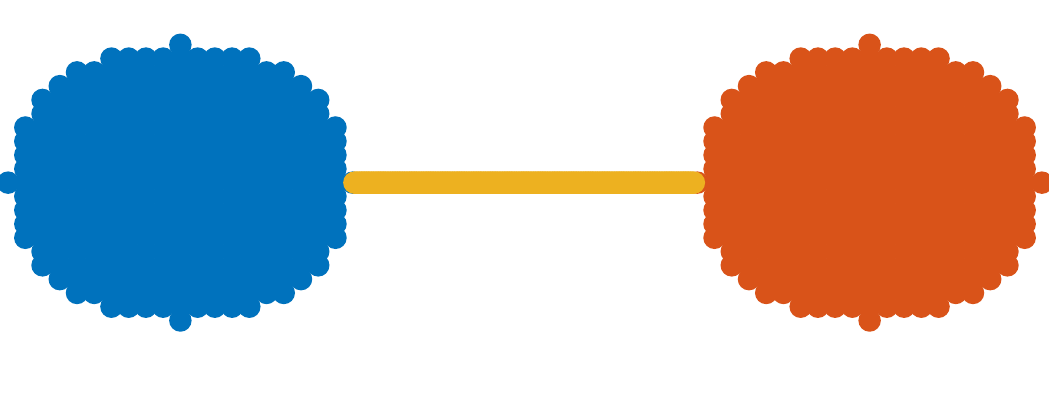} 
            \end{minipage}
            \begin{minipage}[t][0.55cm][t]{1\textwidth}
                \centering
                \includegraphics[width=1\textwidth]{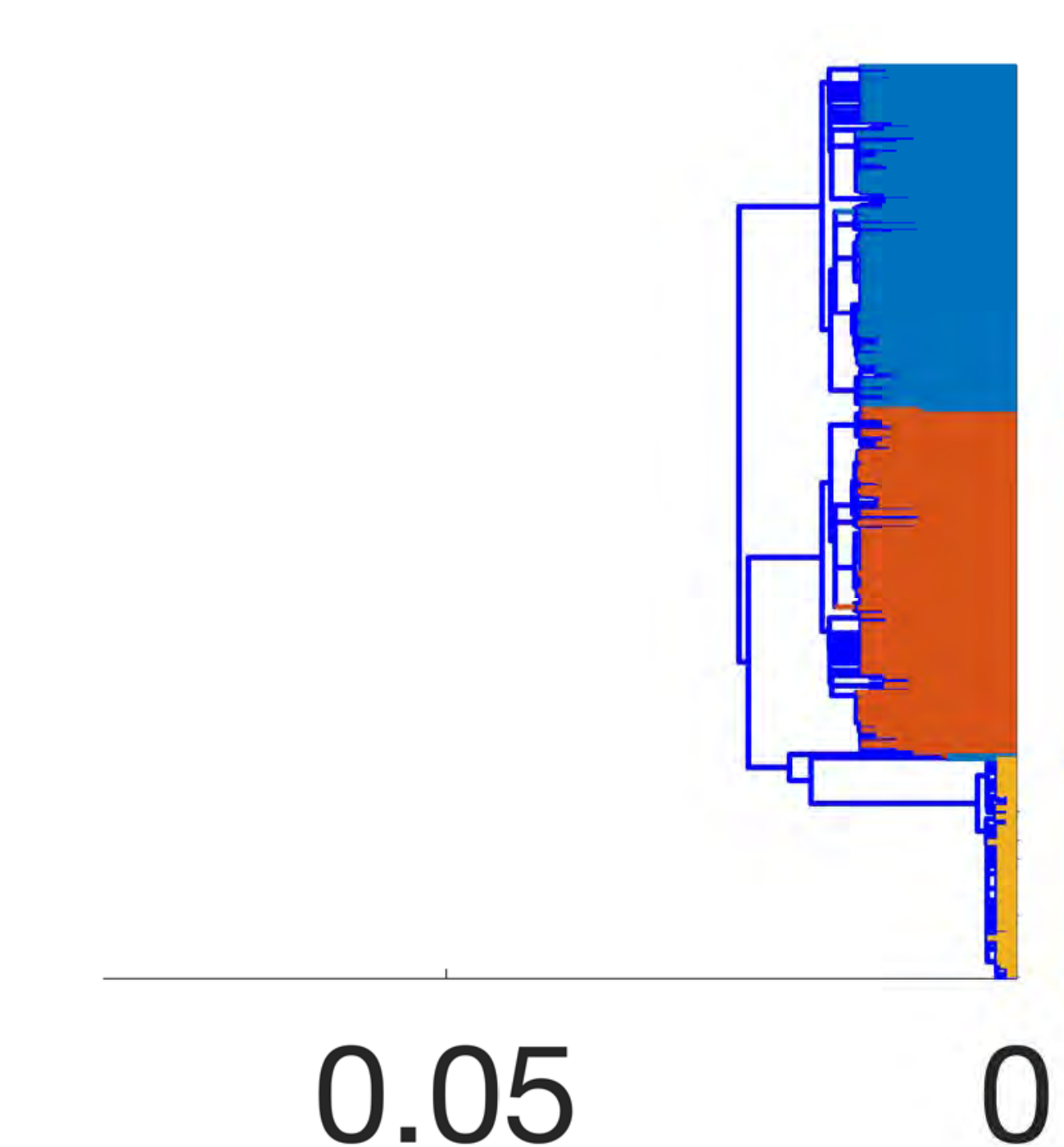}
            \end{minipage}
        \end{minipage}
        \begin{minipage}[t]{0.09\textwidth}
            \centering
            \begin{minipage}[t]{1\textwidth}
                \centering
                $e$:1/0.6 \\ 
                \includegraphics[width=1\textwidth]{figures/ellip/ellip-ori-seq-1-6.pdf} 
            \end{minipage}
            \begin{minipage}[t][0.55cm][t]{1\textwidth}
                \centering
                \includegraphics[width=1\textwidth]{figures/ellip/ms-ellip-seq-dend-1-6.pdf}
            \end{minipage}
            \begin{minipage}[t]{1\textwidth}
                \centering
                $e$:0.6/1 \\ 
                \includegraphics[width=1\textwidth]{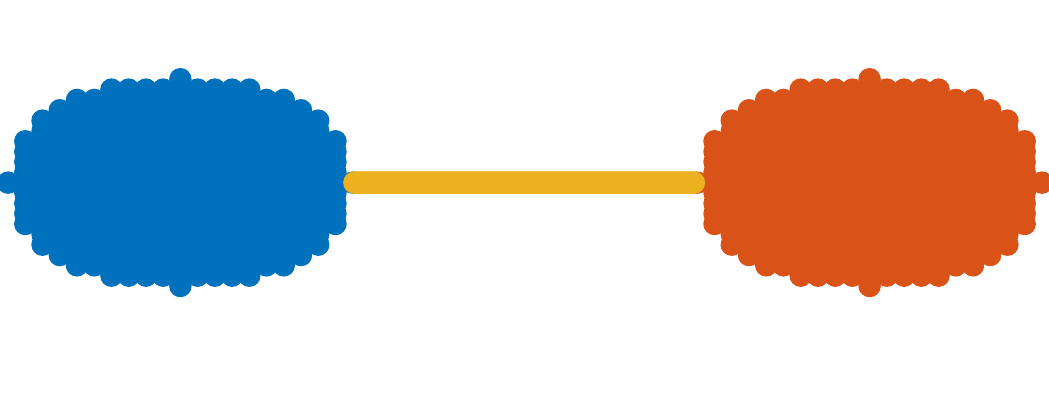} 
            \end{minipage}
            \begin{minipage}[t][0.55cm][t]{1\textwidth}
                \centering
                \includegraphics[width=1\textwidth]{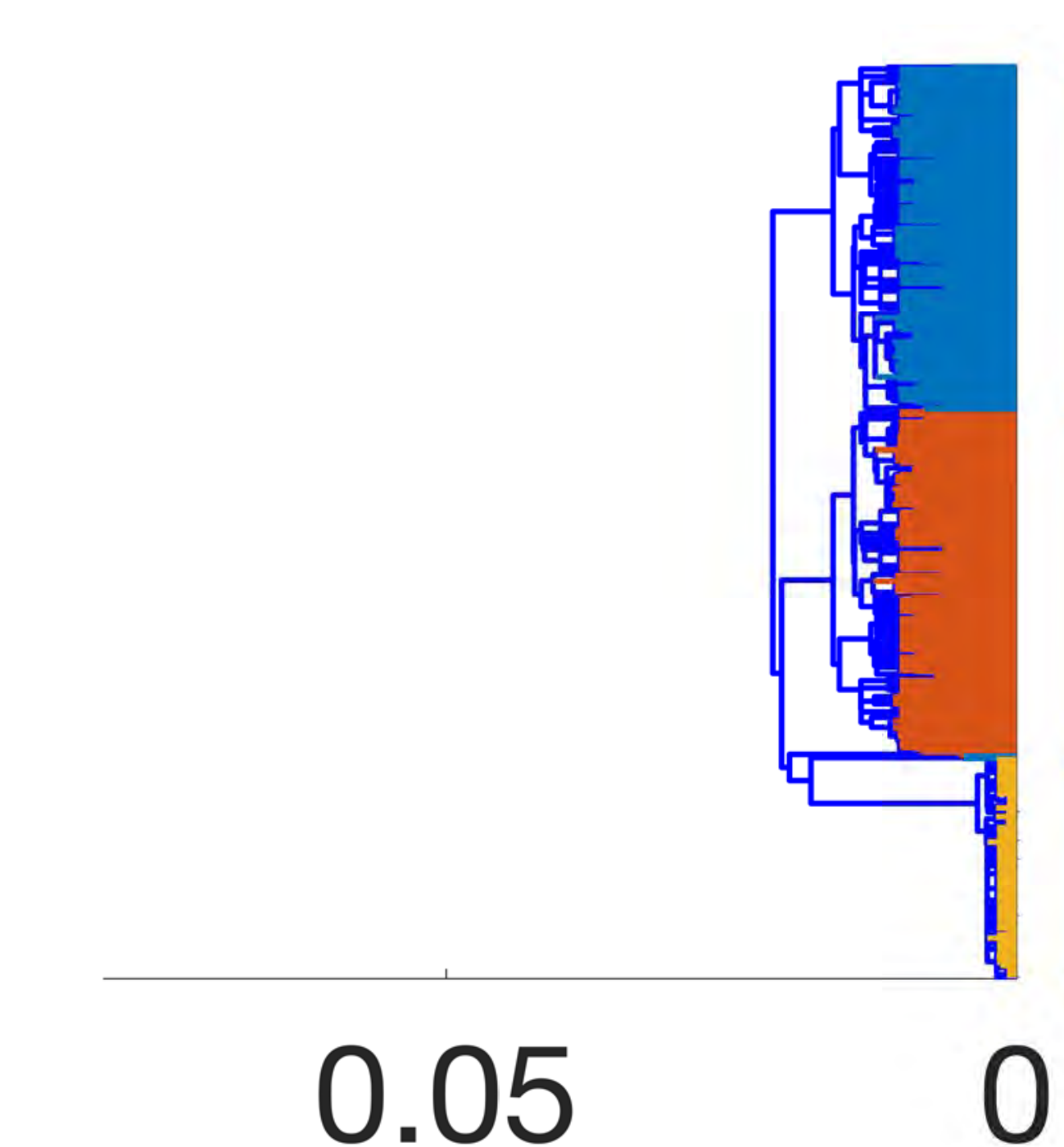}
            \end{minipage}
        \end{minipage}
        \begin{minipage}[t]{0.09\textwidth}
            \centering
            \begin{minipage}[t]{1\textwidth}
                \centering
                $e$:1/0.4 \\ 
                \includegraphics[width=1\textwidth]{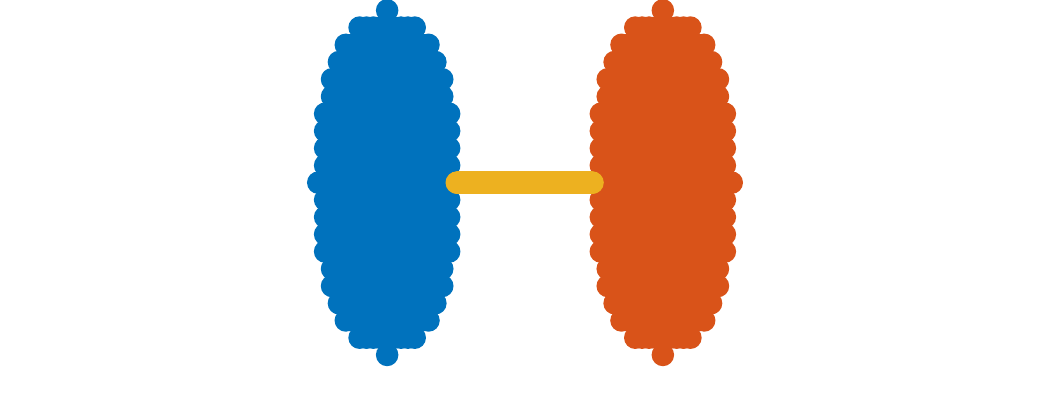} 
            \end{minipage}
            \begin{minipage}[t][0.55cm][t]{1\textwidth}
                \centering
                \includegraphics[width=1\textwidth]{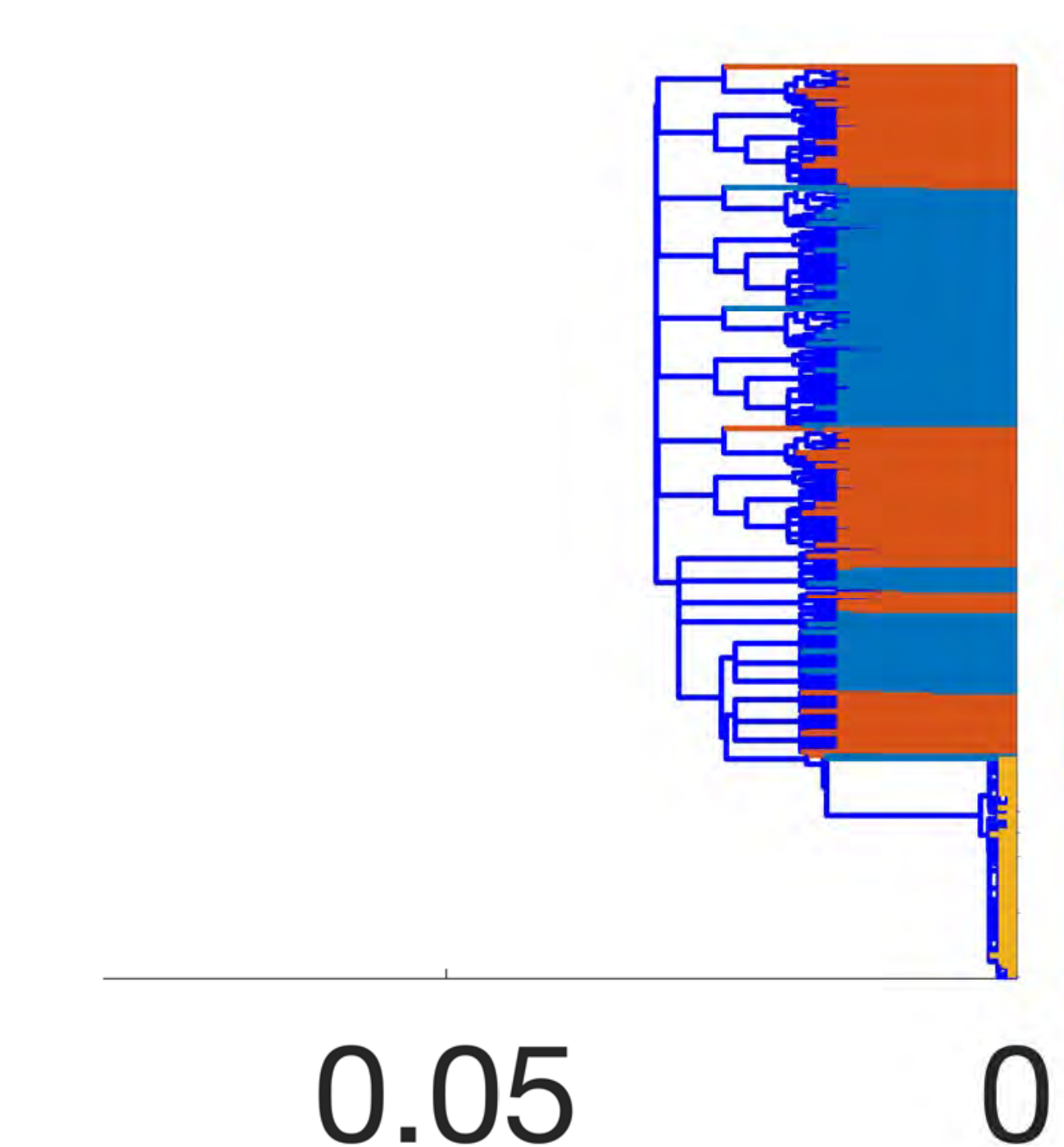}
            \end{minipage}
            \begin{minipage}[t]{1\textwidth}
                \centering
                $e$:0.4/1 \\ 
                \includegraphics[width=1\textwidth]{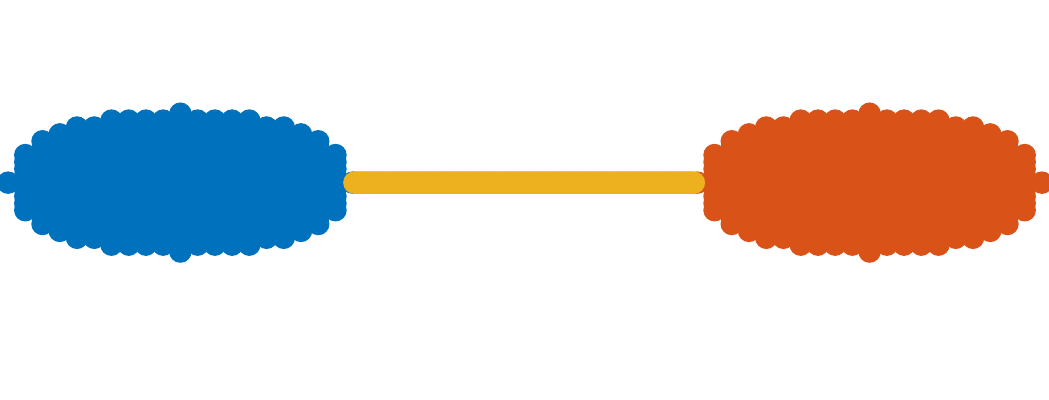} 
            \end{minipage}
            \begin{minipage}[t][0.55cm][t]{1\textwidth}
                \centering
                \includegraphics[width=1\textwidth]{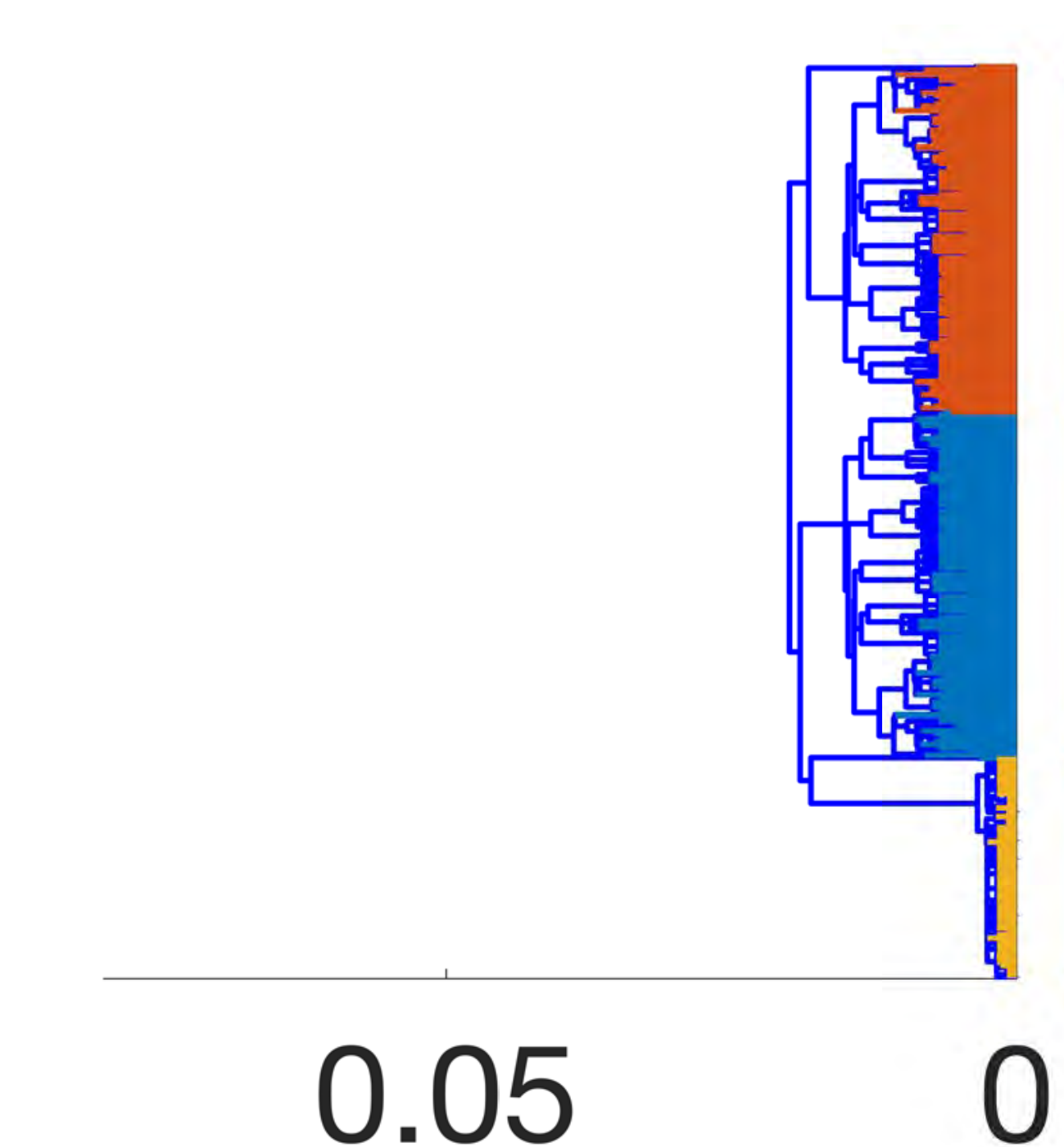}
            \end{minipage}
        \end{minipage}
        \begin{minipage}[t]{0.09\textwidth}
            \centering
            \begin{minipage}[t]{1\textwidth}
                \centering
                $e$:1/0.2 \\ 
                \includegraphics[width=1\textwidth]{figures/ellip/ellip-ori-seq-1-2.pdf} 
            \end{minipage}
            \begin{minipage}[t][0.55cm][t]{1\textwidth}
                \centering
                \includegraphics[width=1\textwidth]{figures/ellip/ms-ellip-seq-dend-1-2.pdf}
            \end{minipage}
            \begin{minipage}[t]{1\textwidth}
                \centering
                $e$:0.2/1 \\ 
                \includegraphics[width=1\textwidth]{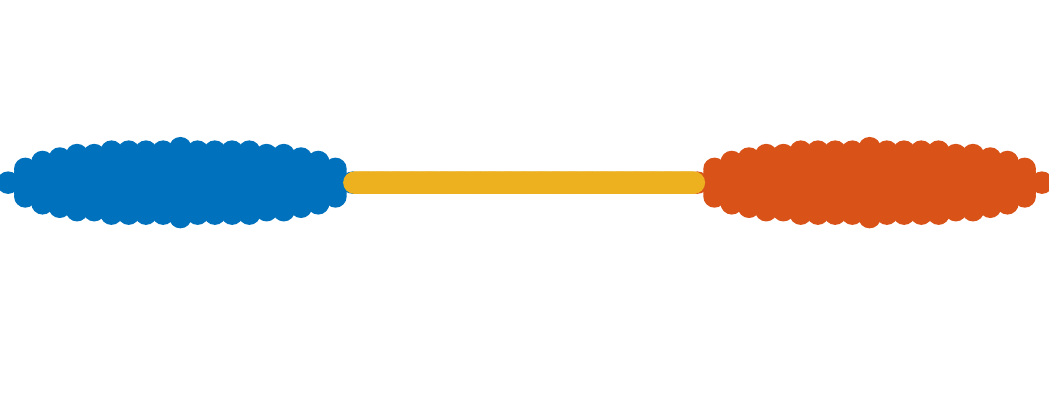} 
            \end{minipage}
            \begin{minipage}[t][0.55cm][t]{1\textwidth}
                \centering
                \includegraphics[width=1\textwidth]{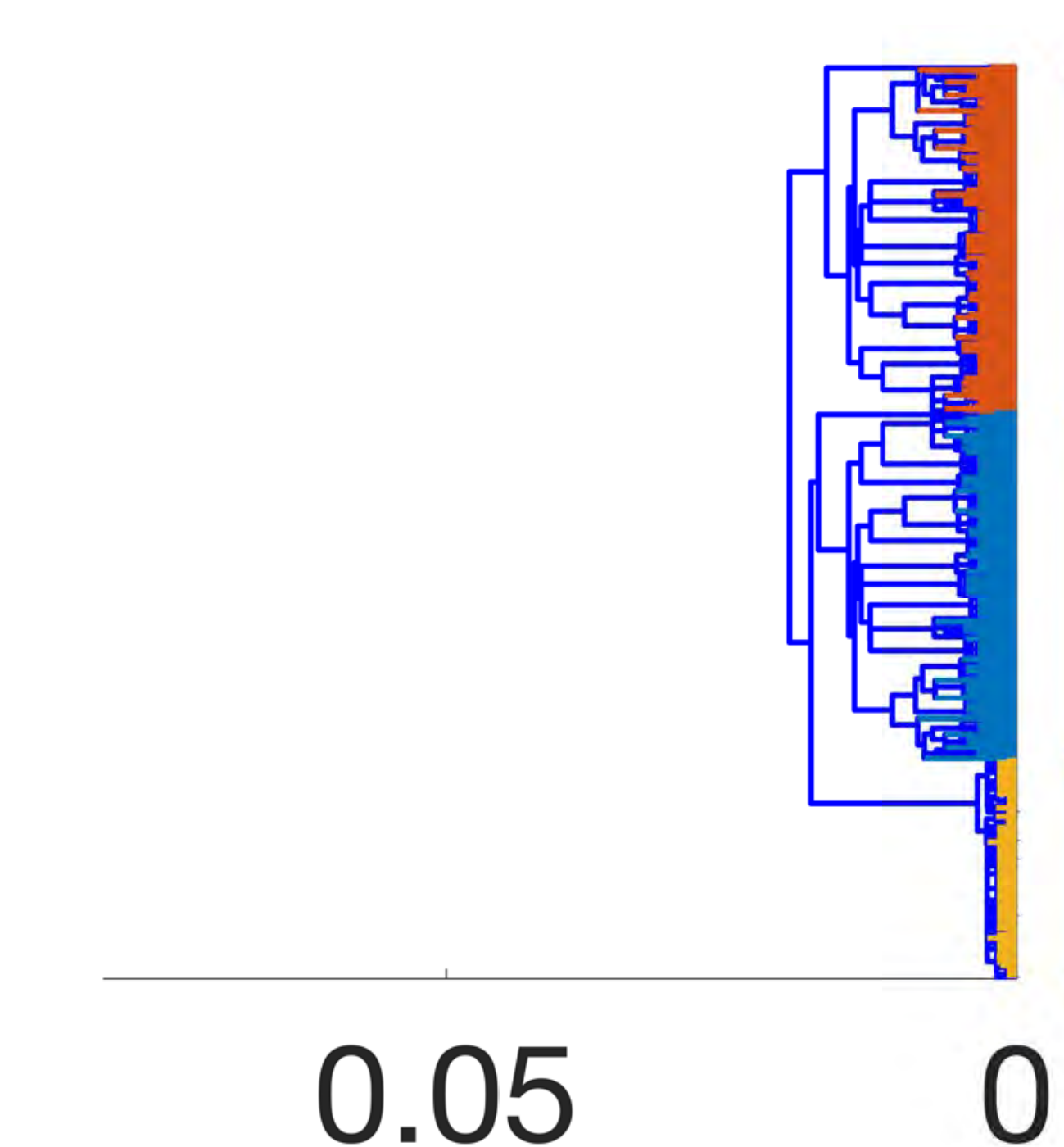}
            \end{minipage}
        \end{minipage}
    } } 
    
    \subfloat[GT-$\lambda$-1]{
    \label{fig:ballchains-gt1}
    \fbox{
        \begin{minipage}[t]{0.09\textwidth}
            \centering
            \begin{minipage}[t]{1\textwidth}
                \centering
                $e$:1/1 \\ 
                \includegraphics[width=1\textwidth]{figures/ellip/ellip-ori-seq-1-10.pdf} 
            \end{minipage}
            \begin{minipage}[t][0.55cm][t]{1\textwidth}
                \centering
                \includegraphics[width=1\textwidth]{figures/ellip/gtv-ellip-lamb-1-seq-dend-1-10.pdf}
            \end{minipage}
            \begin{minipage}[t]{1\textwidth}
                \centering
                $e$:1/1 \\ 
                \includegraphics[width=1\textwidth]{figures/ellip/ellip-ori-seq-1-10.pdf} 
            \end{minipage}
            \begin{minipage}[t][0.55cm][t]{1\textwidth}
                \centering
                \includegraphics[width=1\textwidth]{figures/ellip/gtv-ellip-lamb-1-seq-dend-1-10.pdf}
            \end{minipage}
        \end{minipage}
        \begin{minipage}[t]{0.09\textwidth}
            \centering
            \begin{minipage}[t]{1\textwidth}
                \centering
                $e$:1/0.8 \\ 
                \includegraphics[width=1\textwidth]{figures/ellip/ellip-ori-seq-1-8.pdf} 
            \end{minipage}
            \begin{minipage}[t][0.55cm][t]{1\textwidth}
                \centering
                \includegraphics[width=1\textwidth]{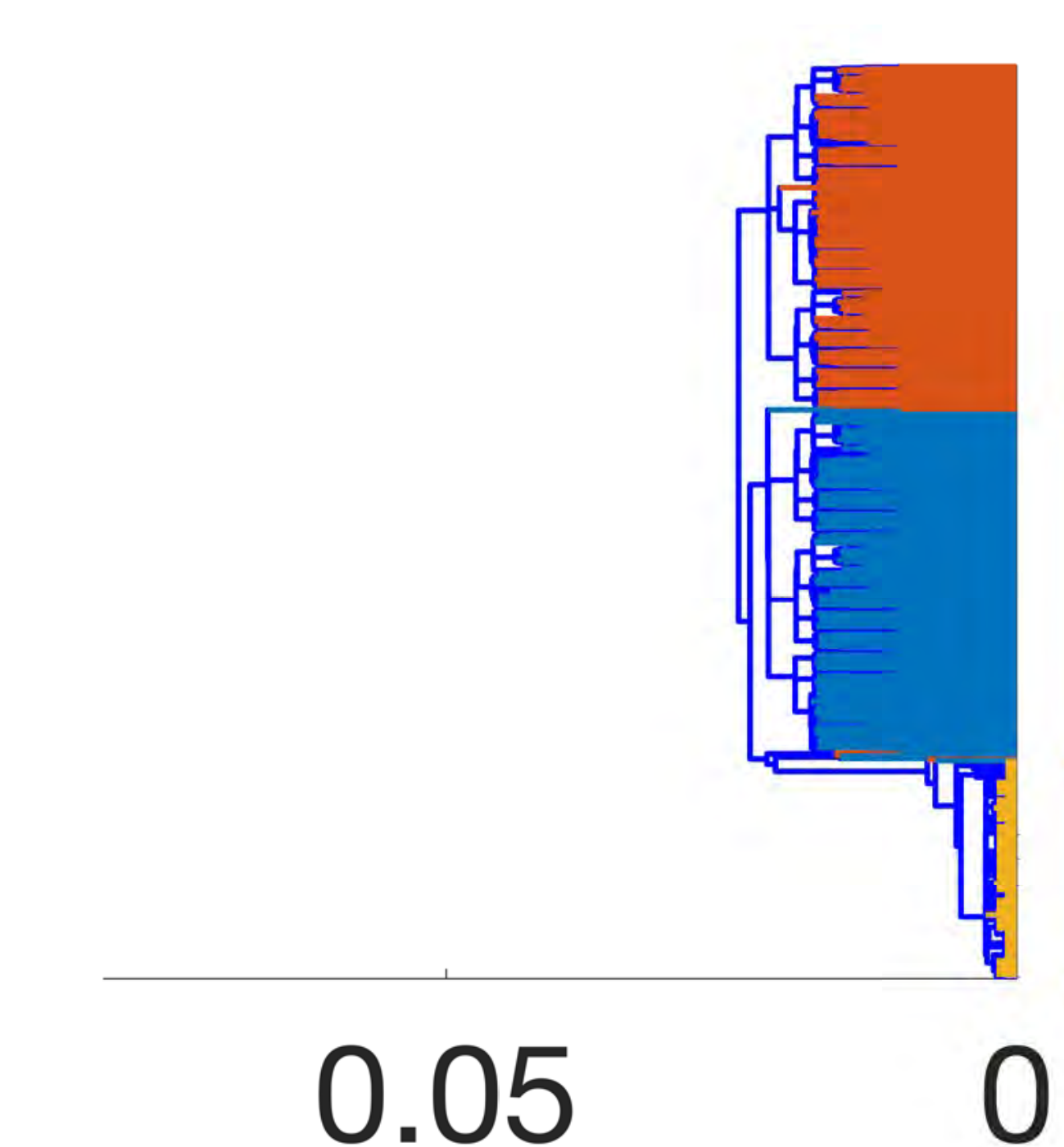}
            \end{minipage}
            \begin{minipage}[t]{1\textwidth}
                \centering
                $e$:0.8/1 \\ 
                \includegraphics[width=1\textwidth]{figures/ellip/ellip-ori-seq-8-1.pdf} 
            \end{minipage}
            \begin{minipage}[t][0.55cm][t]{1\textwidth}
                \centering
                \includegraphics[width=1\textwidth]{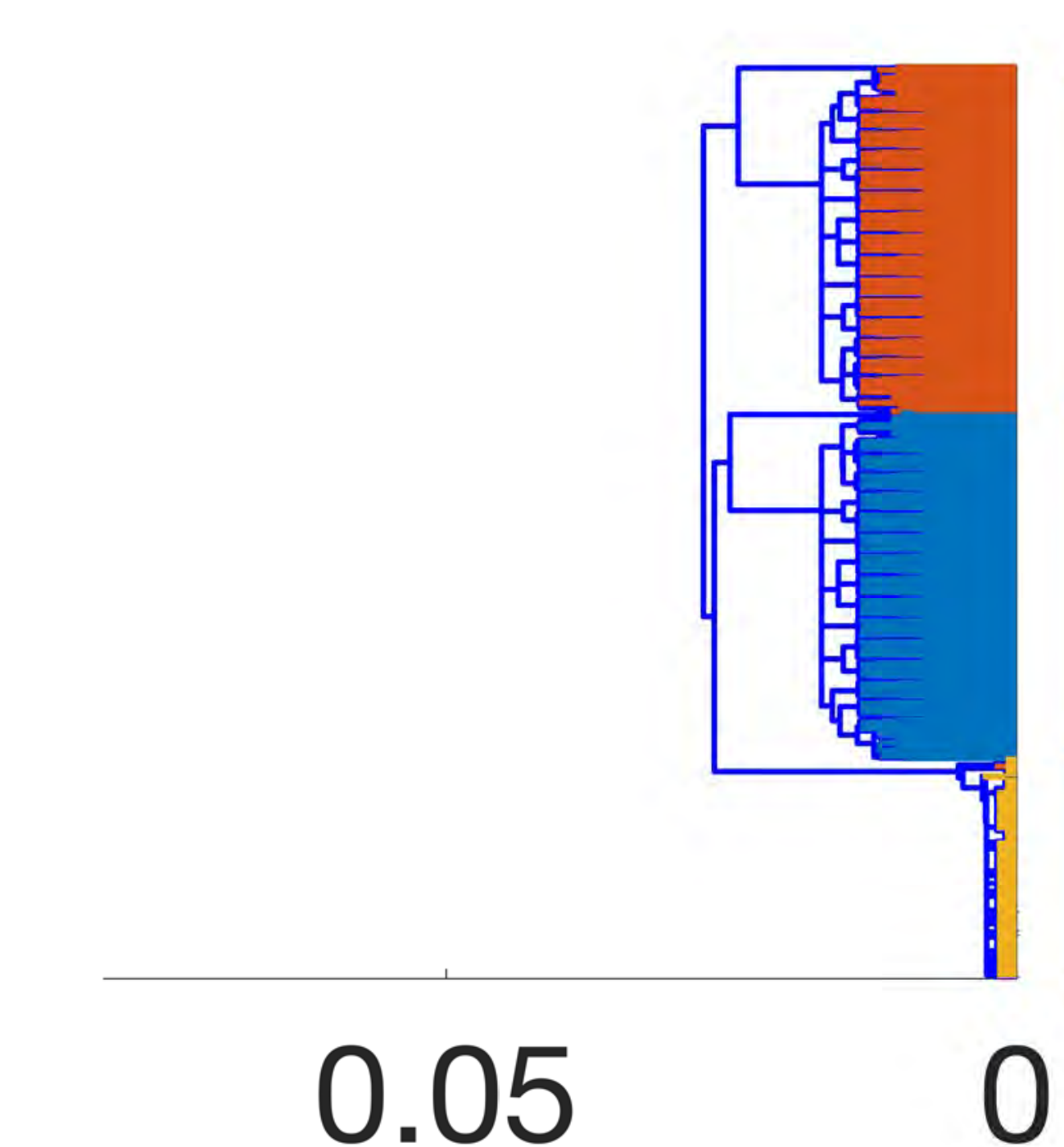}
            \end{minipage}
        \end{minipage}
        \begin{minipage}[t]{0.09\textwidth}
            \centering
            \begin{minipage}[t]{1\textwidth}
                \centering
                $e$:1/0.6 \\ 
                \includegraphics[width=1\textwidth]{figures/ellip/ellip-ori-seq-1-6.pdf} 
            \end{minipage}
            \begin{minipage}[t][0.55cm][t]{1\textwidth}
                \centering
                \includegraphics[width=1\textwidth]{figures/ellip/gtv-ellip-lamb-1-seq-dend-1-6.pdf}
            \end{minipage}
            \begin{minipage}[t]{1\textwidth}
                \centering
                $e$:0.6/1 \\ 
                \includegraphics[width=1\textwidth]{figures/ellip/ellip-ori-seq-6-1.pdf} 
            \end{minipage}
            \begin{minipage}[t][0.55cm][t]{1\textwidth}
                \centering
                \includegraphics[width=1\textwidth]{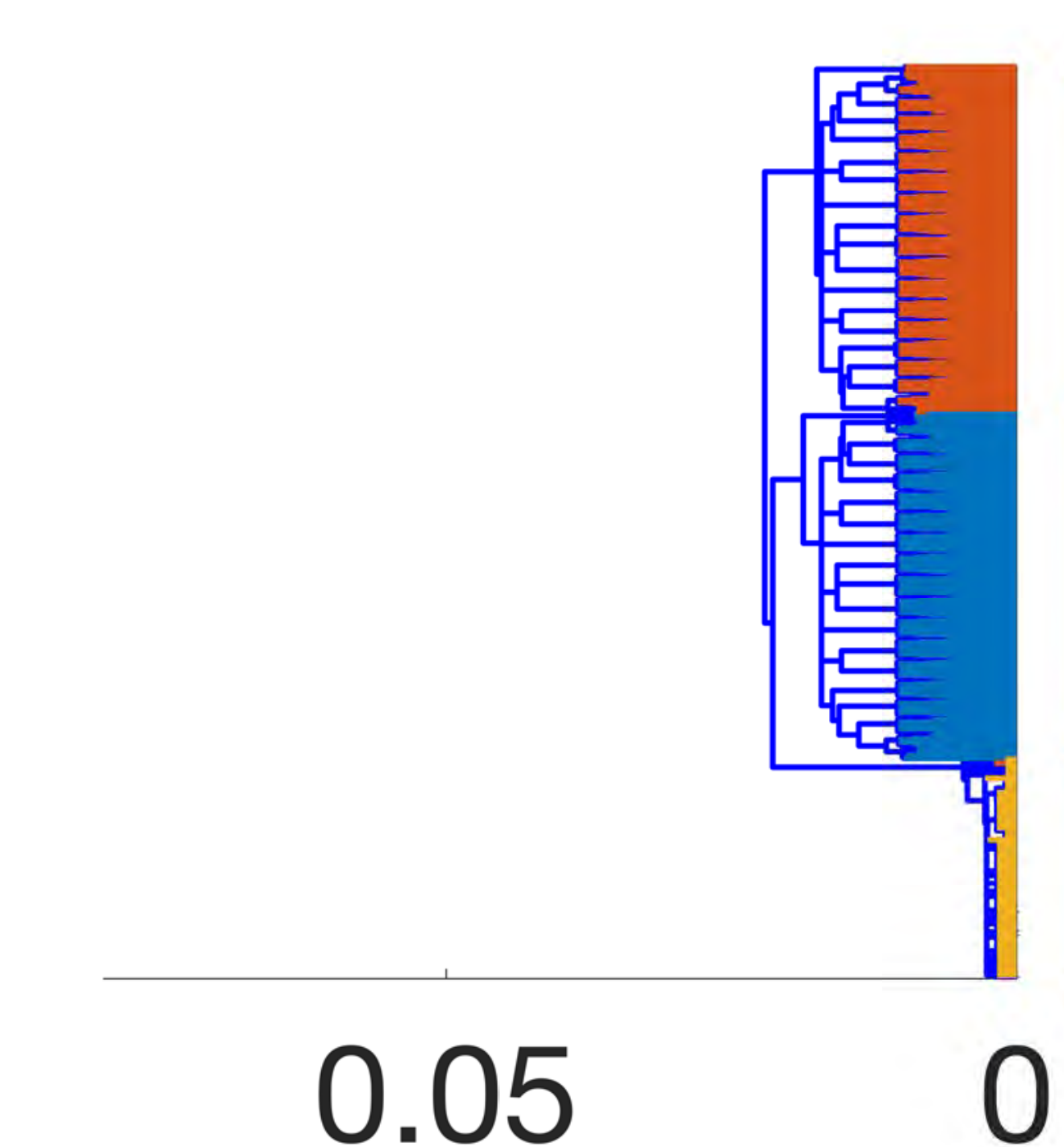}
            \end{minipage}
        \end{minipage}
        \begin{minipage}[t]{0.09\textwidth}
            \centering
            \begin{minipage}[t]{1\textwidth}
                \centering
                $e$:1/0.4 \\ 
                \includegraphics[width=1\textwidth]{figures/ellip/ellip-ori-seq-1-4.pdf} 
            \end{minipage}
            \begin{minipage}[t][0.55cm][t]{1\textwidth}
                \centering
                \includegraphics[width=1\textwidth]{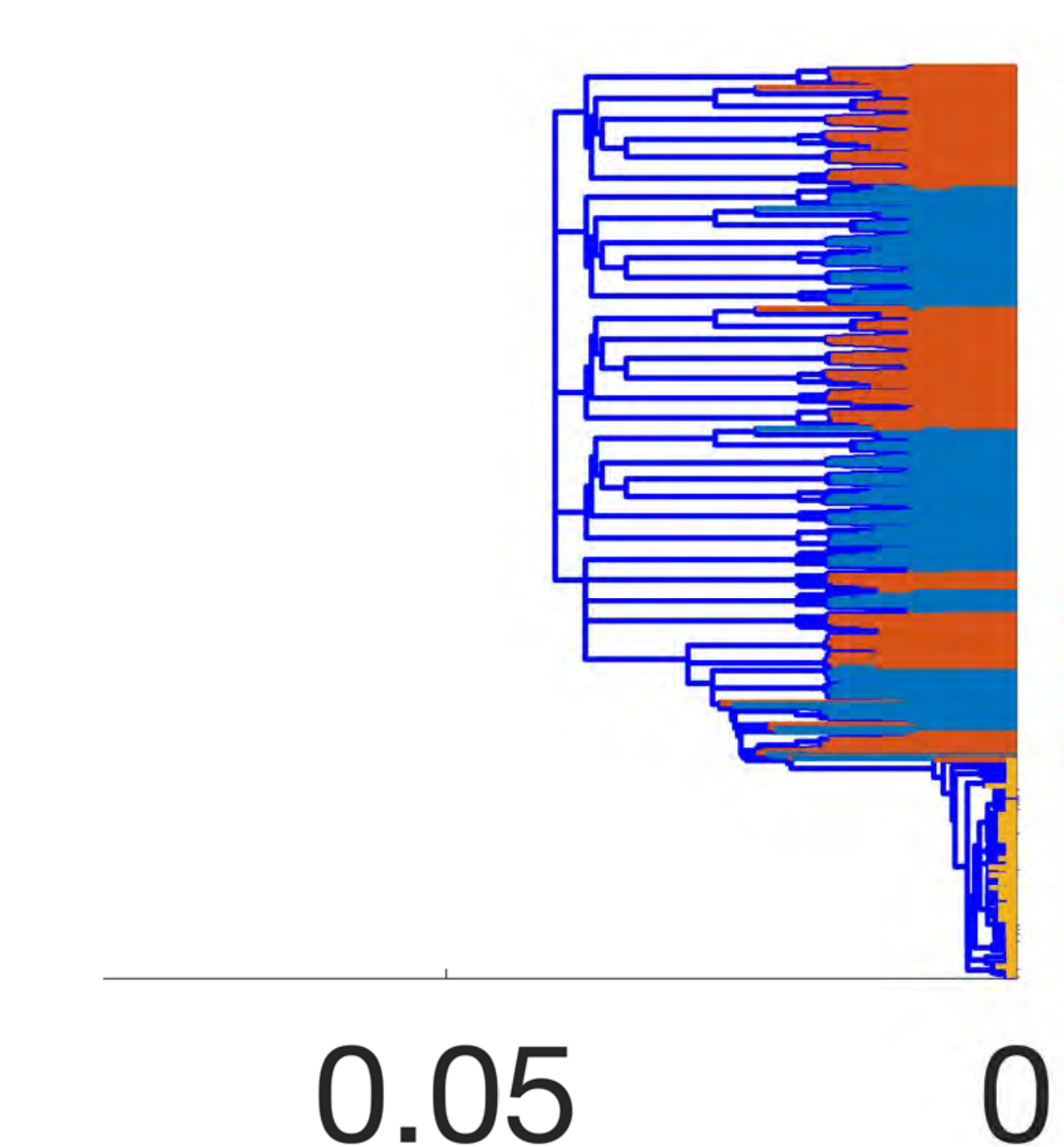}
            \end{minipage}
            \begin{minipage}[t]{1\textwidth}
                \centering
                $e$:0.4/1 \\ 
                \includegraphics[width=1\textwidth]{figures/ellip/ellip-ori-seq-4-1.pdf} 
            \end{minipage}
            \begin{minipage}[t][0.55cm][t]{1\textwidth}
                \centering
                \includegraphics[width=1\textwidth]{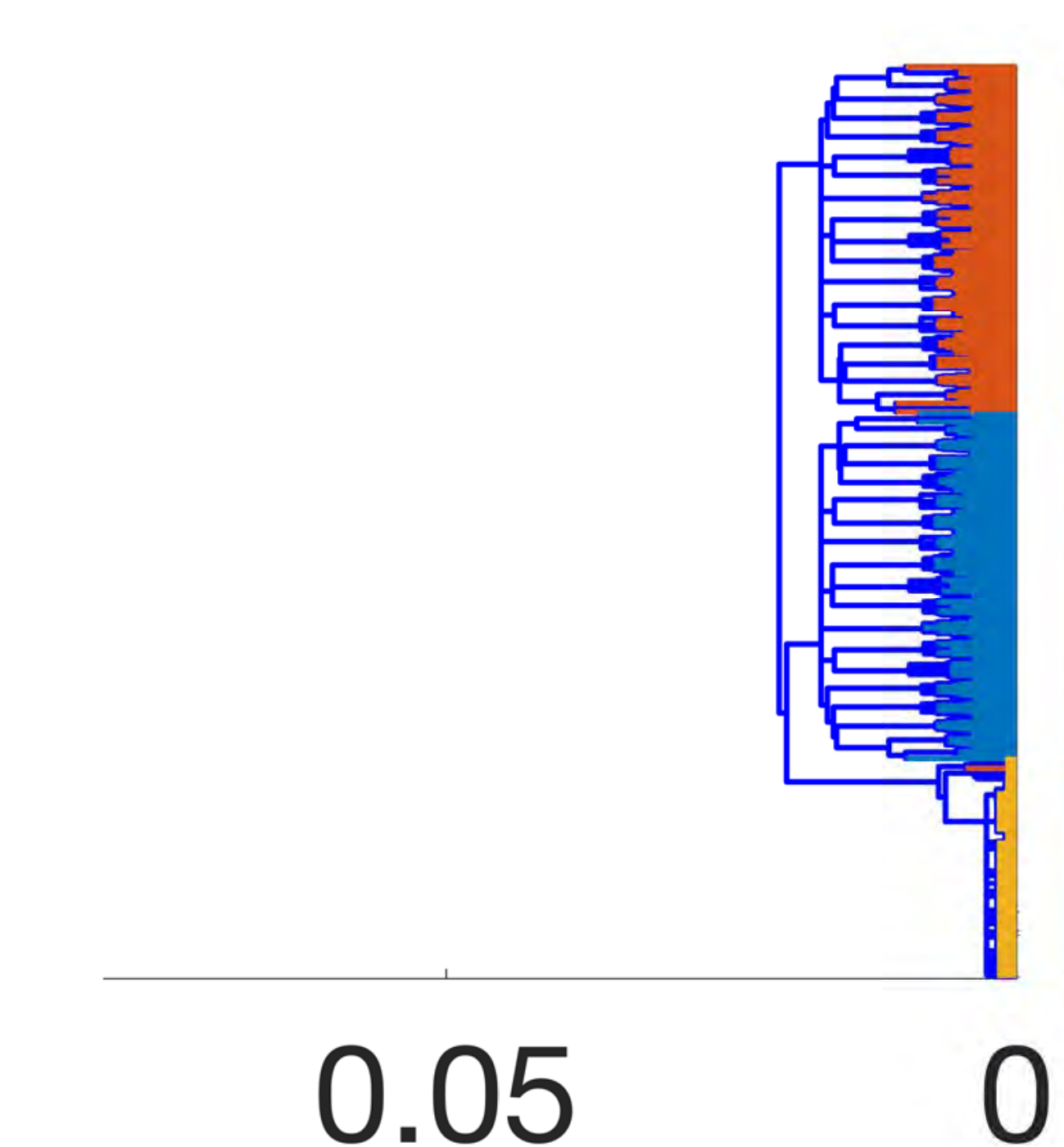}
            \end{minipage}
        \end{minipage}
        \begin{minipage}[t]{0.09\textwidth}
            \centering
            \begin{minipage}[t]{1\textwidth}
                \centering
                $e$:1/0.2 \\ 
                \includegraphics[width=1\textwidth]{figures/ellip/ellip-ori-seq-1-2.pdf} 
            \end{minipage}
            \begin{minipage}[t][0.55cm][t]{1\textwidth}
                \centering
                \includegraphics[width=1\textwidth]{figures/ellip/gtv-ellip-lamb-1-seq-dend-1-2.pdf}
            \end{minipage}
            \begin{minipage}[t]{1\textwidth}
                \centering
                $e$:0.2/1 \\ 
                \includegraphics[width=1\textwidth]{figures/ellip/ellip-ori-seq-2-1.pdf} 
            \end{minipage}
            \begin{minipage}[t][0.55cm][t]{1\textwidth}
                \centering
                \includegraphics[width=1\textwidth]{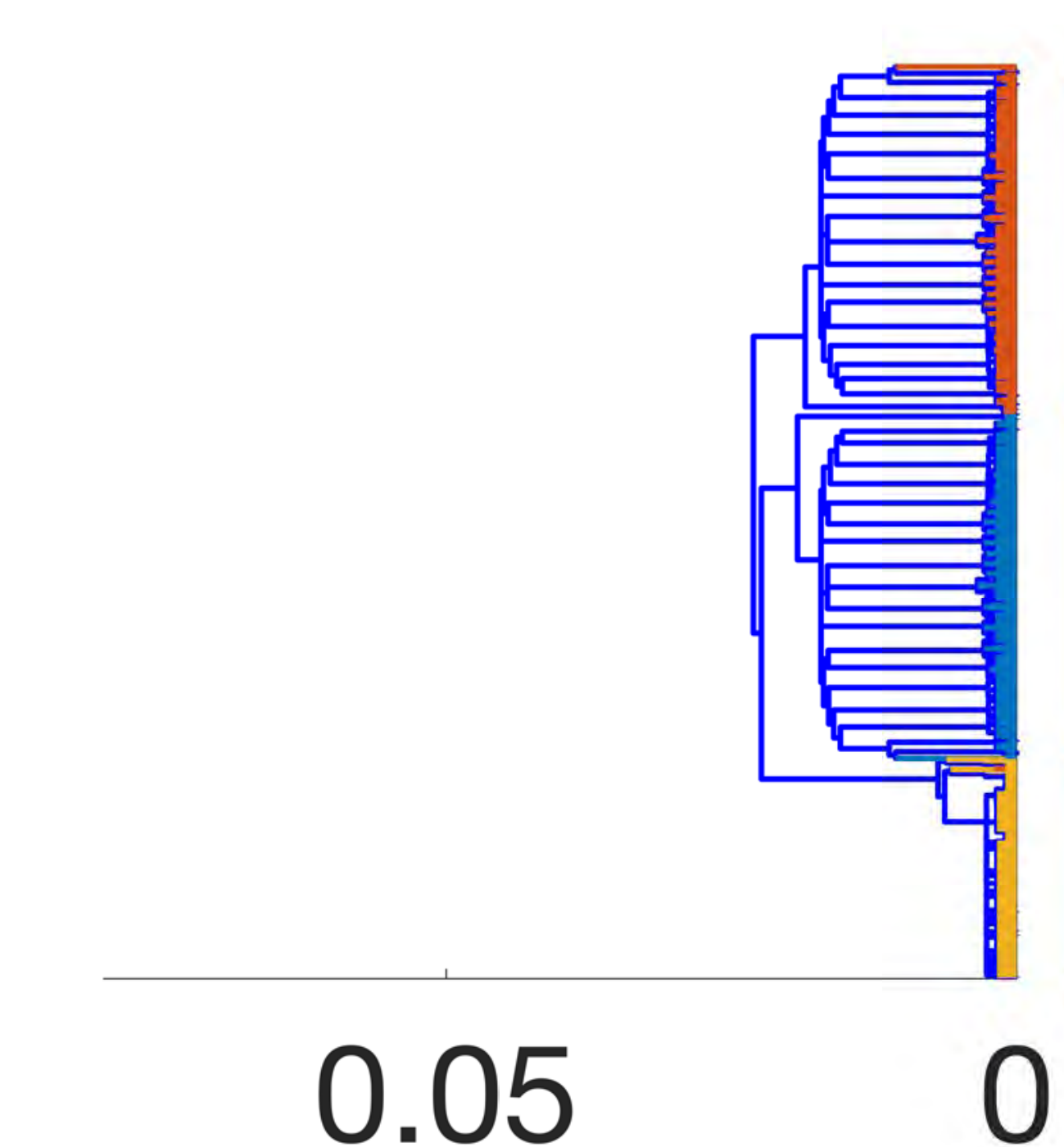}
            \end{minipage}
        \end{minipage}
    } }
    \subfloat[GT-$\lambda$-5]{
    \label{fig:ballchains-gt5}
    \fbox{
        \begin{minipage}[t]{0.09\textwidth}
            \centering
            \begin{minipage}[t]{1\textwidth}
                \centering
                $e$:1/1 \\ 
                \includegraphics[width=1\textwidth]{figures/ellip/ellip-ori-seq-1-10.pdf} 
            \end{minipage}
            \begin{minipage}[t][0.55cm][t]{1\textwidth}
                \centering
                \includegraphics[width=1\textwidth]{figures/ellip/gtv-ellip-lamb-5-seq-dend-1-10.pdf}
            \end{minipage}
            \begin{minipage}[t]{1\textwidth}
                \centering
                $e$:1/1 \\ 
                \includegraphics[width=1\textwidth]{figures/ellip/ellip-ori-seq-1-10.pdf} 
            \end{minipage}
            \begin{minipage}[t][0.55cm][t]{1\textwidth}
                \centering
                \includegraphics[width=1\textwidth]{figures/ellip/gtv-ellip-lamb-5-seq-dend-1-10.pdf}
            \end{minipage}
        \end{minipage}
        \begin{minipage}[t]{0.09\textwidth}
            \centering
            \begin{minipage}[t]{1\textwidth}
                \centering
                $e$:1/0.8 \\ 
                \includegraphics[width=1\textwidth]{figures/ellip/ellip-ori-seq-1-8.pdf} 
            \end{minipage}
            \begin{minipage}[t][0.55cm][t]{1\textwidth}
                \centering
                \includegraphics[width=1\textwidth]{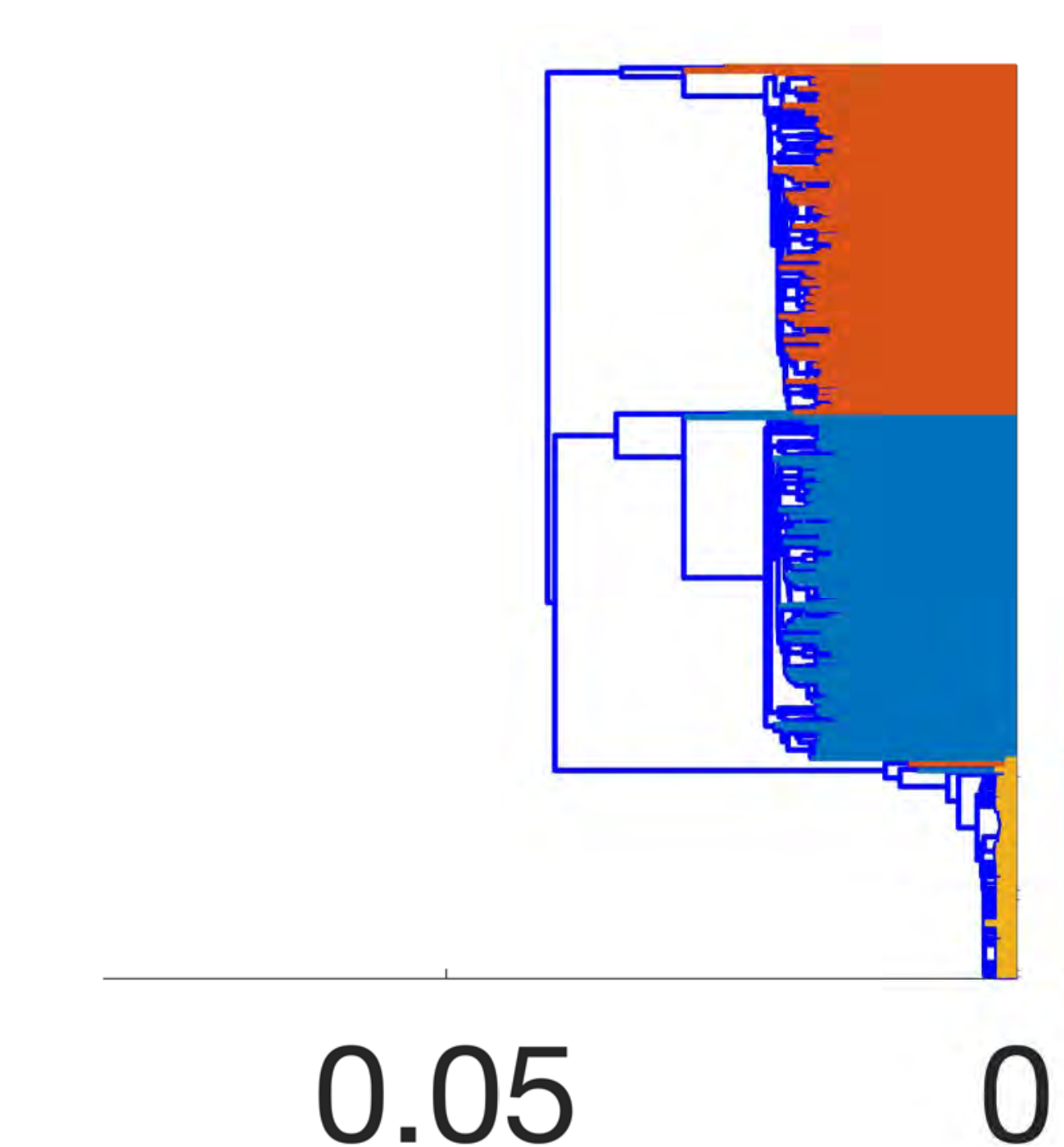}
            \end{minipage}
            \begin{minipage}[t]{1\textwidth}
                \centering
                $e$:0.8/1 \\ 
                \includegraphics[width=1\textwidth]{figures/ellip/ellip-ori-seq-8-1.pdf} 
            \end{minipage}
            \begin{minipage}[t][0.55cm][t]{1\textwidth}
                \centering
                \includegraphics[width=1\textwidth]{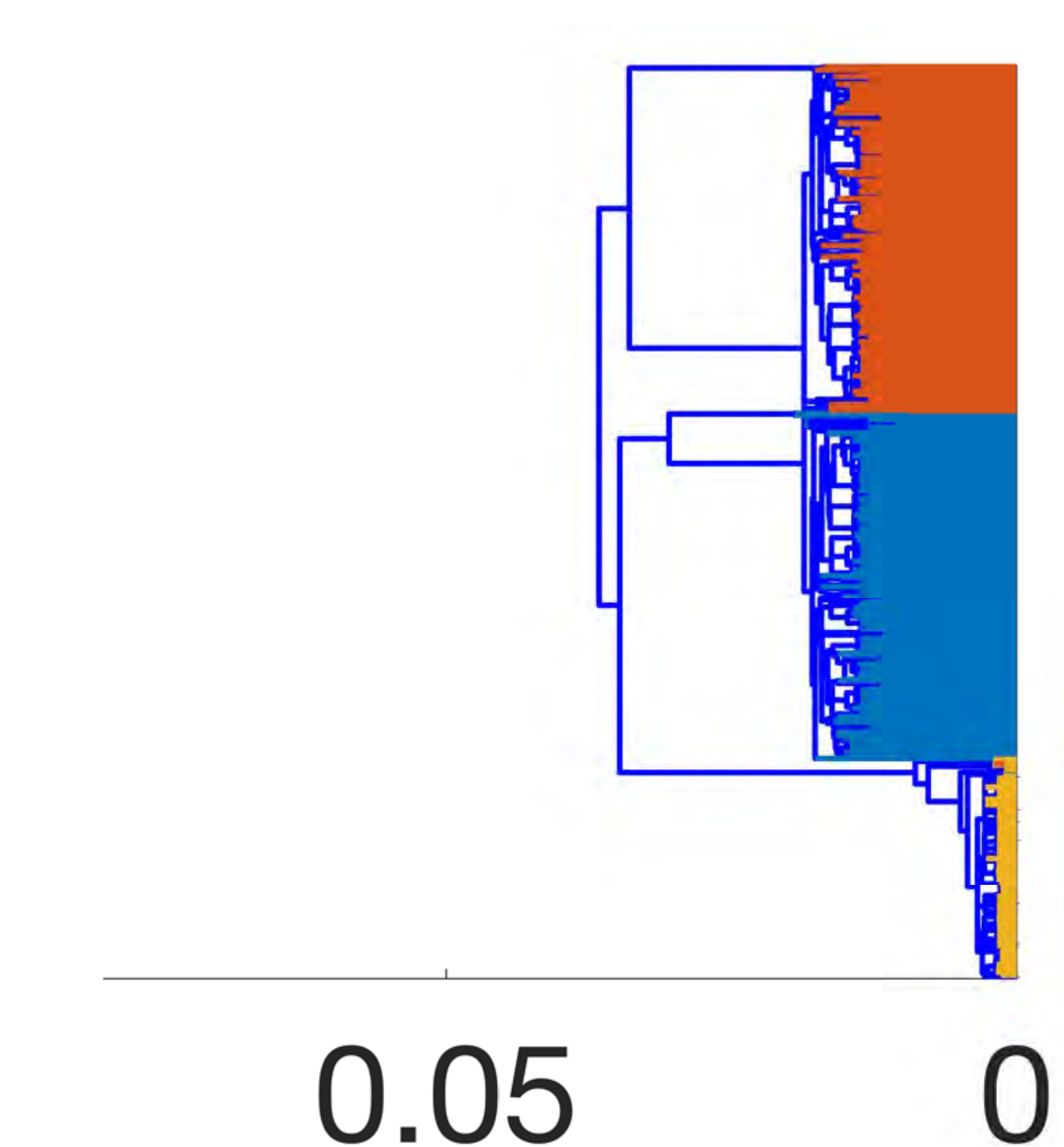}
            \end{minipage}
        \end{minipage}
        \begin{minipage}[t]{0.09\textwidth}
            \centering
            \begin{minipage}[t]{1\textwidth}
                \centering
                $e$:1/0.6 \\ 
                \includegraphics[width=1\textwidth]{figures/ellip/ellip-ori-seq-1-6.pdf} 
            \end{minipage}
            \begin{minipage}[t][0.55cm][t]{1\textwidth}
                \centering
                \includegraphics[width=1\textwidth]{figures/ellip/gtv-ellip-lamb-5-seq-dend-1-6.pdf}
            \end{minipage}
            \begin{minipage}[t]{1\textwidth}
                \centering
                $e$:0.6/1 \\ 
                \includegraphics[width=1\textwidth]{figures/ellip/ellip-ori-seq-6-1.pdf} 
            \end{minipage}
            \begin{minipage}[t][0.55cm][t]{1\textwidth}
                \centering
                \includegraphics[width=1\textwidth]{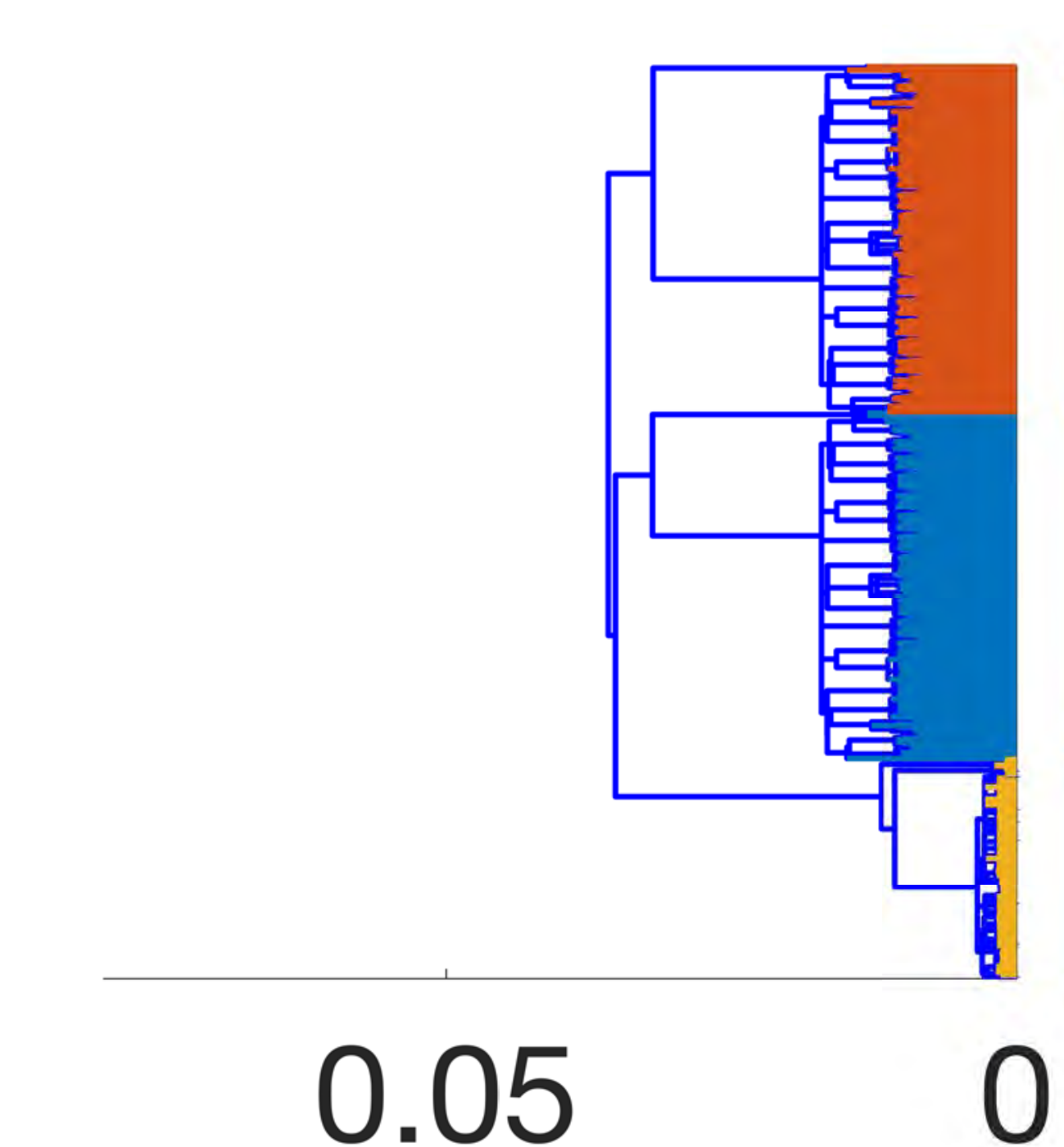}
            \end{minipage}
        \end{minipage}
        \begin{minipage}[t]{0.09\textwidth}
            \centering
            \begin{minipage}[t]{1\textwidth}
                \centering
                $e$:1/0.4 \\ 
                \includegraphics[width=1\textwidth]{figures/ellip/ellip-ori-seq-1-4.pdf} 
            \end{minipage}
            \begin{minipage}[t][0.55cm][t]{1\textwidth}
                \centering
                \includegraphics[width=1\textwidth]{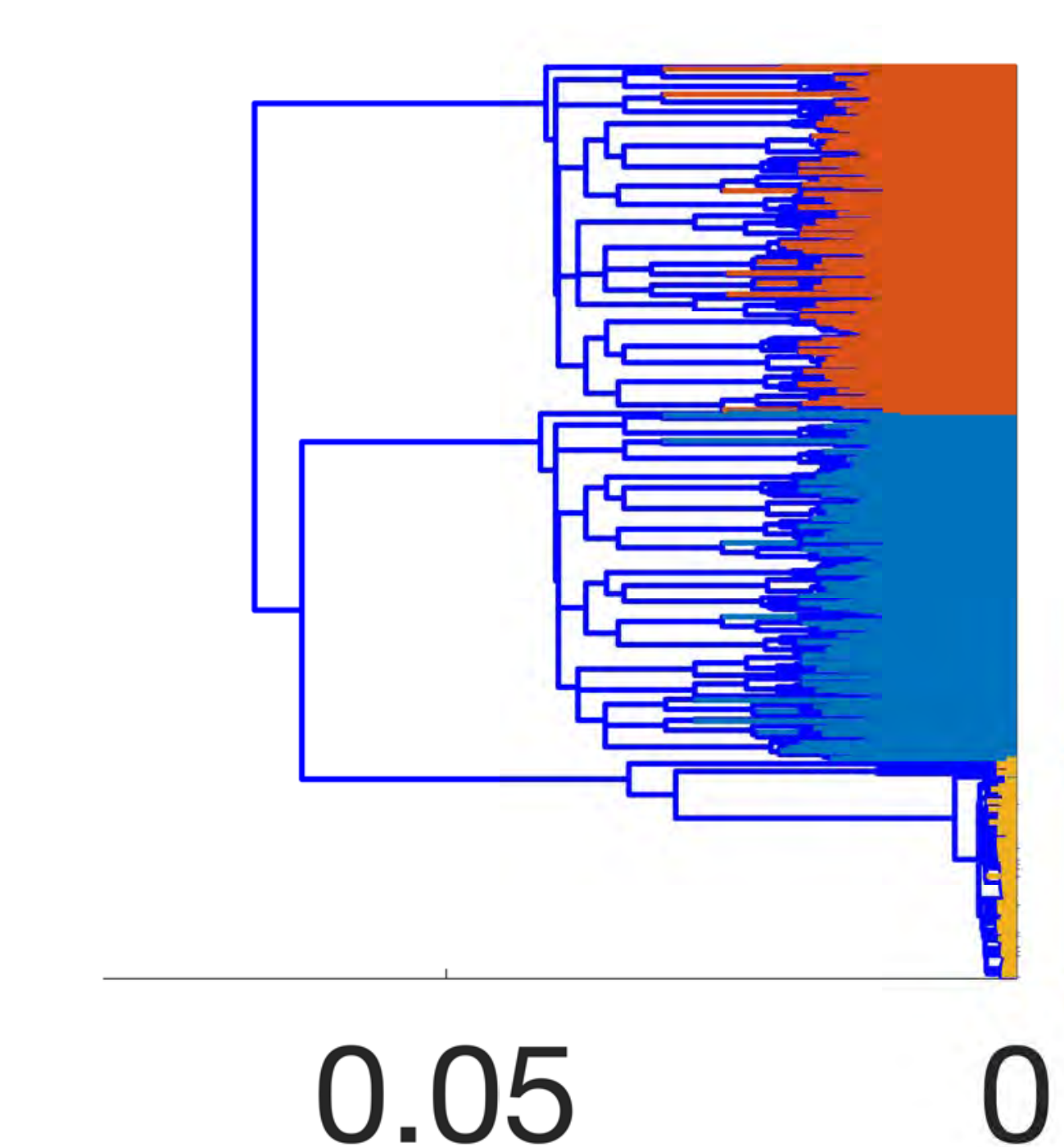}
            \end{minipage}
            \begin{minipage}[t]{1\textwidth}
                \centering
                $e$:0.4/1 \\ 
                \includegraphics[width=1\textwidth]{figures/ellip/ellip-ori-seq-4-1.pdf} 
            \end{minipage}
            \begin{minipage}[t][0.55cm][t]{1\textwidth}
                \centering
                \includegraphics[width=1\textwidth]{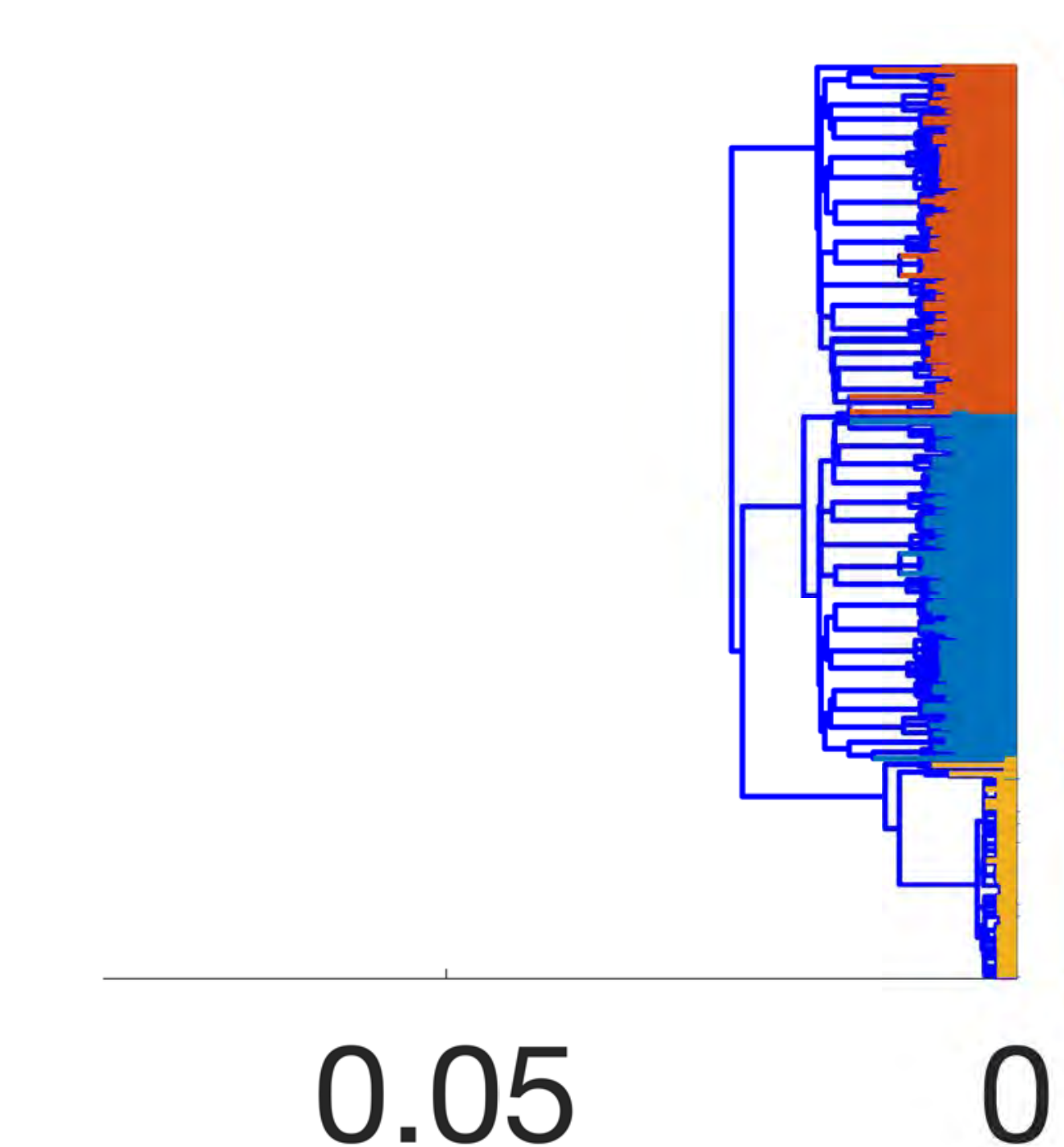}
            \end{minipage}
        \end{minipage}
        \begin{minipage}[t]{0.09\textwidth}
            \centering
            \begin{minipage}[t]{1\textwidth}
                \centering
                $e$:1/0.2 \\ 
                \includegraphics[width=1\textwidth]{figures/ellip/ellip-ori-seq-1-2.pdf} 
            \end{minipage}
            \begin{minipage}[t][0.55cm][t]{1\textwidth}
                \centering
                \includegraphics[width=1\textwidth]{figures/ellip/gtv-ellip-lamb-5-seq-dend-1-2.pdf}
            \end{minipage}
            \begin{minipage}[t]{1\textwidth}
                \centering
                $e$:0.2/1 \\ 
                \includegraphics[width=1\textwidth]{figures/ellip/ellip-ori-seq-2-1.pdf} 
            \end{minipage}
            \begin{minipage}[t][0.55cm][t]{1\textwidth}
                \centering
                \includegraphics[width=1\textwidth]{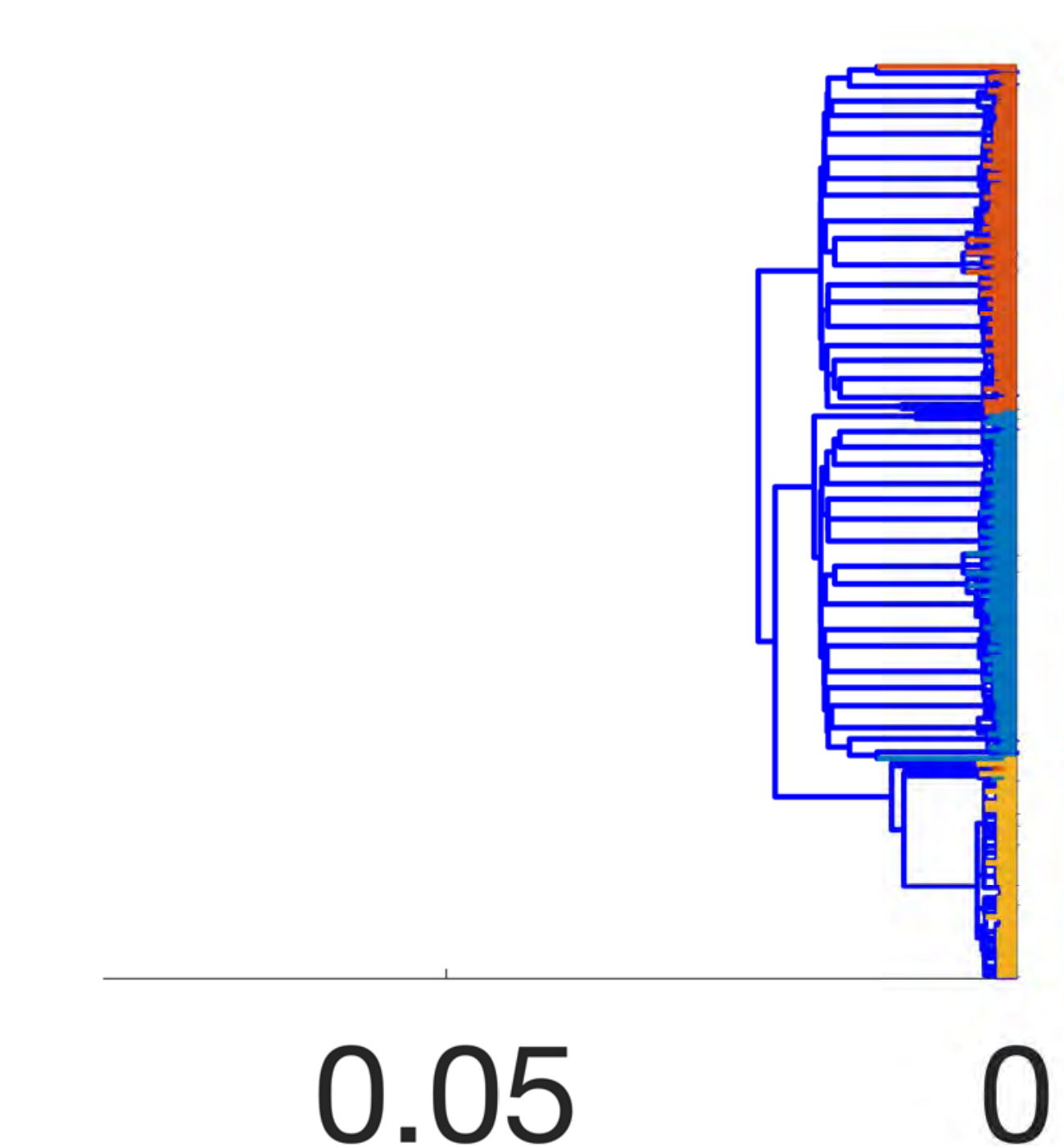}
            \end{minipage}
        \end{minipage}
    } }
    
    \subfloat[WT2]{
    \fbox{
        \begin{minipage}[t]{0.09\textwidth}
            \centering
            \begin{minipage}[t]{1\textwidth}
                \centering
                $e$:1/1 \\ 
                \includegraphics[width=1\textwidth]{figures/ellip/ellip-ori-seq-1-10.pdf} 
            \end{minipage}
            \begin{minipage}[t][0.55cm][t]{1\textwidth}
                \centering
                \includegraphics[width=1\textwidth]{figures/ellip/wt2-fix-emd-ellip-seq-dend-1-10.pdf}
            \end{minipage}
            \begin{minipage}[t]{1\textwidth}
                \centering
                $e$:1/1 \\ 
                \includegraphics[width=1\textwidth]{figures/ellip/ellip-ori-seq-1-10.pdf} 
            \end{minipage}
            \begin{minipage}[t][0.55cm][t]{1\textwidth}
                \centering
                \includegraphics[width=1\textwidth]{figures/ellip/wt2-fix-emd-ellip-seq-dend-1-10.pdf}
            \end{minipage}
        \end{minipage}
        \begin{minipage}[t]{0.09\textwidth}
            \centering
            \begin{minipage}[t]{1\textwidth}
                \centering
                $e$:1/0.8 \\ 
                \includegraphics[width=1\textwidth]{figures/ellip/ellip-ori-seq-1-8.pdf} 
            \end{minipage}
            \begin{minipage}[t][0.55cm][t]{1\textwidth}
                \centering
                \includegraphics[width=1\textwidth]{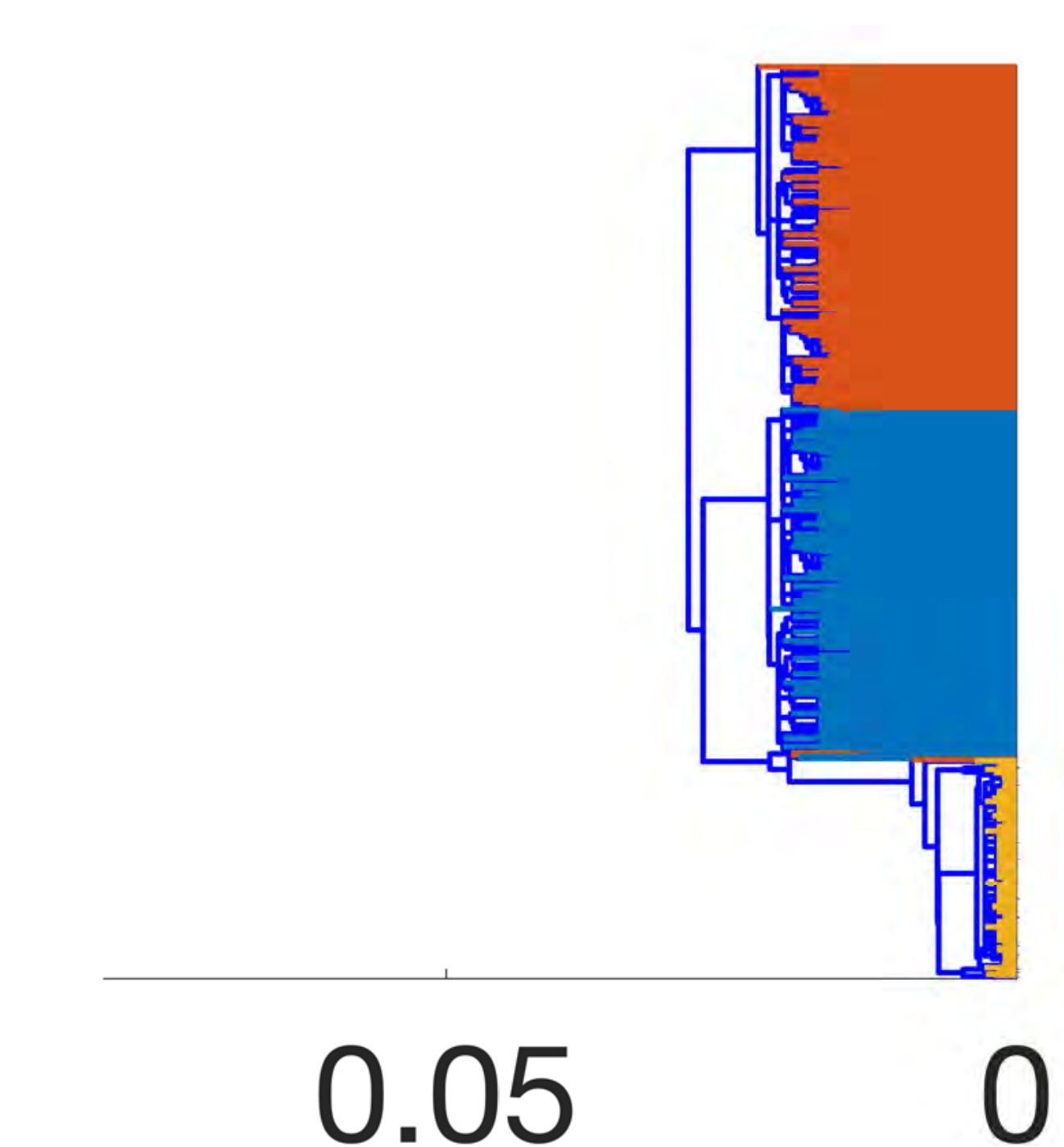}
            \end{minipage}
            \begin{minipage}[t]{1\textwidth}
                \centering
                $e$:0.8/1 \\ 
                \includegraphics[width=1\textwidth]{figures/ellip/ellip-ori-seq-8-1.pdf} 
            \end{minipage}
            \begin{minipage}[t][0.55cm][t]{1\textwidth}
                \centering
                \includegraphics[width=1\textwidth]{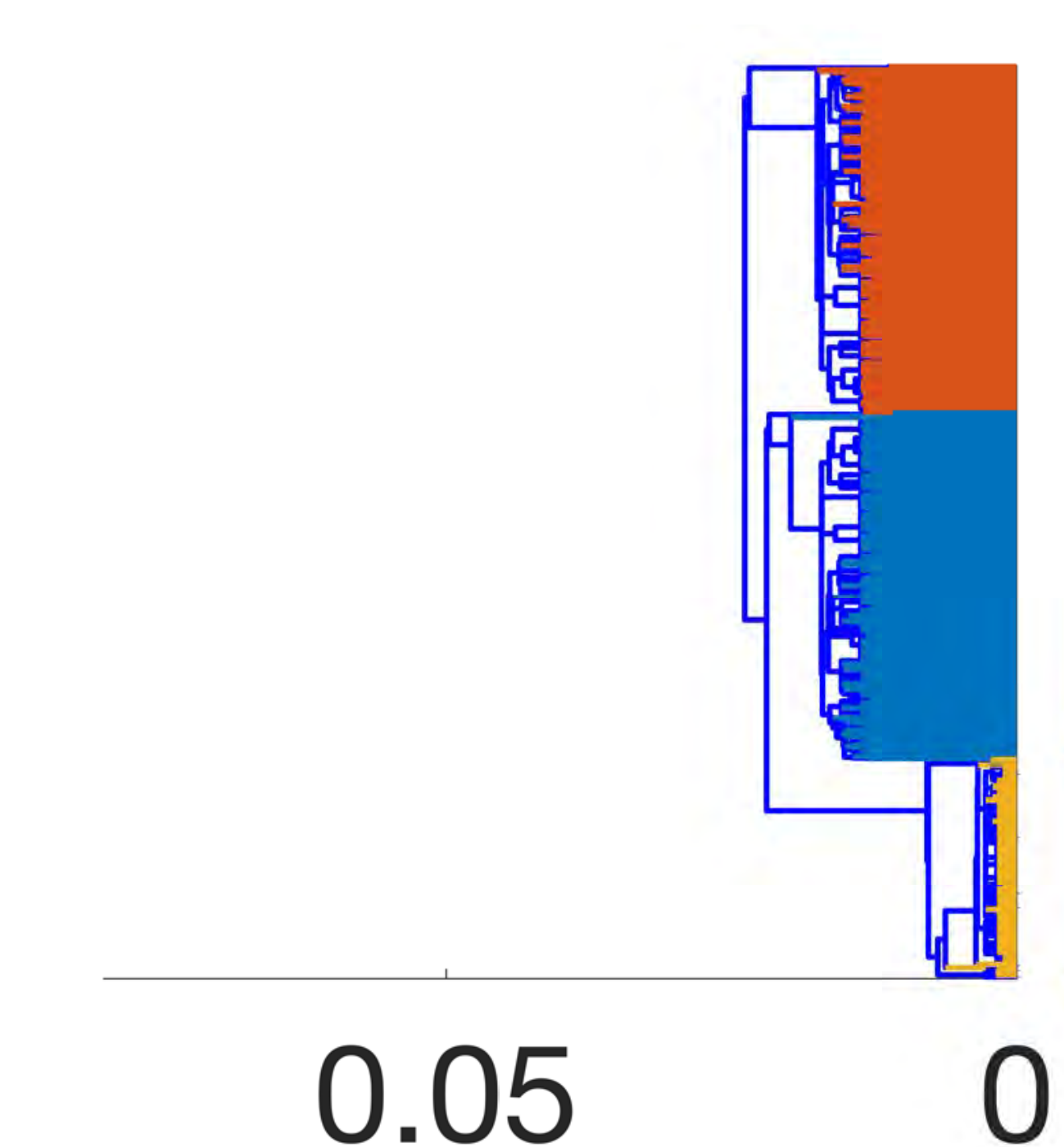}
            \end{minipage}
        \end{minipage}
        \begin{minipage}[t]{0.09\textwidth}
            \centering
            \begin{minipage}[t]{1\textwidth}
                \centering
                $e$:1/0.6 \\ 
                \includegraphics[width=1\textwidth]{figures/ellip/ellip-ori-seq-1-6.pdf} 
            \end{minipage}
            \begin{minipage}[t][0.55cm][t]{1\textwidth}
                \centering
                \includegraphics[width=1\textwidth]{figures/ellip/wt2-fix-emd-ellip-seq-dend-1-6.pdf}
            \end{minipage}
            \begin{minipage}[t]{1\textwidth}
                \centering
                $e$:0.6/1 \\ 
                \includegraphics[width=1\textwidth]{figures/ellip/ellip-ori-seq-6-1.pdf} 
            \end{minipage}
            \begin{minipage}[t][0.55cm][t]{1\textwidth}
                \centering
                \includegraphics[width=1\textwidth]{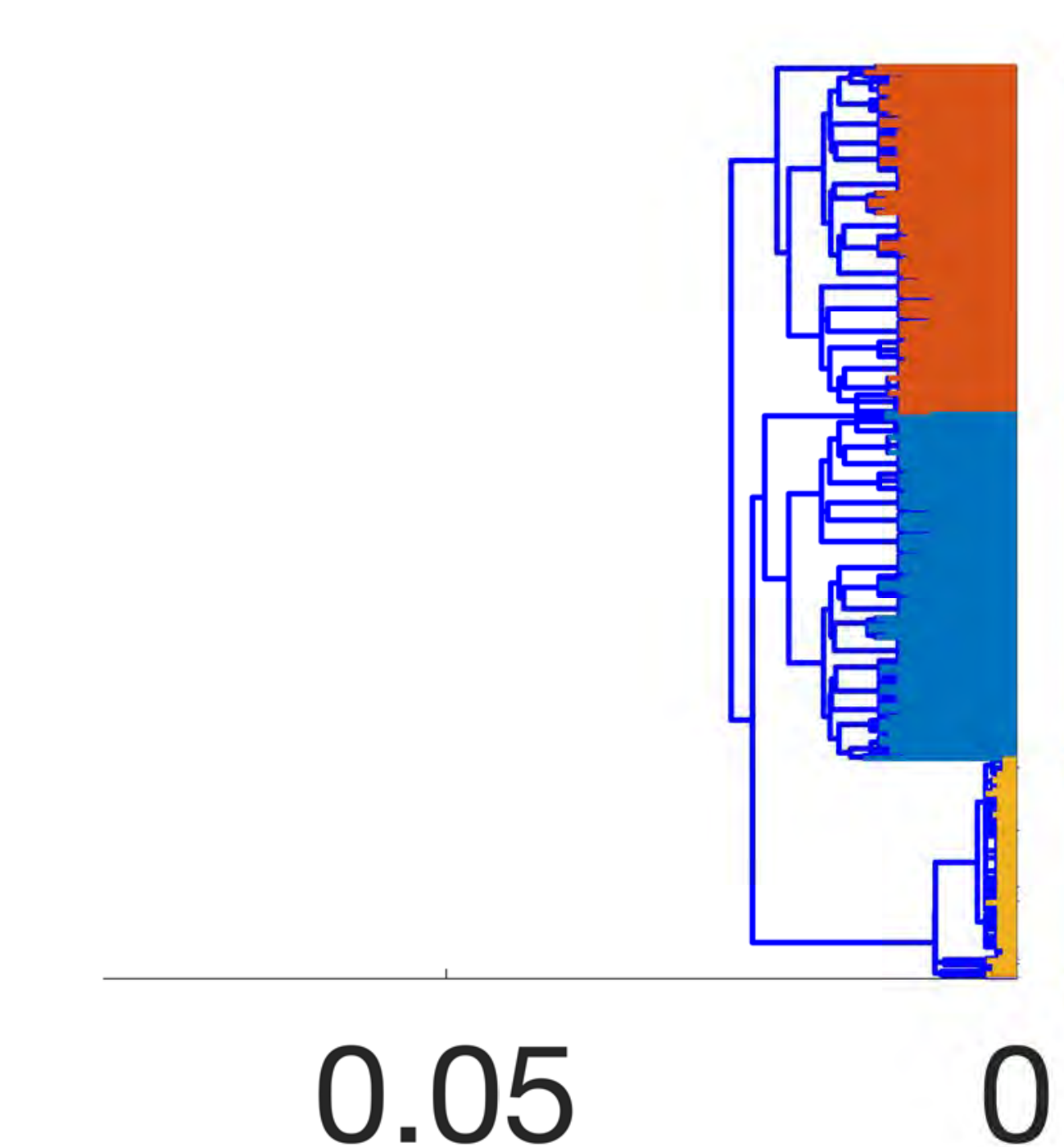}
            \end{minipage}
        \end{minipage}
        \begin{minipage}[t]{0.09\textwidth}
            \centering
            \begin{minipage}[t]{1\textwidth}
                \centering
                $e$:1/0.4 \\ 
                \includegraphics[width=1\textwidth]{figures/ellip/ellip-ori-seq-1-4.pdf} 
            \end{minipage}
            \begin{minipage}[t][0.55cm][t]{1\textwidth}
                \centering
                \includegraphics[width=1\textwidth]{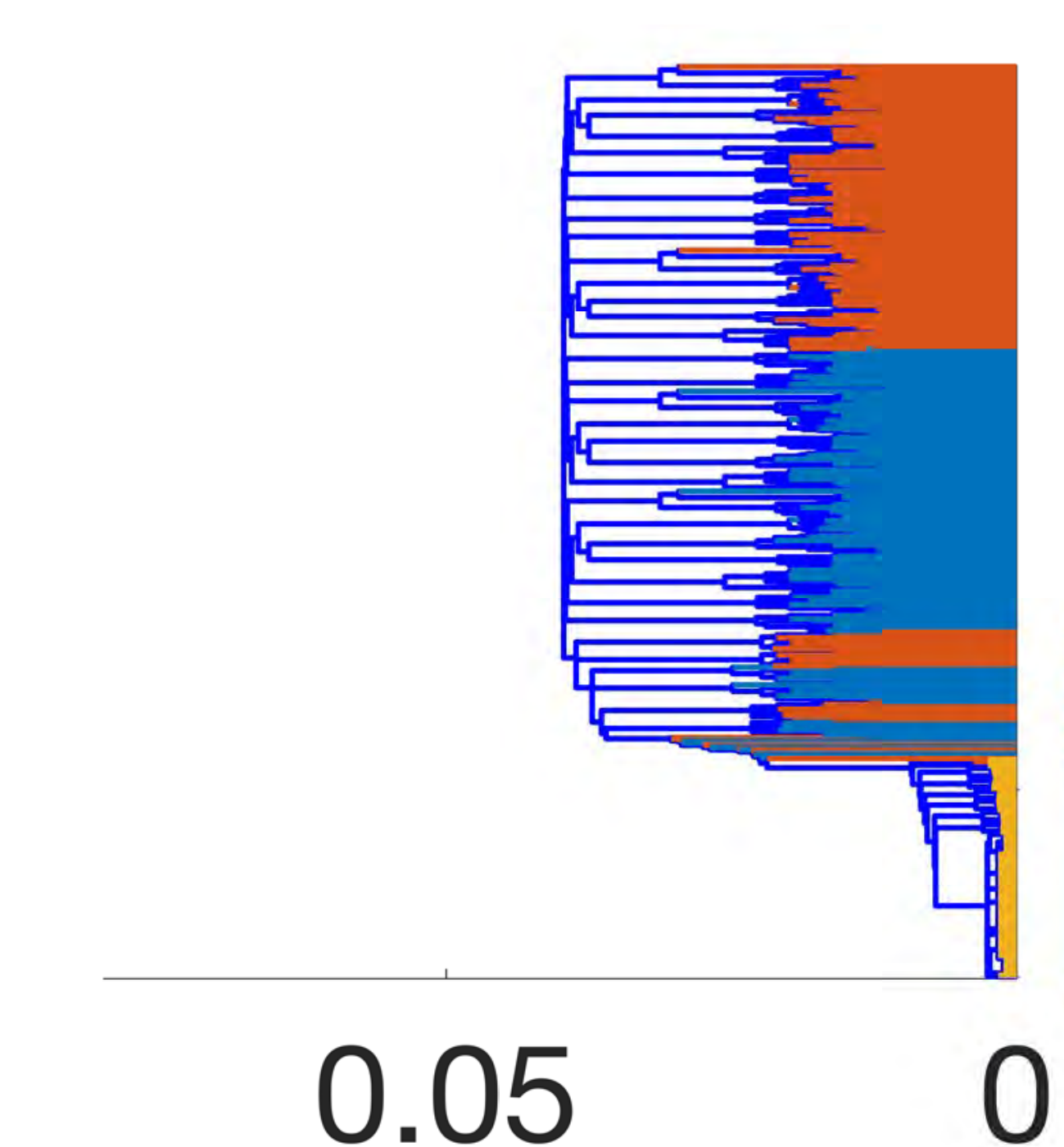}
            \end{minipage}
            \begin{minipage}[t]{1\textwidth}
                \centering
                $e$:0.4/1 \\ 
                \includegraphics[width=1\textwidth]{figures/ellip/ellip-ori-seq-4-1.pdf} 
            \end{minipage}
            \begin{minipage}[t][0.55cm][t]{1\textwidth}
                \centering
                \includegraphics[width=1\textwidth]{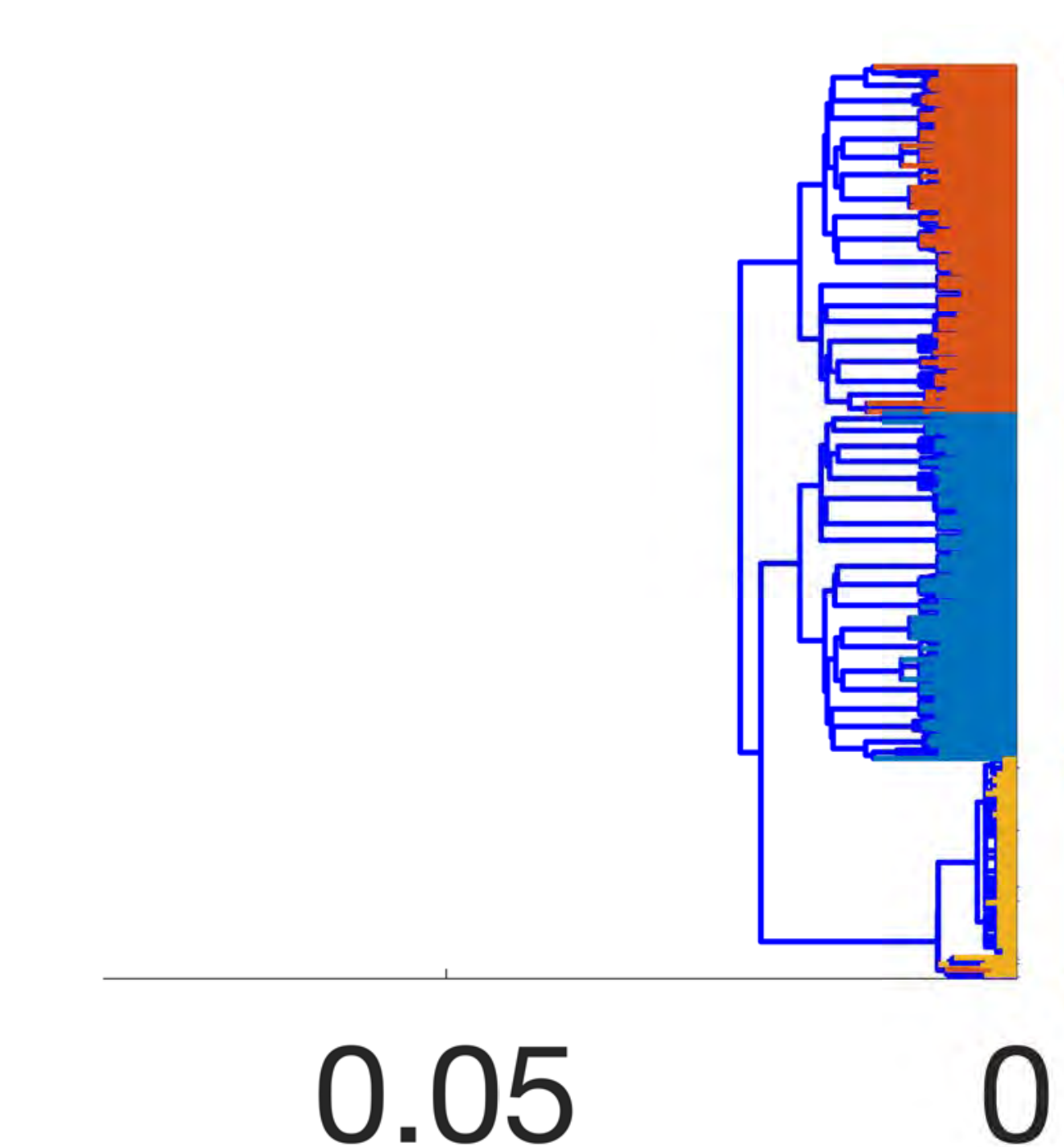}
            \end{minipage}
        \end{minipage}
        \begin{minipage}[t]{0.09\textwidth}
            \centering
            \begin{minipage}[t]{1\textwidth}
                \centering
                $e$:1/0.2 \\ 
                \includegraphics[width=1\textwidth]{figures/ellip/ellip-ori-seq-1-2.pdf} 
            \end{minipage}
            \begin{minipage}[t][0.55cm][t]{1\textwidth}
                \centering
                \includegraphics[width=1\textwidth]{figures/ellip/wt2-fix-emd-ellip-seq-dend-1-2.pdf}
            \end{minipage}
            \begin{minipage}[t]{1\textwidth}
                \centering
                $e$:0.2/1 \\ 
                \includegraphics[width=1\textwidth]{figures/ellip/ellip-ori-seq-2-1.pdf} 
            \end{minipage}
            \begin{minipage}[t][0.55cm][t]{1\textwidth}
                \centering
                \includegraphics[width=1\textwidth]{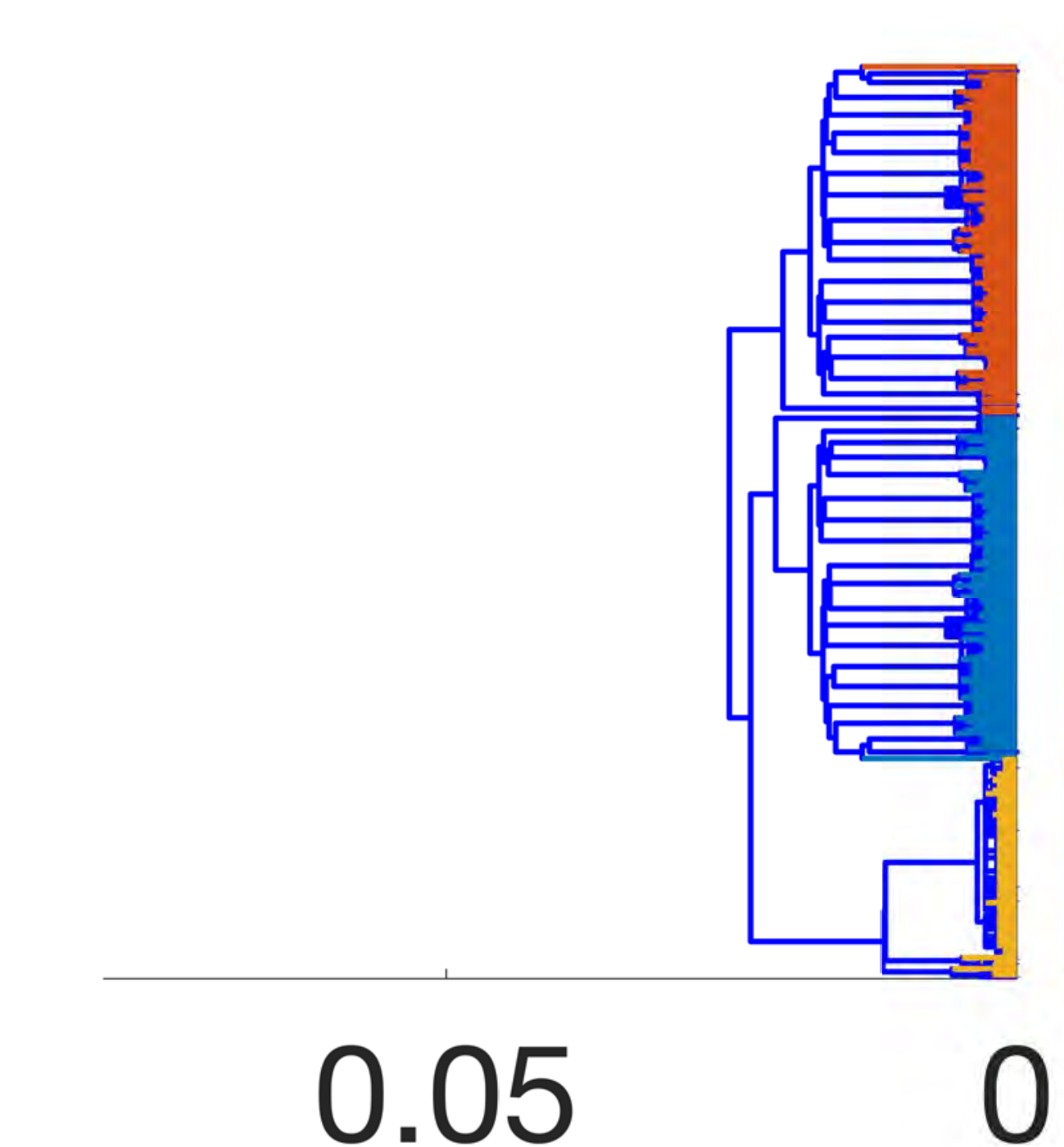}
            \end{minipage}
        \end{minipage}
    } }
    \subfloat[WT1]{
    \fbox{
        \begin{minipage}[t]{0.09\textwidth}
            \centering
            \begin{minipage}[t]{1\textwidth}
                \centering
                $e$:1/1 \\ 
                \includegraphics[width=1\textwidth]{figures/ellip/ellip-ori-seq-1-10.pdf} 
            \end{minipage}
            \begin{minipage}[t][0.55cm][t]{1\textwidth}
                \centering
                \includegraphics[width=1\textwidth]{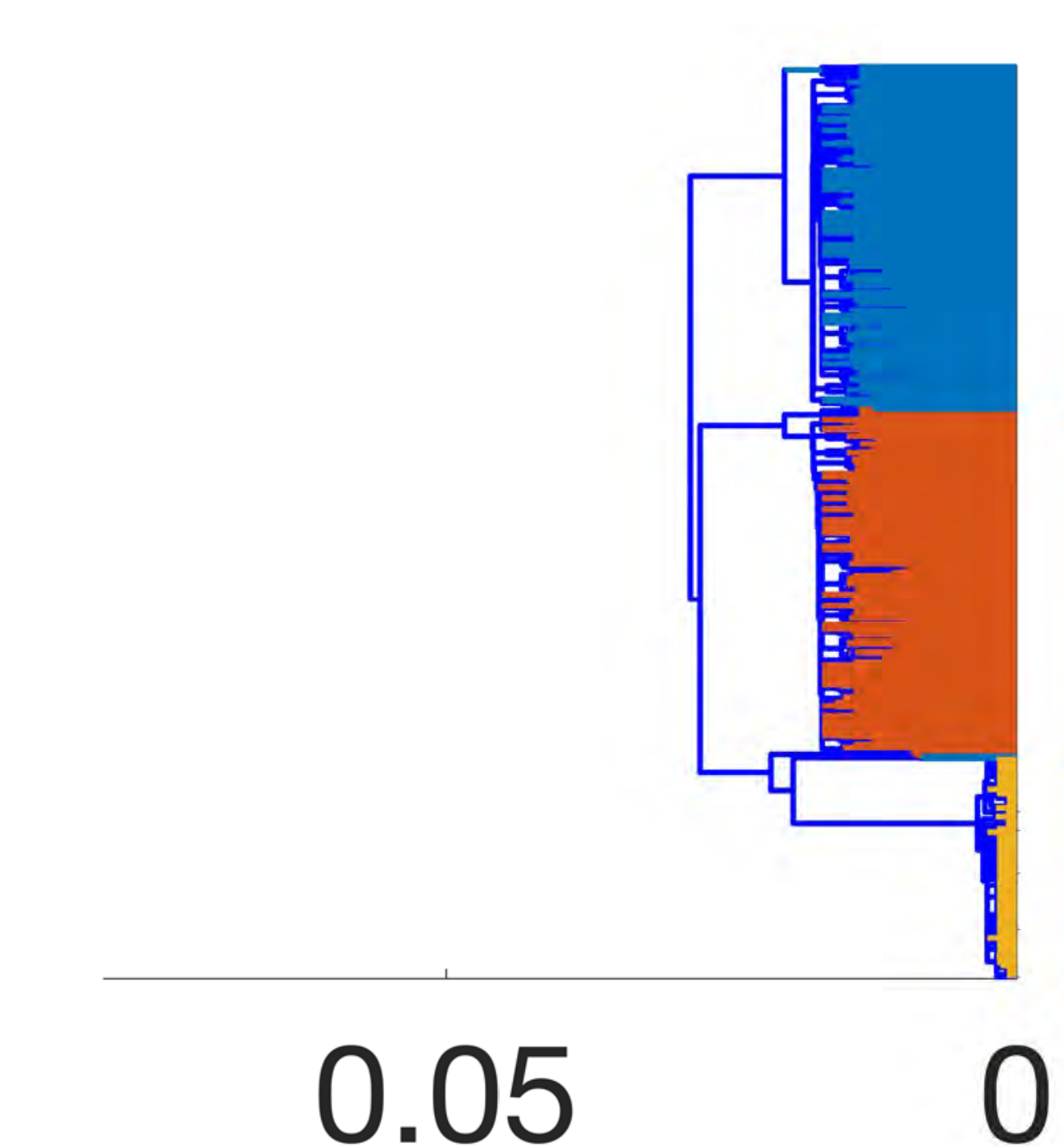}
            \end{minipage}
            \begin{minipage}[t]{1\textwidth}
                \centering
                $e$:1/1 \\ 
                \includegraphics[width=1\textwidth]{figures/ellip/ellip-ori-seq-1-10.pdf} 
            \end{minipage}
            \begin{minipage}[t][0.55cm][t]{1\textwidth}
                \centering
                \includegraphics[width=1\textwidth]{figures/ellip/wt1-fix-emd-ellip-seq-dend-1-10.pdf}
            \end{minipage}
        \end{minipage}
        \begin{minipage}[t]{0.09\textwidth}
            \centering
            \begin{minipage}[t]{1\textwidth}
                \centering
                $e$:1/0.8 \\ 
                \includegraphics[width=1\textwidth]{figures/ellip/ellip-ori-seq-1-8.pdf} 
            \end{minipage}
            \begin{minipage}[t][0.55cm][t]{1\textwidth}
                \centering
                \includegraphics[width=1\textwidth]{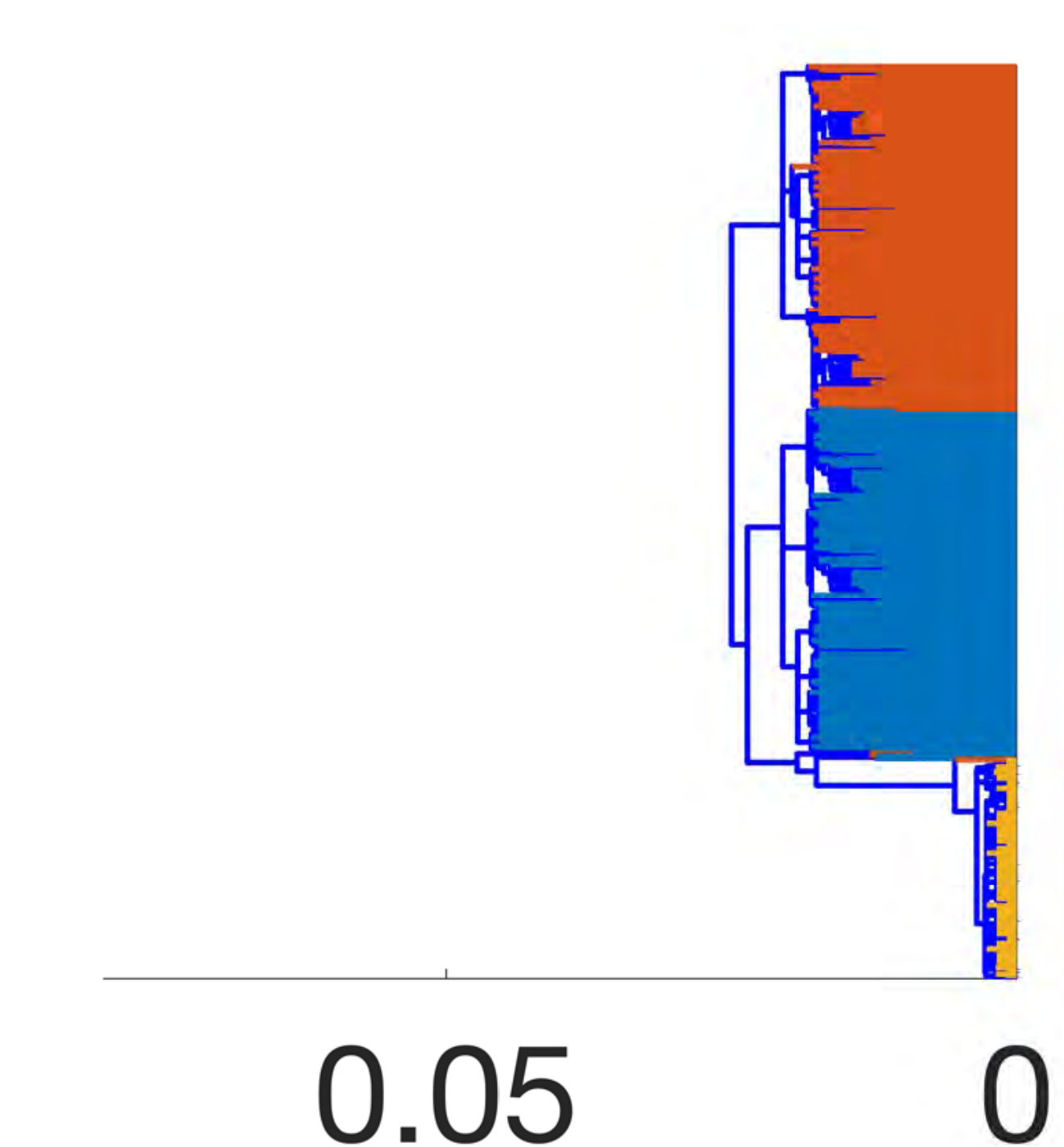}
            \end{minipage}
            \begin{minipage}[t]{1\textwidth}
                \centering
                $e$:0.8/1 \\ 
                \includegraphics[width=1\textwidth]{figures/ellip/ellip-ori-seq-8-1.pdf} 
            \end{minipage}
            \begin{minipage}[t][0.55cm][t]{1\textwidth}
                \centering
                \includegraphics[width=1\textwidth]{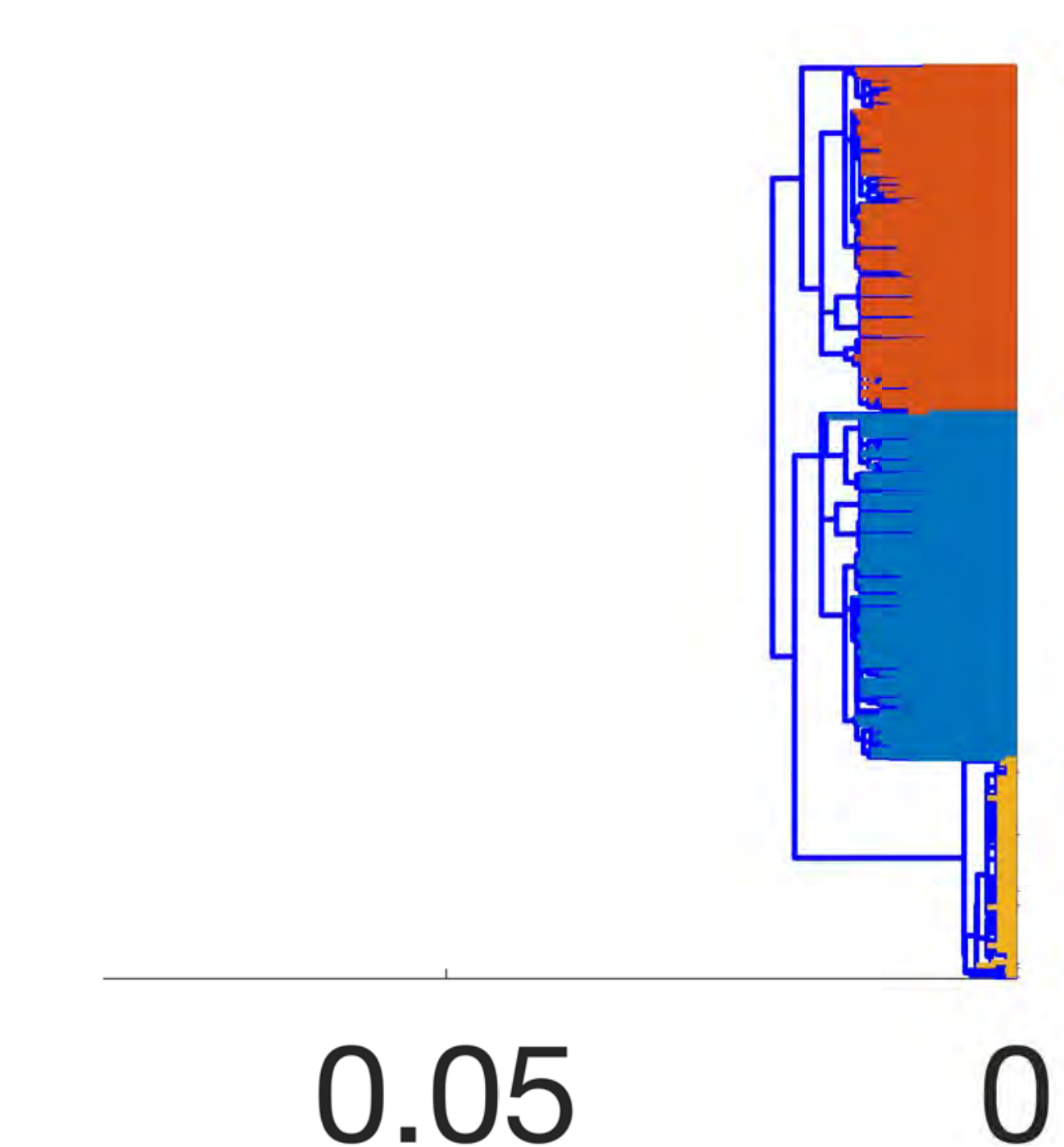}
            \end{minipage}
        \end{minipage}
        \begin{minipage}[t]{0.09\textwidth}
            \centering
            \begin{minipage}[t]{1\textwidth}
                \centering
                $e$:1/0.6 \\ 
                \includegraphics[width=1\textwidth]{figures/ellip/ellip-ori-seq-1-6.pdf} 
            \end{minipage}
            \begin{minipage}[t][0.55cm][t]{1\textwidth}
                \centering
                \includegraphics[width=1\textwidth]{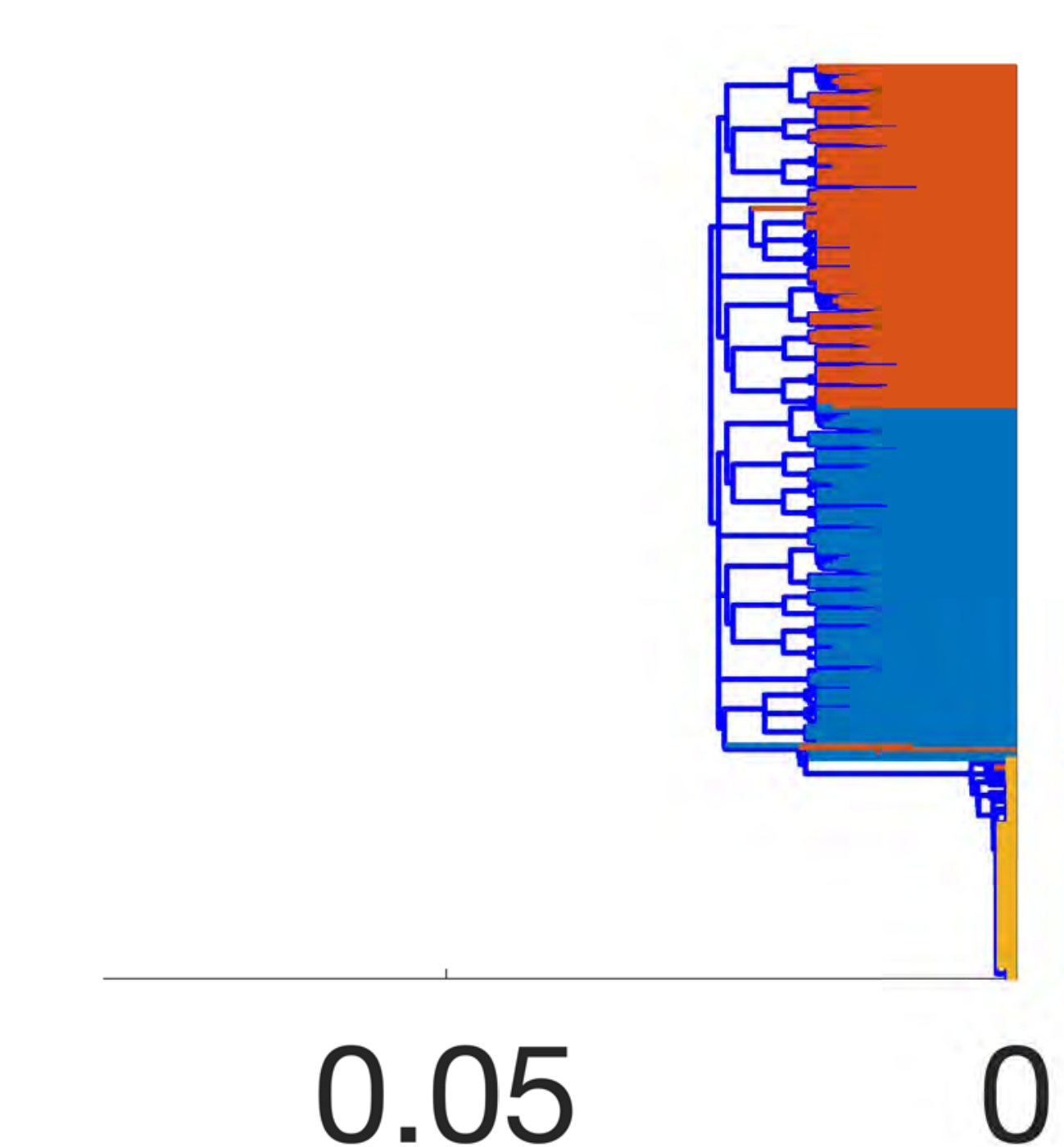}
            \end{minipage}
            \begin{minipage}[t]{1\textwidth}
                \centering
                $e$:0.6/1 \\ 
                \includegraphics[width=1\textwidth]{figures/ellip/ellip-ori-seq-6-1.pdf} 
            \end{minipage}
            \begin{minipage}[t][0.55cm][t]{1\textwidth}
                \centering
                \includegraphics[width=1\textwidth]{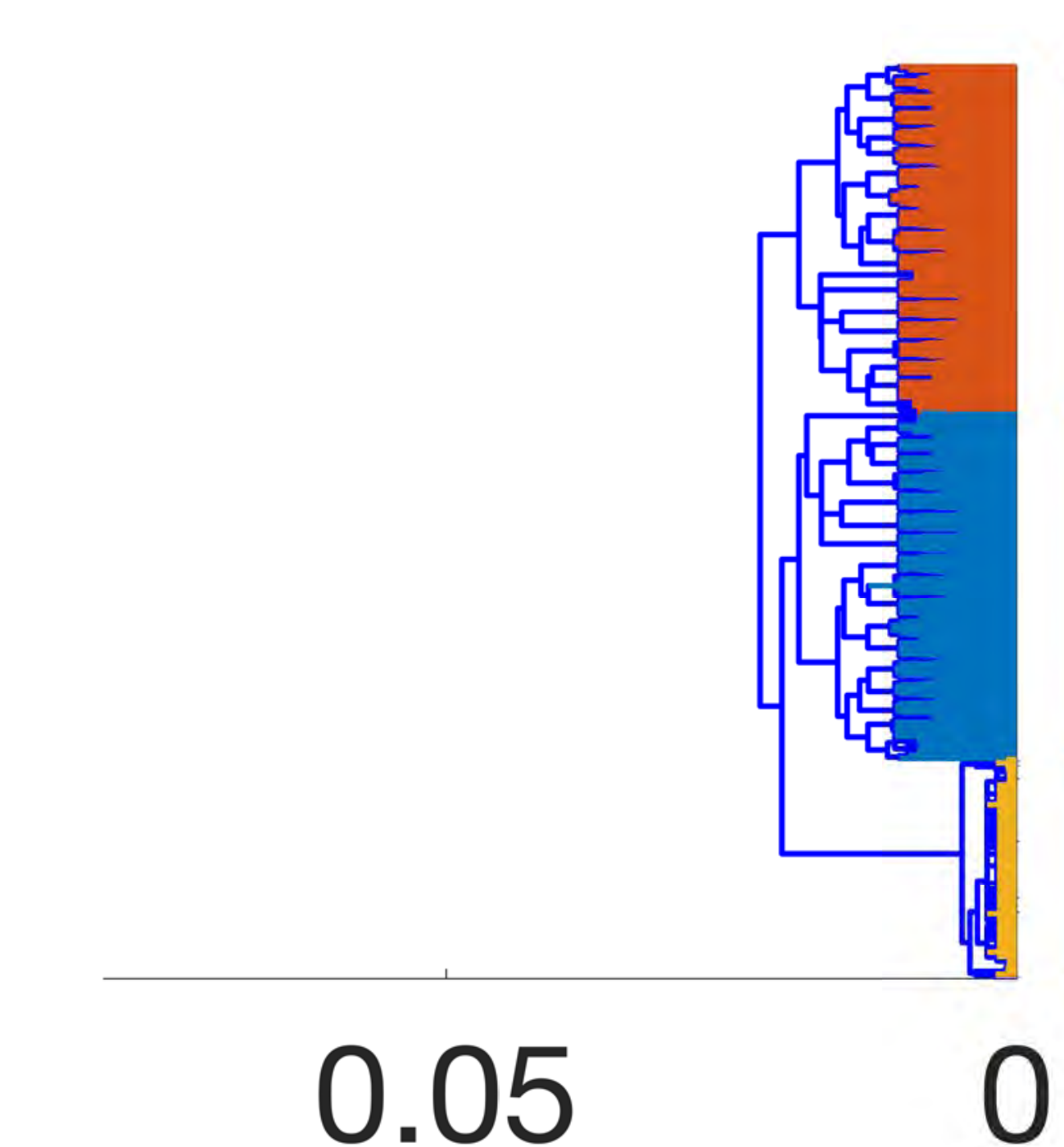}
            \end{minipage}
        \end{minipage}
        \begin{minipage}[t]{0.09\textwidth}
            \centering
            \begin{minipage}[t]{1\textwidth}
                \centering
                $e$:1/0.4 \\ 
                \includegraphics[width=1\textwidth]{figures/ellip/ellip-ori-seq-1-4.pdf} 
            \end{minipage}
            \begin{minipage}[t][0.55cm][t]{1\textwidth}
                \centering
                \includegraphics[width=1\textwidth]{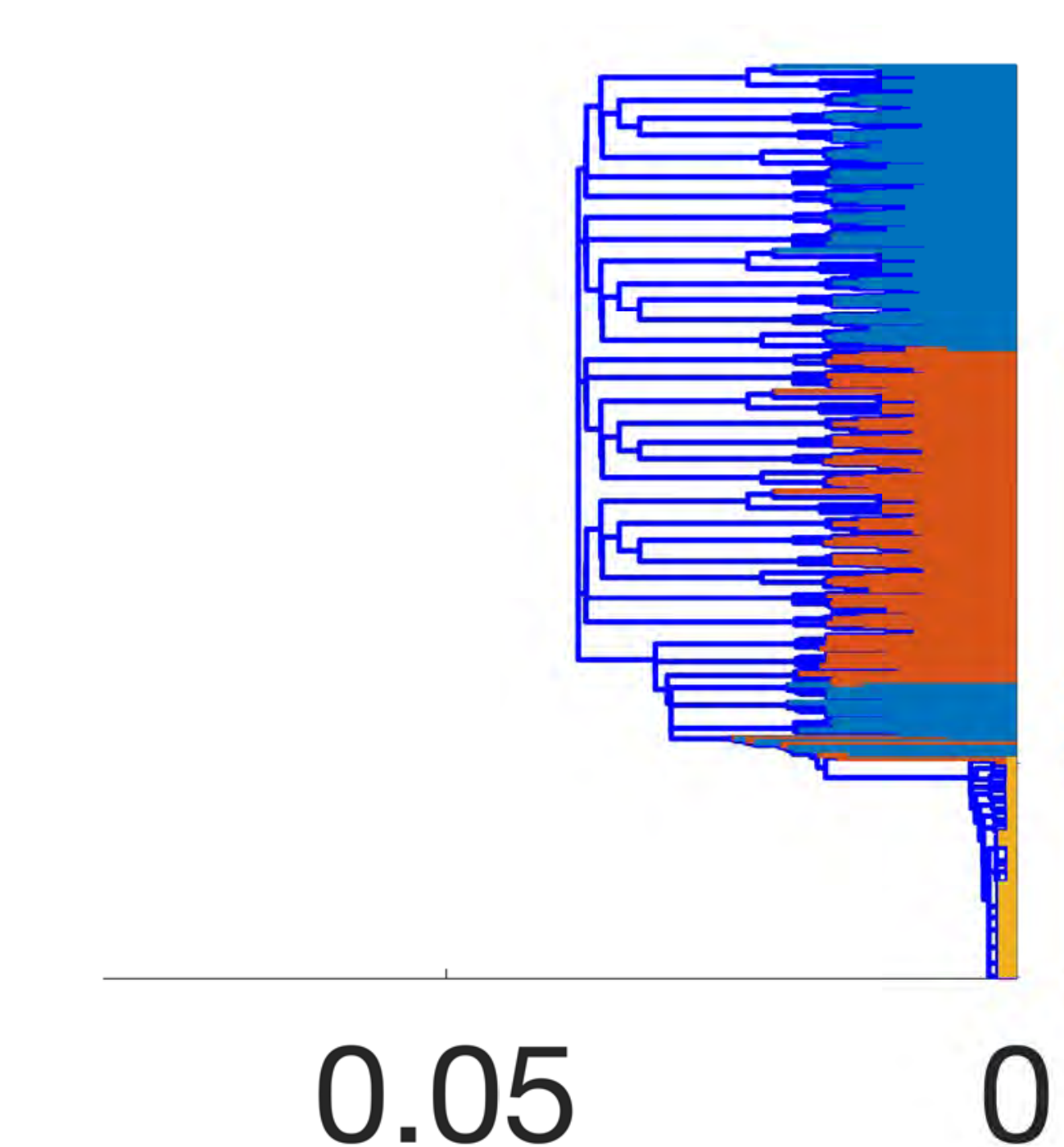}
            \end{minipage}
            \begin{minipage}[t]{1\textwidth}
                \centering
                $e$:0.4/1 \\ 
                \includegraphics[width=1\textwidth]{figures/ellip/ellip-ori-seq-4-1.pdf} 
            \end{minipage}
            \begin{minipage}[t][0.55cm][t]{1\textwidth}
                \centering
                \includegraphics[width=1\textwidth]{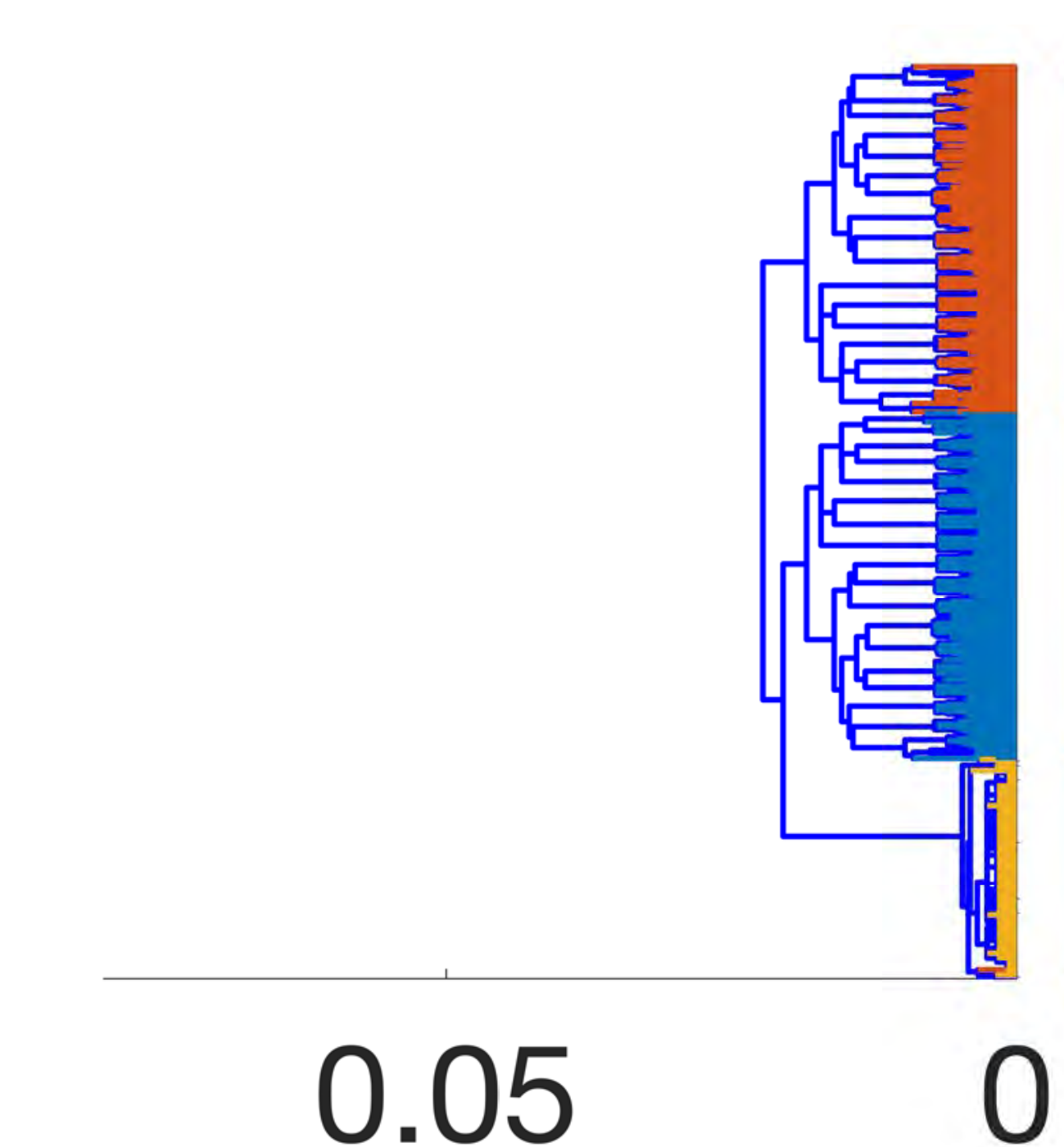}
            \end{minipage}
        \end{minipage}
        \begin{minipage}[t]{0.09\textwidth}
            \centering
            \begin{minipage}[t]{1\textwidth}
                \centering
                $e$:1/0.2 \\ 
                \includegraphics[width=1\textwidth]{figures/ellip/ellip-ori-seq-1-2.pdf} 
            \end{minipage}
            \begin{minipage}[t][0.55cm][t]{1\textwidth}
                \centering
                \includegraphics[width=1\textwidth]{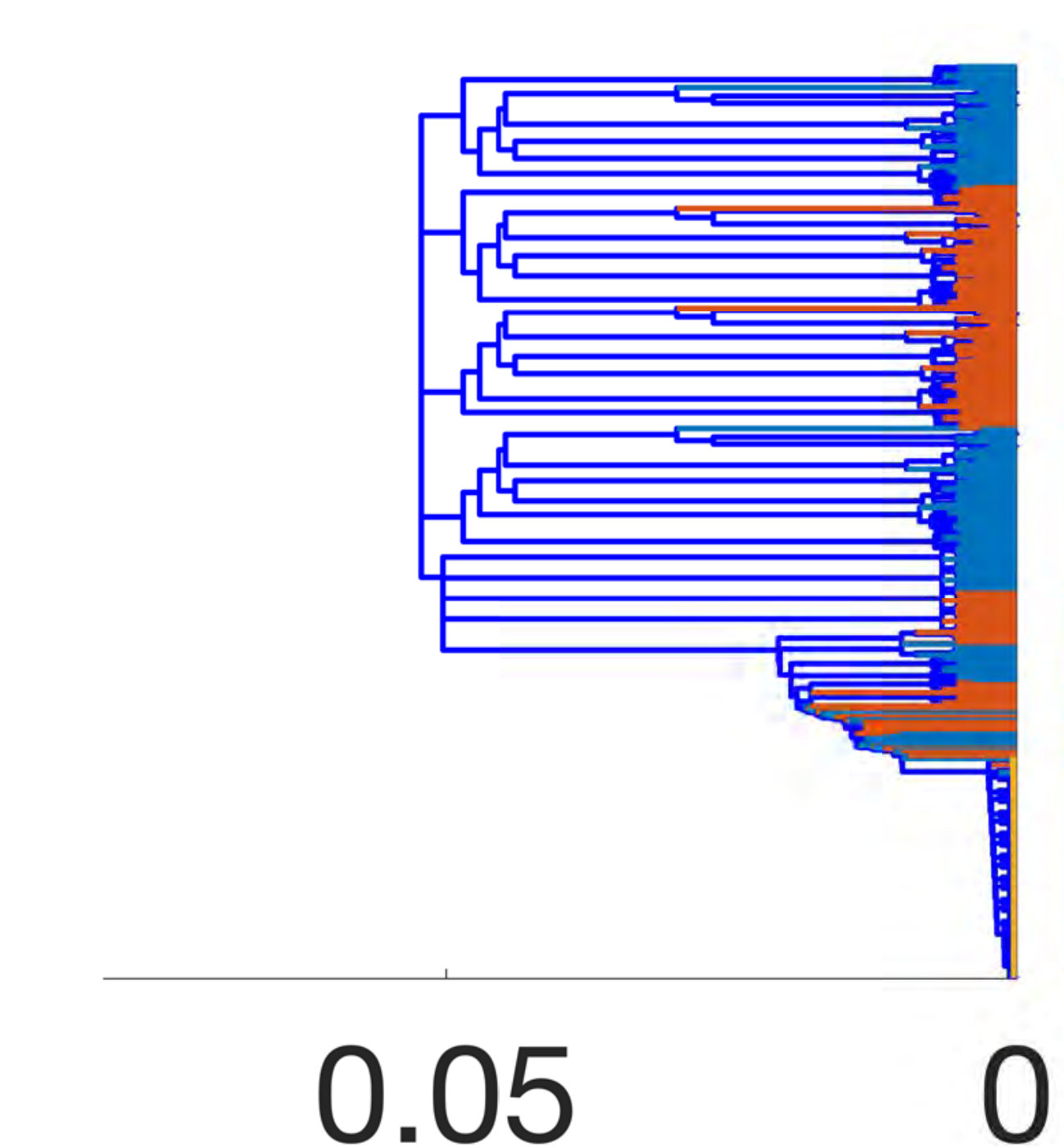}
            \end{minipage}
            \begin{minipage}[t]{1\textwidth}
                \centering
                $e$:0.2/1 \\ 
                \includegraphics[width=1\textwidth]{figures/ellip/ellip-ori-seq-2-1.pdf} 
            \end{minipage}
            \begin{minipage}[t][0.55cm][t]{1\textwidth}
                \centering
                \includegraphics[width=1\textwidth]{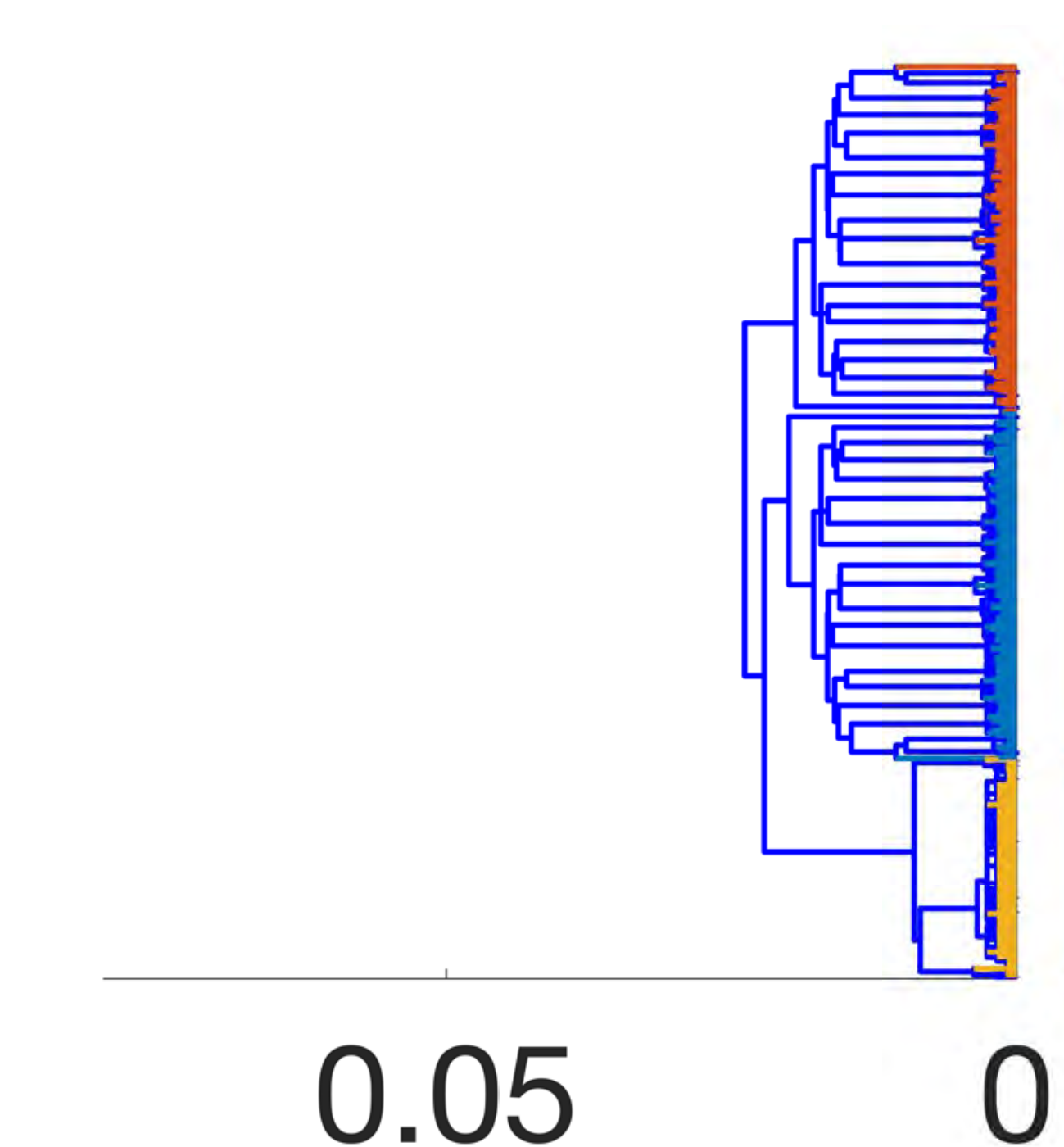}
            \end{minipage}
        \end{minipage}
    } }
    \caption{Chaining effect: influence of geometry. For each subgraph, each (dumbbell-shape) point cloud represents the result after applying the linear transformation $T$ with eccentricity $e$ to the original dataset ($e=1$). All the corresponding dendrograms have the same x-axis limit.}
    \label{fig:supp-ballchains}
\end{figure}

\subsection{Denoising of a spiral}
\label{app:denoising-spiral}
In this example, we analyze a spiral composed of 600 points lying in the square $[-30, 30]^2$ together with 150 outliers (following the uniform distribution). We compare the performance of MS, GT, WT2 and WT1 in the course of 4 iterations. Results are shown in Figure~\ref{fig:supp-spiral}. We see that GT both absorbs outliers faster and resolves the spiral shape with better quality than MS, WT2 and WT1 do. 
\begin{figure}[htb]
    \centering
    \subfloat[MS]{
		\begin{minipage}[t]{0.08\textwidth}
		    \centering
		    $\tau=0$ \\ 
            \includegraphics[width=1\textwidth]{figures/spiral/ms-spiral-0.pdf}
        \end{minipage}
        \begin{minipage}[t]{0.08\textwidth}
		    \centering
		    $\tau=1$ \\ 
            \includegraphics[width=1\textwidth]{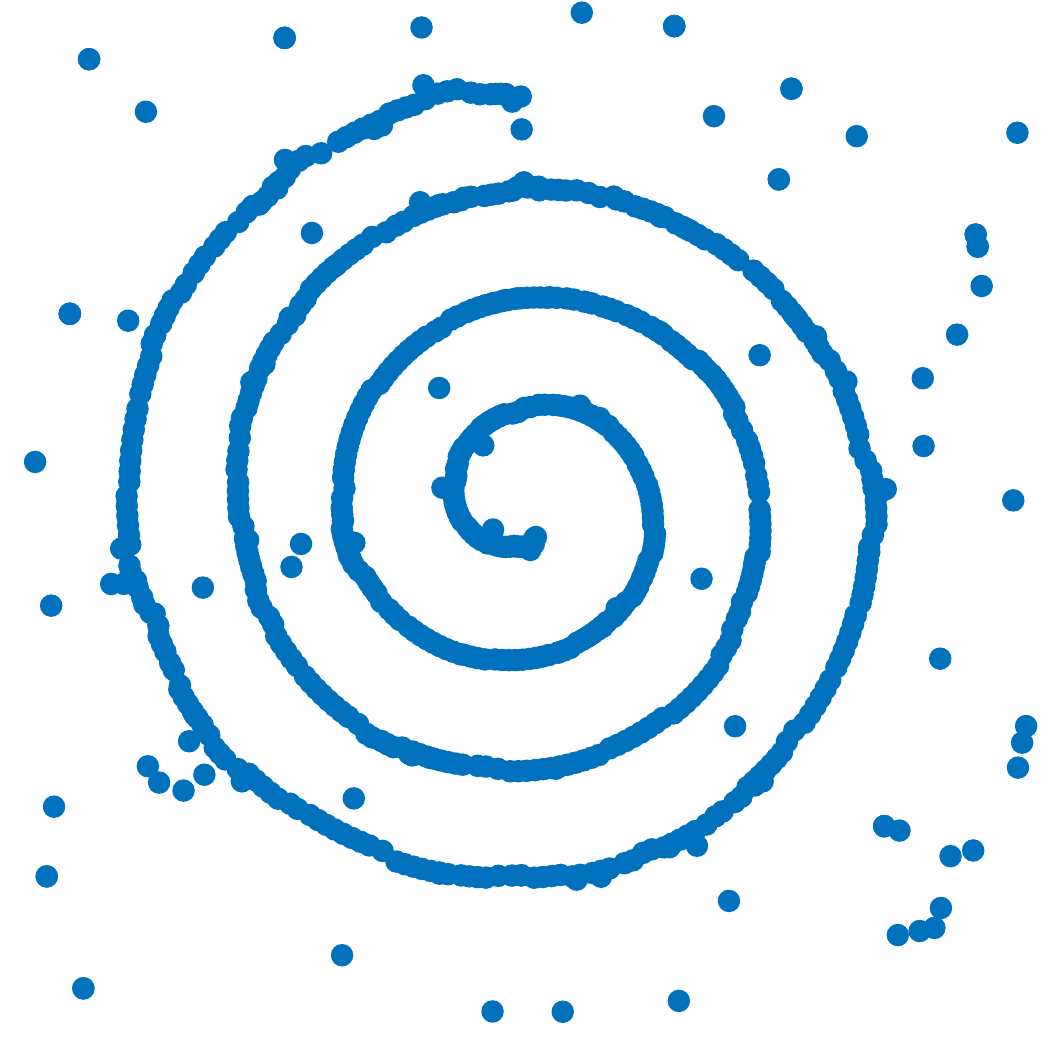}
        \end{minipage}
        \begin{minipage}[t]{0.08\textwidth}
		    \centering
		    $\tau=2$ \\ 
            \includegraphics[width=1\textwidth]{figures/spiral/ms-spiral-2.pdf}
        \end{minipage}
        \begin{minipage}[t]{0.08\textwidth}
		    \centering
		    $\tau=3$ \\ 
            \includegraphics[width=1\textwidth]{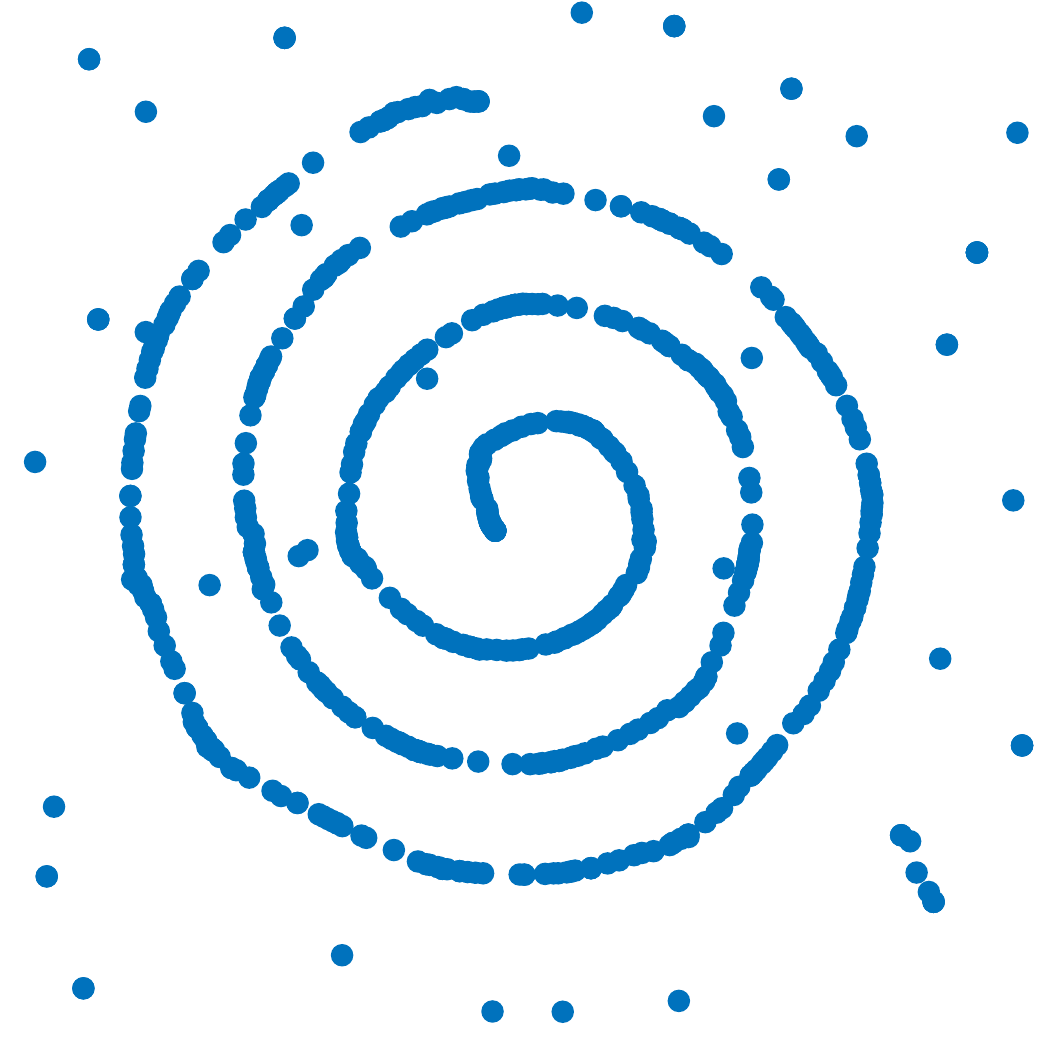}
        \end{minipage}
        \begin{minipage}[t]{0.08\textwidth}
		    \centering
		    $\tau=4$ \\ 
            \includegraphics[width=1\textwidth]{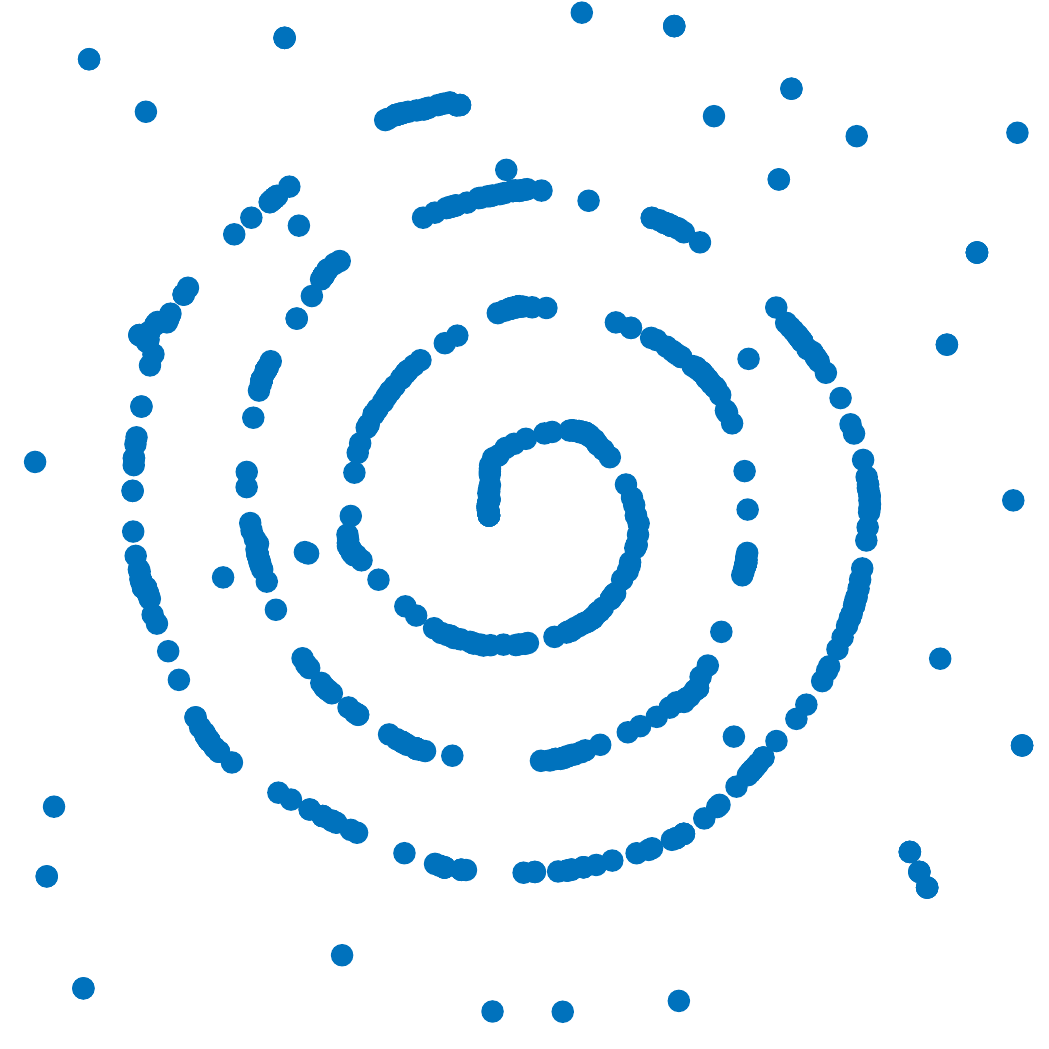}
        \end{minipage}
	}

	\subfloat[GT-$\lambda$-1]{
		\begin{minipage}[t]{0.08\textwidth}
		    \centering
            \includegraphics[width=1\textwidth]{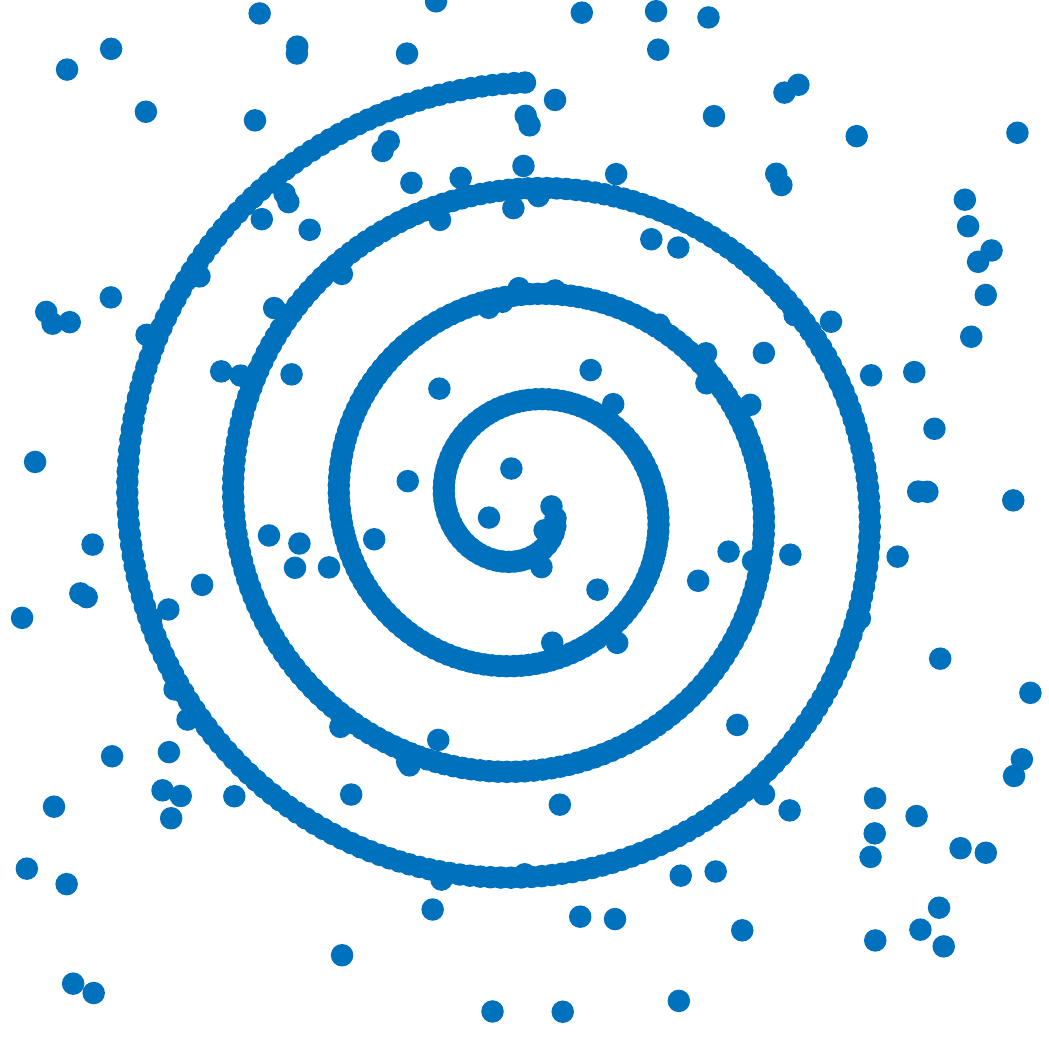}
        \end{minipage}
        \begin{minipage}[t]{0.08\textwidth}
		    \centering
            \includegraphics[width=1\textwidth]{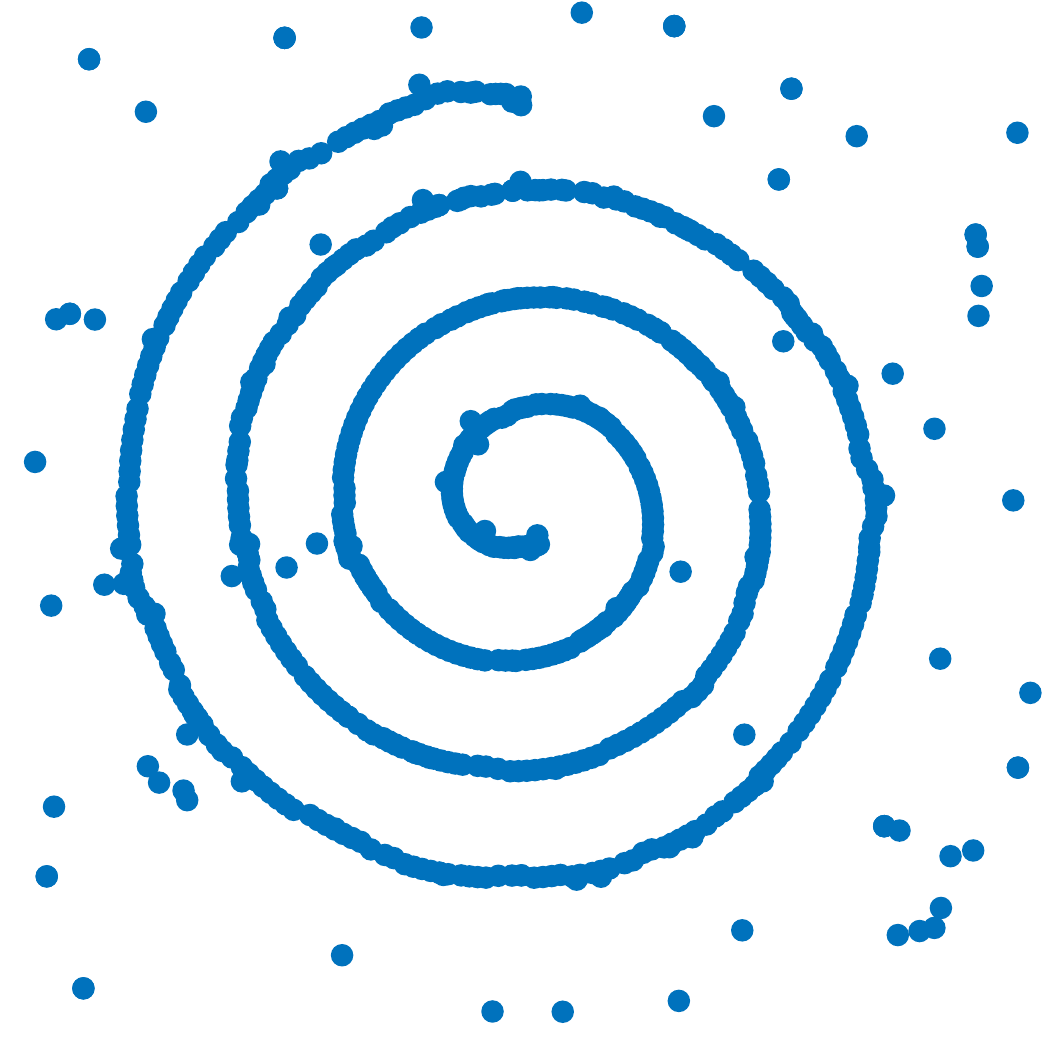}
        \end{minipage}
        \begin{minipage}[t]{0.08\textwidth}
		    \centering
            \includegraphics[width=1\textwidth]{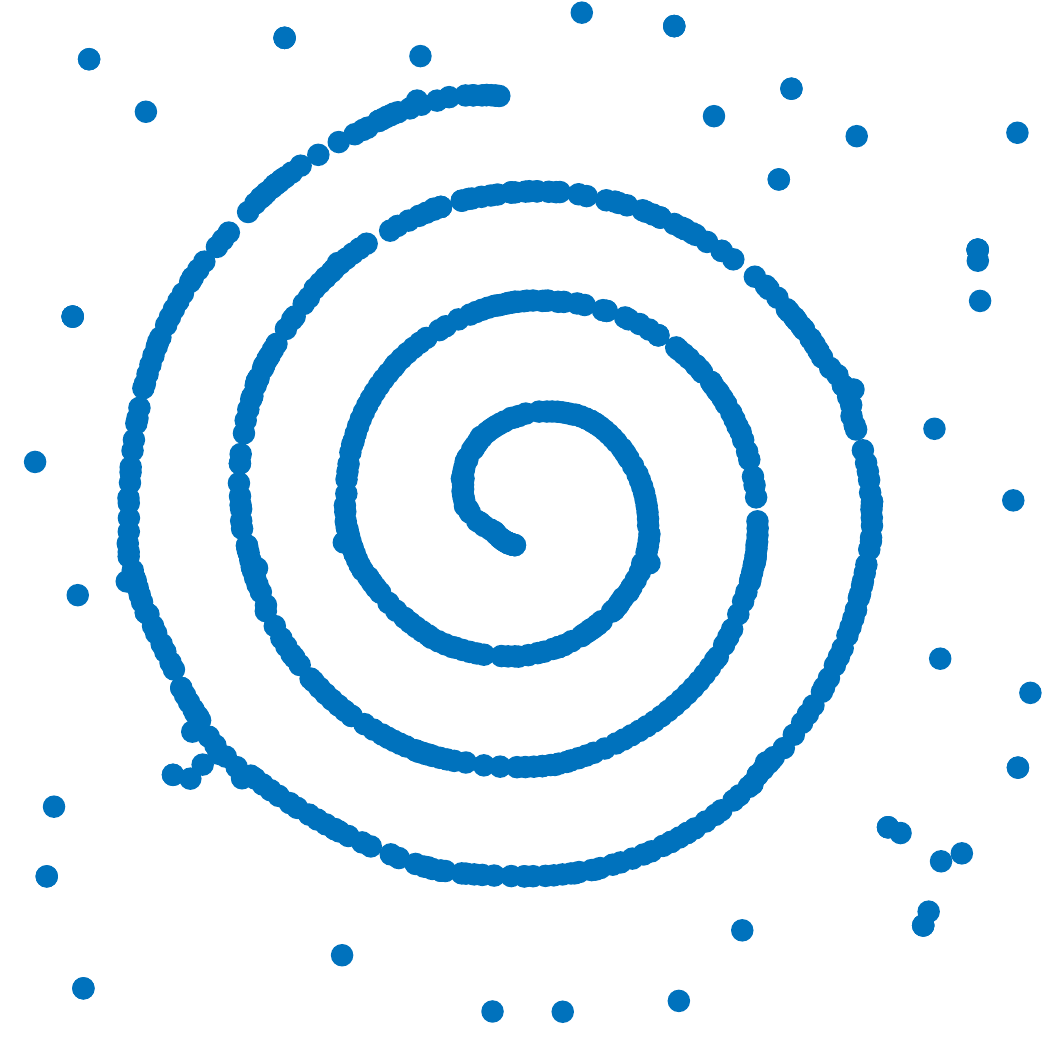}
        \end{minipage}
        \begin{minipage}[t]{0.08\textwidth}
		    \centering
            \includegraphics[width=1\textwidth]{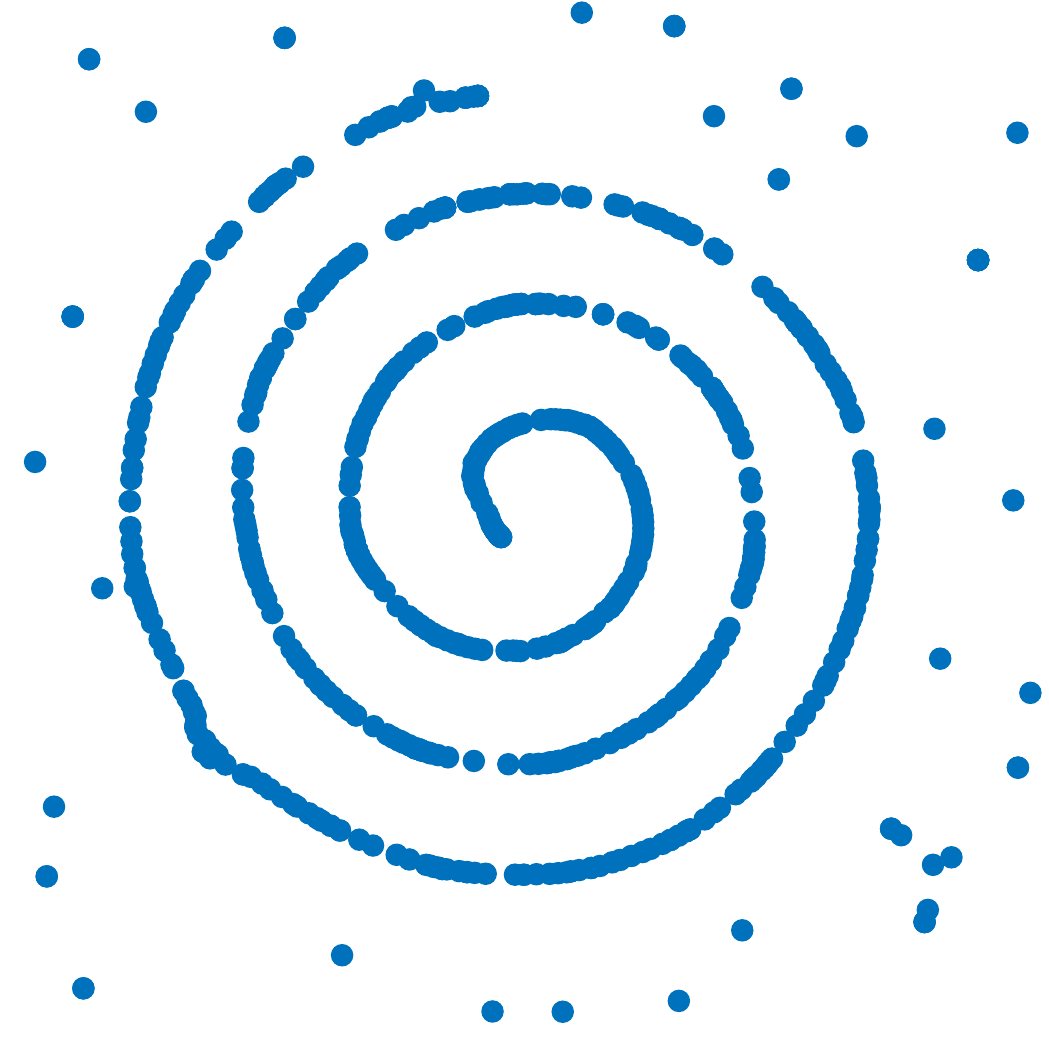}
        \end{minipage}
        \begin{minipage}[t]{0.08\textwidth}
		    \centering
            \includegraphics[width=1\textwidth]{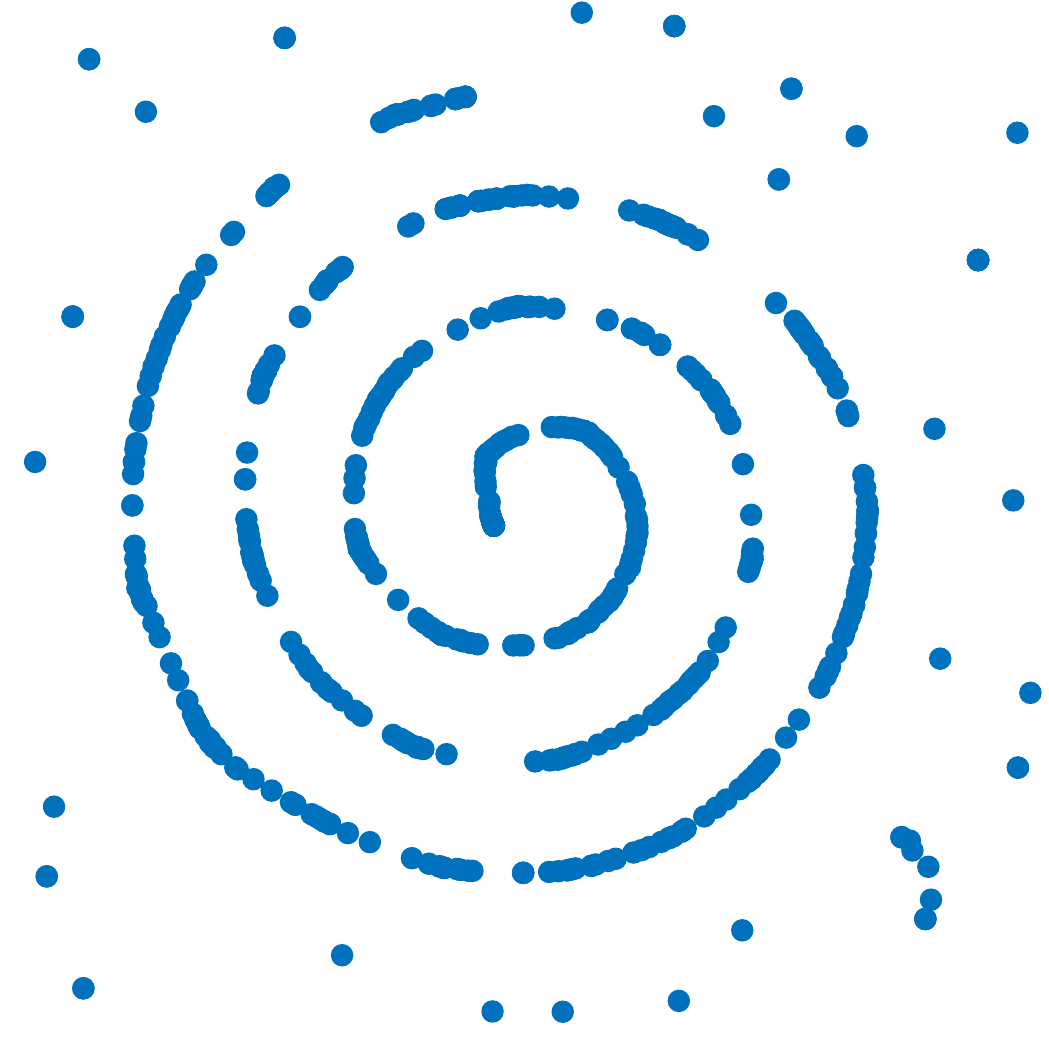}
        \end{minipage}
	}
	
	\subfloat[WT2]{
        \begin{minipage}[t]{0.08\textwidth}
		    \centering
            \includegraphics[width=1\textwidth]{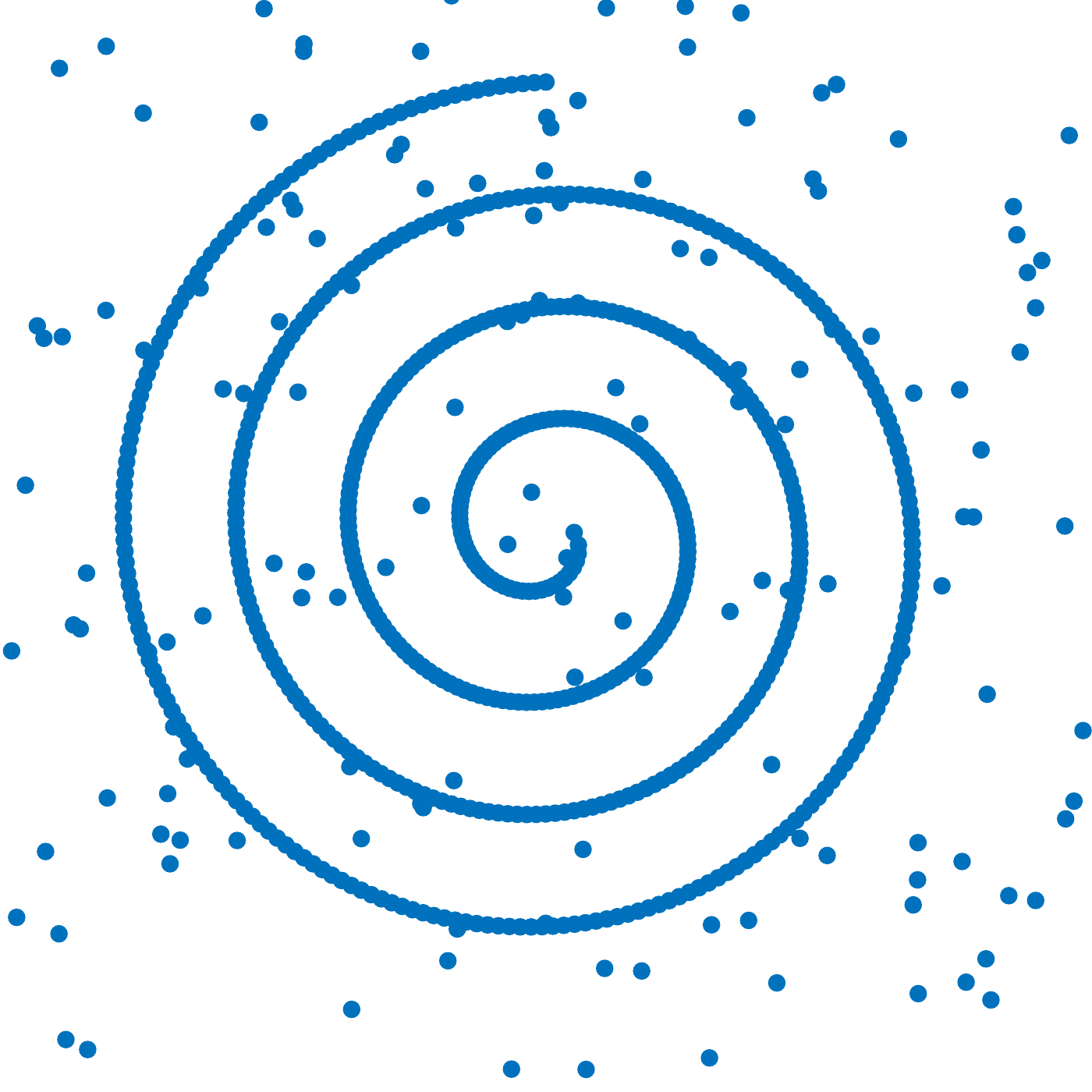}
        \end{minipage}
        \begin{minipage}[t]{0.08\textwidth}
		    \centering
            \includegraphics[width=1\textwidth]{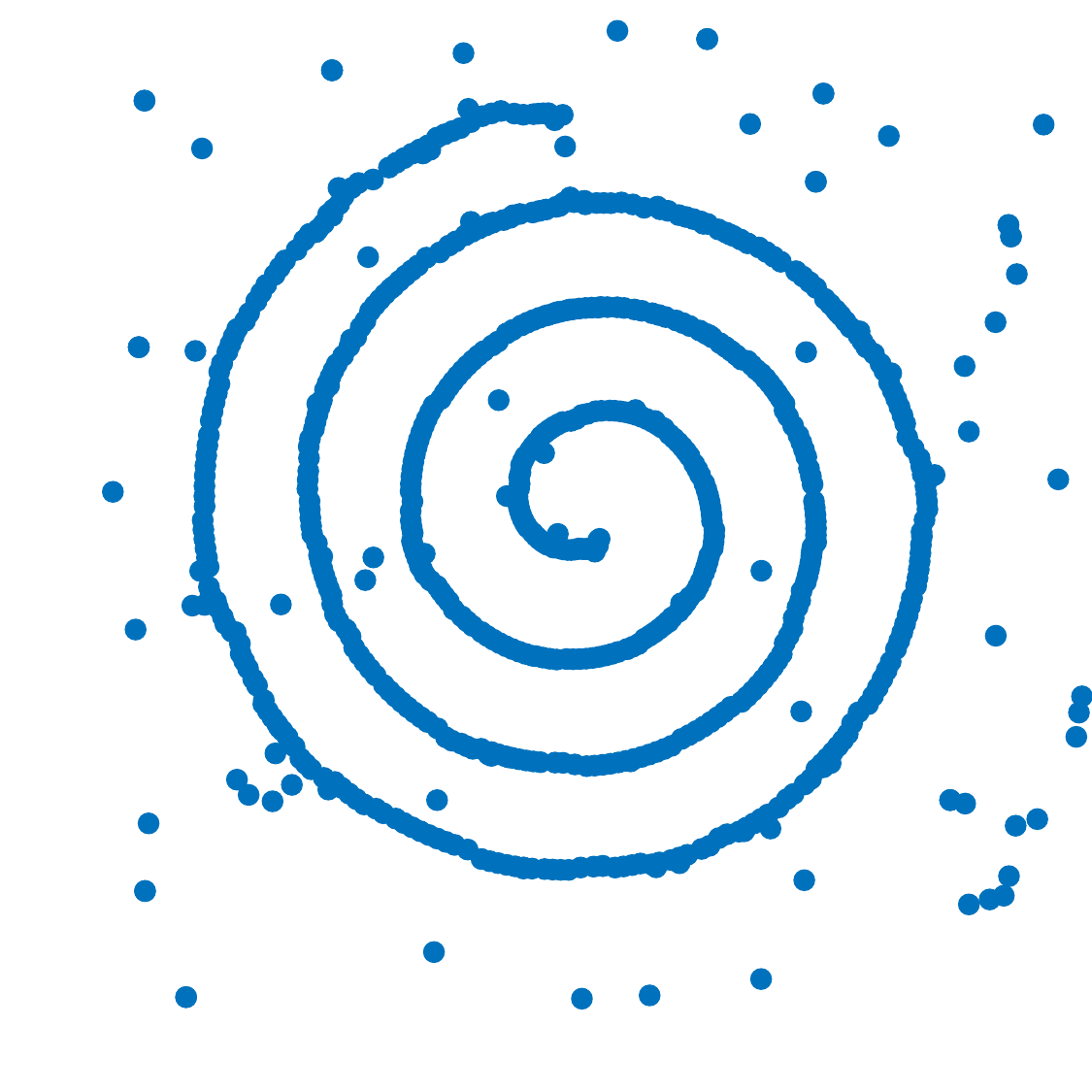}
        \end{minipage}
        \begin{minipage}[t]{0.08\textwidth}
		    \centering
            \includegraphics[width=1\textwidth]{figures/spiral/wt2-fix-emd-spiral-2.pdf}
        \end{minipage}
        \begin{minipage}[t]{0.08\textwidth}
		    \centering
            \includegraphics[width=1\textwidth]{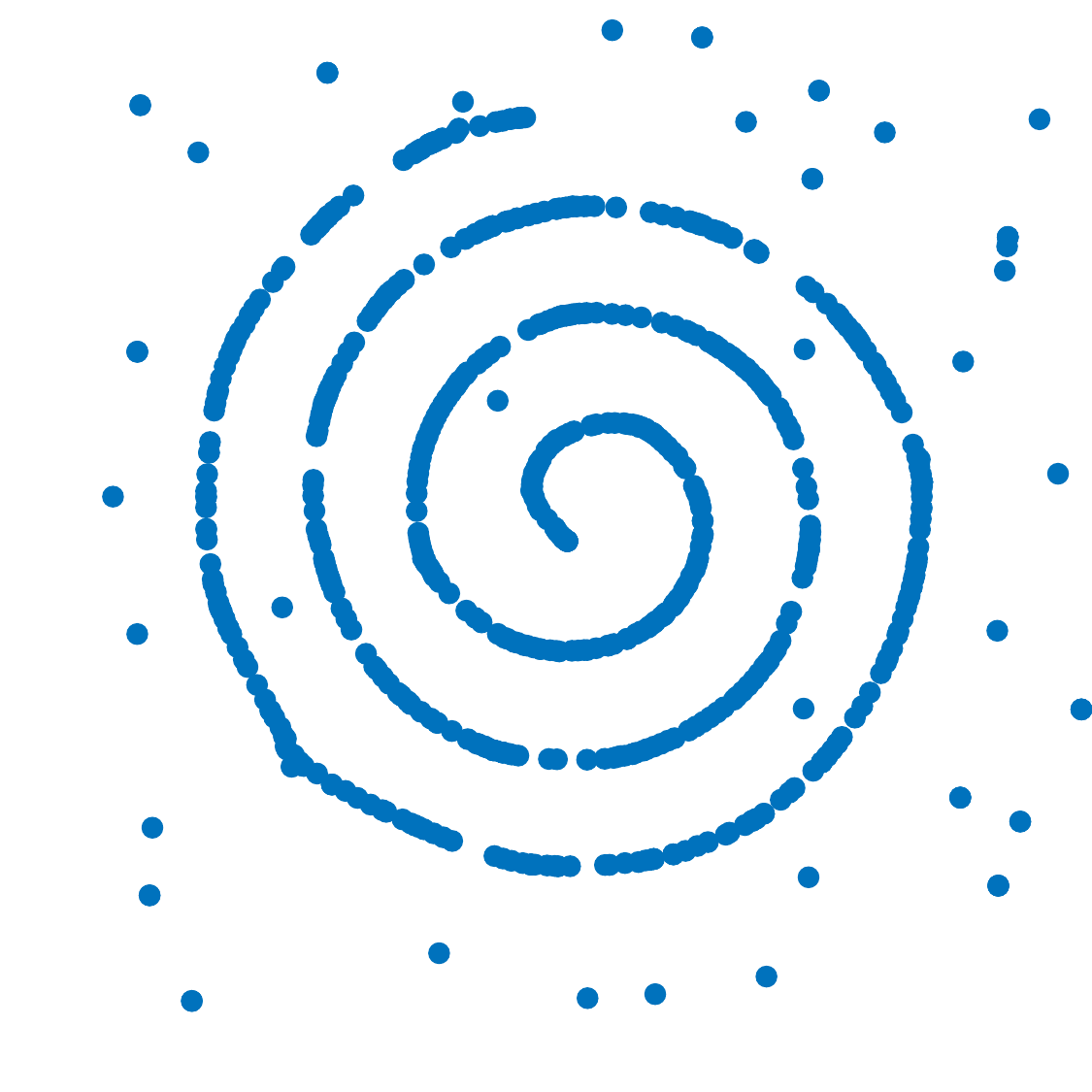}
        \end{minipage}
        \begin{minipage}[t]{0.08\textwidth}
		    \centering
            \includegraphics[width=1\textwidth]{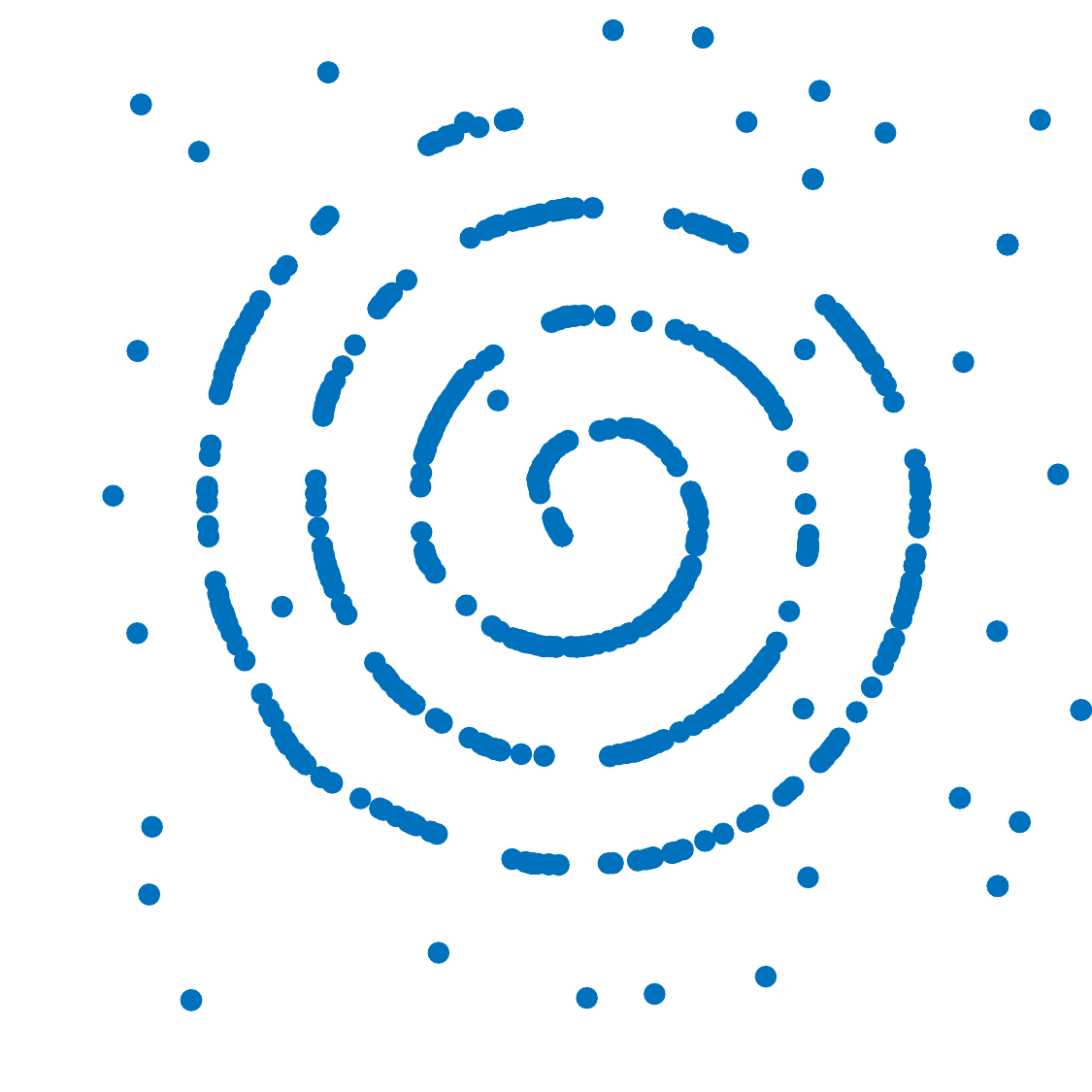}
        \end{minipage}
	}

    \subfloat[WT1]{
        \begin{minipage}[t]{0.08\textwidth}
            \centering
            \includegraphics[width=1\textwidth]{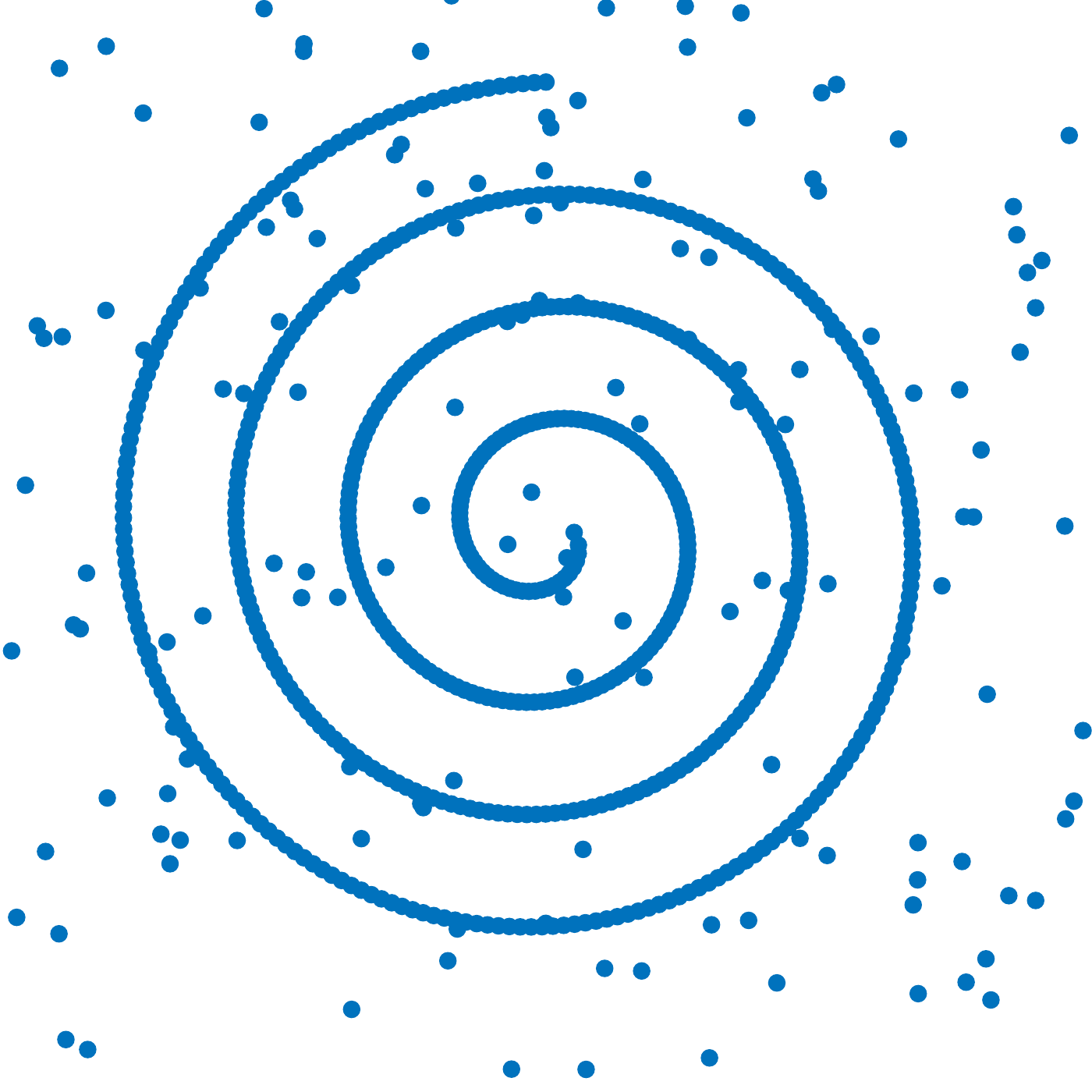}
        \end{minipage}
        \begin{minipage}[t]{0.08\textwidth}
            \centering
            \includegraphics[width=1\textwidth]{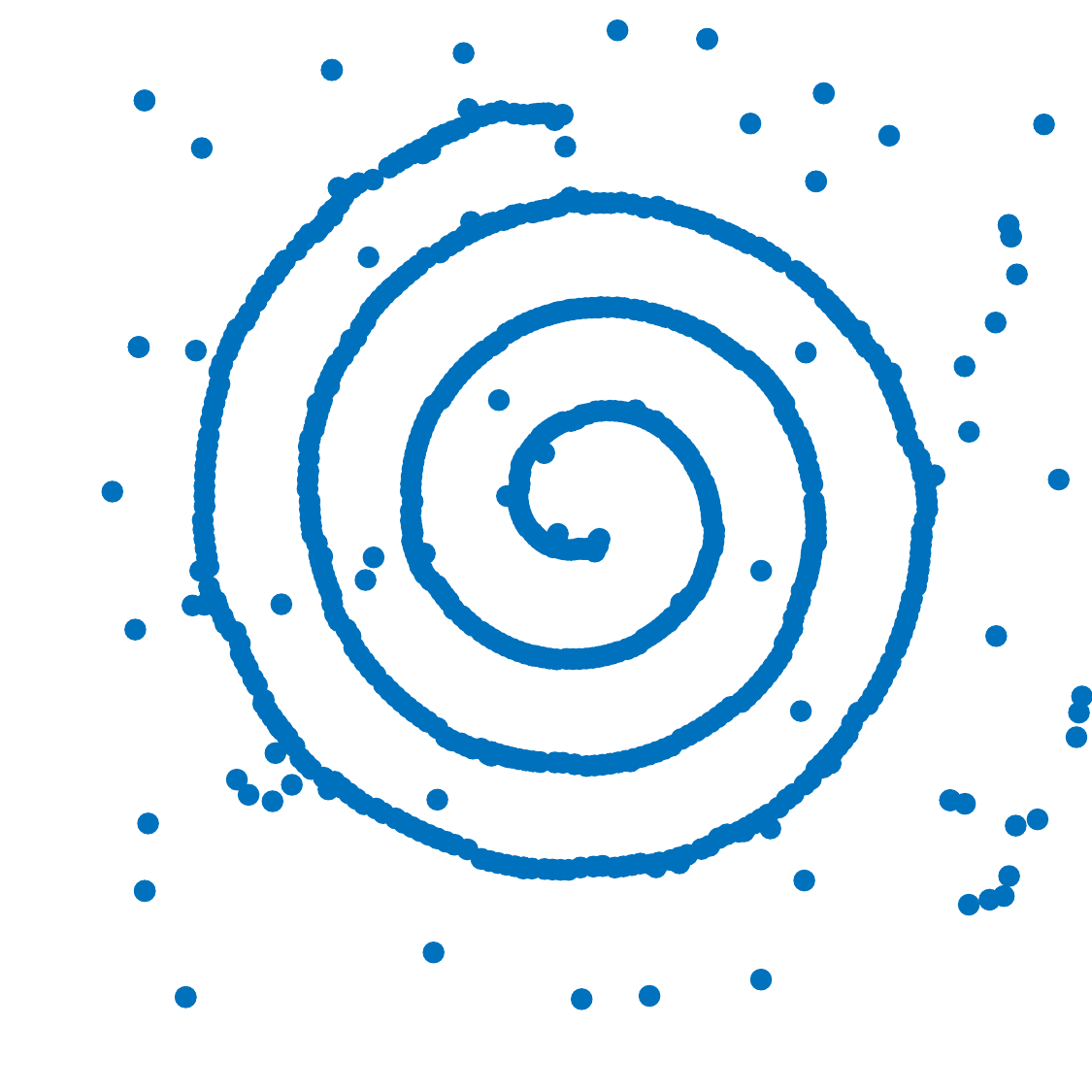}
        \end{minipage}
        \begin{minipage}[t]{0.08\textwidth}
            \centering
            \includegraphics[width=1\textwidth]{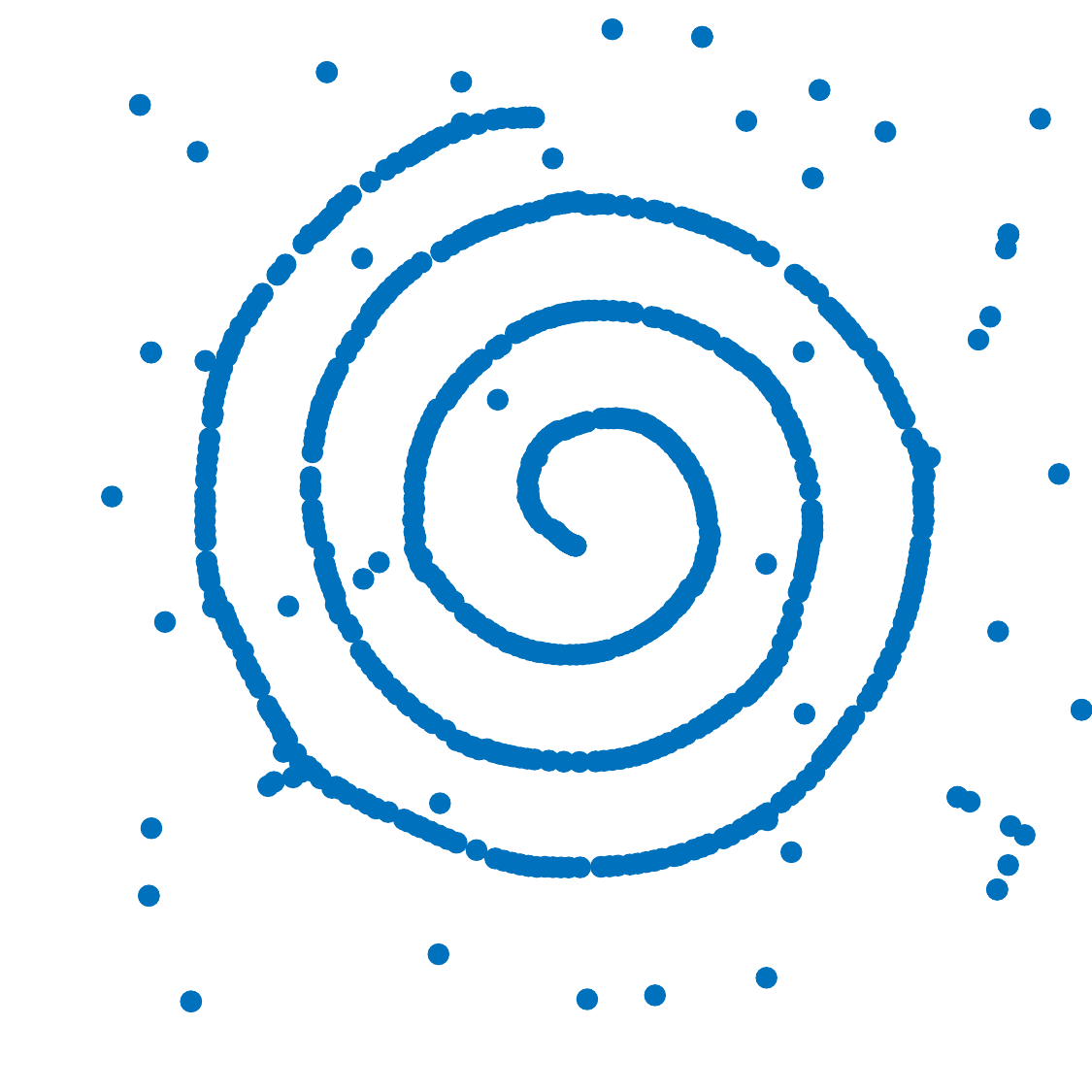}
        \end{minipage}
        \begin{minipage}[t]{0.08\textwidth}
            \centering
            \includegraphics[width=1\textwidth]{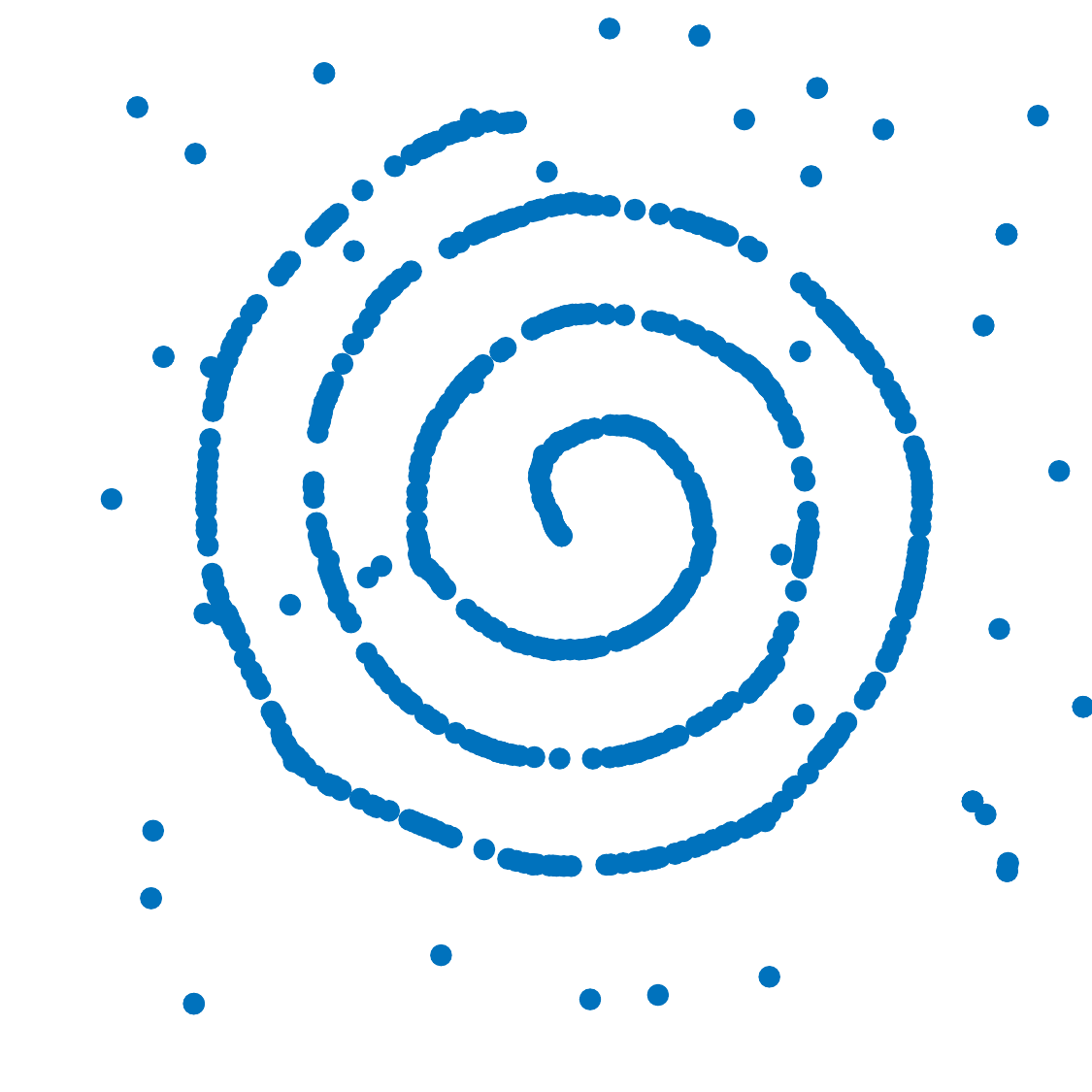}
        \end{minipage}
        \begin{minipage}[t]{0.08\textwidth}
            \centering
            \includegraphics[width=1\textwidth]{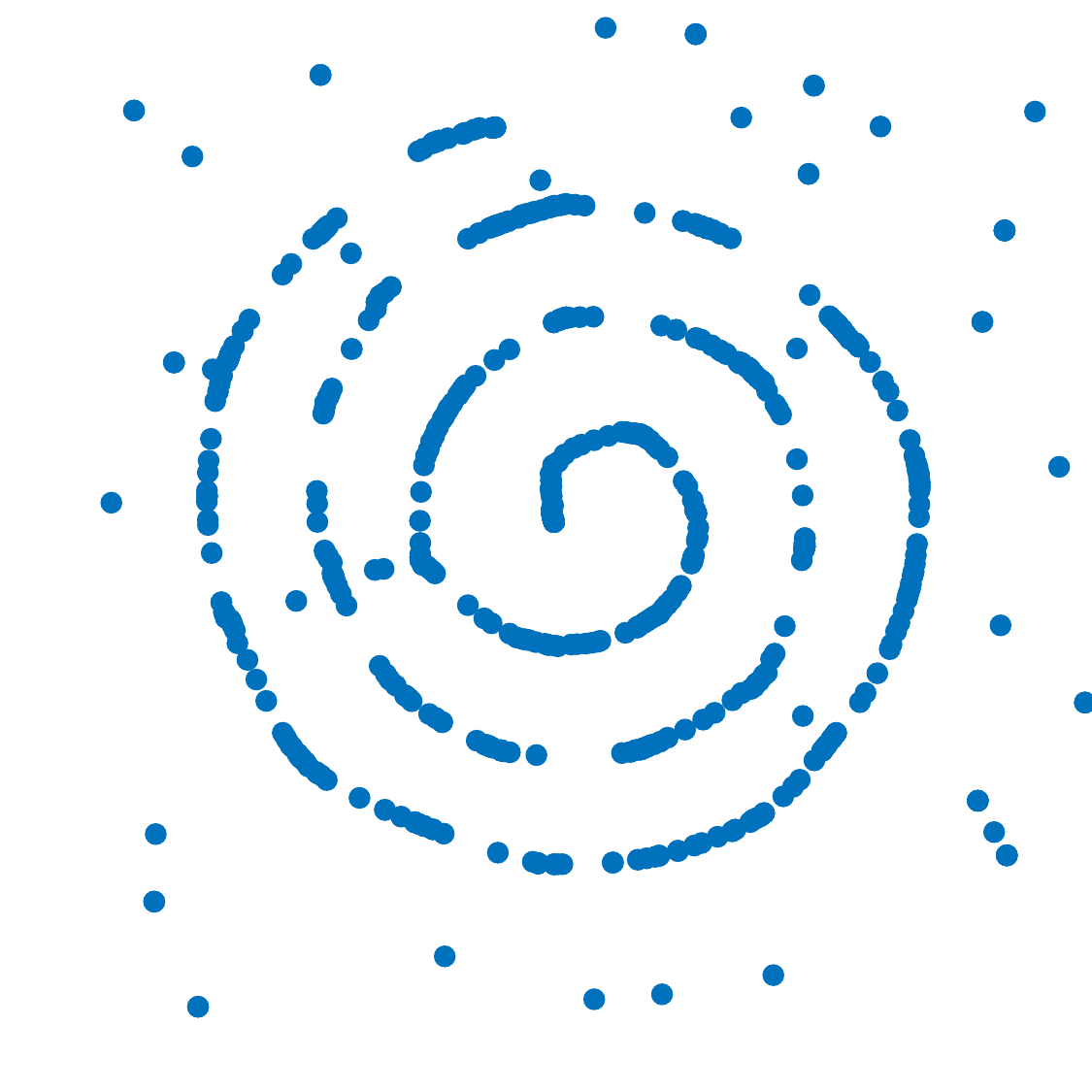}
        \end{minipage}
    }
    \caption{Denoising of a spiral with outliers. }
    \label{fig:supp-spiral}
\end{figure}

\subsection{Denoising of concentric circles}

In this example, we analyze a dataset composed of two concentric circles with random perturbations on points by small values. Each circle has 250 points lying in the square $[-2, 2]^2$. We compare the performance of MS, GT, WT2 and WT1 in the course of 4 iterations. Results are shown in Figure~\ref{fig:supp-concen}. 

\begin{figure}[htb]
    \centering
    \subfloat[MS]{
        \begin{minipage}[t]{0.08\textwidth}
            \centering
            $\tau=0$ \\ 
            \includegraphics[width=1\textwidth]{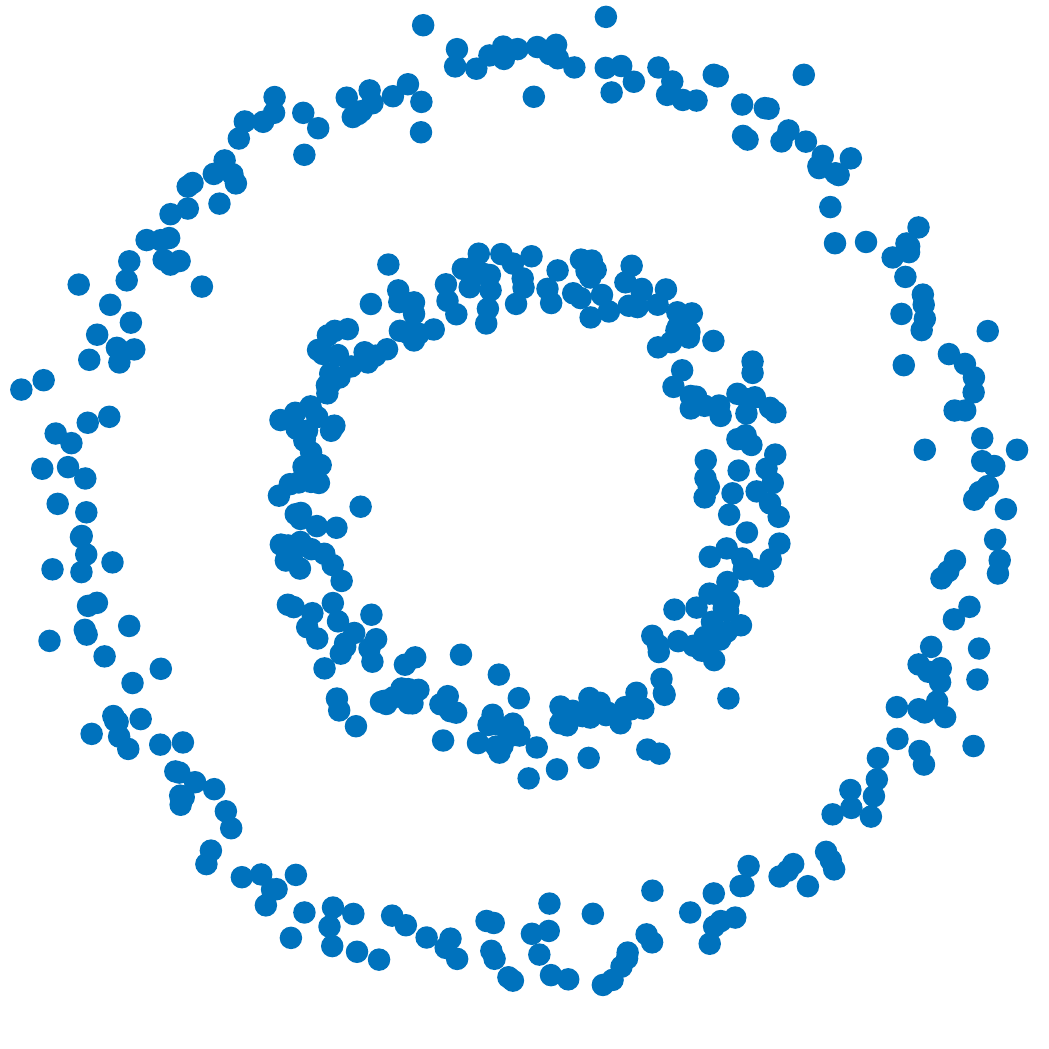}
        \end{minipage}
        \begin{minipage}[t]{0.08\textwidth}
            \centering
            $\tau=1$ \\ 
            \includegraphics[width=1\textwidth]{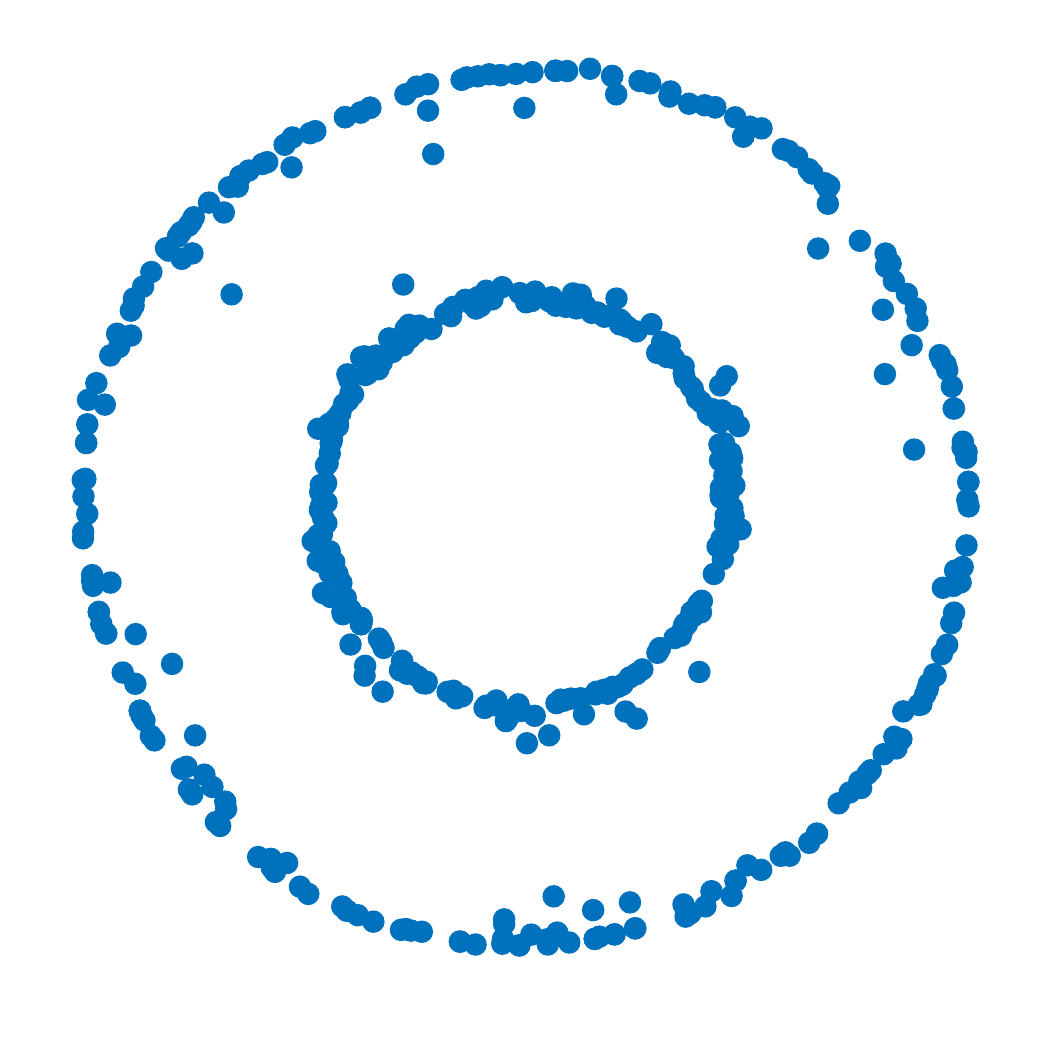}
        \end{minipage}
        \begin{minipage}[t]{0.08\textwidth}
            \centering
            $\tau=2$ \\ 
            \includegraphics[width=1\textwidth]{figures/concen/ms-concen-2.pdf}
        \end{minipage}
        \begin{minipage}[t]{0.08\textwidth}
            \centering
            $\tau=3$ \\ 
            \includegraphics[width=1\textwidth]{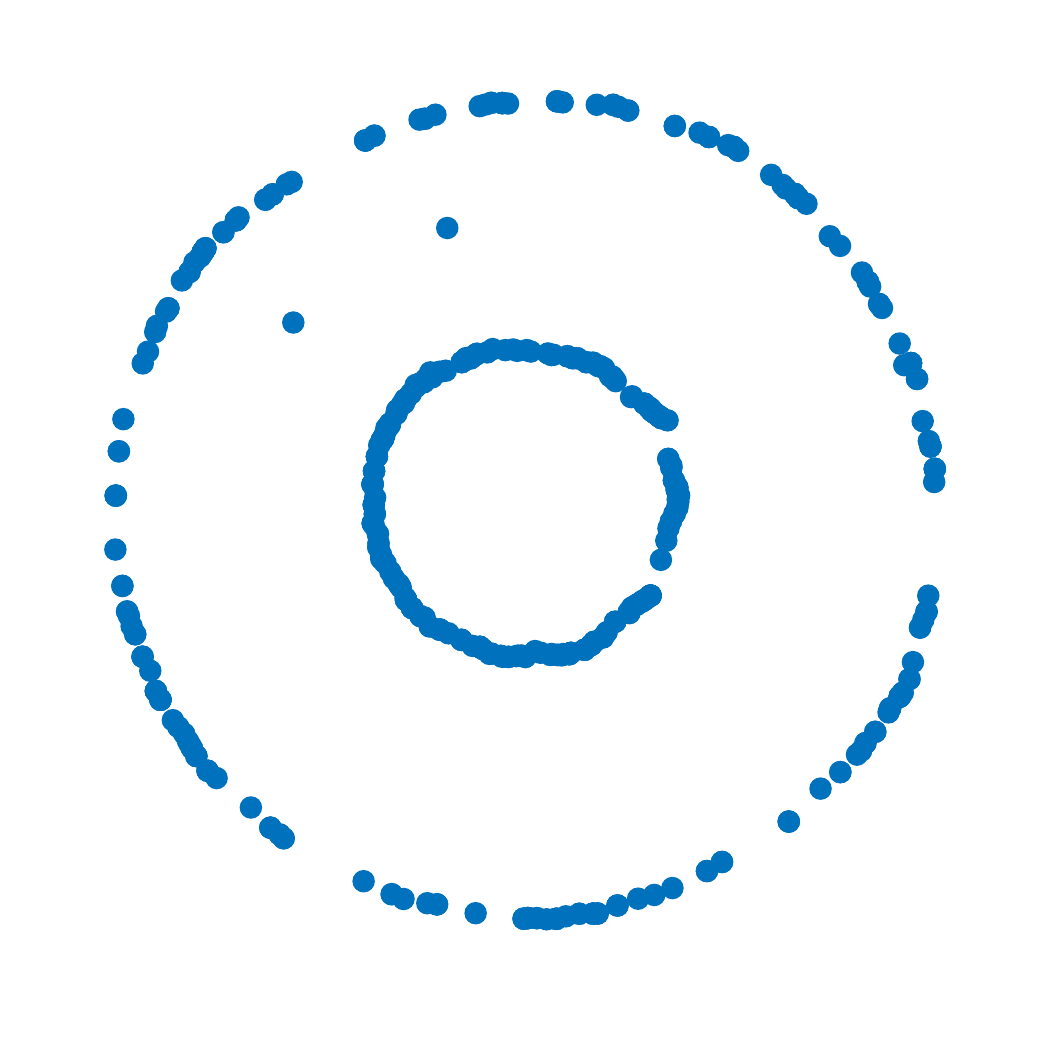}
        \end{minipage}
        \begin{minipage}[t]{0.08\textwidth}
            \centering
            $\tau=4$ \\ 
            \includegraphics[width=1\textwidth]{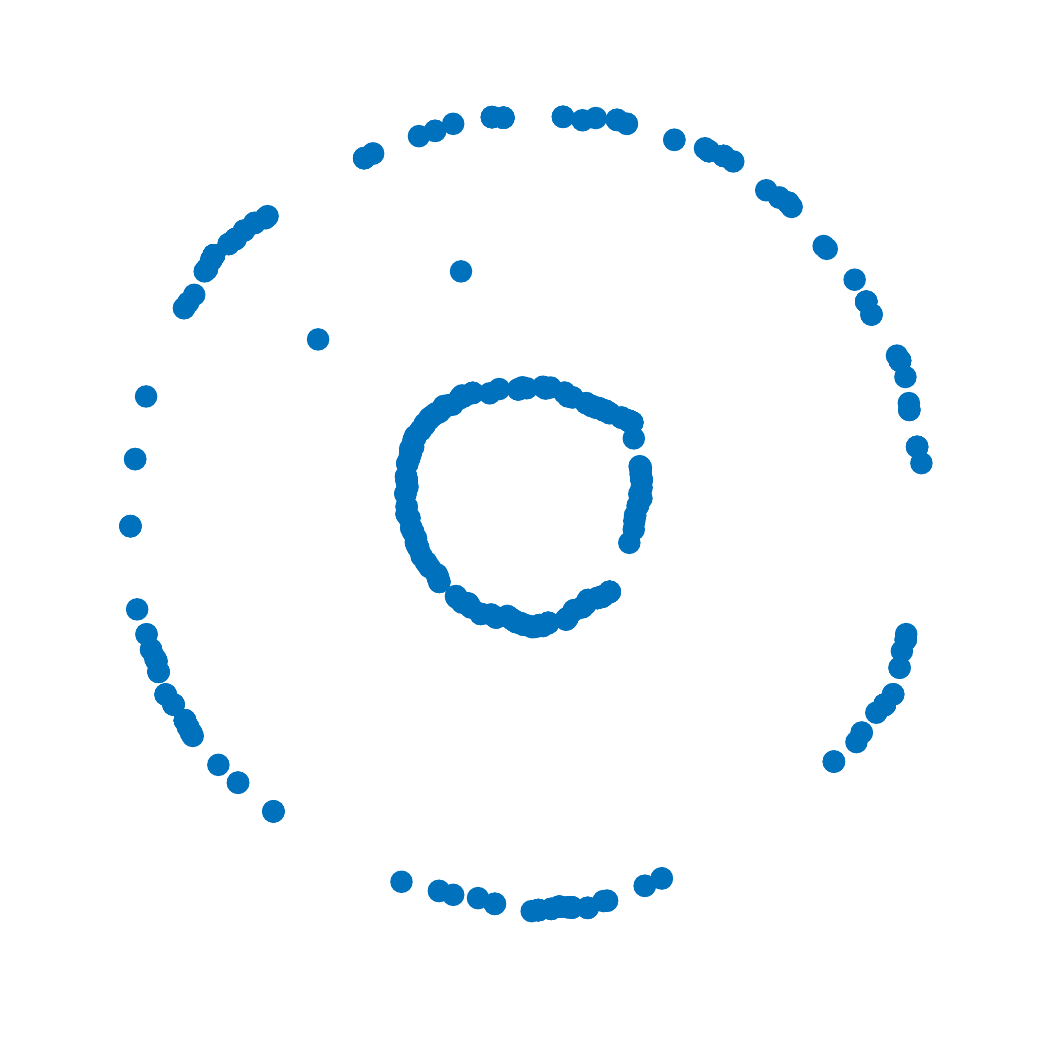}
        \end{minipage}
    }

    \subfloat[GT-$\lambda$-2.5]{
        \begin{minipage}[t]{0.08\textwidth}
            \centering
            \includegraphics[width=1\textwidth]{figures/concen/gtv-concen-0.pdf}
        \end{minipage}
        \begin{minipage}[t]{0.08\textwidth}
            \centering
            \includegraphics[width=1\textwidth]{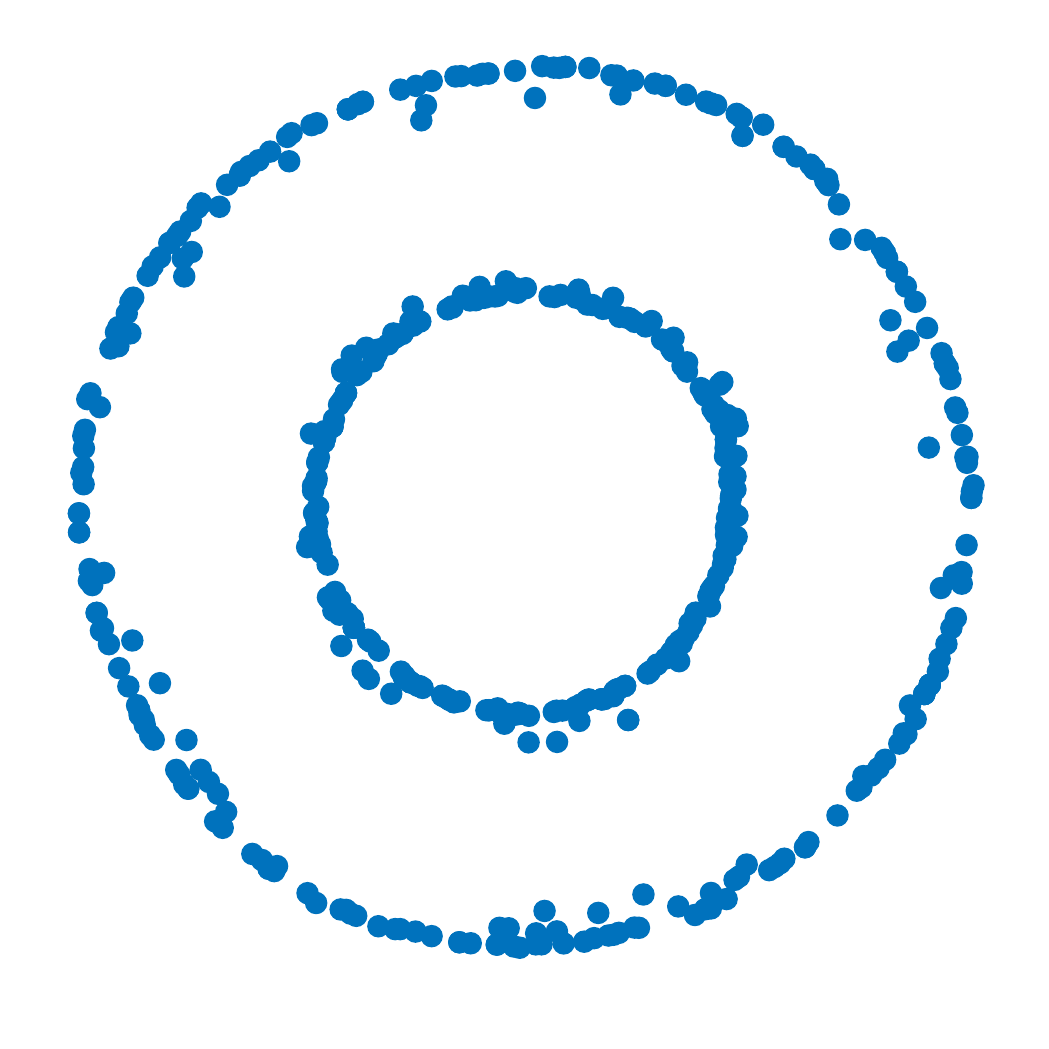}
        \end{minipage}
        \begin{minipage}[t]{0.08\textwidth}
            \centering
            \includegraphics[width=1\textwidth]{figures/concen/gtv-concen-2.pdf}
        \end{minipage}
        \begin{minipage}[t]{0.08\textwidth}
            \centering
            \includegraphics[width=1\textwidth]{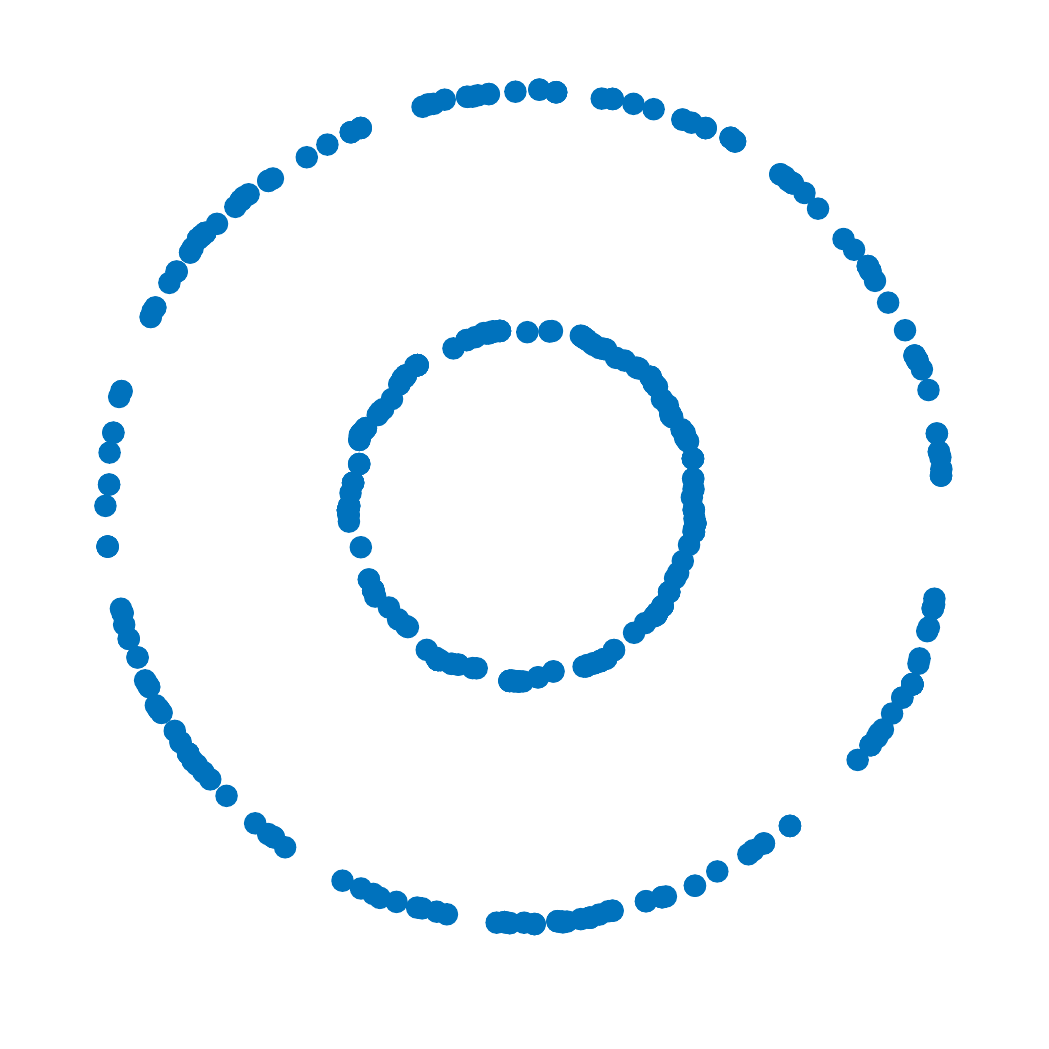}
        \end{minipage}
        \begin{minipage}[t]{0.08\textwidth}
            \centering
            \includegraphics[width=1\textwidth]{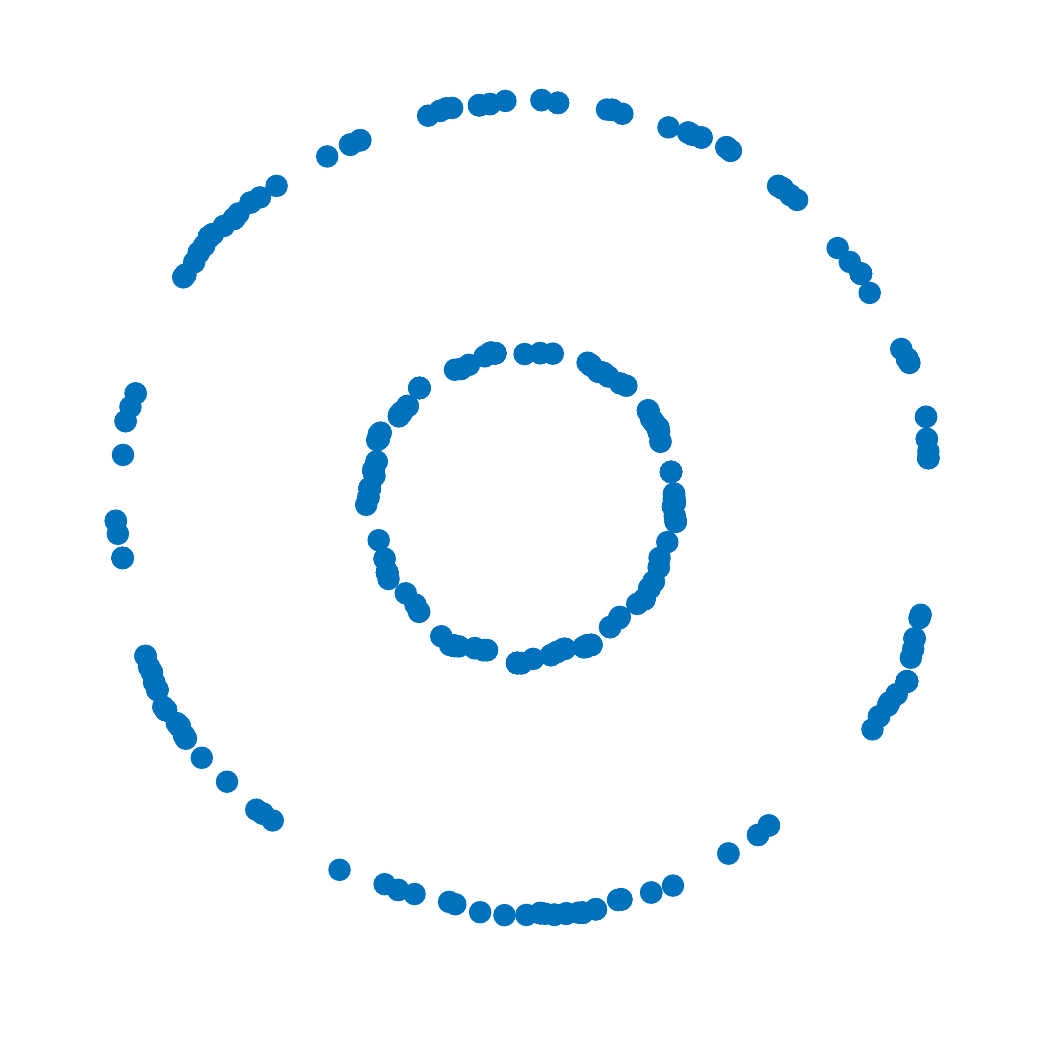}
        \end{minipage}
    }

    \subfloat[WT2]{
        \begin{minipage}[t]{0.08\textwidth}
            \centering
            \includegraphics[width=1\textwidth]{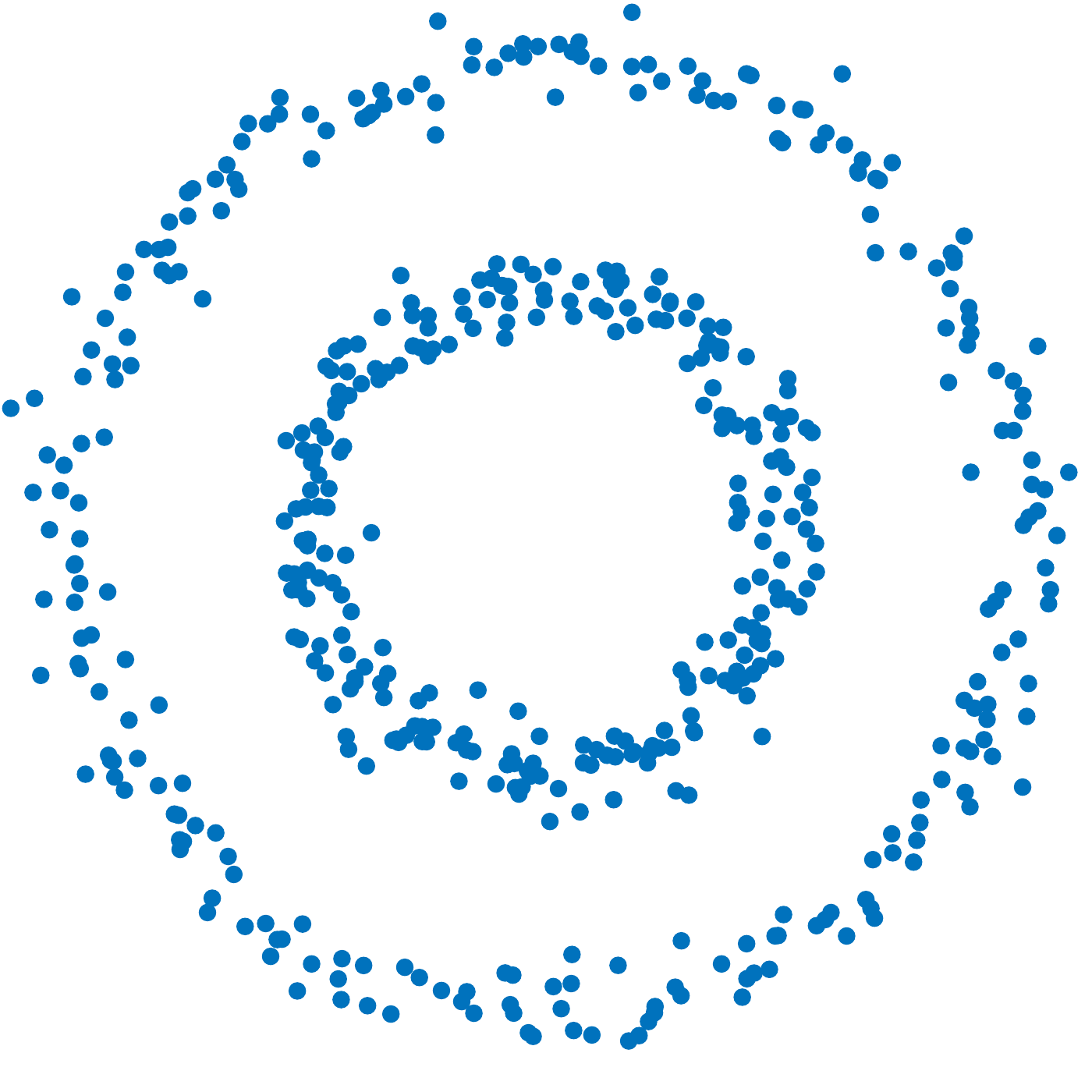}
        \end{minipage}
        \begin{minipage}[t]{0.08\textwidth}
            \centering
            \includegraphics[width=1\textwidth]{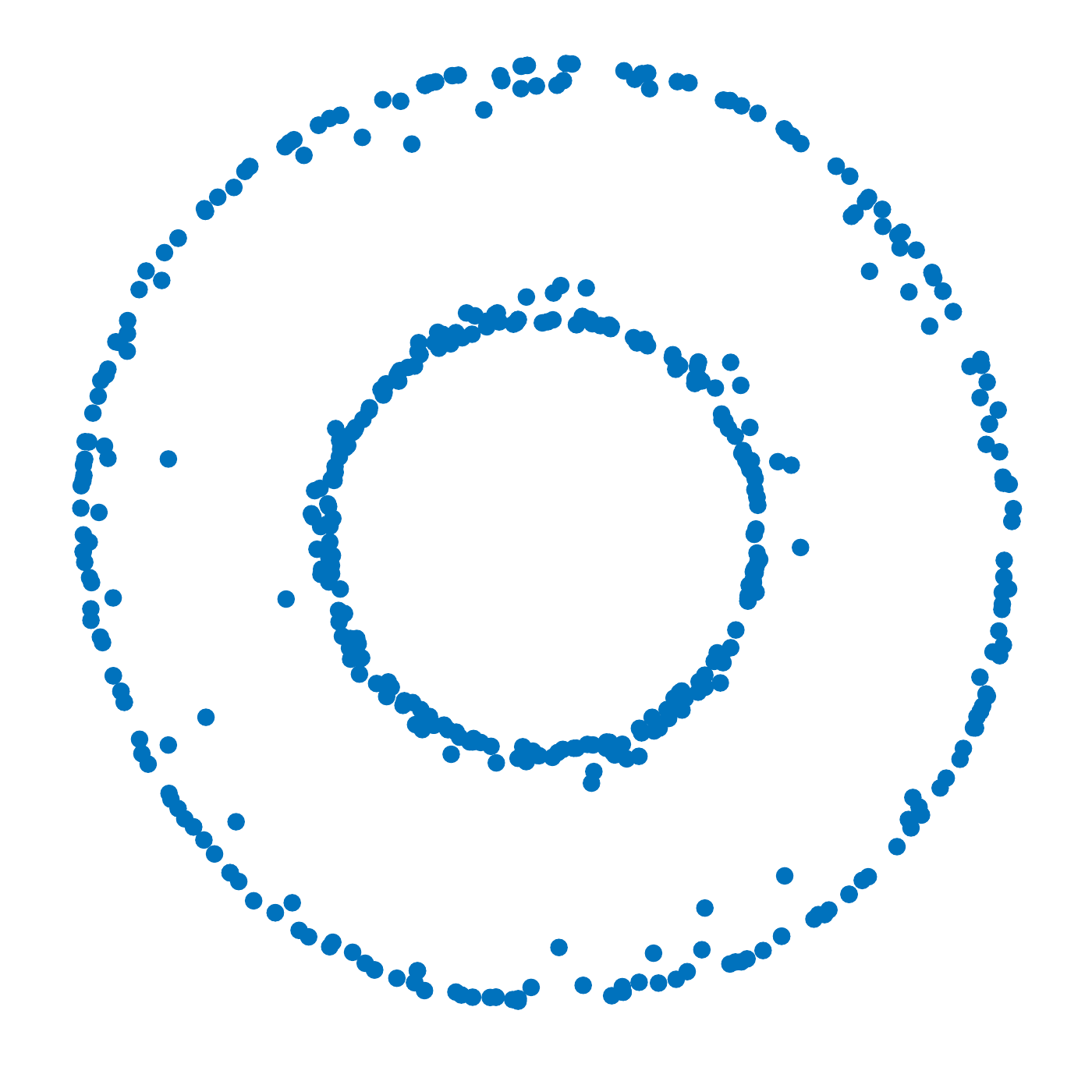}
        \end{minipage}
        \begin{minipage}[t]{0.08\textwidth}
            \centering
            \includegraphics[width=1\textwidth]{figures/concen/wt2-fix-emd-concencircles-2.pdf}
        \end{minipage}
        \begin{minipage}[t]{0.08\textwidth}
            \centering
            \includegraphics[width=1\textwidth]{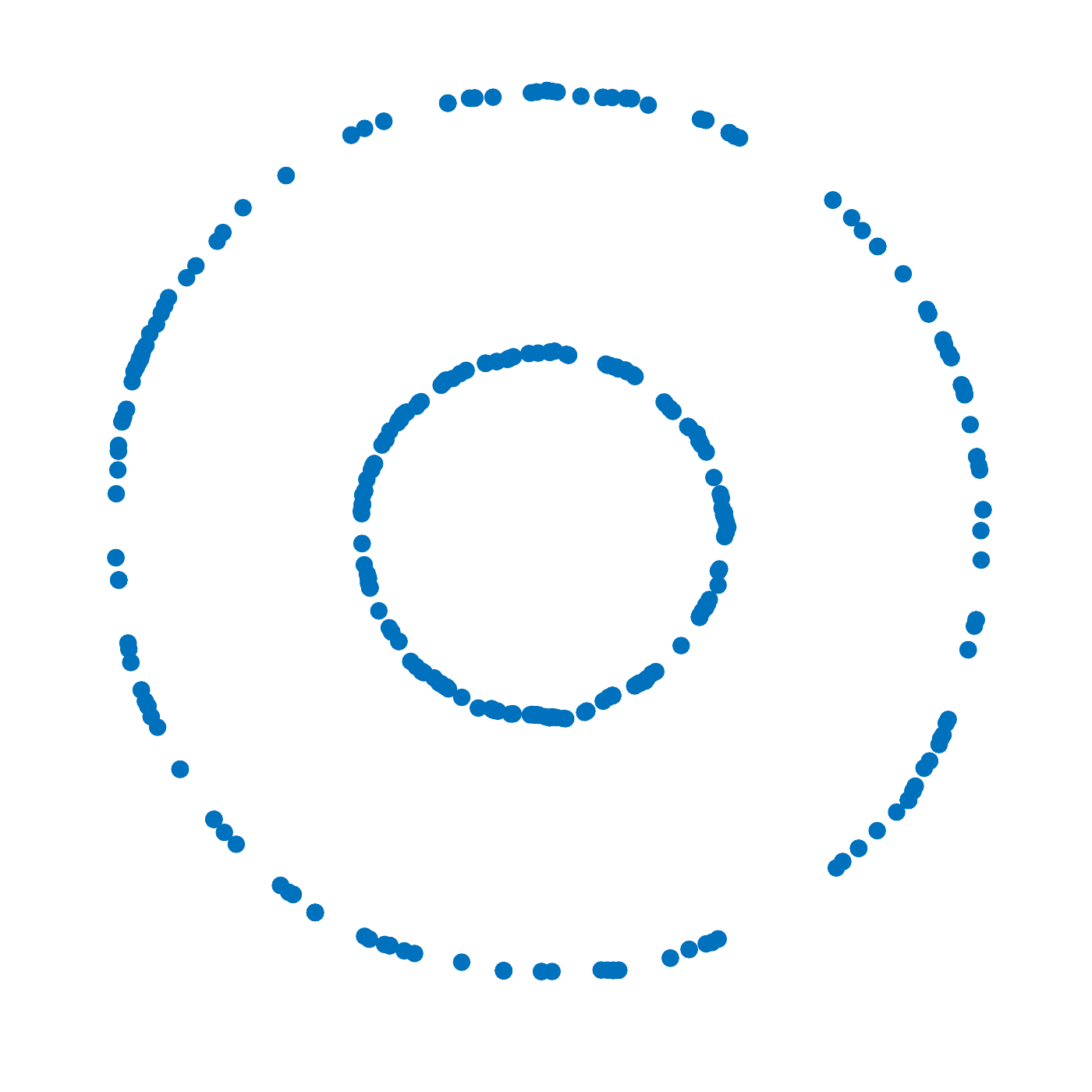}
        \end{minipage}
        \begin{minipage}[t]{0.08\textwidth}
            \centering
            \includegraphics[width=1\textwidth]{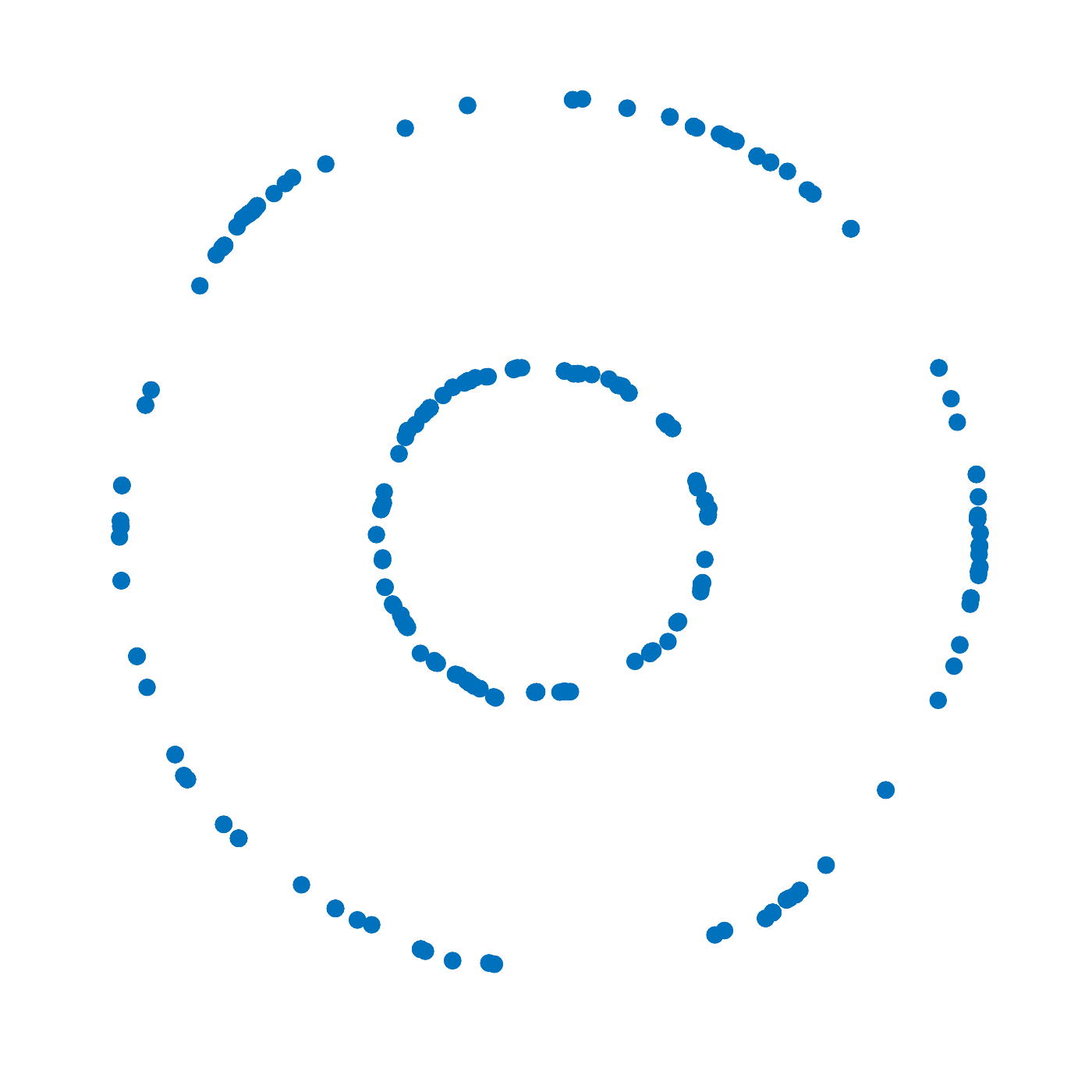}
        \end{minipage}
    }

    \subfloat[WT1]{
        \begin{minipage}[t]{0.08\textwidth}
            \centering
            \includegraphics[width=1\textwidth]{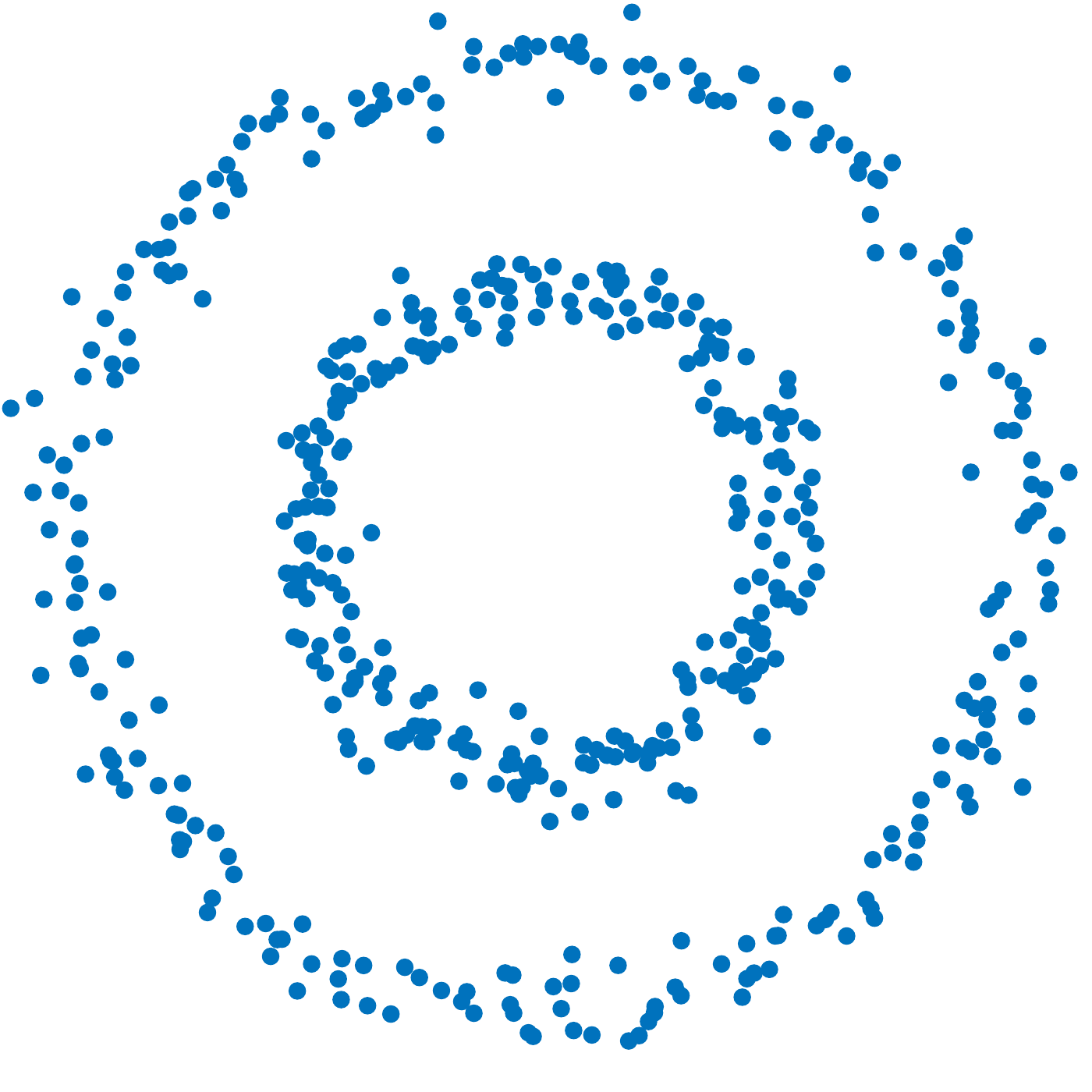}
        \end{minipage}
        \begin{minipage}[t]{0.08\textwidth}
            \centering
            \includegraphics[width=1\textwidth]{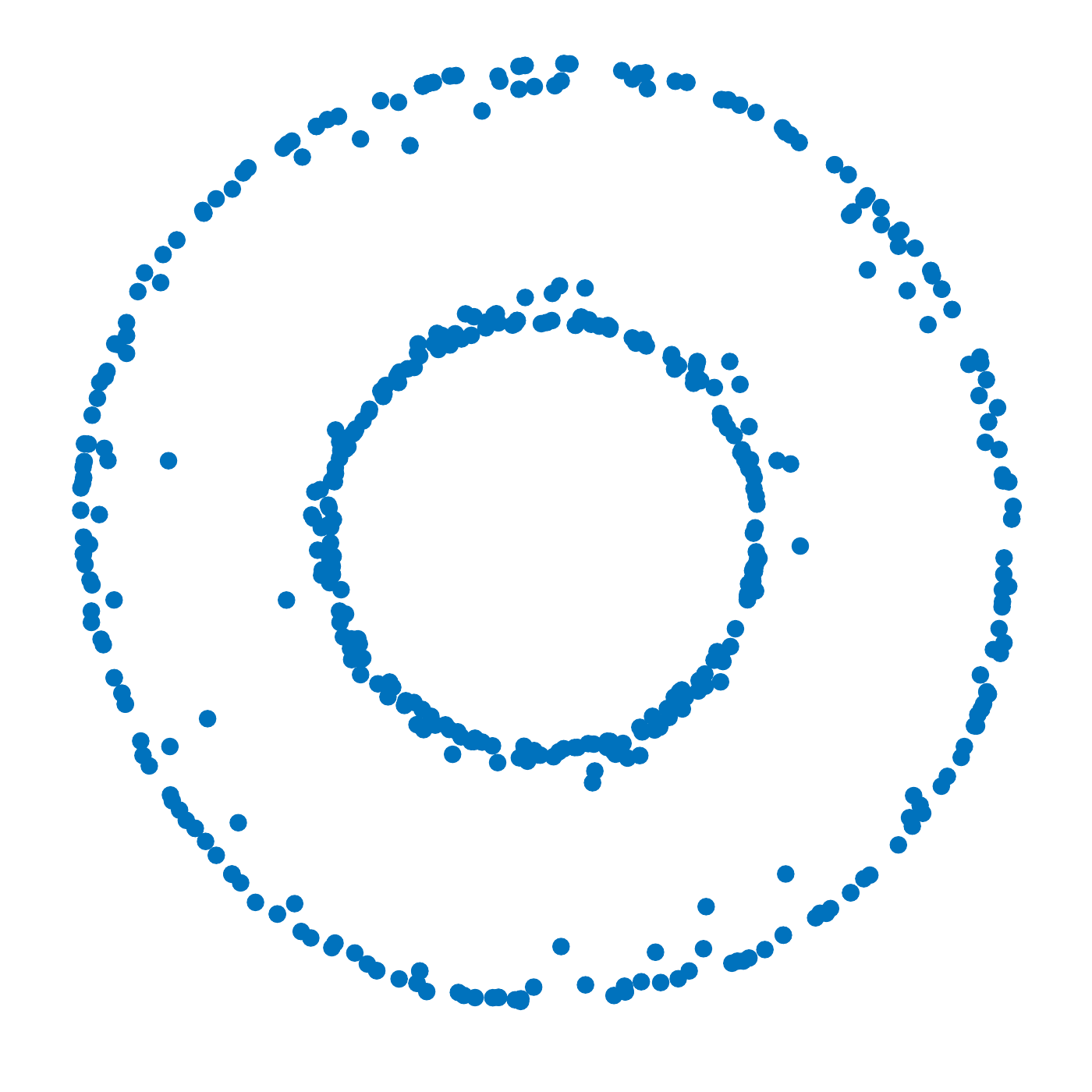}
        \end{minipage}
        \begin{minipage}[t]{0.08\textwidth}
            \centering
            \includegraphics[width=1\textwidth]{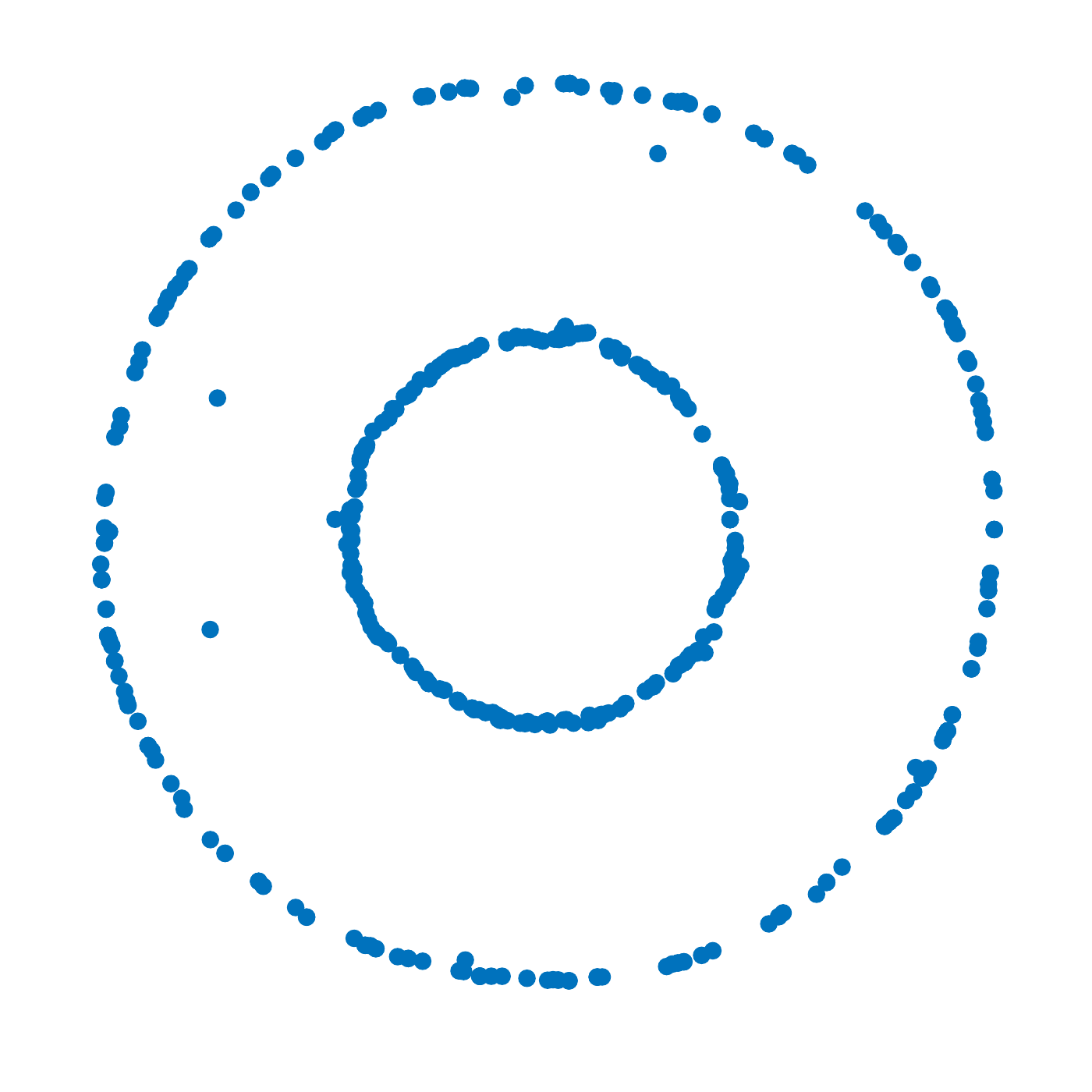}
        \end{minipage}
        \begin{minipage}[t]{0.08\textwidth}
            \centering
            \includegraphics[width=1\textwidth]{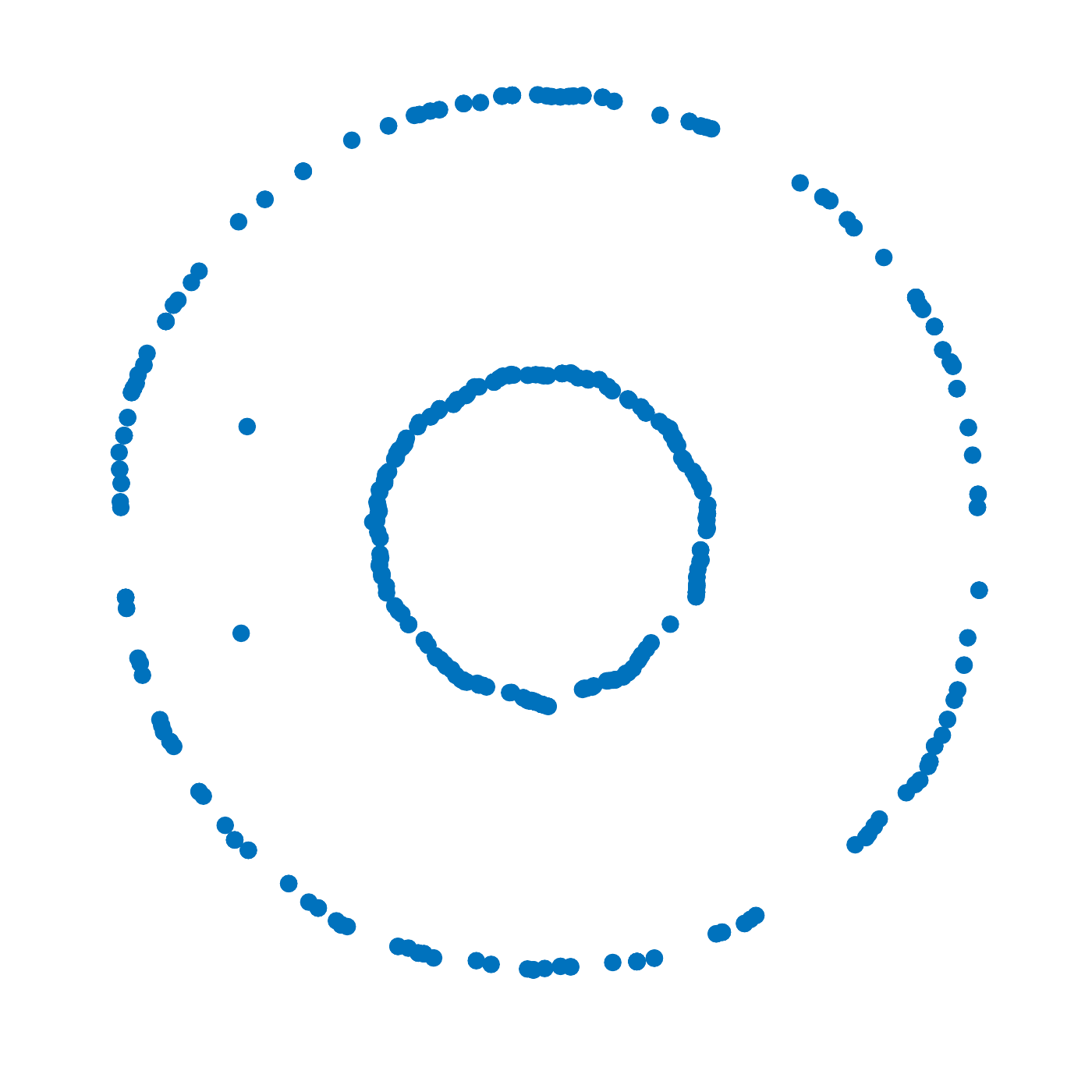}
        \end{minipage}
        \begin{minipage}[t]{0.08\textwidth}
            \centering
            \includegraphics[width=1\textwidth]{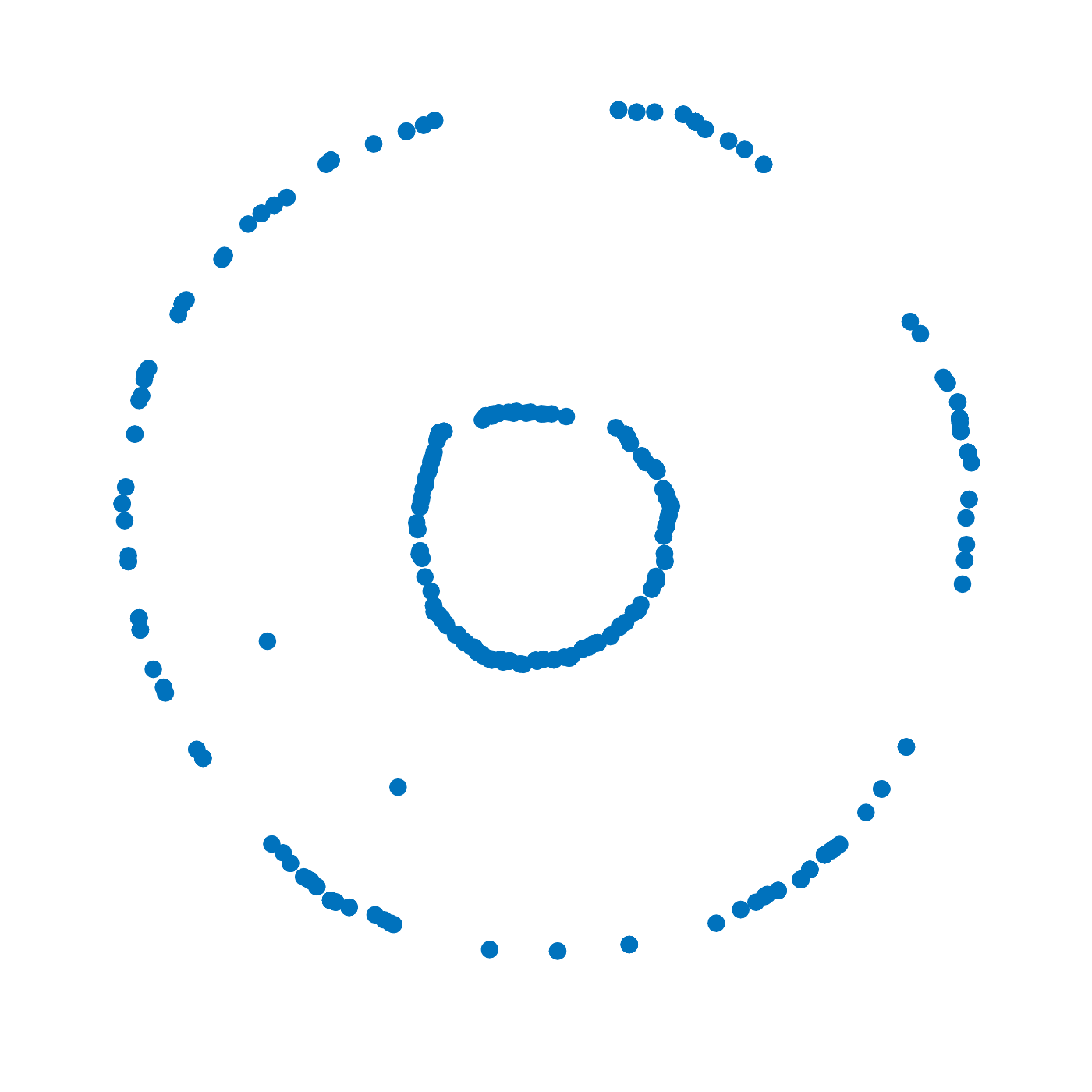}
        \end{minipage}
    }
    \caption{Denoising of concentric circles.}
    \label{fig:supp-concen}
\end{figure}

\subsection{Denoising of a noisy circle}
In this example, we analyze a noisy circle composed of 200 points uniformly spaced on the circle lying in the square $[-1, 1]^2$ together with 500 noisy points (following the uniform distribution). We compare the performance of MS, GT, WT2 and WT1 in the course of 4 iterations. Results are shown in Figure~\ref{fig:supp-S1}. We see that all methods clean the noisy points to some extent and WT2 has the best performance that it absorbs all points within the circle after the fourth iteration. GT with $\lambda=1$ has similar performance as MS and WT. After the fourth iteration, GT with $\lambda=10$ better absorbs noisy points within the circle than GT with $\lambda=1$.

\begin{figure}[htb]
    \centering
    \subfloat[MS]{
        \begin{minipage}[t]{0.08\textwidth}
            \centering
            $\tau=0$ \\ 
            \includegraphics[width=1\textwidth]{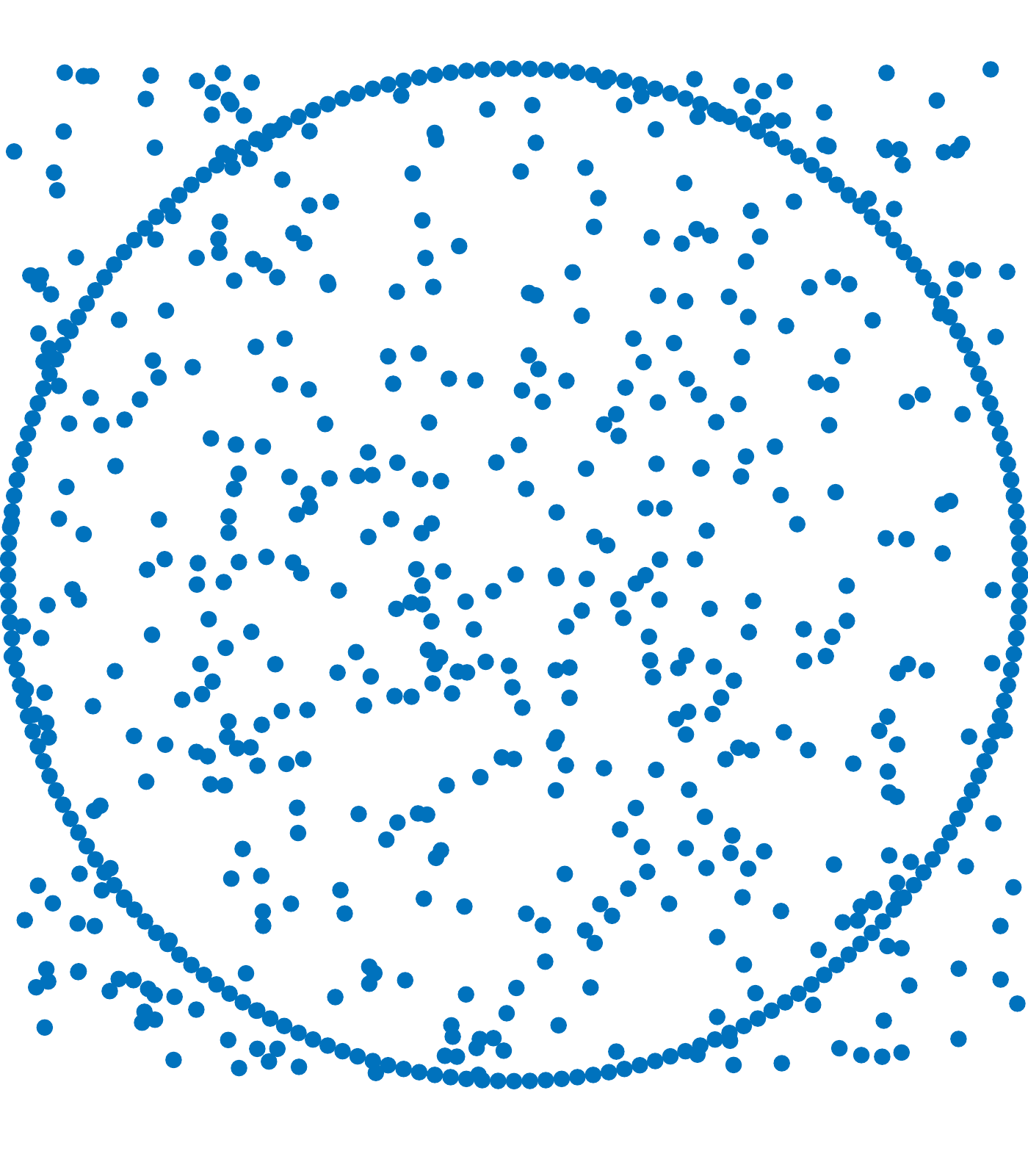}
        \end{minipage}
        \begin{minipage}[t]{0.08\textwidth}
            \centering
            $\tau=1$ \\ 
            \includegraphics[width=1\textwidth]{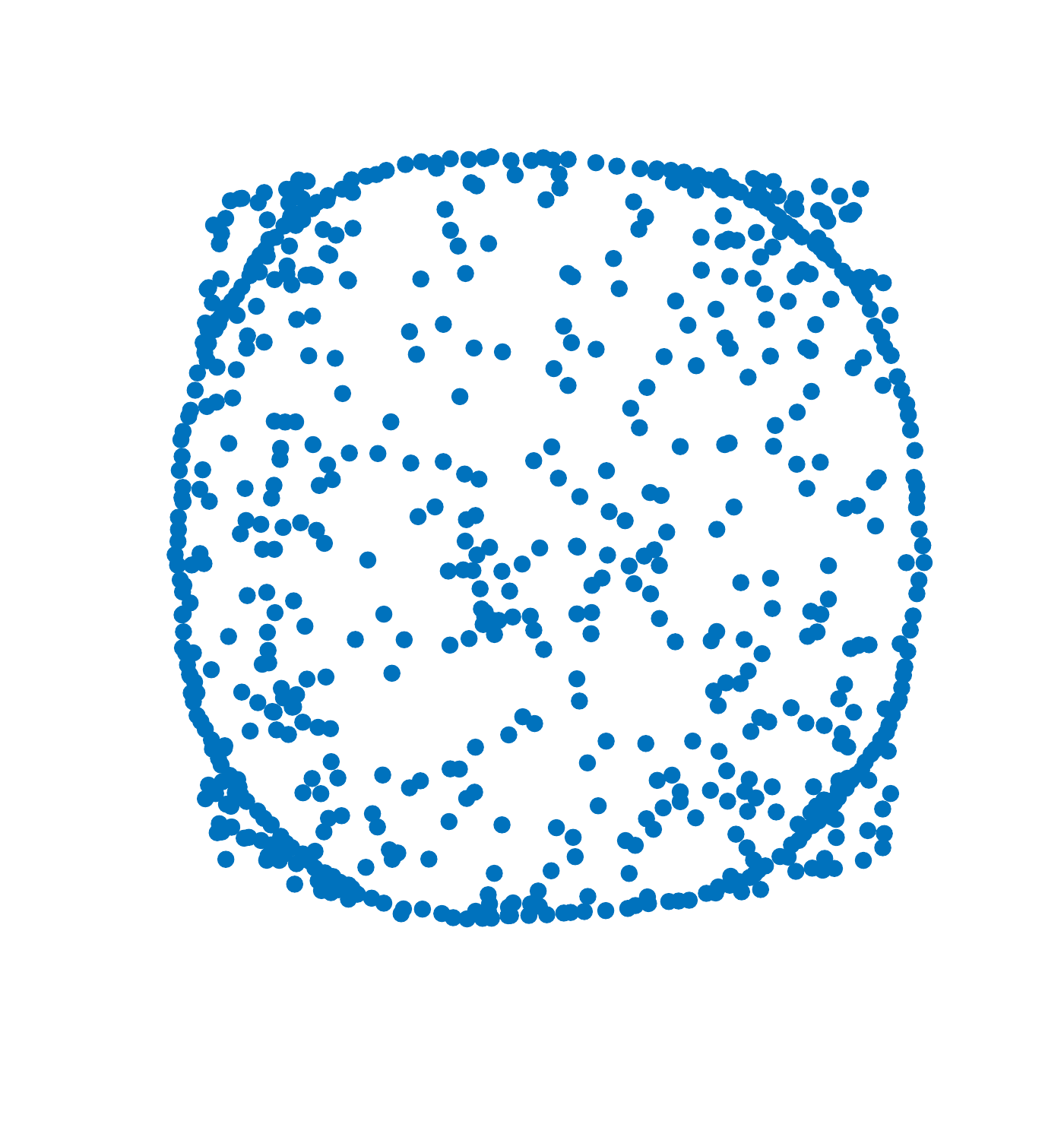}
        \end{minipage}
        \begin{minipage}[t]{0.08\textwidth}
            \centering
            $\tau=2$ \\ 
            \includegraphics[width=1\textwidth]{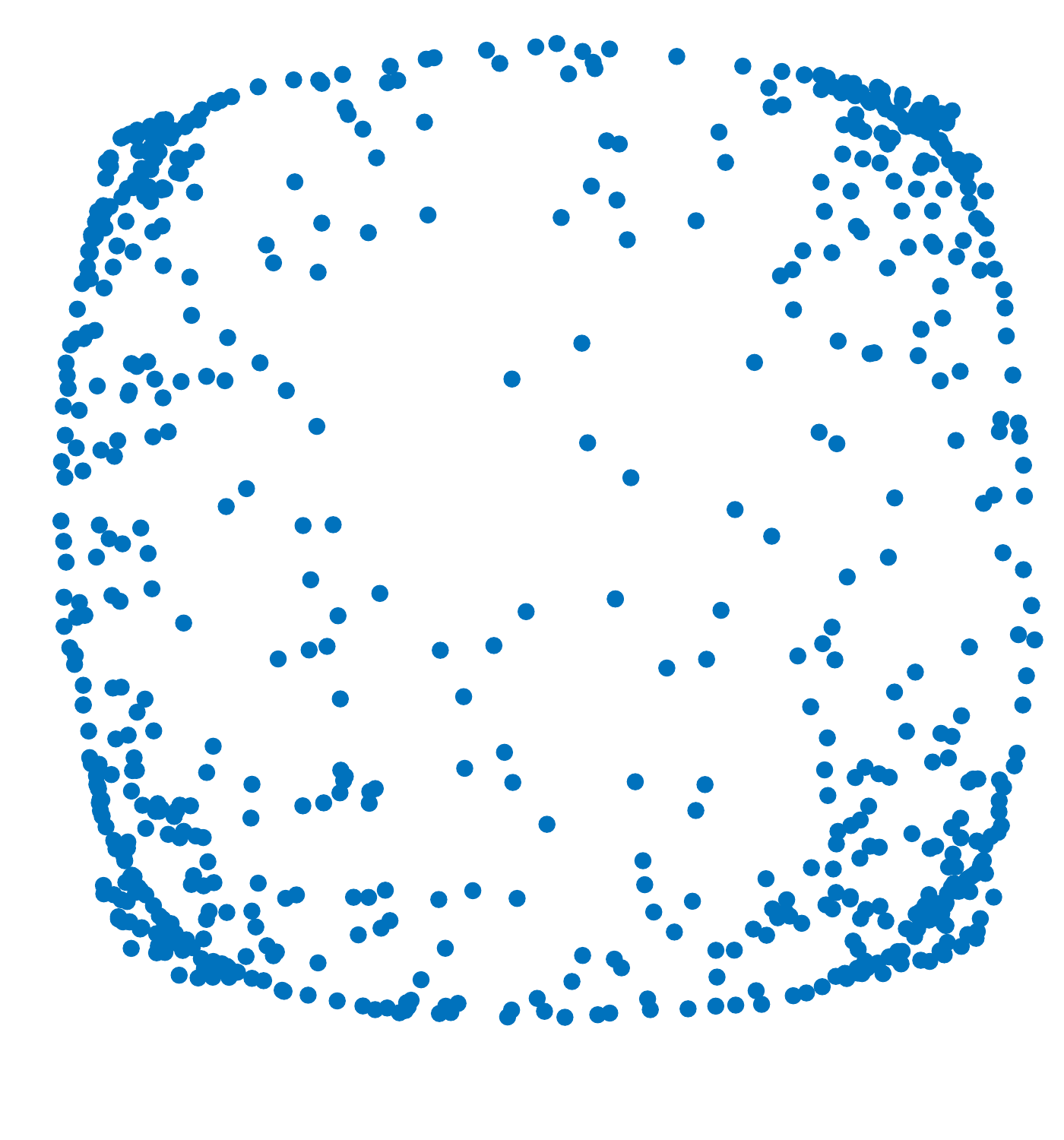}
        \end{minipage}
        \begin{minipage}[t]{0.08\textwidth}
            \centering
            $\tau=3$ \\ 
            \includegraphics[width=1\textwidth]{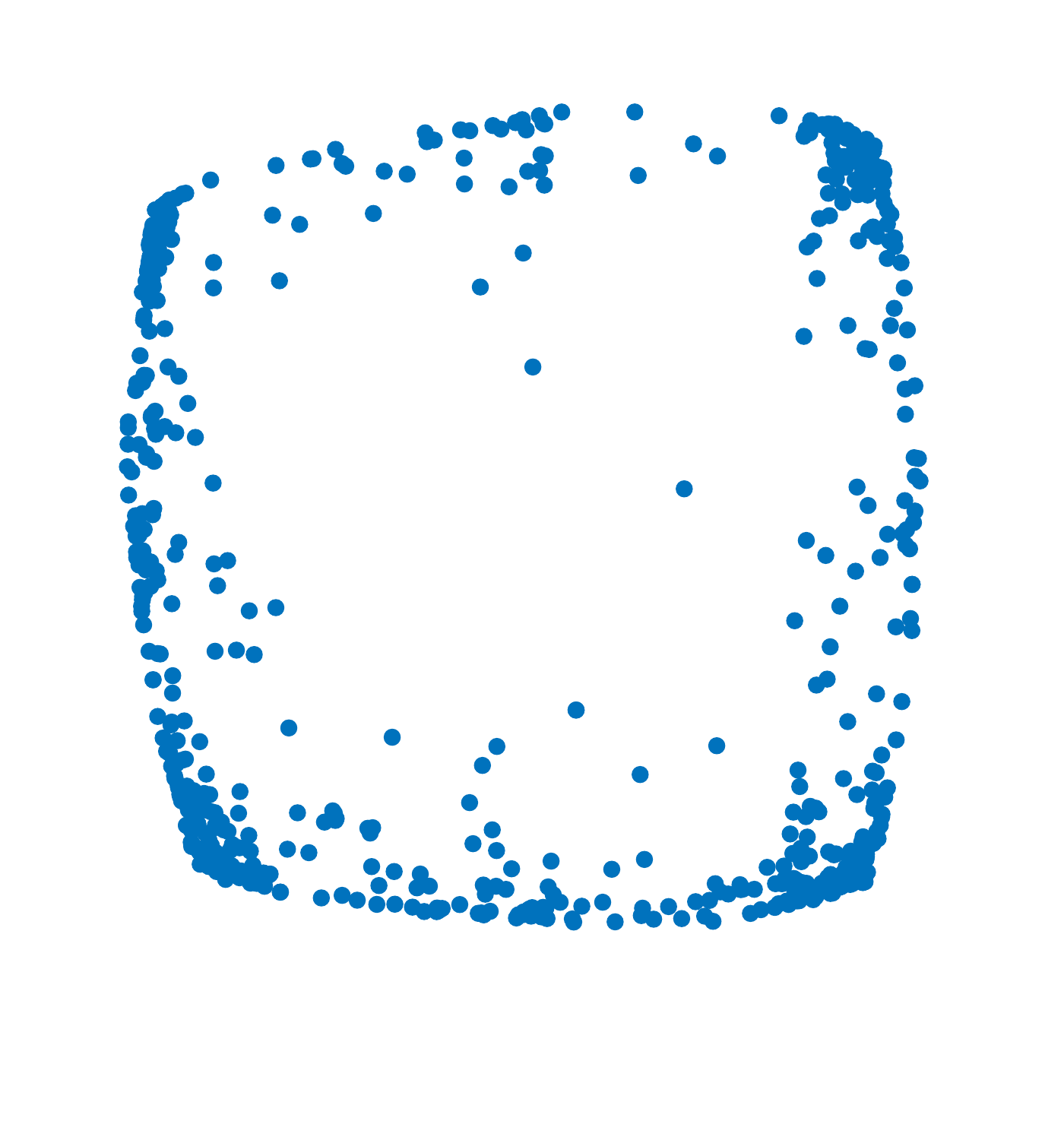}
        \end{minipage}
        \begin{minipage}[t]{0.08\textwidth}
            \centering
            $\tau=4$ \\ 
            \includegraphics[width=1\textwidth]{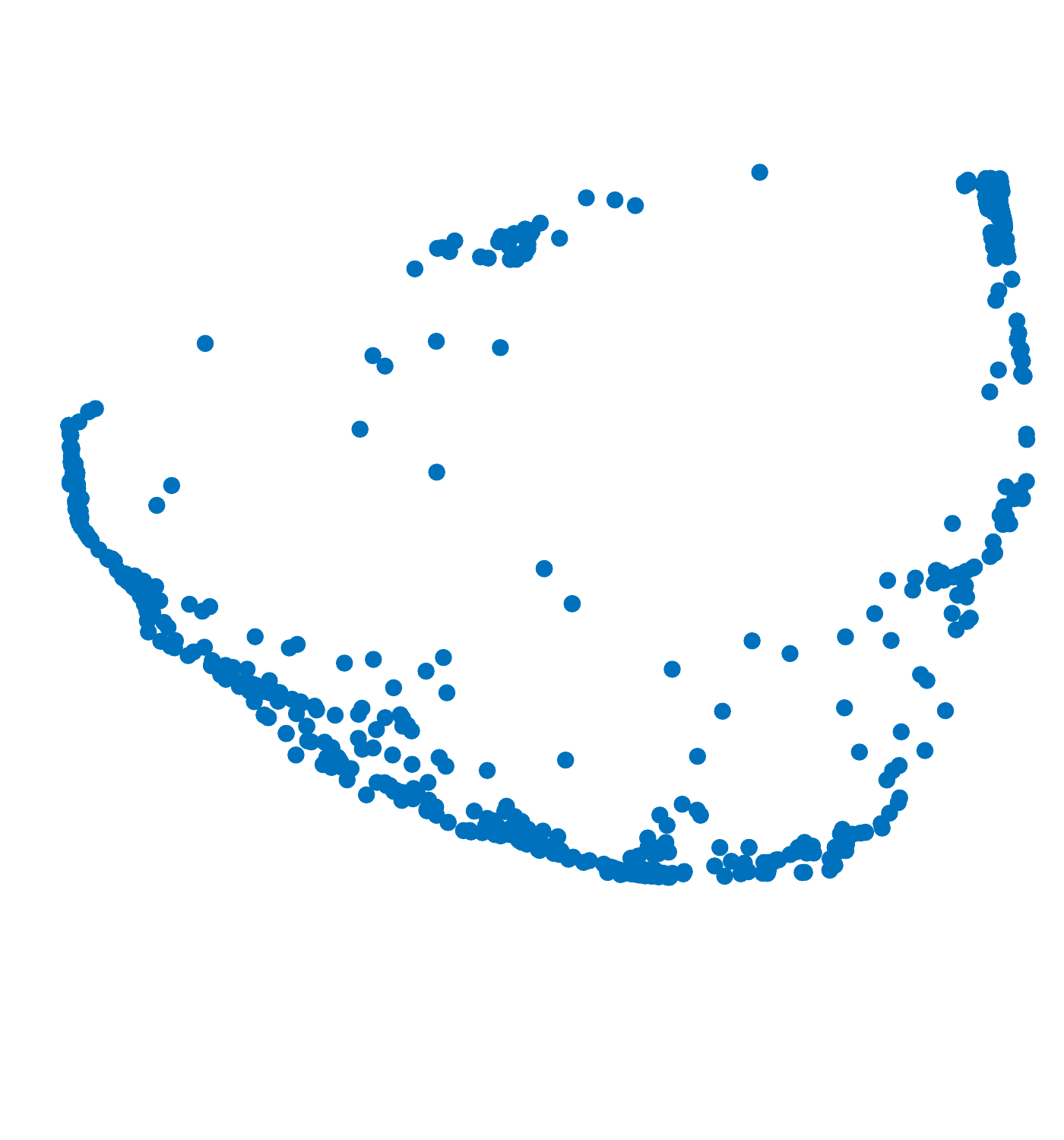}
        \end{minipage}
    }

    \subfloat[GT-$\lambda$-1]{
        \begin{minipage}[t]{0.08\textwidth}
            \centering
            \includegraphics[width=1\textwidth]{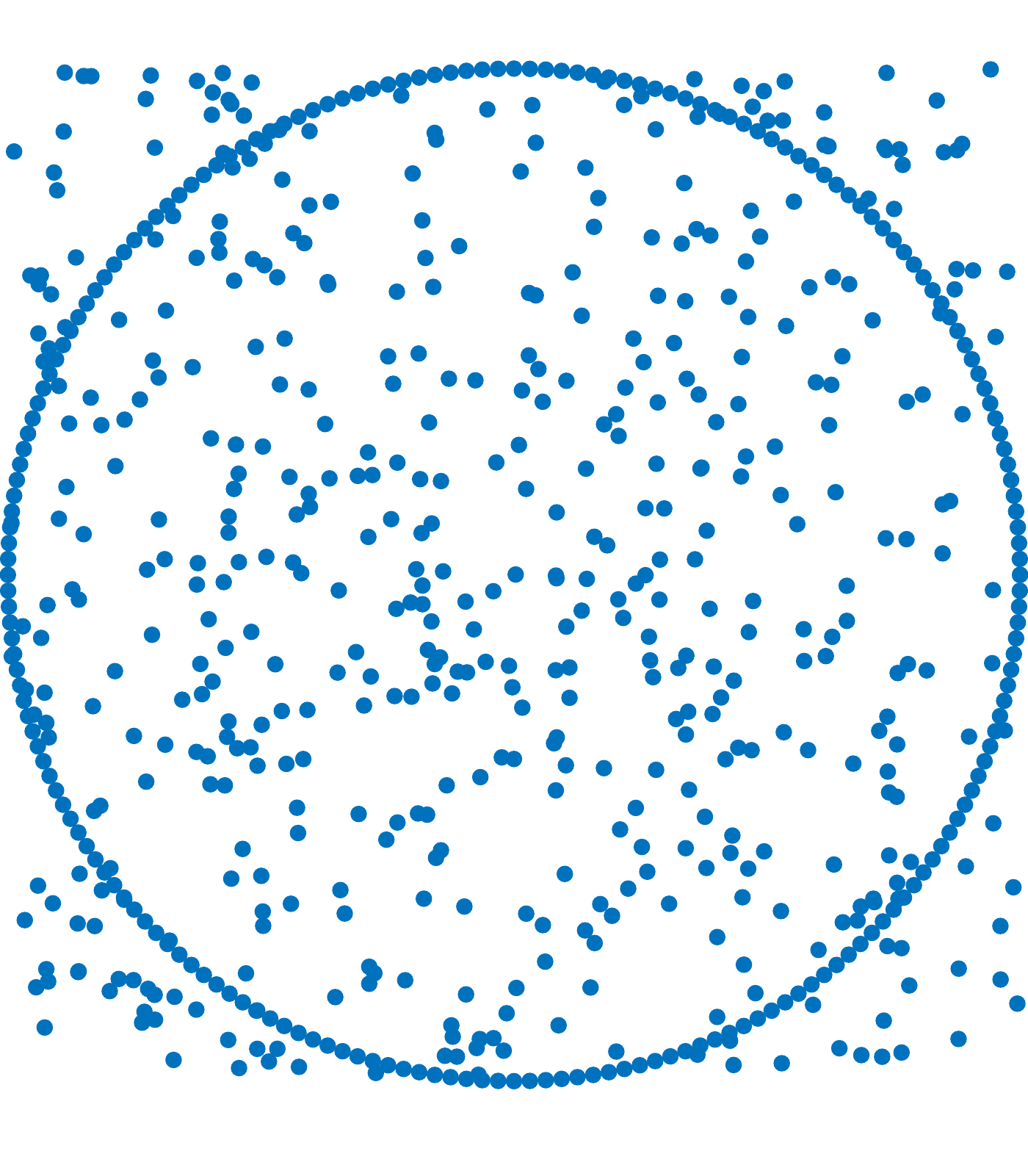}
        \end{minipage}
        \begin{minipage}[t]{0.08\textwidth}
            \centering
            \includegraphics[width=1\textwidth]{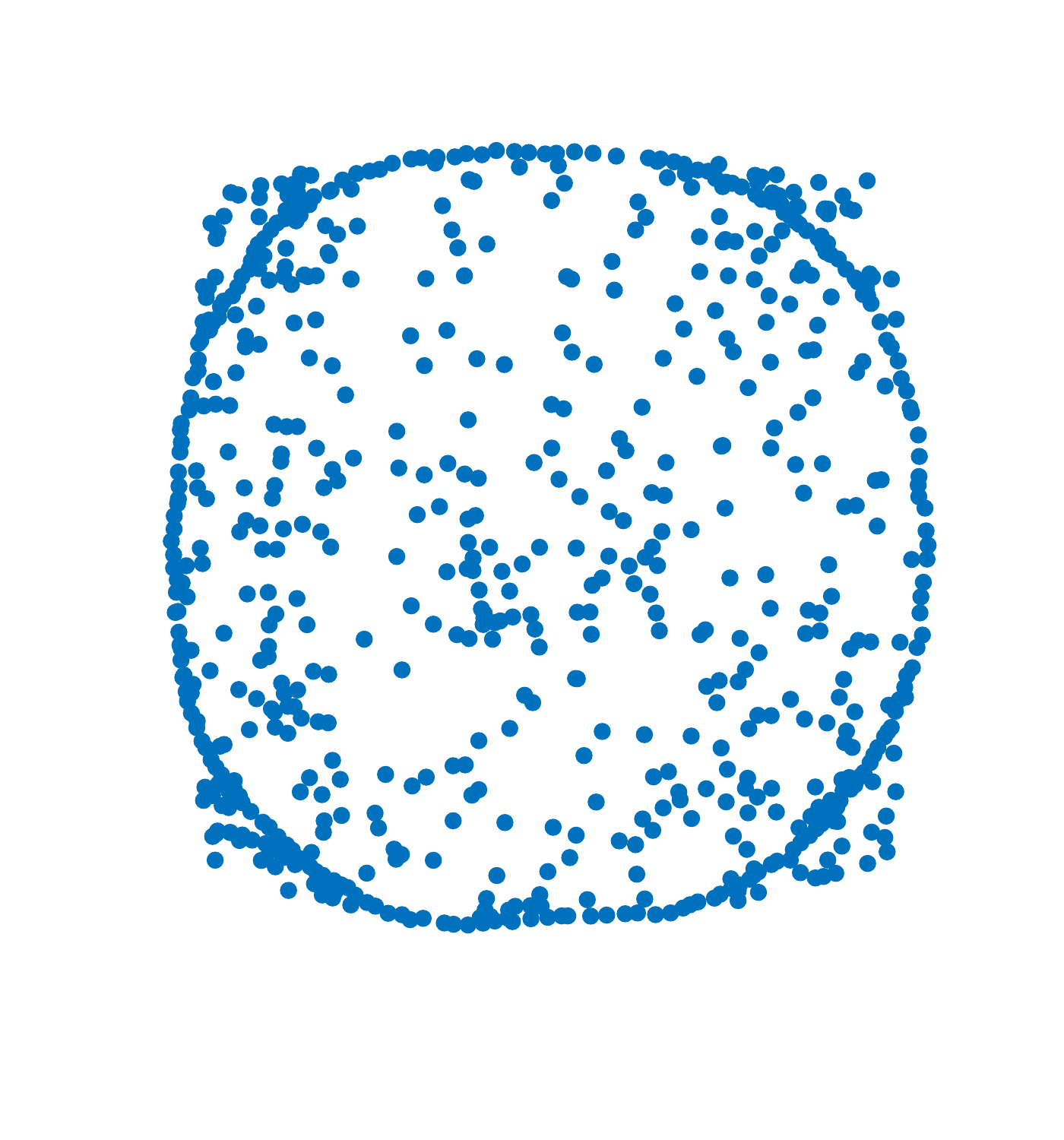}
        \end{minipage}
        \begin{minipage}[t]{0.08\textwidth}
            \centering
            \includegraphics[width=1\textwidth]{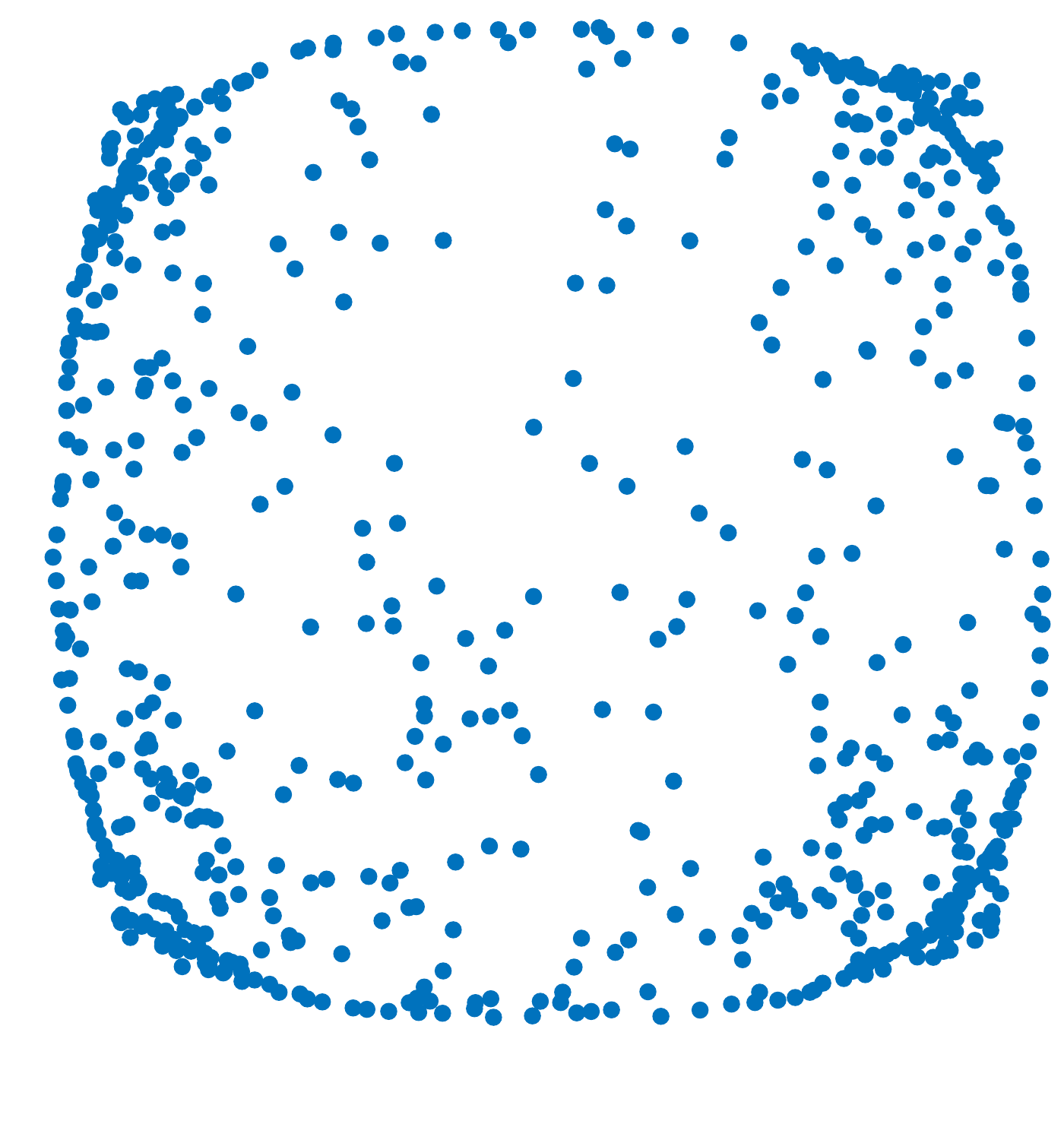}
        \end{minipage}
        \begin{minipage}[t]{0.08\textwidth}
            \centering
            \includegraphics[width=1\textwidth]{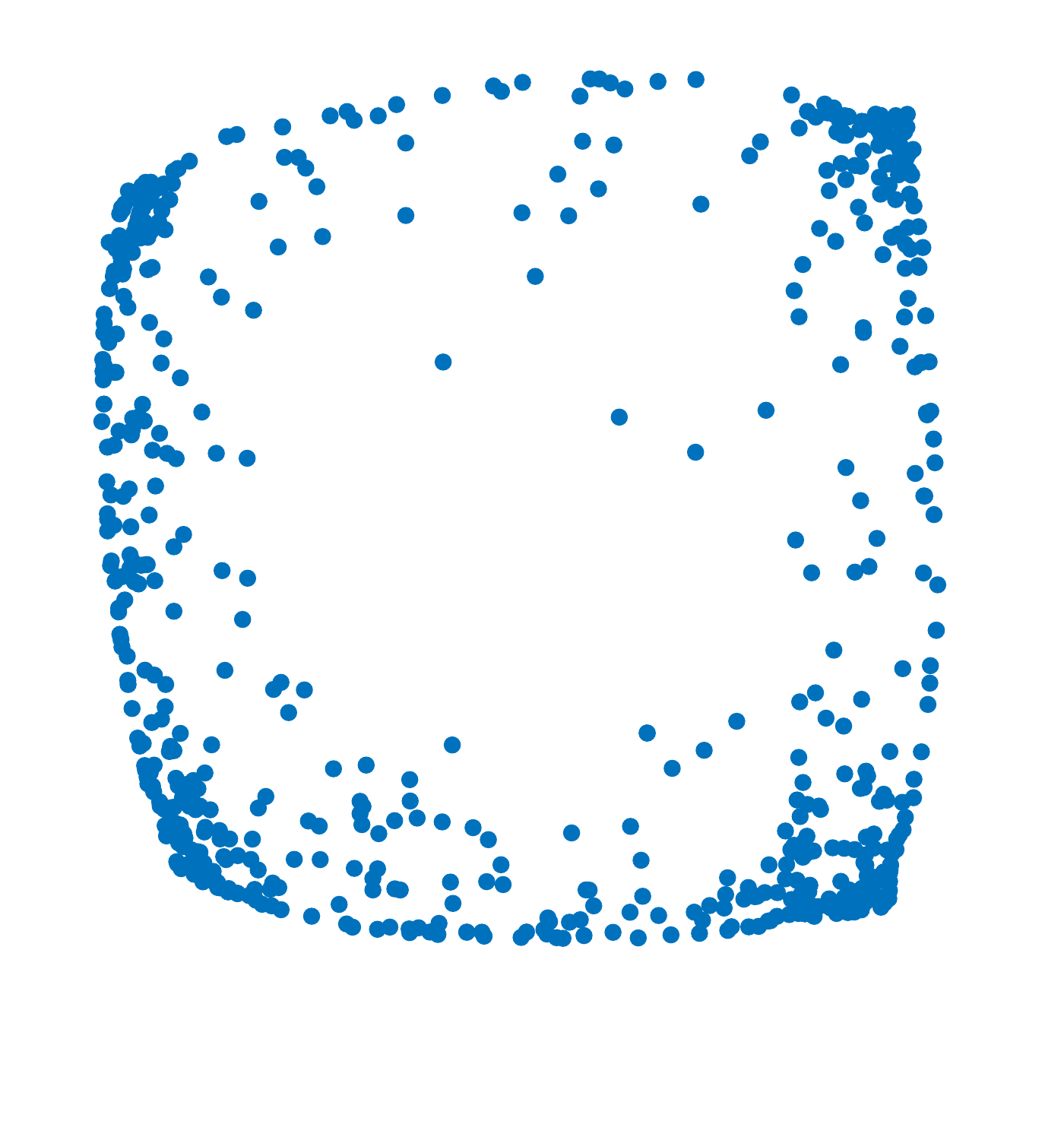}
        \end{minipage}
        \begin{minipage}[t]{0.08\textwidth}
            \centering
            \includegraphics[width=1\textwidth]{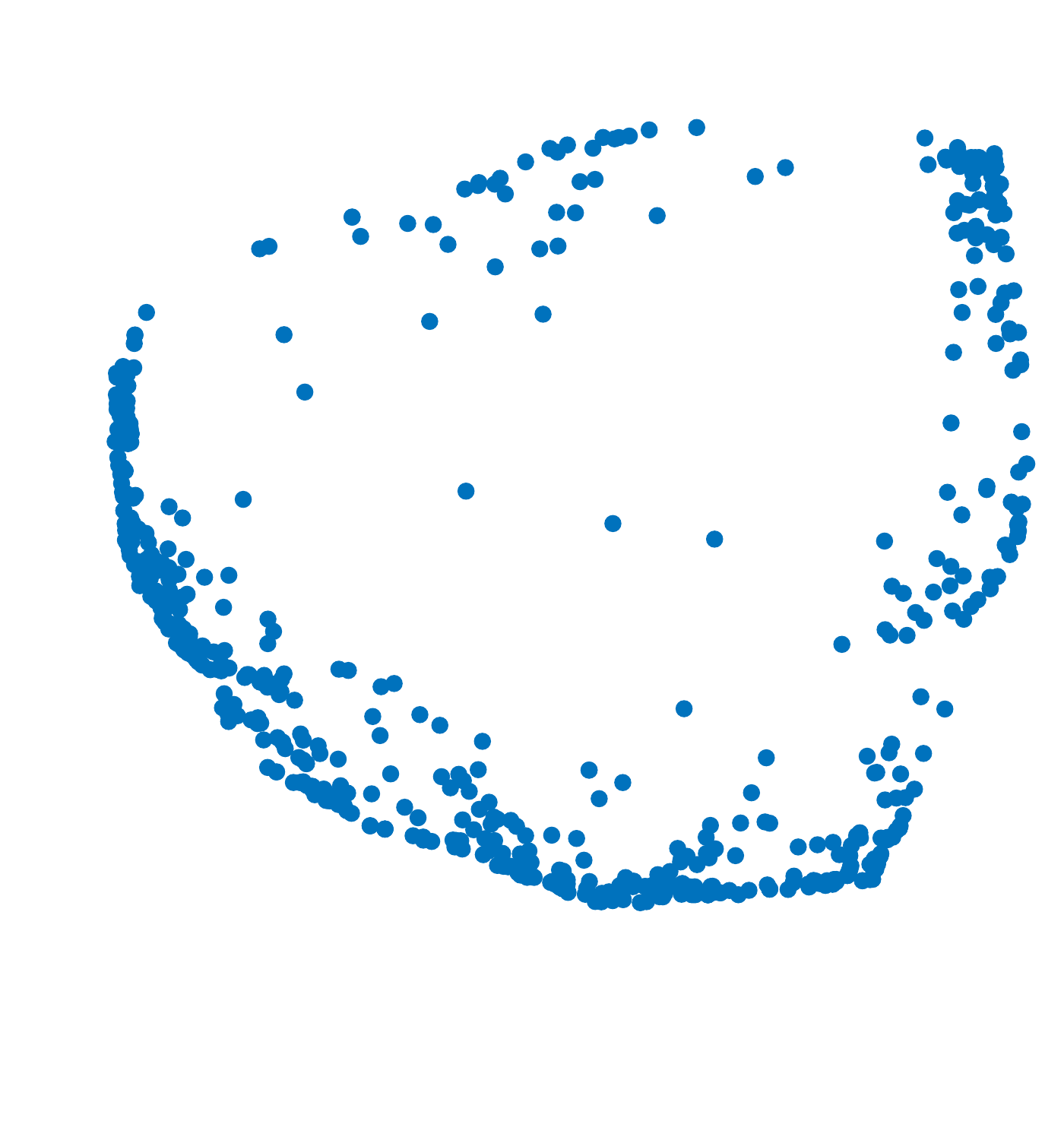}
        \end{minipage}
    }

    \subfloat[GT-$\lambda$-10]{
        \begin{minipage}[t]{0.08\textwidth}
            \centering
            \includegraphics[width=1\textwidth]{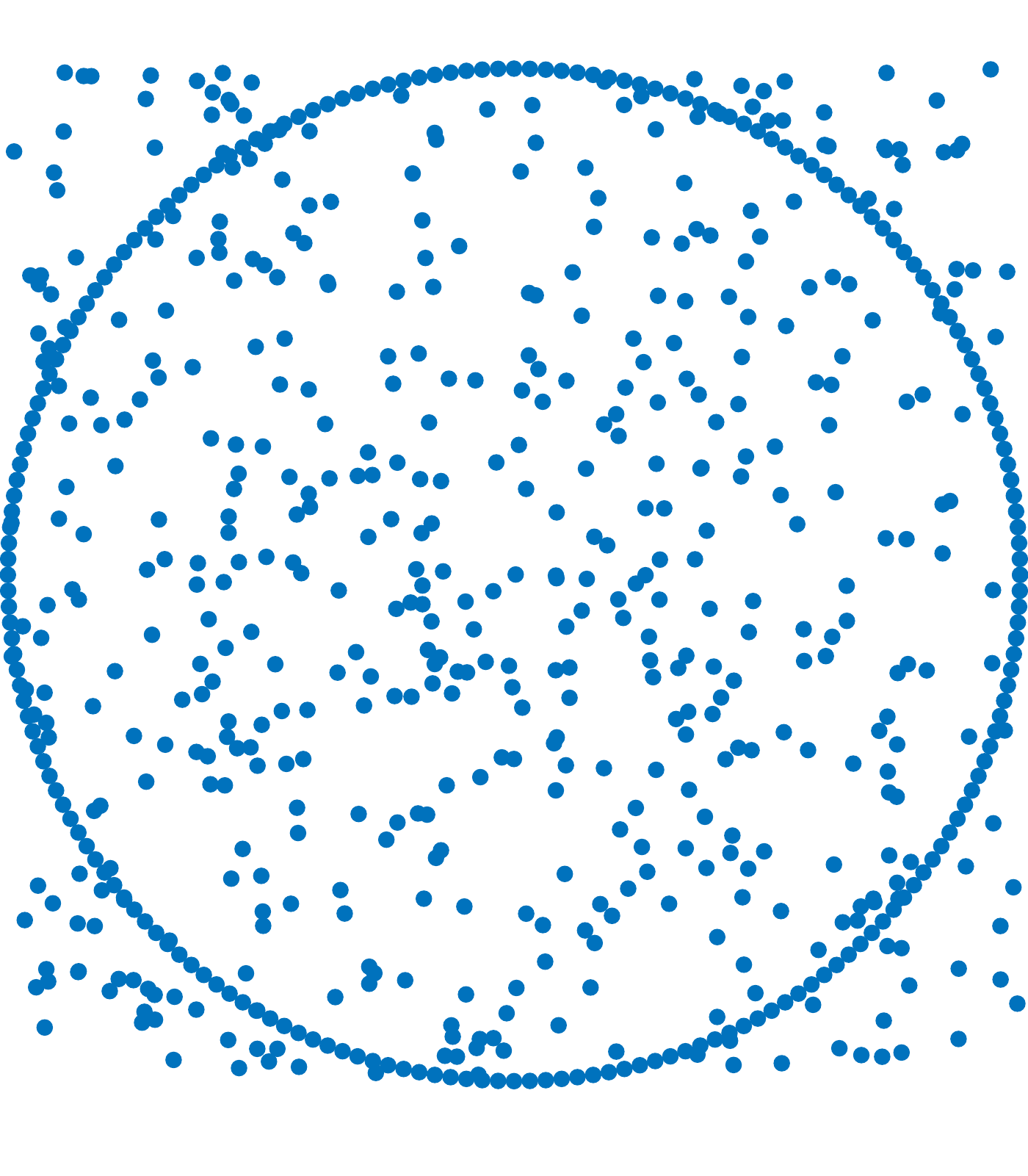}
        \end{minipage}
        \begin{minipage}[t]{0.08\textwidth}
            \centering
            \includegraphics[width=1\textwidth]{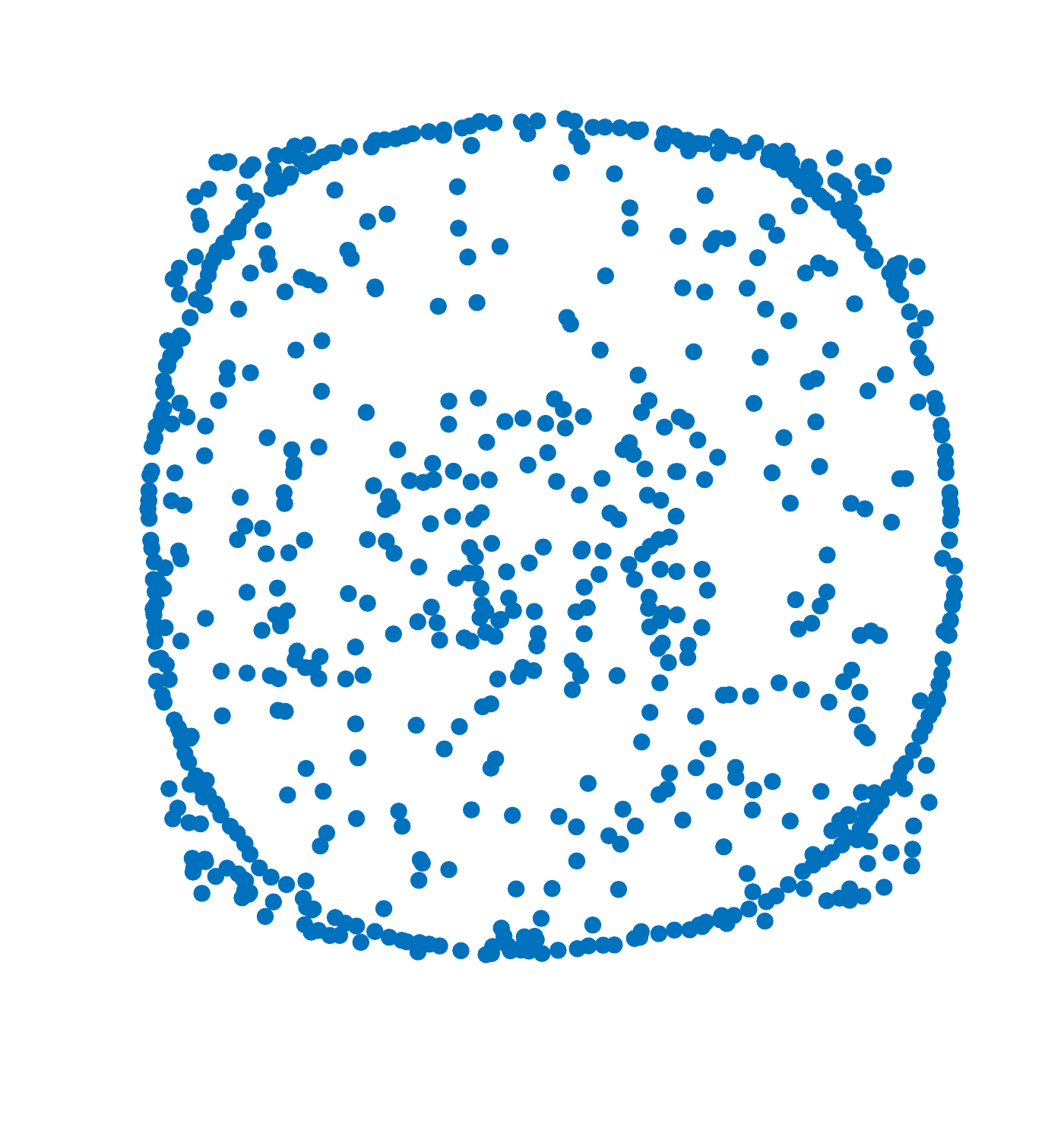}
        \end{minipage}
        \begin{minipage}[t]{0.08\textwidth}
            \centering
            \includegraphics[width=1\textwidth]{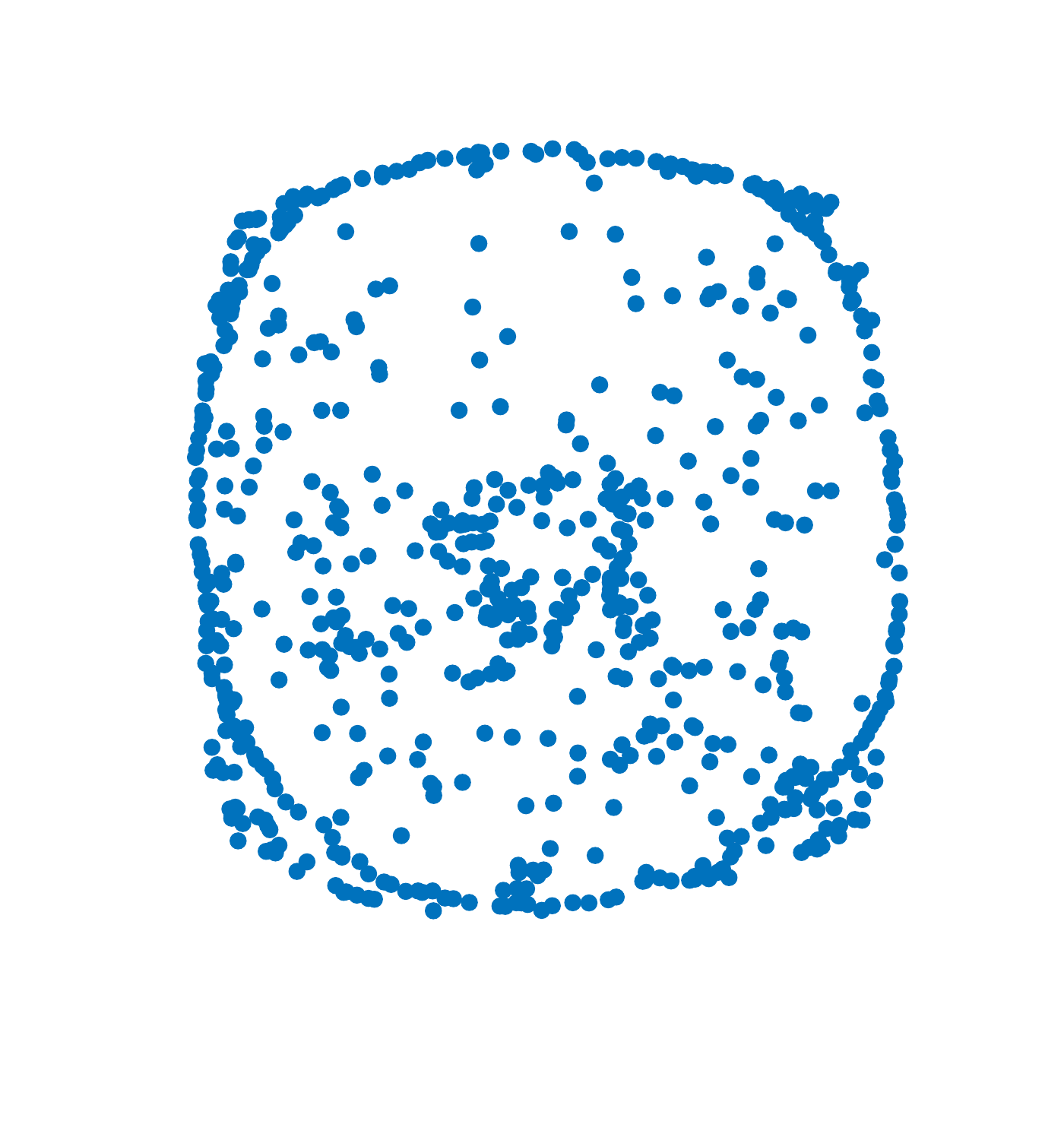}
        \end{minipage}
        \begin{minipage}[t]{0.08\textwidth}
            \centering
            \includegraphics[width=1\textwidth]{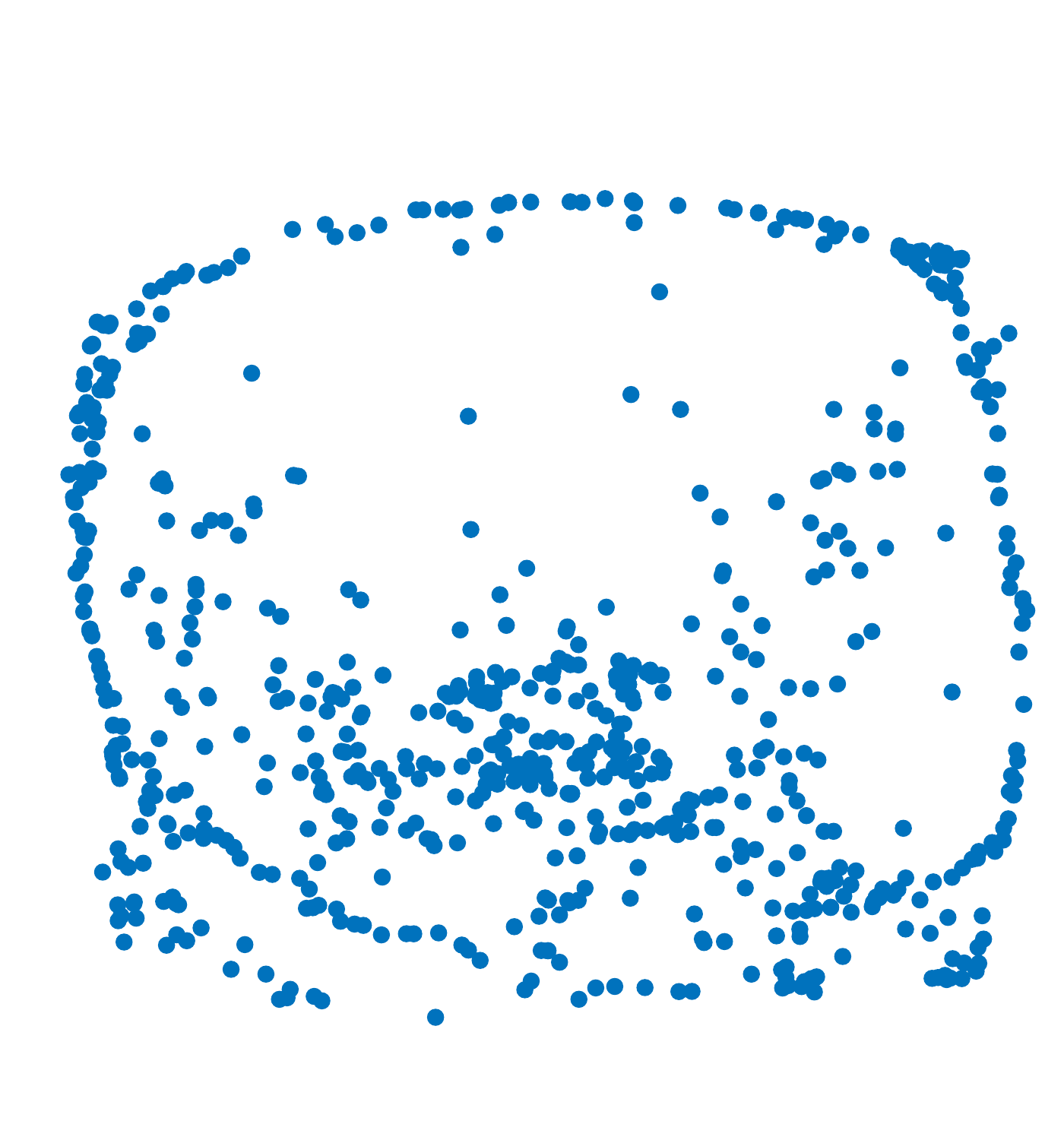}
        \end{minipage}
        \begin{minipage}[t]{0.08\textwidth}
            \centering
            \includegraphics[width=1\textwidth]{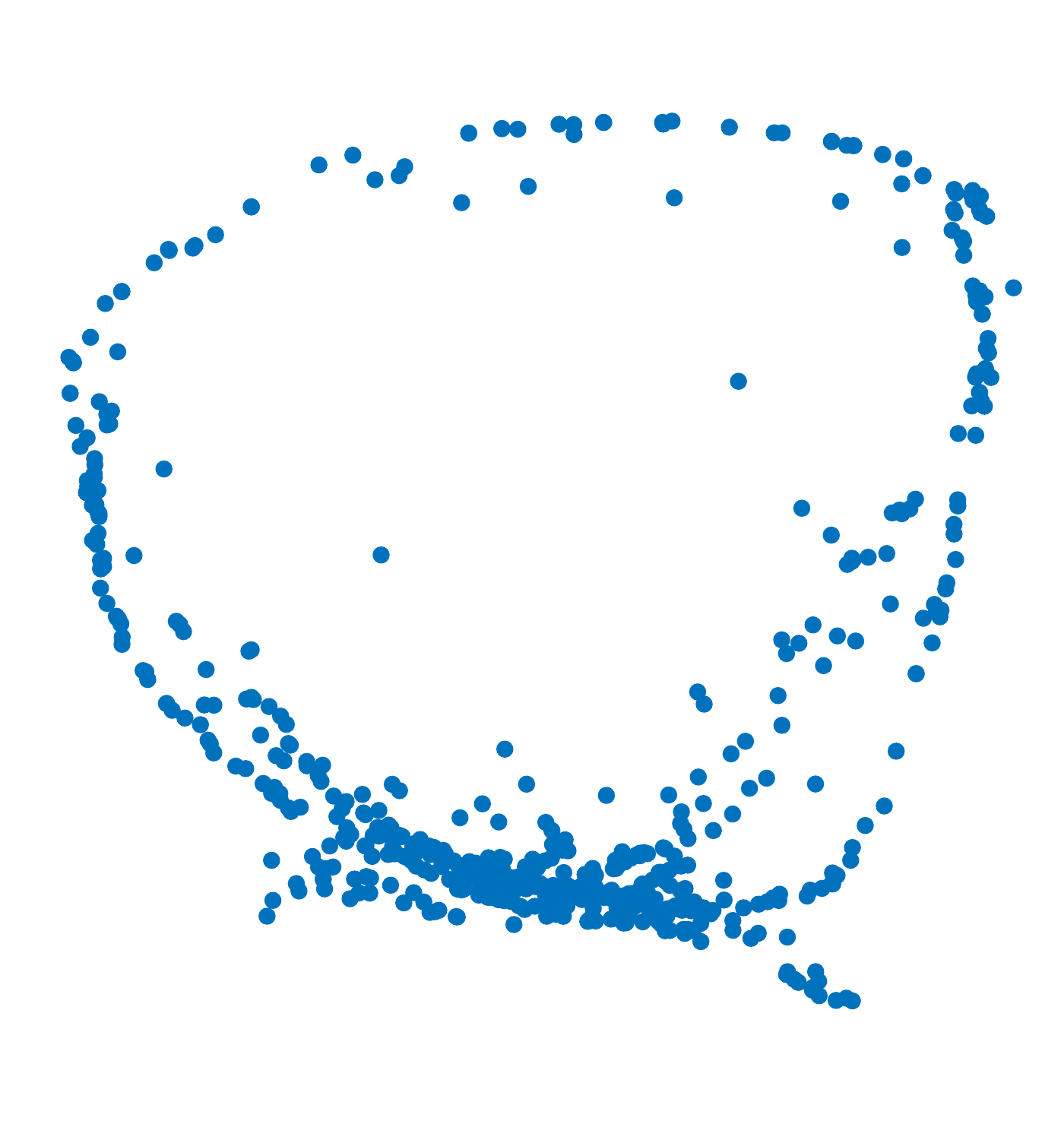}
        \end{minipage}
    }

    \subfloat[WT2]{
        \begin{minipage}[t]{0.08\textwidth}
            \centering
            \includegraphics[width=1\textwidth]{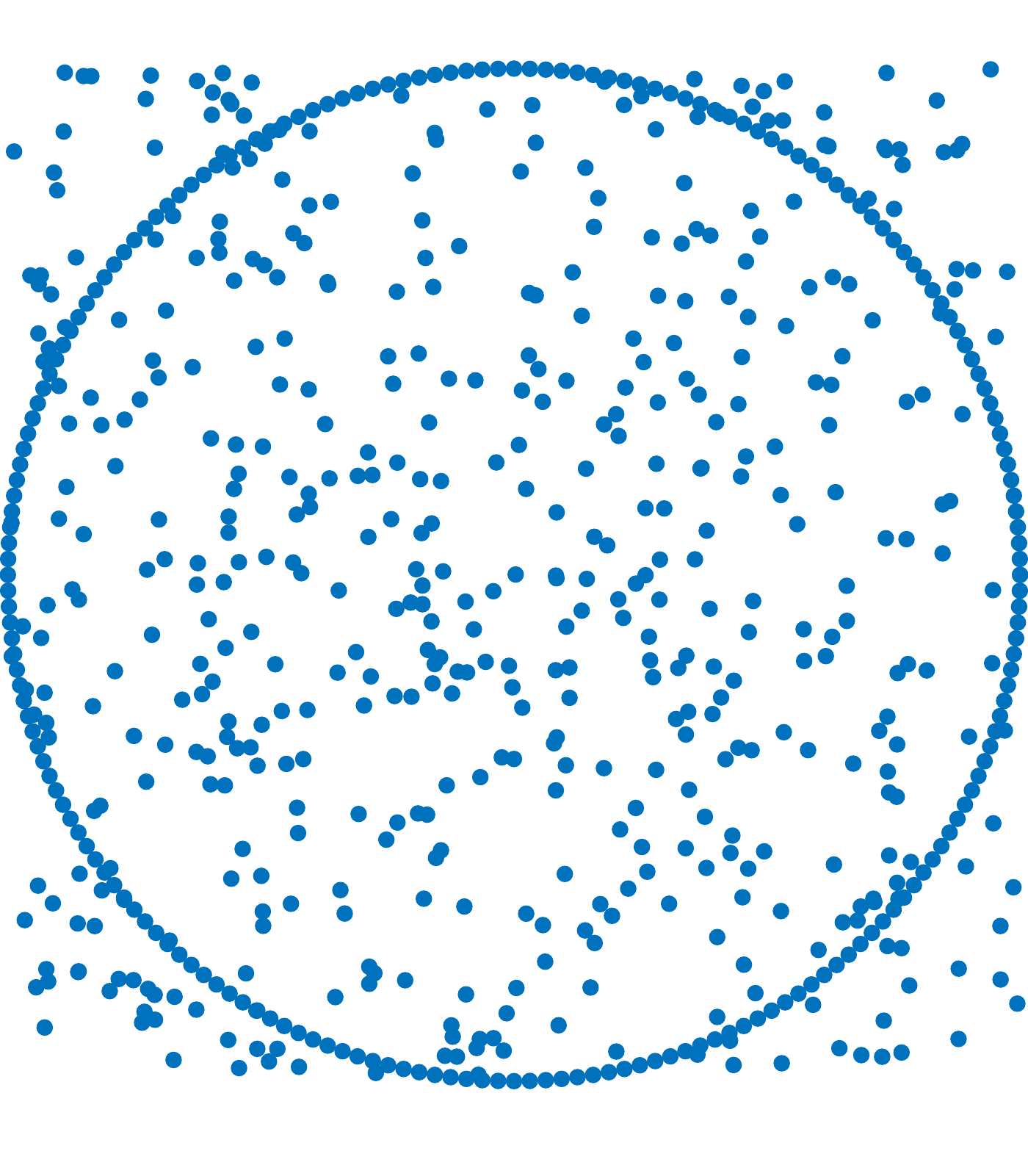}
        \end{minipage}
        \begin{minipage}[t]{0.08\textwidth}
            \centering
            \includegraphics[width=1\textwidth]{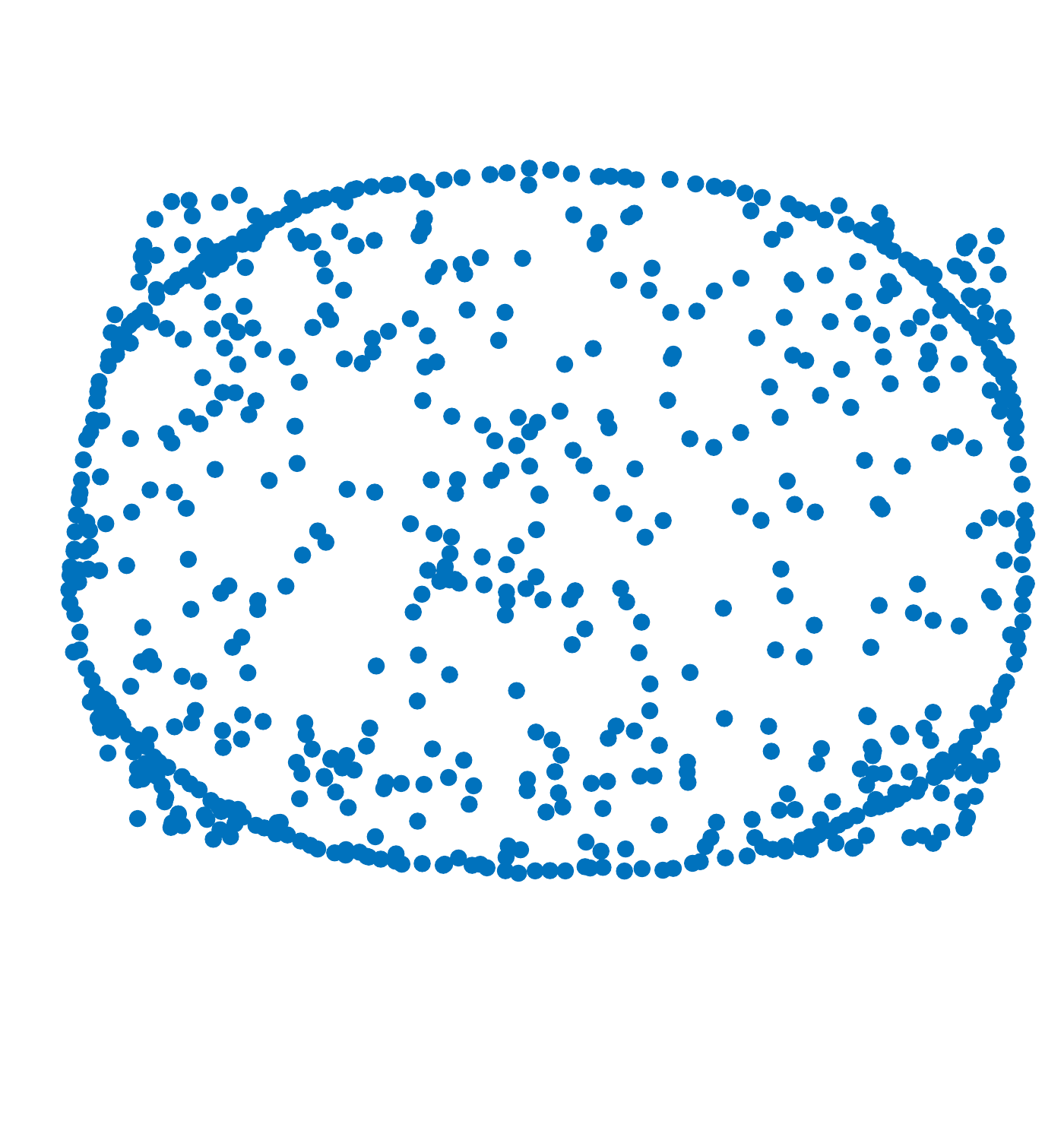}
        \end{minipage}
        \begin{minipage}[t]{0.08\textwidth}
            \centering
            \includegraphics[width=1\textwidth]{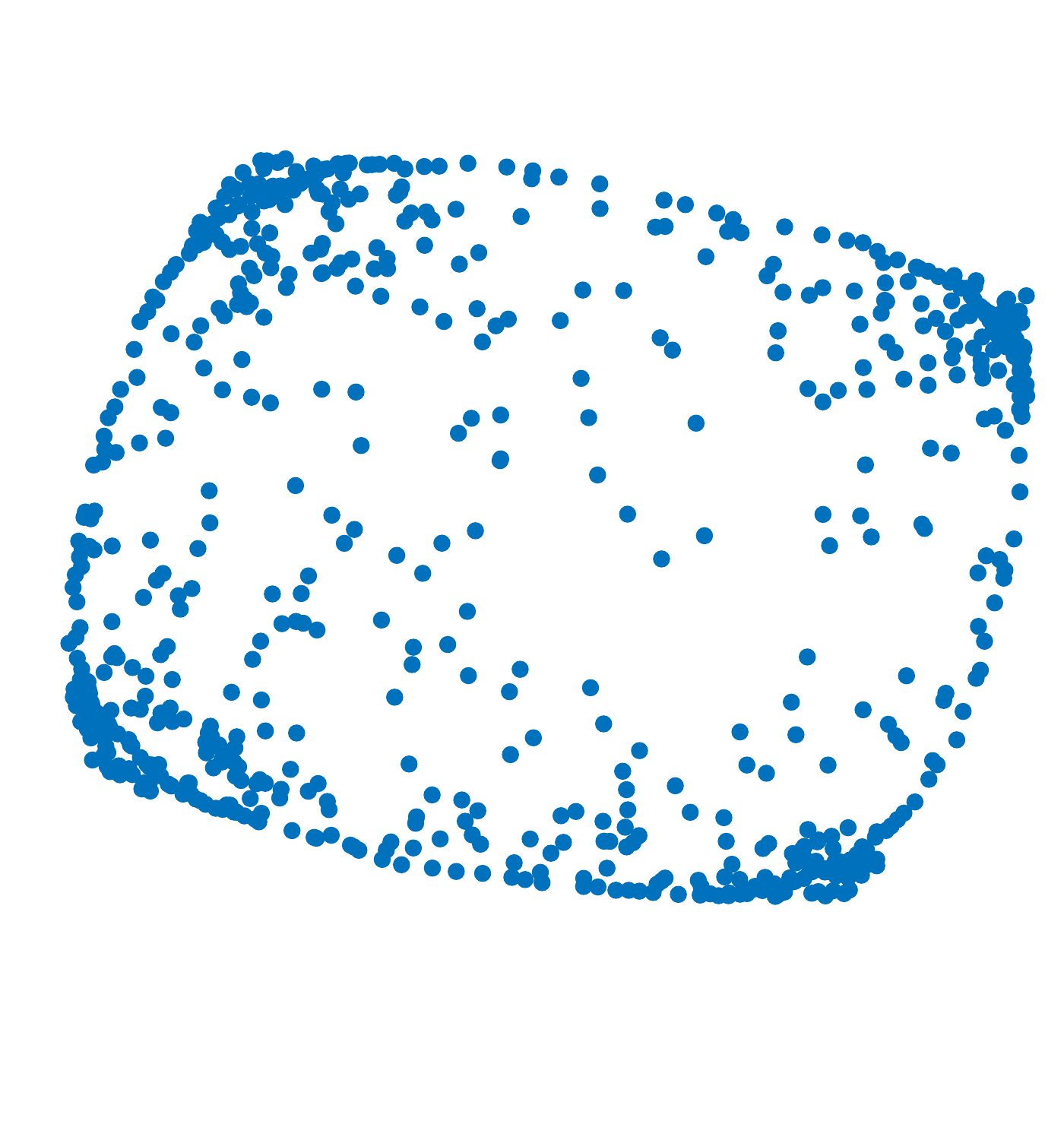}
        \end{minipage}
        \begin{minipage}[t]{0.08\textwidth}
            \centering
            \includegraphics[width=1\textwidth]{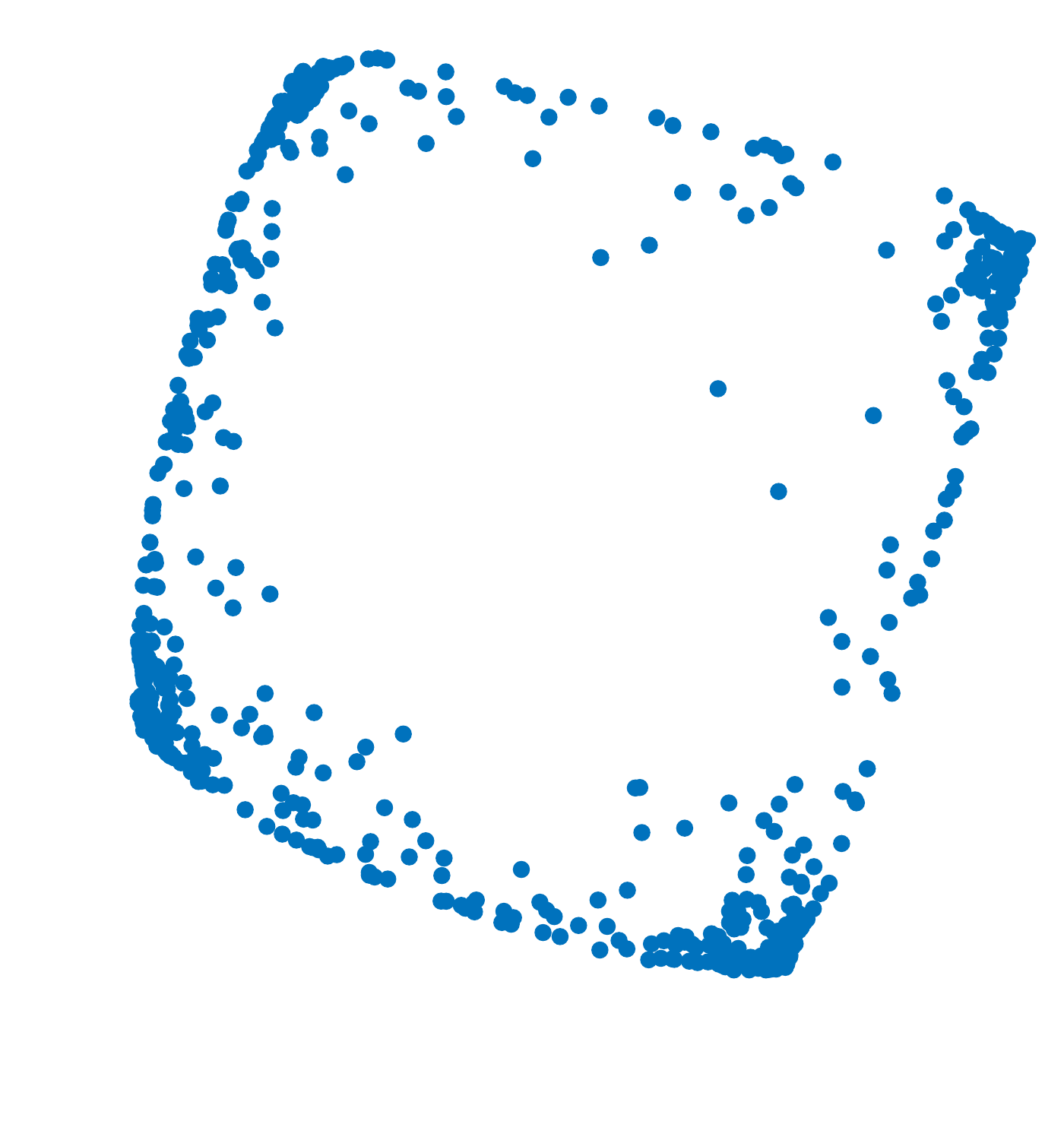}
        \end{minipage}
        \begin{minipage}[t]{0.08\textwidth}
            \centering
            \includegraphics[width=1\textwidth]{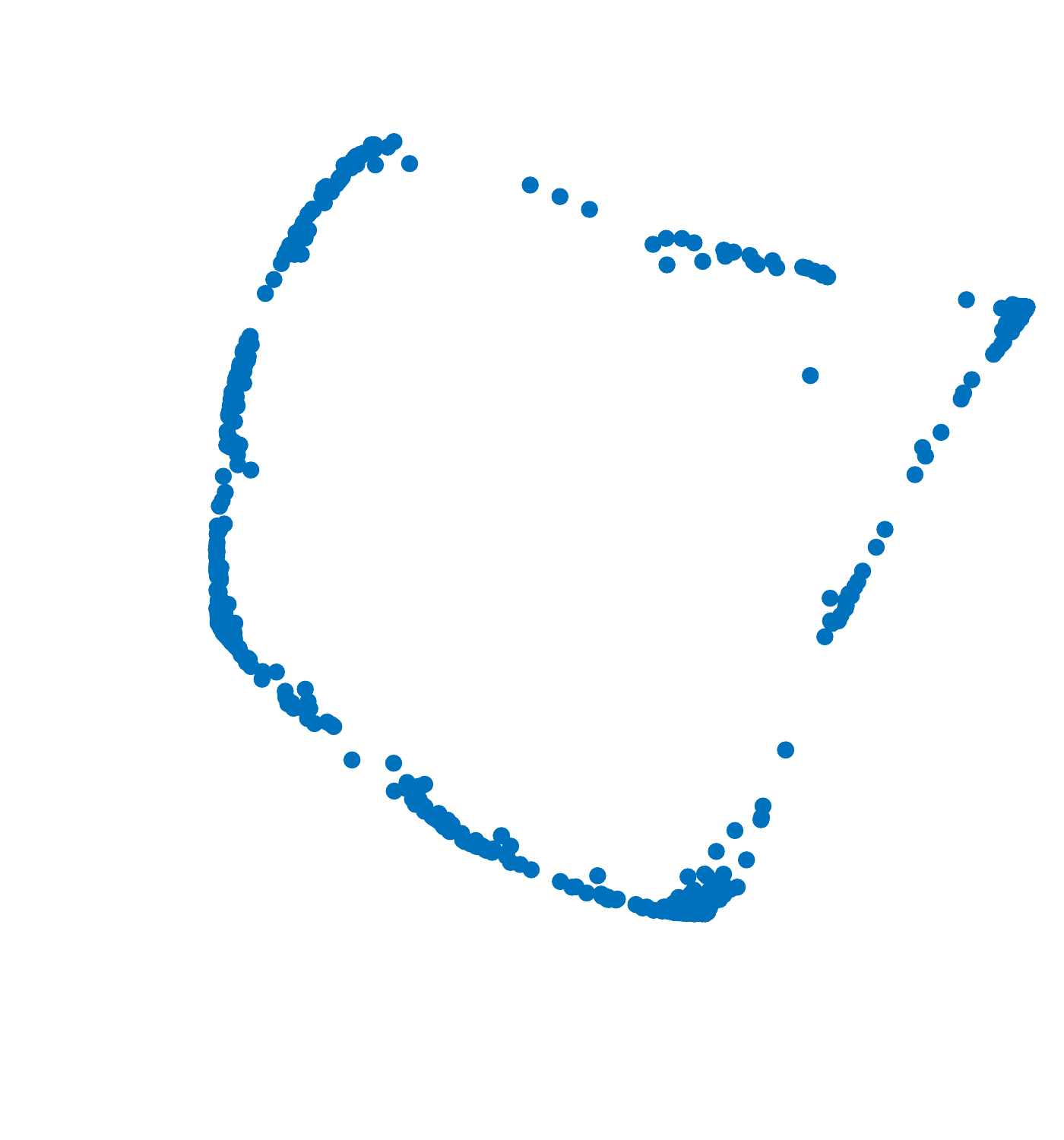}
        \end{minipage}
    }

    \subfloat[WT1]{
        \begin{minipage}[t]{0.08\textwidth}
            \centering
            \includegraphics[width=1\textwidth]{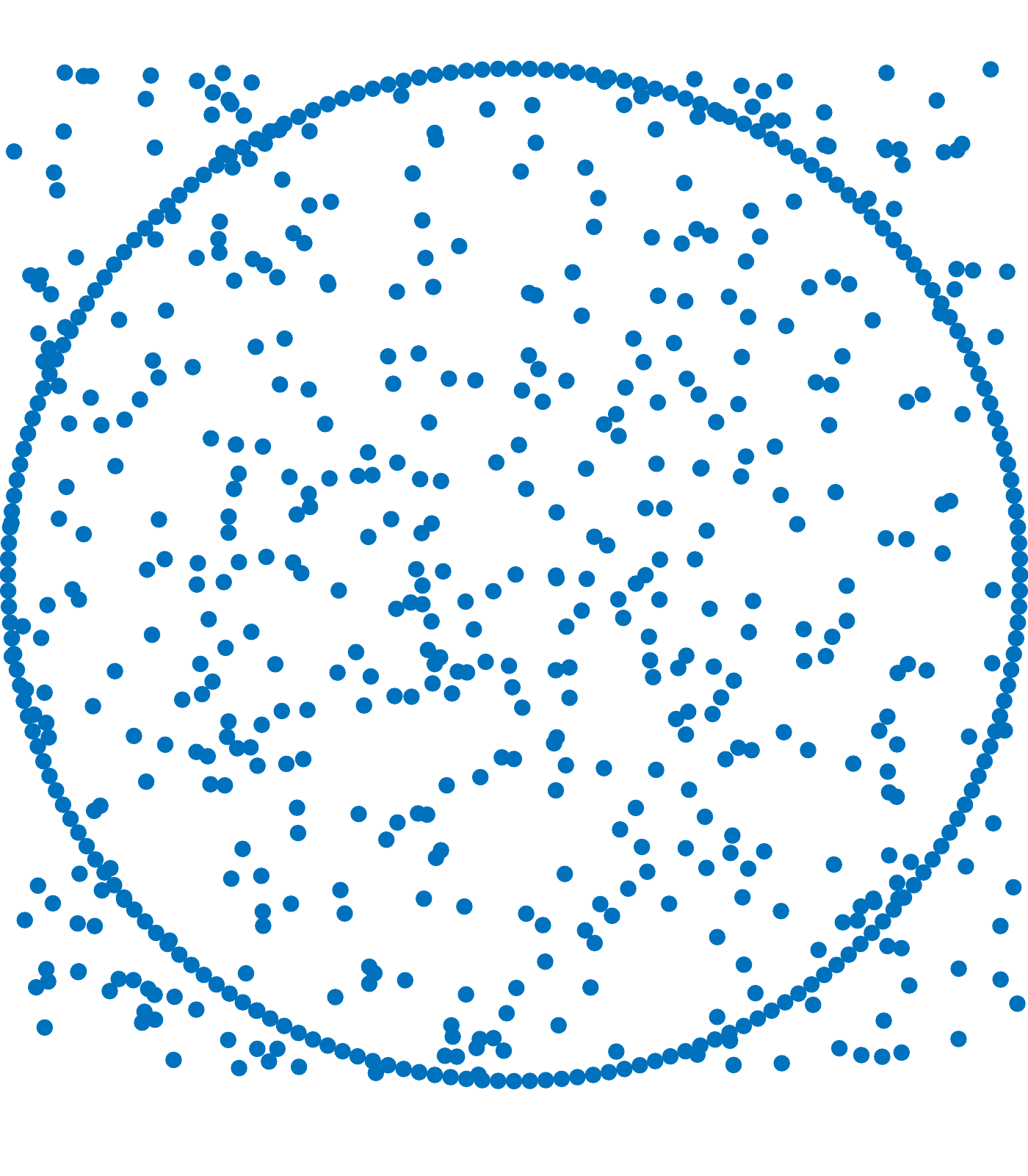}
        \end{minipage}
        \begin{minipage}[t]{0.08\textwidth}
            \centering
            \includegraphics[width=1\textwidth]{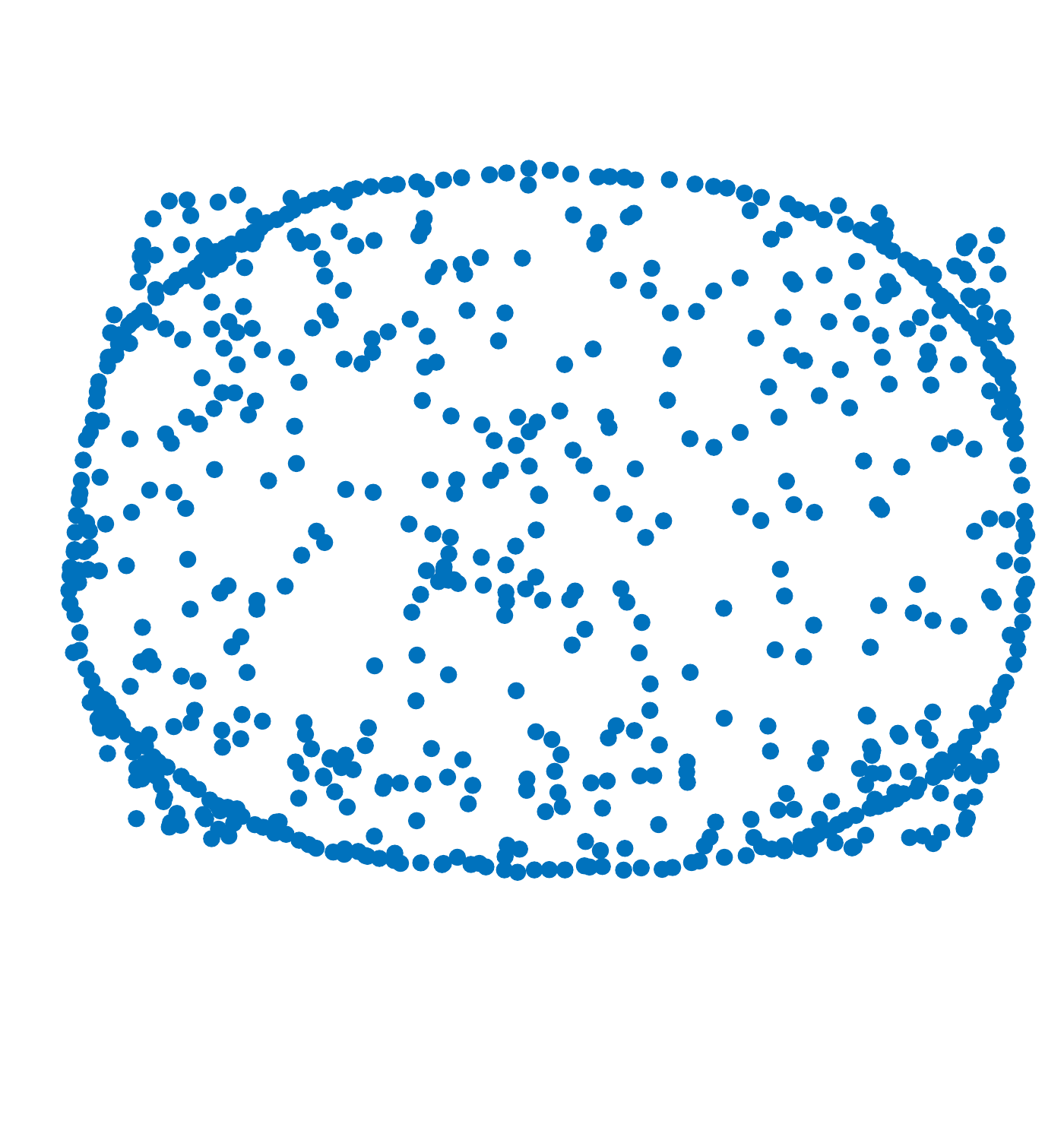}
        \end{minipage}
        \begin{minipage}[t]{0.08\textwidth}
            \centering
            \includegraphics[width=1\textwidth]{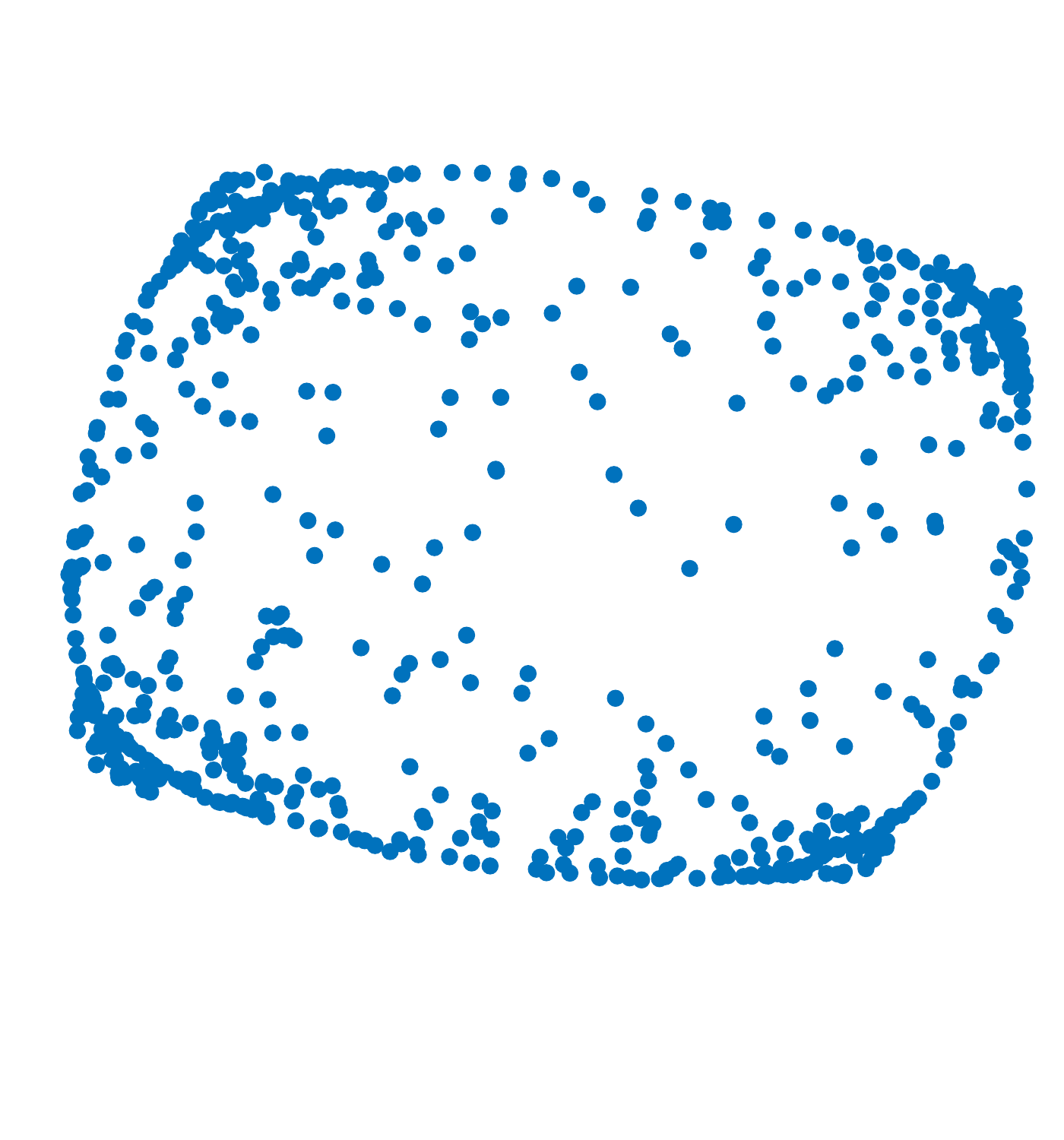}
        \end{minipage}
        \begin{minipage}[t]{0.08\textwidth}
            \centering
            \includegraphics[width=1\textwidth]{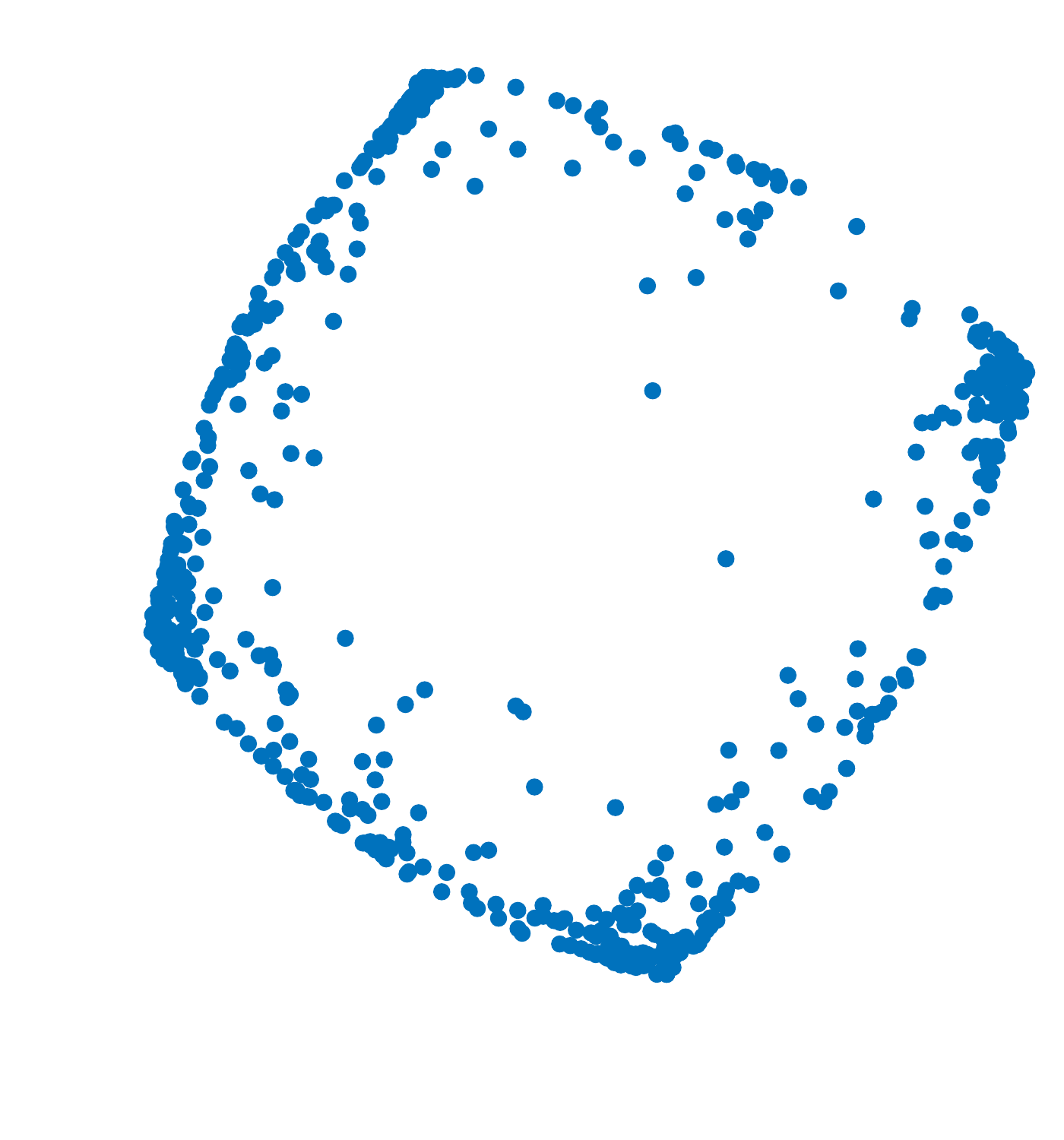}
        \end{minipage}
        \begin{minipage}[t]{0.08\textwidth}
            \centering
            \includegraphics[width=1\textwidth]{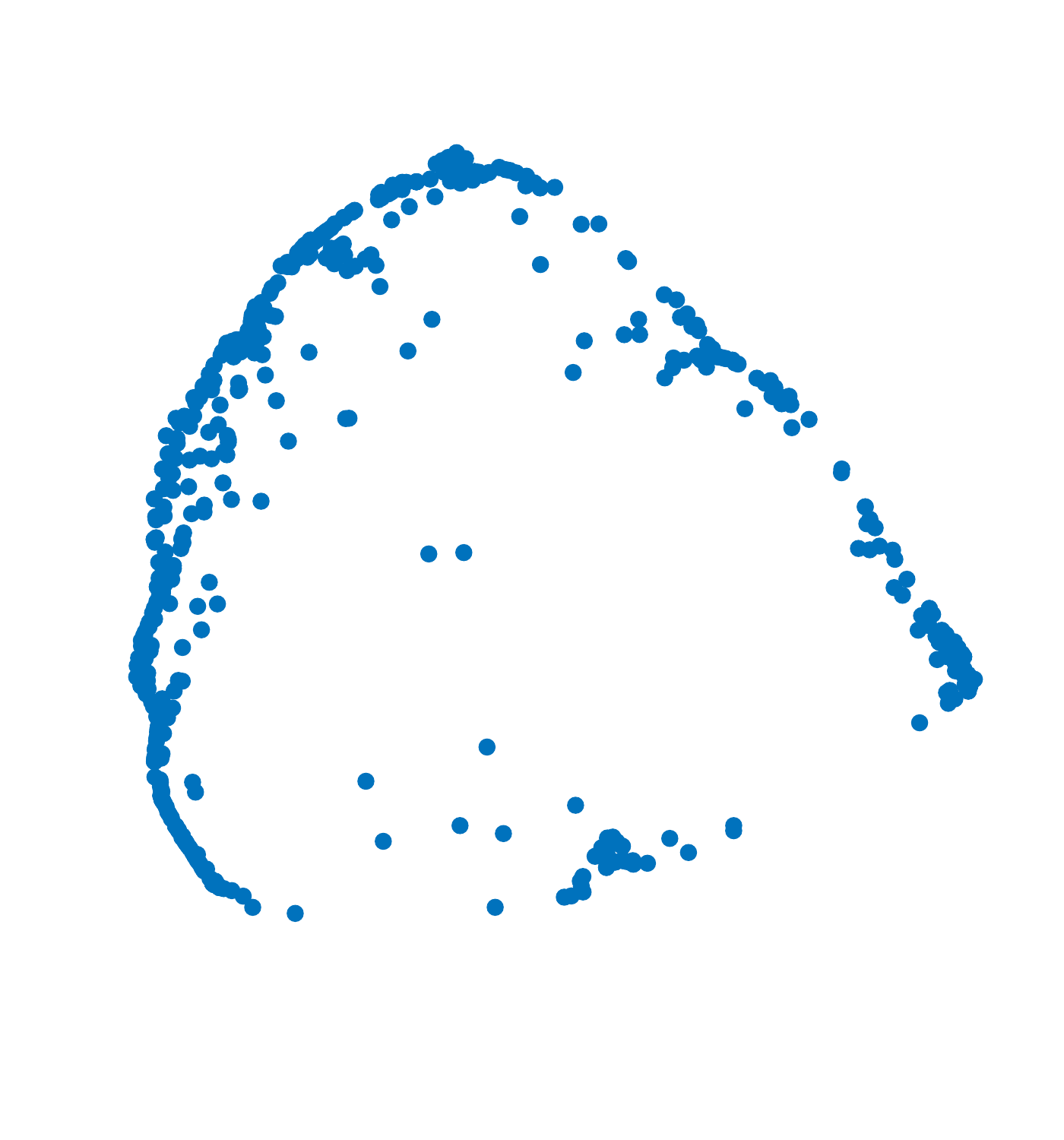}
        \end{minipage}
    }
    \caption{Denoising of a noisy circle. In each iteration, we use a fixed radius $\eps=0.7$. We apply GT with $\lambda=10$.}
    \label{fig:supp-S1}
\end{figure}

\subsection{Details about implementation of GT for image segmentation}\label{app:image-seg}

There are multiple variants of MS for image segmentation. We follow the implementation in \cite{demirovic2019implementation}, where in each iteration, the $(\eps_s,\eps_r)$-neighborhood of $T(x)=(T^s(x),T^r(x))$ is defined to be the set of all pixels $y=(y^s,y^r)$ such that $\| y^s - x^s \| \leq \eps_s$ and $\| y^r - T^r(x) \| \leq \eps_r$ (note that the $\eps_s$-spatial neighborhood of $T(x)$ is always the same as the one of $x$ through iterations because of the first inequality). We adapt our GT algorithm according to this modified version of MS and use a variant (cf. Equation~(\ref{eq:image-gt})) below of Equation~(\ref{eq:gtd}) to compute the GT distance such that when $\lambda=0$, our GT algorithm boils down to the MS based algorithm of \cite{demirovic2019implementation}.

We implement GT for image segmentation through the following  precise procedures:

\begin{enumerate}
    \item Initialization:
    \begin{enumerate}
        \item transfer pixels into 5-dimensional feature points $x_i = (x_i^s, x_i^r)$; 
        \item specify spatial and range bandwidth parameters $\eps_s$ and $\eps_r$. 
        \item compute the 2-dimensional covariance matrix $\Sigma(x_i^s)$ of the $(\eps_s,\eps_r)$-neighborhood of $x_i$ using only spatial features.
    \end{enumerate}
    \item Associate a cluster point $T(x_i) = (T^s(x_i),T^r(x_i))$ to every pixel $x_i$, and initialize it to be $(x_i^s,x_i^r)$. Repeat the following steps for each $i$ until $T(x_i)$ converges:
    \begin{enumerate}
        \item for each $x_j$ within the $(\eps_s,\eps_r)$-neighborhood of $T(x_i)$ (i.e., $\| x_j^s - x_i^s \| \leq \eps_s$ and $\| x_j^r - T^r(x_i) \| \leq \eps_r$), we compute via the following formula a variant of GT distance between the spatial features of $x_j$ and $T(x_i)$, denoted by $d^{\mathsmaller{(\eps,\lambda)}}_{\alpha,d_s} (x_j^s, T^s(x_i))$ where $d_s$ refers to the Euclidean distance on spatial features:
        \begin{equation}\label{eq:image-gt}
         \lc\norm{x_i^s-x_j^s}^2+\lambda\cdot\dcov^2\left(\Sigma(x_j^s),\Sigma(T^s(x_i))\right)\rc^\frac{1}{2}. 
        \end{equation}
        This is slightly different from Equation (\ref{eq:gtd}) that we use fixed $x_i^s$ for the Euclidean part and  $T^s(x_i)$ for the $\dcov$ part in the iteration to be comparable with MS, i.e., when $\lambda=0$, it reduces to the MS implementation.
        
        \item determine the $(\eps_s,\eps_r)$-GT-neighborhood of $T(x_i)$, which consists of all pixels $x_j$ satisfying $d^{\mathsmaller{(\eps,\lambda)}}_{\alpha,d_s} (x_j^s, T^s(x_i)) \leq \eps_s$ and $\| T^r(x_i) - x_j^r \| \leq \eps_r$; 
        \item update $T(x_i)$ with the mean of the neighborhood and compute the 2-dimensional covariance matrix $\Sigma(T^s(x_i))$ of the $(\eps_s,\eps_r)$-GT-neighborhood of $T(x_i)$ using spatial features.
    \end{enumerate}
    \item Identify clusters of convergence points $T(x_i)$: we construct a graph taking all convergence points as vertices. We connect $T(x_i)$ with $T(x_j)$ with an edge if and only if $\norm{T^s(x_i)-T^s(x_j)}\leq\eps_s$ and $\norm{T^r(x_i)-T^r(x_j)}\leq\eps_r$. Then, each connected component of the graph forms a cluster of the set of convergence points. Finally, we cluster the set of all pixels such that $x_i$ and $x_j$ belong to the same cluster if and only if $T(x_i)$ and $T(x_j)$ belong to the same cluster.
\end{enumerate}

We apply GT and MS to image segmentation task on cameraman images with different resolutions. The results are shown in Figure~(\ref{fig:supp-cameraman}). When the image is of high resolution (Figure (\ref{fig:supp-image1})), GT performs as well as MS. When the image is of low resolution (Figure (\ref{fig:supp-image2})), we see that GT generates a reasonably better segmentation than MS does. 

\begin{figure}[htb]
    \centering
    \subfloat[Test image1]{
        \label{fig:supp-image1}
        \includegraphics[width=0.15\textwidth]{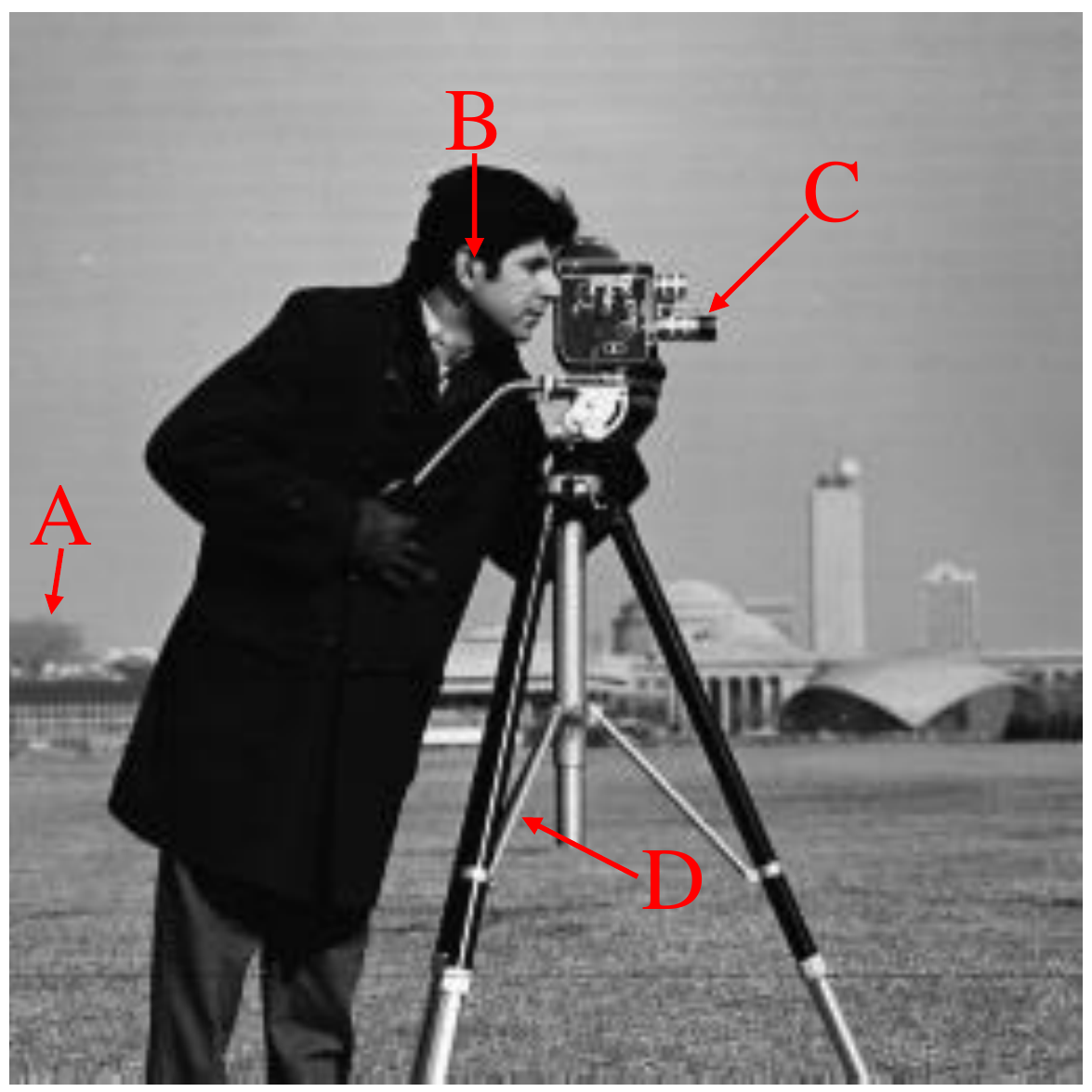}
	}
	\subfloat[MS]{
        \includegraphics[width=0.15\textwidth]{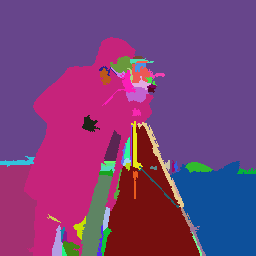}
	}
	\subfloat[GT]{
        \includegraphics[width=0.15\textwidth]{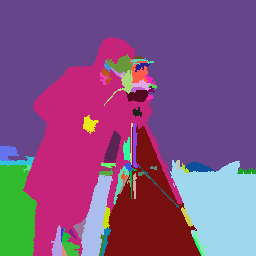}
	}
	\subfloat[Test image2]{
        \label{fig:supp-image2}
        \includegraphics[width=0.15\textwidth]{figures/compcamera/labelminicameraman.pdf}
	}
	\subfloat[MS]{
        \includegraphics[width=0.15\textwidth]{figures/cameraman/segmented_ms_minicameraman.png}
	}
	\subfloat[GT]{
        \includegraphics[width=0.15\textwidth]{figures/cameraman/segmented_cov_minicameraman_v2.png}
	}
    \caption{Image segmentation. 
    (a) Test image1 (cameraman $256 \times 256$ grayscale).  (b) MS segmentation with $\eps_s=8$ and $\eps_r=8$. 
    (c) GT segmentation with $\eps_s=8, \eps_r = 8, \lambda=4.8$. (d) Test image2 (cameraman $128\times 128$ grayscale).  (e) MS segmentation with $\eps_s=6$ and $\eps_r=6$. (f) GT segmentation with $\eps_s=6, \eps_r = 6, \lambda=5$.  Check differences between GT and MS results at the labeled areas in test images. } 
    \label{fig:supp-cameraman}
\end{figure}

\subsection{Details about word embeddings}\label{app:word}

In this section, we provide details about our implementation of GT for word embeddings. We do not compare our results with MS because we consider that MS is not applicable. In NLP, one compares words by comparing their contexts. We found that the Euclidean mean associated to context neighborhood of a word does not well represent the word itself. Indeed, we observe in practice that the Euclidean means of most neighborhoods selected from corpus contexts are concentrated around a point, and it also deteriorates the performance of the original word embedding. 
We also do not compare WT with GT in this experiment since WT is not commensurable with GT: in our following implementation, we modify our GT construction by replacing the original covariance matrix (cf. Equation~(\ref{eq:supp-cov})) with the covariation around each given word vector (cf. Equation~(\ref{eq:cov-nlp})).

\subsubsection{Open-source pre-trained word embeddings}

There are many open-source
embeddings \footnote{https://nlp.stanford.edu/projects/glove/} \footnote{https://gluon-nlp.mxnet.io/model\_zoo/bert}  which have been pre-trained on very large and rich corpora (such as wikipedia) and could potentially be directly applied to a given task. In this experiment, we use the GloVe embeddings pre-trained on Wikipedia2014 and Gigaword 5. 

\subsubsection{GT for word embeddings}
As mentioned in Section \ref{sec:word-emb} of the main text, we regard $c_\mathcal{C}(w)$ as the neighborhood of a given word $w$, where $c_\mathcal{C}(w)$ is the collection of all words in the corpus $\mathcal{C}$ that are found in the context of $w$ with a given window size $W$. To apply GT, we compute the covariance for each $c_\mathcal{C}(w)$ as follows according to the empirical covariance\footnote{Strictly speaking, $\Sigma_w$ is not the covariance matrix of $c_\mathcal{C}(w)$ but instead the covariation of points in $c_\mathcal{C}(w)$ around $w$.} in~\cite{Vilnis2014WordRV}:

\begin{equation}\label{eq:cov-nlp}
    \Sigma_w = \frac{1}{|c_\mathcal{C}(w)|} \sum_{i\in c_\mathcal{C}(w)} (c_\mathcal{C}(w)_{i} - w)(c_\mathcal{C}(w)_{i} - w)^\mathrm{T},
\end{equation}
where $c_\mathcal{C}(w)_i$ denotes the $i$th context word for $w$ in $c_\mathcal{C}(w)$. If there are no context words for $w$, we set $\Sigma_w = \textbf{0}$. 
Then the GT distance between any pair of words $w_1,w_2\in\mathcal{C}$ is computed as follows: 
\begin{equation}
    d^{\mathsmaller{(W,\lambda)}}_{\norm{\cdot}}(w_1,w_2) = \sqrt{\|w_1 - w_2\|^2 + \lambda \,  d^2_{cov}(\Sigma_{w_1}, \Sigma_{w_2})}.
\end{equation}

In practice, we only compute the GT distance between pairs occurring in the evaluation datasets mentioned in Section \ref{sec:eval} below. We use $-d^{\mathsmaller{(W,\lambda)}}_{\norm{\cdot}}(w_1,w_2)$ as the similarity between two words $w_1$ and $w_2$ (note the minus sign).

\subsubsection{Experiment details}\label{sec:eval}

We use the pre-computed GloVe embeddings $\Omega:\mathrm{Dict}\rightarrow\R^m$ and for each word $w\in \mathrm{Dict}$, we abuse notation and also use $w$ to represent the embedding $\Omega(w)$. We normalize the data set such that each word $w\in \mathrm{Dict}$ has magnitude $\norm{w}=1$. 
We choose corpus text8~\footnote{http://mattmahoney.net/dc/text8.zip} to be $\mathcal{C}$ to retrieve context words for a given word. We preprocess the corpus text8 in two steps: (1) we drop those rare words whose occurrence frequencies are  fewer than 5; (2) we drop frequent words with a probability following the strategy proposed in~\cite{mikolov2013distributed} that the more frequently that a word appears in the corpus, the higher probability that the word will be discarded. We then apply GT and train GloVe Embeddings (GloVe*text8) \cite{pennington2014glove} and Word2Vec (W2V*text8) \cite{mikolov2013distributed} on the preprocessed corpus text8.

We evaluate the embeddings on 13 different standard word similarity benchmarks: 
MC-30~\cite{miller1991contextual}, MEN-TR-3k~\cite{bruni2014multimodal}, MTurk-287~\cite{radinsky2011word}, 
MTurk-771~\cite{halawi2012large}, RG-65~\cite{rubenstein1965contextual}, RW-STANFORD~\cite{luong2013better}, SIMPLEX-999~\cite{hill-etal-2015-simlex}, SimVerb-3500~\cite{gerz2016simverb}, 
VERB-143~\cite{baker2014unsupervised}, WS-353~\cite{finkelstein2001placing}, WS-YP-130~\cite{yang2006verb}. 

In these benchmarks, similarity scores between certain pairs of words are provided. We refer to them as human similarity scores. Then, we calculate the Spearman rank correlation coefficient~\cite{spearman1961proof} between the human similarity scores and the similarity scores on the word pairs for all embeddings described above.

In table \ref{tab:word-compare}, for each evaluation dataset, we compare the Spearman rank correlation coefficients corresponding to GloVe+GT, GloVe, GloVe*text8 and W2V*text8. We observe the following: 
{GloVe+GT} outperforms GloVe in most of the evaluation datasets, and has comparable performance on the remaining datasets. Moreover, GloVe+GT outperforms models GloVe*text8 and W2V*text8 trained specifically on text8 in most evaluation datasets.

We also compare the similarity scores of GloVe+GT with the ones given by Elliptical Embeddings (Ell) \cite{muzellec2018generalizing} and Diagonal Gaussian Embeddings (W2G) \cite{Vilnis2014WordRV} trained on larger corpora ukWaC and WaCkypedia. Ell and W2G models require training high dimensional parameters and might not be suitable for small corpora such as text8. Note that, the performance of GloVe+GT based on a small corpus text8 is comparable with the performance of Ell and W2G trained on a much larger corpus.

Our experiments show the effectiveness of applying GT to improve the performance of pre-trained embeddings.

\begin{table*}[htb]
\caption{Spearman correlation for word similarity datasets.  Column ``{GloVe}" and column 
``{GloVe+GT}" are the same as the corresponding columns in Table 2 of the main paper. ``{GloVe*text8}" represents the embeddings trained on text8 using the GloVe model, ``{W2V*text8}" represents the embeddings trained on text8 using the Word2Vec model, ``{Ell}" represents the embeddings trained on ukWaC and WaCkypedia using the Elliptical Embeddings model (result is directly from ~\cite{muzellec2018generalizing}), ``{W2G}" represents the embeddings trained on ukWaC and WaCkypedia using the Diagonal Gaussian Embeddings model (result is directly from ~\cite{muzellec2018generalizing}).}
    \centering
    \begin{tabular}{| c| c| c| c| c|| c| c|}
    \hline
    \textbf{Dataset} & \textbf{GloVe} & \textbf{GloVe*text8} & \textbf{W2V*text8} & \textbf{GloVe+GT}  & \textbf{Ell} & \textbf{W2G}\\
    \hline
        MC-30 & 0.56  & 0.34 & 0.57 &  \textbf{0.67} & 0.65 & 0.59  \\ \hline 
        MEN-TR-3k & \textbf{0.65}  & 0.37 & 0.59 & \textbf{0.65}  & 0.65 & 0.65 \\ \hline 
        MTurk-287 & 0.61 & 0.49 & 0.61 &  \textbf{0.62} & 0.59 & 0.61  \\ \hline 
        MTurk-771 & 0.55 & 0.36 & 0.50 &  \textbf{0.56}& 0.56 & 0.57  \\ \hline 
         RG-65 & 0.60 & 0.33 &  0.56 &  \textbf{0.62}& 0.65 & 0.69  \\ \hline 
         RW-STANFORD & 0.34  & 0.20 & 0.25 &  \textbf{0.38} & 0.29 & 0.40 \\ \hline 
         SIMLEX-999 & 0.26 &  0.13 & 0.22 &  \textbf{0.27} & 0.24 & 0.25 \\ \hline 
        SimVerb-3500 & \textbf{0.15} & 0.07 & 0.08  & 0.14  & - & - \\ \hline
        VERB-143 & 0.25  & 0.28 & \textbf{0.32}  & 0.24  & - & - \\ \hline 
         WS-353-ALL & 0.49  & 0.43 & \textbf{0.62} & 0.51  & 0.66 & 0.53 \\ \hline 
        WS-353-REL & 0.46  & 0.41  & \textbf{0.59} & 0.47  & 0.71 & 0.61\\ \hline 
        WS-353-SIM & 0.57  & 0.51 & \textbf{0.66}  & 0.60  & 0.60 & 0.48 \\ \hline 
        WS-YP-130 & \textbf{0.37} & 0.19 &  0.23  & \textbf{0.37}  & 0.25 & 0.37 \\ \hline 
    \end{tabular}
        \label{tab:word-compare}
\end{table*}

%% file: nips2020.bbl
\newcommand{\etalchar}[1]{$^{#1}$}
\begin{thebibliography}{WTXC04}

\bibitem[ACB17]{arjovsky2017wasserstein}
Martin Arjovsky, Soumith Chintala, and L{\'e}on Bottou.
\newblock Wasserstein generative adversarial networks.
\newblock In {\em International conference on machine learning}, pages
  214--223, 2017.

\bibitem[AH19]{pmlr-v97-allen19a}
Carl Allen and Timothy Hospedales.
\newblock Analogies explained: Towards understanding word embeddings.
\newblock In Kamalika Chaudhuri and Ruslan Salakhutdinov, editors, {\em
  Proceedings of the 36th International Conference on Machine Learning},
  volume~97 of {\em Proceedings of Machine Learning Research}, pages 223--231,
  Long Beach, California, USA, 09--15 Jun 2019. PMLR.

\bibitem[AKS{\etalchar{+}}19]{anand2019asynchronous}
Avishek Anand, Megha Khosla, Jaspreet Singh, Jan-Hendrik Zab, and Zijian Zhang.
\newblock Asynchronous training of word embeddings for large text corpora.
\newblock In {\em Proceedings of the Twelfth ACM International Conference on
  Web Search and Data Mining}, pages 168--176, 2019.

\bibitem[AMJ18]{alvarez-melis-jaakkola-2018-gromov}
David Alvarez-Melis and Tommi Jaakkola.
\newblock {G}romov-{W}asserstein alignment of word embedding spaces.
\newblock In {\em Proceedings of the 2018 Conference on Empirical Methods in
  Natural Language Processing}, pages 1881--1890, Brussels, Belgium,
  October-November 2018. Association for Computational Linguistics.

\bibitem[AMO93]{ahujia1993network}
RK~Ahujia, Thomas~L Magnanti, and James~B Orlin.
\newblock Network flows: Theory, algorithms and applications.
\newblock {\em New Jersey: Rentice-Hall}, 1993.

\bibitem[BGJ19]{bhatia2019matrix}
Rajendra Bhatia, Stephane Gaubert, and Tanvi Jain.
\newblock Matrix versions of the {H}ellinger distance.
\newblock {\em Letters in Mathematical Physics}, pages 1--28, 2019.

\bibitem[Bha13]{bhatia2013matrix}
Rajendra Bhatia.
\newblock {\em Matrix analysis}, volume 169.
\newblock Springer Science \& Business Media, 2013.

\bibitem[BRK14]{baker2014unsupervised}
Simon Baker, Roi Reichart, and Anna Korhonen.
\newblock An unsupervised model for instance level subcategorization
  acquisition.
\newblock In {\em Proceedings of the 2014 Conference on Empirical Methods in
  Natural Language Processing (EMNLP)}, pages 278--289, 2014.

\bibitem[BTB14]{bruni2014multimodal}
Elia Bruni, Nam-Khanh Tran, and Marco Baroni.
\newblock Multimodal distributional semantics.
\newblock {\em Journal of Artificial Intelligence Research}, 49:1--47, 2014.

\bibitem[Bur69]{bures1969extension}
Donald Bures.
\newblock An extension of {K}akutani's theorem on infinite product measures to
  the tensor product of semifinite w*-algebras.
\newblock {\em Transactions of the American Mathematical Society},
  135:199--212, 1969.

\bibitem[CFTR16]{courty2016optimal}
Nicolas Courty, R{\'e}mi Flamary, Devis Tuia, and Alain Rakotomamonjy.
\newblock Optimal transport for domain adaptation.
\newblock {\em IEEE transactions on pattern analysis and machine intelligence},
  39(9):1853--1865, 2016.

\bibitem[Che95]{cheng1995mean}
Yizong Cheng.
\newblock Mean shift, mode seeking, and clustering.
\newblock {\em IEEE transactions on pattern analysis and machine intelligence},
  17(8):790--799, 1995.

\bibitem[CM02]{comaniciu2002mean}
Dorin Comaniciu and Peter Meer.
\newblock Mean shift: A robust approach toward feature space analysis.
\newblock {\em IEEE Transactions on pattern analysis and machine intelligence},
  24(5):603--619, 2002.

\bibitem[Cut13]{cuturi2013sinkhorn}
Marco Cuturi.
\newblock Sinkhorn distances: Lightspeed computation of optimal transport.
\newblock In {\em Advances in neural information processing systems}, pages
  2292--2300, 2013.

\bibitem[DCLT19]{Devlin2019BERTPO}
Jacob Devlin, Ming-Wei Chang, Kenton Lee, and Kristina Toutanova.
\newblock Bert: Pre-training of deep bidirectional transformers for language
  understanding.
\newblock In {\em NAACL-HLT}, 2019.

\bibitem[DDH07]{demmel2007fast}
James Demmel, Ioana Dumitriu, and Olga Holtz.
\newblock Fast linear algebra is stable.
\newblock {\em Numerische Mathematik}, 108(1):59--91, 2007.

\bibitem[Dem19]{demirovic2019implementation}
Damir Demirovi{\'c}.
\newblock An implementation of the mean shift algorithm.
\newblock {\em Image Processing On Line}, 9:251--268, 2019.

\bibitem[DGG17]{das2017named}
Arjun Das, Debasis Ganguly, and Utpal Garain.
\newblock Named entity recognition with word embeddings and wikipedia
  categories for a low-resource language.
\newblock {\em ACM Transactions on Asian and Low-Resource Language Information
  Processing (TALLIP)}, 16(3):1--19, 2017.

\bibitem[FGM{\etalchar{+}}01]{finkelstein2001placing}
Lev Finkelstein, Evgeniy Gabrilovich, Yossi Matias, Ehud Rivlin, Zach Solan,
  Gadi Wolfman, and Eytan Ruppin.
\newblock Placing search in context: The concept revisited.
\newblock In {\em Proceedings of the 10th international conference on World
  Wide Web}, pages 406--414, 2001.

\bibitem[FH75]{fukunaga1975estimation}
Keinosuke Fukunaga and Larry Hostetler.
\newblock The estimation of the gradient of a density function, with
  applications in pattern recognition.
\newblock {\em IEEE Transactions on information theory}, 21(1):32--40, 1975.

\bibitem[GAA{\etalchar{+}}17]{gulrajani2017improved}
Ishaan Gulrajani, Faruk Ahmed, Martin Arjovsky, Vincent Dumoulin, and Aaron~C
  Courville.
\newblock Improved training of {W}asserstein {GAN}s.
\newblock In {\em Advances in neural information processing systems}, pages
  5767--5777, 2017.

\bibitem[Gel90]{gelbrich1990formula}
Matthias Gelbrich.
\newblock On a formula for the {L}2 {W}asserstein metric between measures on
  {E}uclidean and {H}ilbert spaces.
\newblock {\em Mathematische Nachrichten}, 147(1):185--203, 1990.

\bibitem[GS{\etalchar{+}}84]{givens1984class}
Clark~R Givens, Rae~Michael Shortt, et~al.
\newblock A class of {W}asserstein metrics for probability distributions.
\newblock {\em The Michigan Mathematical Journal}, 31(2):231--240, 1984.

\bibitem[GS02]{gibbs2002choosing}
Alison~L Gibbs and Francis~Edward Su.
\newblock On choosing and bounding probability metrics.
\newblock {\em International statistical review}, 70(3):419--435, 2002.

\bibitem[GVH{\etalchar{+}}16]{gerz2016simverb}
Daniela Gerz, Ivan Vuli{\'c}, Felix Hill, Roi Reichart, and Anna Korhonen.
\newblock Simverb-3500: A large-scale evaluation set of verb similarity.
\newblock {\em arXiv preprint arXiv:1608.00869}, 2016.

\bibitem[HDGK12]{halawi2012large}
Guy Halawi, Gideon Dror, Evgeniy Gabrilovich, and Yehuda Koren.
\newblock Large-scale learning of word relatedness with constraints.
\newblock In {\em Proceedings of the 18th ACM SIGKDD international conference
  on Knowledge discovery and data mining}, pages 1406--1414, 2012.

\bibitem[HRK15]{hill-etal-2015-simlex}
Felix Hill, Roi Reichart, and Anna Korhonen.
\newblock {S}im{L}ex-999: Evaluating semantic models with (genuine) similarity
  estimation.
\newblock {\em Computational Linguistics}, 41(4):665--695, December 2015.

\bibitem[KGK{\etalchar{+}}15]{koch2015review}
Christian Koch, Kristina Georgieva, Varun Kasireddy, Burcu Akinci, and Paul
  Fieguth.
\newblock A review on computer vision based defect detection and condition
  assessment of concrete and asphalt civil infrastructure.
\newblock {\em Advanced Engineering Informatics}, 29(2):196--210, 2015.

\bibitem[LBB{\etalchar{+}}98]{lecun1998gradient}
Yann LeCun, L{\'e}on Bottou, Yoshua Bengio, Patrick Haffner, et~al.
\newblock Gradient-based learning applied to document recognition.
\newblock {\em Proceedings of the IEEE}, 86(11):2278--2324, 1998.

\bibitem[LSM13]{luong2013better}
Minh-Thang Luong, Richard Socher, and Christopher~D Manning.
\newblock Better word representations with recursive neural networks for
  morphology.
\newblock In {\em Proceedings of the Seventeenth Conference on Computational
  Natural Language Learning}, pages 104--113, 2013.

\bibitem[Mat]{eigen}
MathStackExchange.
\newblock Eigenvalues of product of positive semidefinite matrices are greater
  or equal to zero.
\newblock
  https://math.stackexchange.com/questions/2910177/eigenvalues-of-product-of-positive-semidefinite-matrices-are-greater-or-equal-to.

\bibitem[MC91]{miller1991contextual}
George~A Miller and Walter~G Charles.
\newblock Contextual correlates of semantic similarity.
\newblock {\em Language and cognitive processes}, 6(1):1--28, 1991.

\bibitem[MC18]{muzellec2018generalizing}
Boris Muzellec and Marco Cuturi.
\newblock Generalizing point embeddings using the {W}asserstein space of
  elliptical distributions.
\newblock In {\em Advances in Neural Information Processing Systems}, pages
  10237--10248, 2018.

\bibitem[MMM13]{martinez2013multiscale}
Diego H~Diaz Mart{\'i}nez, Facundo M{\'e}moli, and Washington Mio.
\newblock Multiscale covariance fields, local scales, and shape transforms.
\newblock In {\em International Conference on Geometric Science of
  Information}, pages 794--801. Springer, 2013.

\bibitem[MMM20]{martinez2020shape}
Diego H~D{\'\i}az Mart{\'\i}nez, Facundo M{\'e}moli, and Washington Mio.
\newblock The shape of data and probability measures.
\newblock {\em Applied and Computational Harmonic Analysis}, 48(1):149--181,
  2020.

\bibitem[MSC{\etalchar{+}}13]{mikolov2013distributed}
Tomas Mikolov, Ilya Sutskever, Kai Chen, Greg~S Corrado, and Jeff Dean.
\newblock Distributed representations of words and phrases and their
  compositionality.
\newblock In {\em Advances in neural information processing systems}, pages
  3111--3119, 2013.

\bibitem[MSW19]{pmlr-v97-memoli19a}
Facundo M\'emoli, Zane Smith, and Zhengchao Wan.
\newblock The {W}asserstein transform.
\newblock In Kamalika Chaudhuri and Ruslan Salakhutdinov, editors, {\em
  Proceedings of the 36th International Conference on Machine Learning},
  volume~97 of {\em Proceedings of Machine Learning Research}, pages
  4496--4504, Long Beach, California, USA, 09--15 Jun 2019. PMLR.

\bibitem[PC{\etalchar{+}}19]{peyre2019computational}
Gabriel Peyr{\'e}, Marco Cuturi, et~al.
\newblock Computational optimal transport.
\newblock {\em Foundations and Trends{\textregistered} in Machine Learning},
  11(5-6):355--607, 2019.

\bibitem[PCZ{\etalchar{+}}98]{pan1998complexity}
Victor~Y Pan, Z~Chen, Ailong Zheng, et~al.
\newblock The complexity of the algebraic eigenproblem.
\newblock {\em Mathematical Sciences Research Institute, Berkeley}, pages
  1998--71, 1998.

\bibitem[PM90]{perona1990scale}
Pietro Perona and Jitendra Malik.
\newblock Scale-space and edge detection using anisotropic diffusion.
\newblock {\em IEEE Transactions on pattern analysis and machine intelligence},
  12(7):629--639, 1990.

\bibitem[PSM14]{pennington2014glove}
Jeffrey Pennington, Richard Socher, and Christopher Manning.
\newblock Glove: Global vectors for word representation.
\newblock In {\em Proceedings of the 2014 conference on empirical methods in
  natural language processing (EMNLP)}, pages 1532--1543, 2014.

\bibitem[QSS10]{quarteroni2010numerical}
Alfio Quarteroni, Riccardo Sacco, and Fausto Saleri.
\newblock {\em Numerical mathematics}, volume~37.
\newblock Springer Science \& Business Media, 2010.

\bibitem[RAGM11]{radinsky2011word}
Kira Radinsky, Eugene Agichtein, Evgeniy Gabrilovich, and Shaul Markovitch.
\newblock A word at a time: computing word relatedness using temporal semantic
  analysis.
\newblock In {\em Proceedings of the 20th international conference on World
  wide web}, pages 337--346, 2011.

\bibitem[RG65]{rubenstein1965contextual}
Herbert Rubenstein and John~B Goodenough.
\newblock Contextual correlates of synonymy.
\newblock {\em Communications of the ACM}, 8(10):627--633, 1965.

\bibitem[RTG98]{rubner1998metric}
Yossi Rubner, Carlo Tomasi, and Leonidas~J Guibas.
\newblock A metric for distributions with applications to image databases.
\newblock In {\em Computer Vision, 1998. Sixth International Conference on},
  pages 59--66. IEEE, 1998.

\bibitem[Spe61]{spearman1961proof}
Charles Spearman.
\newblock The proof and measurement of association between two things.
\newblock 1961.

\bibitem[Sze10]{szeliski2010computer}
Richard Szeliski.
\newblock {\em Computer vision: algorithms and applications}.
\newblock Springer Science \& Business Media, 2010.

\bibitem[Vil08]{villani2008optimal}
C{\'e}dric Villani.
\newblock {\em Optimal transport: old and new}, volume 338.
\newblock Springer Science \& Business Media, 2008.

\bibitem[VM14]{Vilnis2014WordRV}
Luke Vilnis and Andrew McCallum.
\newblock Word representations via {G}aussian embedding.
\newblock {\em CoRR}, abs/1412.6623, 2014.

\bibitem[WTXC04]{wang2004image}
Jue Wang, Bo~Thiesson, Yingqing Xu, and Michael Cohen.
\newblock Image and video segmentation by anisotropic kernel mean shift.
\newblock In {\em European conference on computer vision}, pages 238--249.
  Springer, 2004.

\bibitem[XJRN03]{xing2003distance}
Eric~P Xing, Michael~I Jordan, Stuart~J Russell, and Andrew~Y Ng.
\newblock Distance metric learning with application to clustering with
  side-information.
\newblock In {\em Advances in neural information processing systems}, pages
  521--528, 2003.

\bibitem[YH10]{yamaguchi2010fast}
Tomoyuki Yamaguchi and Shuji Hashimoto.
\newblock Fast crack detection method for large-size concrete surface images
  using percolation-based image processing.
\newblock {\em Machine Vision and Applications}, 21(5):797--809, 2010.

\bibitem[YP06]{yang2006verb}
Dongqiang Yang and David~Martin Powers.
\newblock {\em Verb similarity on the taxonomy of WordNet}.
\newblock Masaryk University, 2006.

\bibitem[ZSCM13]{zou2013bilingual}
Will~Y Zou, Richard Socher, Daniel Cer, and Christopher~D Manning.
\newblock Bilingual word embeddings for phrase-based machine translation.
\newblock In {\em Proceedings of the 2013 Conference on Empirical Methods in
  Natural Language Processing}, pages 1393--1398, 2013.

\end{thebibliography}
